\documentclass{article}

\usepackage[parfill]{parskip}

\usepackage{amsmath,amsthm,amsfonts,amssymb}
\usepackage{mathtools}

\usepackage{fontspec}
\defaultfontfeatures{Ligatures=TeX}

\usepackage{unicode-math}
\usepackage[dvipsnames]{xcolor}
\colorlet{linkcol}{MidnightBlue!85!black} 
\colorlet{citecol}{Mulberry!90!black}    
\colorlet{urlcol}{Mahogany!85!black}      
\usepackage{bbm}
\usepackage{booktabs}
\usepackage{nicefrac}
\usepackage{microtype}
\usepackage{multirow}
\usepackage[inline]{enumitem}
\usepackage{diagbox}
\usepackage{comment}
\usepackage{graphicx}
\usepackage{wrapfig}
\usepackage[font=small]{caption}
\usepackage{subcaption}
\usepackage{algorithm}
\usepackage{algorithmicx}
\usepackage{algpseudocode}
\usepackage{setspace}
\usepackage{placeins}
\usepackage{float}

\usepackage[authoryear,round]{natbib}

\usepackage{hyperref}
\hypersetup{
  colorlinks   = true,
  linkcolor    = linkcol,
  citecolor    = citecol,
  urlcolor     = urlcol,
  pdfborder    = {0 0 0},
  breaklinks   = true
}
\usepackage[capitalise,noabbrev]{cleveref}

\makeatletter
\renewenvironment{abstract}{%
  \par\begingroup
  \begin{center}
    {\bfseries\small \abstractname\par}
    \vspace{0.4em}
    \begin{minipage}{0.92\textwidth}
      \small
}{
    \end{minipage}
  \end{center}
  \endgroup
}
\makeatother

\usepackage{pifont}

\theoremstyle{plain}
\newtheorem{theorem}{Theorem}
\newtheorem{assumption}[theorem]{Assumption} 

\newtheorem{lemma}{Lemma}[section]
\newtheorem{corollary}[lemma]{Corollary}
\newtheorem{proposition}[theorem]{Proposition}
\theoremstyle{definition}
\newtheorem{definition}{Definition}
\theoremstyle{remark}
\newtheorem{remark}[lemma]{Remark}

\newtheorem*{theorem*}{Theorem}

\title{Sessa: Selective State Space Attention}
\author{Liubomyr Horbatko\\ \texttt{liubomir.horbatko@gmail.com}}
\date{}

\newcommand{\R}{\mathbb{R}}

\newcommand{\Assm}{A_{\mathrm{ssm}}}
\newcommand{\Bssm}{B_{\mathrm{ssm}}}
\newcommand{\Cssm}{C_{\mathrm{ssm}}}
\newcommand{\Dssm}{D_{\mathrm{ssm}}}

\RenewDocumentCommand{\Assm}{e{^_}}{%
  {A}_{\mathrm{ssm}\IfValueT{#2}{,#2}}%
  \IfValueT{#1}{^{#1}}%
}
\RenewDocumentCommand{\Bssm}{e{^_}}{%
  {B}_{\mathrm{ssm}\IfValueT{#2}{,#2}}%
  \IfValueT{#1}{^{#1}}%
}
\RenewDocumentCommand{\Cssm}{e{^_}}{%
  {C}_{\mathrm{ssm}\IfValueT{#2}{,#2}}%
  \IfValueT{#1}{^{#1}}%
}
\RenewDocumentCommand{\Dssm}{e{^_}}{%
  {D}_{\mathrm{ssm}\IfValueT{#2}{,#2}}%
  \IfValueT{#1}{^{#1}}%
}
\newcommand{\Assmt}[1]{\Assm_{#1}}
\newcommand{\Bssmt}[1]{\Bssm_{#1}}
\newcommand{\Cssmt}[1]{\Cssm_{#1}}
\newcommand{\Dssmt}[1]{\Dssm_{#1}}

\newcommand{\Mdom}{M_{\mathcal D}}
\newcommand{\msc}{m_{\mathrm{sc}}}

\newcommand{\Bbatch}{B_{\mathrm{batch}}}
\newcommand{\Tctx}{T}
\newcommand{\Nstate}{\mathsf{N}}

\newcommand{\VisSet}[1]{\mathcal{W}_{#1}} 
\newcommand{\rhospec}[1]{\rho_{\mathrm{spec}}\!\left(#1\right)}
\newcommand{\rhonbhd}{\rho_{\mathrm{nbhd}}}
\newcommand{\Pref}{\mathcal{P}^{\mathrm{pref}}}
\newcommand{\Gammafb}[1]{\Gamma_{\mathrm{fb}}\!\left(#1\right)}
\newcommand{\idxlogit}{i}
\newcommand{\dstate}{d_{\mathrm{state}}}

\ExplSyntaxOn
\NewDocumentCommand{\Kset}{}{
  \Kset_maybe_subscript:
}
\cs_new_protected:Npn \Kset_maybe_subscript:
{
  \peek_meaning:NTF _
    { \Kset_with_subscript:w }
    { \mathcal K_{\mathrm{set}} }
}
\cs_new_protected:Npn \Kset_with_subscript:w _ #1
{
  \mathcal K_{\mathrm{set},#1}
}
\ExplSyntaxOff

\providecommand{\alphadrv}{}%
\providecommand{\alphafb}{}%
\RenewDocumentCommand{\alphadrv}{e{^_}}{%
  \alpha^{\mathrm{fwd}\IfValueT{#1}{,#1}}%
  \IfValueT{#2}{_{#2}}%
}

\RenewDocumentCommand{\alphafb}{e{^_}}{%
  \alpha^{\mathrm{fb}\IfValueT{#1}{,#1}}%
  \IfValueT{#2}{_{#2}}%
}

\providecommand{\Bfb}{B_{\mathrm{fb}}}
\RenewDocumentCommand{\Bfb}{e{^_}}{%
  B_{\mathrm{fb}\IfValueT{#2}{,#2}}%
  \IfValueT{#1}{^{#1}}%
}

\newcommand{\Gssm}{G_{\mathrm{ssm}}} %
\RenewDocumentCommand{\Gssm}{e{^_}}{%
  {G}_{\mathrm{ssm}\IfValueT{#2}{,#2}}%
  \IfValueT{#1}{^{#1}}%
}
\newcommand{\Gssmt}[1]{\Gssm_{#1}}

\newcommand{\BfbEnt}[2]{\big[\Bfb\big]_{#1,#2}}

\newcommand{\betatail}{\beta_{\mathrm{tail}}}

\newcommand{\xnorm}{\tilde{x}}
\newcommand{\xLN}{\xnorm}

\newcommand{\sigk}{\sigma_k}
\newcommand{\GR}{G_R}

\newcommand{\Vmax}{V_{\max}}  
\newcommand{\Gmax}{G_{\max}}

\newcommand{\Alphafb}{\Alpha_{\mathrm{fb}}}

\DeclareMathOperator{\Id}{Id}
\DeclareMathOperator{\Norm}{Norm}
\DeclareMathOperator{\LN}{LN}
\DeclareMathOperator{\Softmax}{softmax}

\NewDocumentCommand{\Keffdeep}{e{_}}{%
  K^{(\Nlayer)}_{\mathrm{eff}\IfValueT{#1}{,#1}}%
}

\newcommand{\Nlayer}{N_{\mathrm{layer}}}
\newcommand{\idxlayer}{n_{\mathrm{layer}}}
\newcommand{\idxlayeri}[1]{n_{\mathrm{layer},#1}}

\DeclareMathOperator{\GELU}{GELU}
\DeclareMathOperator{\Sat}{Sat}

\begin{document}
\maketitle
\begin{abstract}
Modern sequence modeling is dominated by two families: Transformers, whose self-attention can access arbitrary elements of the visible sequence, and structured state-space models, which propagate information through an explicit recurrent state. These mechanisms face different limitations on long contexts: when attention is diffuse, the influence of individual tokens is diluted across the effective support, while recurrent state propagation can lose long-range sensitivity unless information is actively preserved. As a result, both mechanisms face challenges in preserving and selectively retrieving information over long contexts.
We propose Sessa, a decoder that places attention inside a recurrent feedback path. This creates many attention-based paths through which past tokens can influence future states, rather than relying on a single attention read or a single recurrent chain. We prove that, under explicit assumptions and matched regimes, Sessa admits power-law memory tails $O(\ell^{-\beta})$ for $0 < \beta < 1$, with slower decay than in the corresponding Transformer and Mamba-style baselines. We further give an explicit construction that achieves this power-law rate. Under the same assumptions, Sessa is the only model class among those considered that realizes flexible selective retrieval, including profiles whose influence does not decay with distance.
Consistent with this theoretical advantage, across matched experiments, Sessa achieves the strongest performance on long-context benchmarks while remaining competitive with Transformer and Mamba-style baselines on short-context language modeling.
\end{abstract}

\section{Introduction}
\label{sec:intro}

Long-context sequence modeling is central to modern foundation models across language, vision, speech, time series, and genomics \citep{bommasani2021foundation,brown2020language,dosovitskiy2020image,baevski2020wav2vec2,ansari2024chronos,dallatorre2025nucleotidetransformer}. Despite the architectural flexibility of the foundation-model paradigm, state-of-the-art systems are still overwhelmingly based on the Transformer and its self-attention mechanism \citep{vaswani2017attention}.
 
A useful lens is to describe modern sequence mixers by how they route information from the past and how they maintain memory over time.
In many modern architectures, routing decisions are input-dependent: the model uses the current token and its context to decide which parts of the visible history to consult.
Under this view, self-attention implements an input-dependent \emph{direct-read} mechanism:
at each position, it computes a query-dependent pattern of relevance over the visible context and uses it to read out information from selected past positions.
This framing highlights attention's key strength, a selection mechanism over variable support length, but also a structural limitation:
the retrieval is performed in a single pass, without an internal feedback loop that would repeatedly incorporate past readouts into an evolving state.
Separately, standard implementations are also computationally expensive at long contexts due to quadratic time/memory scaling \citep{vaswani2017attention,rabe2021selfattention}.

In parallel, structured recurrent sequence models, especially state space models (SSMs), which realize long-range dynamics through a latent state and an explicit feedback path, have re-emerged as a compelling alternative for long-context modeling \citep{gu2022efficiently, gu2022s4d}. SSMs can be interpreted as modern descendants of classical dynamical systems \citep{kalman1960new} and admit linear (or near-linear) scaling in sequence length. 
However, for information-dense discrete data, a persistent challenge is that stable feedback dynamics often exhibit rapid attenuation of distant information (commonly exponential forgetting \citep{huang2025inputselectivitymamba}), which can hinder integrating multiple far-apart evidence snippets under heavy distractors. Selective SSMs (e.g., Mamba) can conditionally slow this attenuation by modulating the effective transition \citep{gu2023mamba, dao2024transformersssms} (e.g., $\Assmt{t} \approx I$ on selected steps, ``freeze time'' \citep{huang2025inputselectivitymamba}),
but this mechanism is input-dependent and can fail when relevant and irrelevant positions induce similar local representations, leading to preserving or overwriting the wrong content.

These perspectives suggest complementary long-context failure modes. Stable feedback dynamics can suffer from
exponential forgetting. Attention, while input-dependent, can suffer from dilution:
when attention mass is spread across a large effective support of competing tokens (e.g., many near-tied logits),
individual weights, and thus per-token contributions and sensitivities, decrease roughly inversely with that support
(often behaving like $O(1/S_{\mathrm{eff}}(t))$, and in the worst case like $O(1/\Tctx)$ when the effective support grows
proportionally with context length $\Tctx$)\citep{mudarisov2025limitations}. In practice, both effects can limit reliable long-range evidence integration.

We introduce \textbf{Sessa}, a decoder architecture that injects input-dependent attention into a feedback (recurrent) path, combining direct-read input-dependent routing with stateful aggregation through the feedback channel.
Viewed through a temporal routing lens, for a fixed source token $\tau$ and target position $t$ (lag $\ell=t-\tau$), a single self-attention layer routes influence via a \emph{single routing step} (a direct edge $\tau\!\to\!t$), while chain-structured state-space recurrences propagate along the unique length-$\ell$ temporal chain.
Sessa introduces route diversity within a single layer: its attention-induced feedback operator aggregates contributions over multiple internal routing depths (and, in dense patterns, many temporal paths), which can help sustain long-range sensitivity when routing is diffuse (formalized in Section~\ref{sec:theory_memory}).
Concretely, while self-attention corresponds to an input-dependent direct-read system (in the values), Sessa realizes an
input-dependent feedback system: it maintains a latent state over unbounded horizons, while the feedback dynamics remain
input-dependent via attention-based routing inside the loop (potentially over variable-support patterns). Intuitively, Sessa
retains the representational benefits of recurrence for long-range accumulation while leveraging attention as an
input-dependent mechanism within the feedback pathway.

Related architectural ideas have introduced recurrence or feedback into sequence modeling \citep{dai2019transformerxl,fan2020feedbacktransformer,bulatov2022recurrentmemorytransformer,hutchins2022blockrecurrent,hwang2024transformerfam}.
These approaches span a variety of feedback constructions and are typically presented in architecture-specific terms.
Our contribution is complementary but mathematically different: we propose a routing-induced systems perspective that separates how context produces routing/mixing coefficients from how those coefficients are composed over time, and we use this lens to relate input-dependent routing directly to long-context sensitivity and memory-decay behavior.

Our contributions are:
\begin{itemize}
  \item \textbf{Architecture.} We propose the Sessa sequence mixer, integrating attention into the recurrent feedback pathway under an otherwise standard decoder macro-architecture.
  \item \textbf{Memory.} We characterize long-range sensitivity of Sessa and identify a heavy-tail memory regime in which the feedback solve induces a \textbf{power-law influence tail in the lag $\ell$ of order $O(\ell^{-\betatail})$ with $0<\betatail<1$}. In this diffuse, low-separation routing regime, attenuation is asymptotically slower than the exponential forgetting exhibited by many stable or contractive SSM regimes, and it mitigates inverse-support dilution effects under the stated assumptions (Section~\ref{sec:theory_memory}; Theorem~\ref{thm:poly_decay_main}).
  \item \textbf{Selective retrieval.} In the matched theoretical regime, we show that deep Sessa realizes flexible selective retrieval profiles, including non-decaying ones, whereas diffuse fixed-depth Transformers and failed-freeze-time fixed-depth Mamba do not (Section~\ref{sec:theory_uniform_finite_horizon_retrieval}; Theorem~\ref{thm:flexible_route_B_sessa}; Proposition~\ref{prop:comparison_class_impossibility_same_regime}).
  \item \textbf{Empirics.} Under matched architectures and training budgets, Sessa achieves the strongest performance on our long-context benchmarks while remaining competitive on short-context language modeling.
\end{itemize}

We additionally prove a universal approximation result for a broad class of causal sequence mappings in Appendix~\ref{app:sessa_uat} (Theorem~\ref{thm:sessa_uap_main}).

\section{Background}
\label{sec:background}

We separate two largely independent aspects of causal mixers:
\begin{enumerate}[label=(\roman*), leftmargin=*, nosep]
\item how routing/mixing coefficients are produced from context, and
\item whether information is accessed via a single read or accumulated through feedback.
\end{enumerate}

\paragraph*{Terminology}
We use \emph{system} to refer to the memory mechanism (direct-read or feedback).
We use \emph{routing} to refer to the coefficients that specify how information flows over time
for example attention weights $\alphadrv$, the induced feedback matrix $\Bfb$, or the transition operators in a recurrence.
Routing is the collection of coefficients, meaning weights or operators, that determine information flow over time.
The system determines whether routing is applied once (direct-read) or repeatedly composed via feedback.
 
\subsection{Direct-read and feedback systems}
\label{sec:directread_feedback_ops}

We model a broad class of sequence mixers by expressing each output as a mixture of a chosen stream $u_t$
with coefficients that may depend on the available context $x_{0:t}$.

\begin{definition}[Direct-read variable-support system]
\label{def:directread_onehop}
We say that $\mathcal{F}$ is a direct-read system with respect to a chosen stream $u_t$ if, for every $t$,
\begin{equation}
y_t \;=\; \sum_{\tau\in S_t} K_{t,\tau}(x_{0:t})\,u_\tau,
\qquad S_t \subseteq \{0,\dots,t\},
\end{equation}
so each $y_t$ is produced by a single input-addressed read, i.e., a mixture over the visible index set $S_t$.
If $|S_t|$ varies with $t$, we call the system \emph{variable-support}.
If there exists $W\ge 1$ such that $K_{t,\tau}\equiv 0$ whenever $t-\tau\ge W$, equivalently, $S_t\subseteq\{\max(0,t-W+1),\dots,t\}$, we call it \emph{bounded-support direct-read}.
\end{definition}

\begin{remark}[Kernel representations alone do not distinguish direct-read or feedback]
On any finite horizon $\Tctx$, any causal linear map admits a lower-triangular kernel representation \citep{kalman1960new,antsaklis2006linear}.
$y_t=\sum_{\tau\le t}K_{t,\tau}u_\tau$, so kernel form alone does not identify whether influence is produced by a single read or by an internal recurrence.
Here, direct-read refers to the computation graph: $y_t$ is formed by one read/mix over a visible set, without repeated composition of the same mixing primitive inside the layer.
\end{remark}

\paragraph{Dimensions.}
$u_\tau\in\mathbb{R}^D$, $y_t\in\mathbb{R}^D$, and $K_{t,\tau}(x_{0:t})$ is a linear map of the appropriate shape.

In contrast, models with an explicit state and feedback naturally take a feedback form.

\begin{definition}[Feedback system: state-space or operator form]
\label{def:feedback_multihop}
We say that $\mathcal{G}$ is a feedback system with respect to a chosen stream $u_t$ if there exist
states $h_t$ in a possibly time-varying state space $\mathcal H_t$ such that, for each $t\ge 0$,
\begin{equation}
\text{with, e.g., }h_{-1}=0\text{,}\qquad
h_t \;=\; \Assmt{t}(x_{0:t})\,h_{t-1} \;+\; \Bssmt{t}(x_{0:t})\,u_t,
\qquad
y_t \;=\; \Cssmt{t}(x_{0:t})\,h_t \;+\; \Dssmt{t}(x_{0:t})\,u_t.
\end{equation}
The recurrence composes the routing over time, so $y_t$ can depend on arbitrarily old inputs even when each update is local in $h_{t-1}$.
\end{definition}

\begin{remark}[One-hop and multi-hop routing]
We view routing as propagation on a directed acyclic graph (DAG) over time indices induced by the mixing coefficients.
Fix a horizon $\Tctx$ and nodes $\{0,\dots,\Tctx-1\}$.

\begin{figure}[H]
  \centering
  \begin{minipage}{0.95\linewidth}
    \centering
    \includegraphics[width=\linewidth,height=0.40\textheight,keepaspectratio]{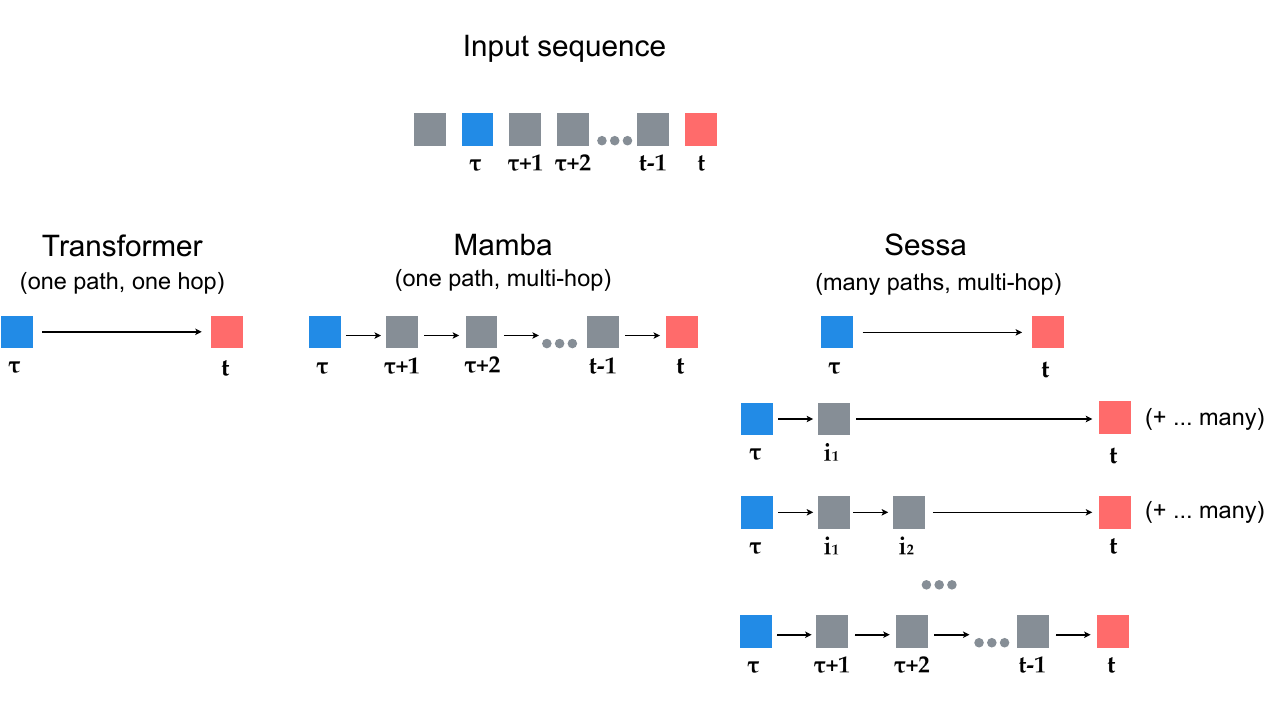}

    \captionsetup{
      format=plain,
      justification=raggedright,
      singlelinecheck=false,
      labelsep=colon,
      margin={2.2em,0em}
    }

    \caption{One-hop and multi-hop temporal routing within a single mixer layer.\newline
    \textbf{Transformer:} influence from $\tau$ to $t$ follows a single direct edge (one-hop).\newline
    \textbf{Mamba:} influence from $\tau$ to $t$ follows the chain $\tau\to\cdots\to t$ (multi-hop along a single path).\newline
    \textbf{Sessa:} influence from $\tau$ to $t$ aggregates over many paths with varying hop counts (multi-hop over many paths).}
    \label{fig:architecture_comparison}
  \end{minipage}
\end{figure}

\par\noindent\emph{Direct-read (one-hop).}
A direct-read system forms $y_t$ by a single read from a visible set $S_t$ using coefficients $K_{t,\tau}$:
in the routing graph, this corresponds to using only direct edges $\tau\to t$.
Influence from $\tau$ reaches $t$ in one routing step.

\par\noindent\emph{Feedback (multi-hop).}
A feedback mechanism can apply routing repeatedly through an internal state or solve, allowing influence from $\tau$
to reach $t$ through paths with intermediate nodes. This repeated composition is what we call multi-hop routing.
\end{remark}

The classical finite-dimensional state-space case corresponds to $\mathcal H_t=\mathbb R^{\Nstate}$ with fixed $\Nstate$ for all $t$.
Structured SSM layers (e.g., S4/S4D and Mamba) are instances of this special case.

\paragraph*{Hop counts in the solve}
Sessa's mixer output $s$ is defined by a causal lower-triangular solve
\begin{equation}
(I-\Bfb)\,s=f,
\qquad \BfbEnt{t}{j}=0 \ \text{for } j\ge t,
\end{equation}
On any finite horizon $\Tctx$, $\Bfb$ is strictly lower-triangular and hence nilpotent ($\Bfb^{\Tctx}=0$) \citep{horn2012matrix}. Hence,
\begin{equation}
(I-\Bfb)^{-1}=\sum_{k=0}^{\Tctx-1}\Bfb^k,
\qquad\text{and}\qquad
s=\sum_{k=0}^{\Tctx-1}\Bfb^k f.
\label{eq:neumann_finite}
\end{equation}
Each term $\Bfb^k f$ corresponds to routing through $k$ feedback steps, a $k$-hop contribution.
Equivalently, for indices $\tau\le t$,
\begin{equation}
(\Bfb^k)_{t,\tau}
=\!\!\!\sum_{\tau=i_0< i_1<\cdots<i_k=t}\ \prod_{r=1}^{k} \BfbEnt{i_r}{i_{r-1}},
\qquad k\ge 1,
\end{equation}
which is a sum over all length-$k$ directed paths from $\tau$ to $t$ in the feedback-induced routing graph.
This explicit path expansion is the mechanism behind heavy-tail regimes analyzed later:
even if individual edges are small under diffuse routing, the number of admissible paths grows with lag,
and the solve aggregates contributions across all hop counts.

\subsection{Self-attention as direct-read}
\label{sec:attn_directread_onehop}

Standard causal self-attention fits Definition~\ref{def:directread_onehop} when the mixed stream is the sequence of value vectors.
At position $t$, over a visible index set $\VisSet{t}\subseteq\{0,\dots,t\}$:
\begin{equation}
y_t=\sum_{j\in \VisSet{t}}\alphadrv_{t,j}\, v_j,
\qquad
\alphadrv_{t,j}=
\frac{\exp\!\left(\sigk\,q_t^\top k_j\right)}
{\sum_{i\in \VisSet{t}}\exp\!\left(\sigk\,q_t^\top k_i\right)},
\end{equation}
with $q_t=W_Q x_t$, $k_j=W_K x_j$, and $v_j=W_V x_j$.

\begin{lemma}[Self-attention is a direct-read system in $V$]
\label{lem:attn_directread_onehop}
At each position $t$, self-attention computes $y_t$ by a single input-addressed read from the visible set $\VisSet{t}$,
mixing the value vectors $(v_j)_{j\in \VisSet{t}}$ with context-dependent weights $\alphadrv_{t,j}$.
\end{lemma}

Full-prefix, windowed, and sparse attention all fit the same direct-read template through the choice of visible set $\VisSet{t}$ \citep{child2019sparse,beltagy2020longformer,zaheer2020bigbird,ding2023longnet}.

\subsection{State-space models as feedback}
\label{sec:ssm_feedback_multihop}

Structured state-space models (SSMs) implement sequence mixing through a latent state and a (possibly selective) recurrence.
A standard form is
\begin{equation}
h_t = \Assm\,h_{t-1} + \Bssm\,x_t,
\qquad
y_t = \Cssm\,h_t,
\end{equation}
where $\Assm\in\mathbb{R}^{\Nstate\times\Nstate}$ encodes temporal dynamics and is typically constrained
(diagonal/structured/low-rank) for efficiency.

Modern language-oriented SSMs such as Mamba often employ input-dependent recurrences that fit Definition~\ref{def:feedback_multihop}:
\begin{equation}
h_t \;=\; \Assmt{t}(x_{0:t})\,h_{t-1} \;+\; \Bssmt{t}(x_{0:t})\,x_t,
\qquad
y_t \;=\; \Cssmt{t}(x_{0:t})\,h_t.
\end{equation}

In Mamba, the discrete transition commonly takes the form
\[
\Assmt{t}=\operatorname{diag}(\exp(-\lambda_n \Delta_t)),
\]
so a lag-$\ell$ memory factor contains terms of the form
\[
\exp\!\Big(-\lambda_n\sum_{r=t-\ell+1}^{t}\Delta_r\Big).
\]
Accordingly, long-range memory is preserved only when the model can create a long \emph{preserve corridor}
of steps with $\Delta_r\approx 0$.

This suggests the matched comparison principle used later in the paper.
For attention, broken sharp selection means that softmax mass cannot concentrate on a small set of indices.
For Mamba, the analogous failure mode is \emph{failed freeze time}: the model cannot sustain a long preserve corridor on the relevant interval.
For the three-way comparison in this paper, we say that a Mamba layer is in a \emph{failed freeze-time regime} on an input set of interest if there exists $c_\Delta>0$ such that for every relevant pair $\tau<t$,
\[
\sum_{r=\tau+1}^{t}\Delta_r \;\ge\; c_\Delta (t-\tau).
\]
Equivalently, the average discretization step along every relevant interval is bounded below by a positive constant.
In Mamba this implies
\[
\exp\!\Big(-\lambda_n\sum_{r=\tau+1}^{t}\Delta_r\Big)
\;\le\;
e^{-\lambda_n c_\Delta (t-\tau)},
\]
so long-range influence is exponentially small in the lag.
This is the Mamba counterpart of diffuse attention used in the matched comparisons below:
in attention, the selector cannot concentrate mass on a few indices;
in Mamba, the model cannot maintain $\Delta_r\approx 0$ on a long relevant corridor.

\section{Model Architecture}
\label{sec:model_arch}

We instantiate the one-hop and multi-hop routing viewpoint of Section~\ref{sec:directread_feedback_ops} with a concrete layer, Sessa.
Sessa uses a single gated-MLP-style block that wraps a recurrent mixer, rather than alternating separate attention and MLP blocks.
The mixer itself combines (i) a standard causal forward-attention signal and (ii) a feedback term that mixes past mixer outputs.

\noindent The official implementation is available at \url{https://github.com/LibratioAI/sessa}.

\paragraph*{Notation.}
Inputs and outputs have shape $x,y\in\mathbb{R}^{\Bbatch\times \Tctx\times D}$ with $t\in\{0,\dots,\Tctx-1\}$.
We use an internal key and query width $d_k$ and scale $\sigk=d_k^{-1/2}$.
All definitions apply per batch element; we omit the batch index when clear.

\subsection{Sessa block}
\label{sec:sessa_block}

Given $x\in\mathbb{R}^{\Bbatch\times \Tctx\times D}$, the block applies pre-norm, a gated projection, the mixer, and a residual connection:
\begin{align}
\xLN &= \LN(x), \label{eq:sessa_ln}\\
(a,g) &= \mathrm{split}\!\big(\xLN W^{\mathrm{in}}+b^{\mathrm{in}}\big),
\qquad a,g\in\mathbb{R}^{\Bbatch\times \Tctx\times D}, \label{eq:sessa_split}\\
\bar a &= \mathrm{GELU}(a), \label{eq:sessa_act}\\
s &= \mathrm{Mixer}(\bar a)\in\mathbb{R}^{\Bbatch\times \Tctx\times D}, \label{eq:sessa_mixer_call}\\
y &= x + \big((s\odot g)W^{\mathrm{out}} + b^{\mathrm{out}}\big). \label{eq:sessa_residual}
\end{align}
We use Layer Normalization \citep{ba2016layernorm} and the GELU nonlinearity \citep{hendrycks2016gelu}.
Here $W^{\mathrm{in}}\in\mathbb{R}^{D\times 2D}$ and $W^{\mathrm{out}}\in\mathbb{R}^{D\times D}$.
The elementwise gate $g$ plays the usual role of gated MLP variants \citep{hua2022tqlt,shazeer2020glu}: it modulates the mixer output before the residual add.

\begin{figure}[H]
  \centering
  \includegraphics[width=0.95\linewidth,height=0.35\textheight,keepaspectratio]{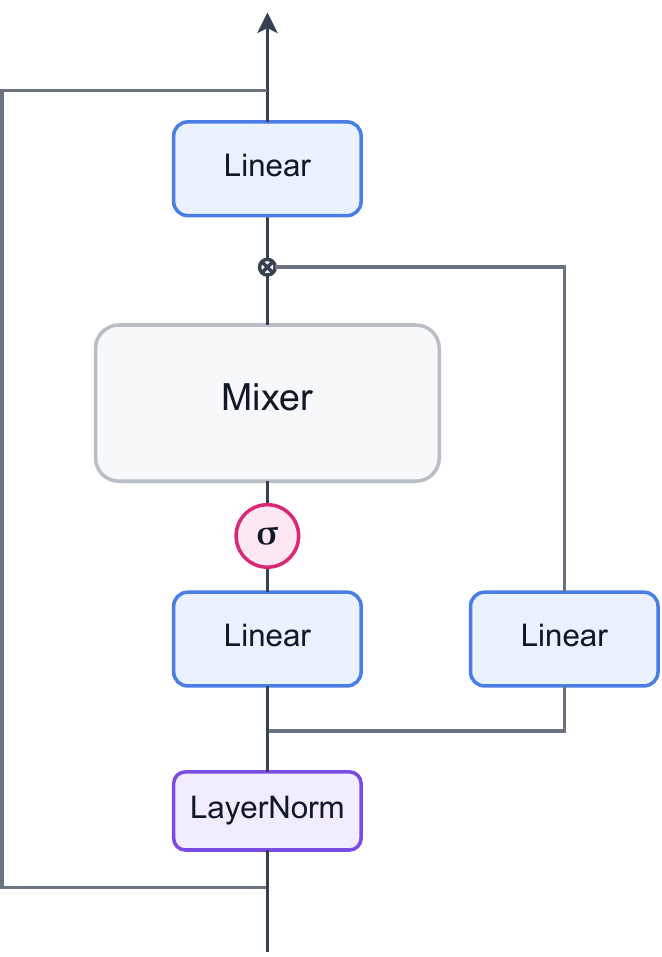}
  \caption{Sessa Layer.}
  \label{fig:sessa_block}
\end{figure}

\subsection{Sessa mixer}
\label{sec:sessa_mixer}

The mixer maps $\bar a\in\mathbb{R}^{\Bbatch\times \Tctx\times D}$ to $s\in\mathbb{R}^{\Bbatch\times \Tctx\times D}$.
It uses two causal attention mechanisms:
(i) a forward causal attention that produces a forward signal $f_t\in\mathbb{R}^D$, and
(ii) a feedback attention that produces weights over the strict past, used inside a causal feedback solve.

\paragraph*{Projections.}
At each time $t$, we form forward queries, keys, and values, as well as feedback queries and keys, using standard linear projections:
\begin{equation}
q^f_t=\bar a_t W_{Qf},\quad k^f_t=\bar a_t W_{Kf},\quad v_t=\bar a_t W_V,\qquad
q^b_t=\bar a_t W_{Qb},\quad k^b_t=\bar a_t W_{Kb},
\end{equation}
where $q^f,k^f,q^b,k^b\in\mathbb{R}^{d_k}$ and $v_t\in\mathbb{R}^{D}$.
We apply RoPE to the forward pair $(q^f,k^f)$.
We use rotary position embeddings in the forward branch \citep{su2021roformer}.

\paragraph*{Forward attention.}
Define causal weights over $j\le t$:
\begin{equation}
\alphadrv_{t,j}=\mathrm{softmax}_{0\le j\le t}\!\Big(\sigk\langle \mathrm{RoPE}(q^f_t), \mathrm{RoPE}(k^f_j)\rangle\Big),
\end{equation}
and the forward signal
\begin{equation}
f_t=\sum_{j=0}^{t} \alphadrv_{t,j}\,v_j\in\mathbb{R}^{D}.
\label{eq:def_f}
\end{equation}
This is a one-hop mixture of values $(v_j)_{j\le t}$ over a finite visible set.
\paragraph*{Feedback attention.}
Define feedback weights over the strict past $j<t$:
\begin{equation}
\alphafb_{t,j}=
\begin{cases}
\mathrm{softmax}_{0\le j\le t-1}\!\Big(\sigk\langle q^b_t, k^b_j\rangle\Big), & t\ge 1,\ j<t,\\
0, & j\ge t,
\end{cases}
\qquad \alphafb_{0,j}=0\ \ \forall j.
\label{eq:def_alpha}
\end{equation}

\paragraph*{Feedback gain.}
We modulate the feedback with a scalar gain $\gamma_t\in(-1,1)$:
\begin{equation}
\gamma_t=\tanh\!\big(\langle \bar a_t,w^\gamma\rangle + b^\gamma\big).
\label{eq:def_gamma}
\end{equation}
The bound controls feedback magnitude: since $\alphafb_{t,\cdot}$ is a convex distribution over $j<t$, the feedback term is a convex combination of past states scaled by $|\gamma_t|<1$.

\paragraph*{Feedback routing matrix.}
\begin{equation}
\BfbEnt{t}{j}=\gamma_t\,\alphafb_{t,j},\qquad \BfbEnt{t}{j}=0\ \text{for } j\ge t.
\label{eq:def_B}
\end{equation}
\paragraph*{Scalar routing and feature-wise solve.}
Here $\Bfb$ is a scalar strictly lower-triangular routing matrix (each $\BfbEnt{t}{j}\in\mathbb R$).
The solve $(I-\Bfb)s=f$ is applied independently to each feature dimension of $s,f\in\mathbb R^{\Tctx\times D}$:
for every $d\in\{1,\dots,D\}$,
\[
(I-\Bfb)\,s_{:,d}=f_{:,d},
\]
\noindent In vectorized form,
\[
\big(I_D\otimes (I-\Bfb)\big)\,\mathrm{vec}(s)=\mathrm{vec}(f).
\]
The resulting recurrence \eqref{eq:mixer_recurrence} therefore uses scalar--vector multiplication
($\BfbEnt{t}{j}\,s_j$ with $\BfbEnt{t}{j}\in\mathbb R$ and $s_j\in\mathbb R^D$).

\paragraph*{Lower-triangular solve.}
The mixer output $s\in\mathbb{R}^{\Tctx\times D}$ is the unique solution of
\begin{equation}
(I-\Bfb)s=f
\label{eq:triangular_system}
\end{equation}
which is a unit-lower-triangular solve with $D$ right-hand sides.
This can be implemented with optimized triangular-solve routines (e.g., batched \texttt{solve\_triangular}/TRSM kernels), avoiding explicit formation of $(I-\Bfb)^{-1}$.
Thus, in the dense full-prefix formulation, the mixer remains quadratic in $\Tctx$.
Equivalently, forward substitution gives the explicit recurrence
\begin{align}
s_0 &= f_0,\\
s_t &= f_t + \sum_{j=0}^{t-1}\BfbEnt{t}{j}\,s_j
    \;=\; f_t + \gamma_t\sum_{j=0}^{t-1}\alphafb_{t,j}\,s_j,\qquad t\ge 1.
\label{eq:mixer_recurrence}
\end{align}

\begin{remark}[Multi-hop routing view: exact on finite horizons]
Since $\Bfb$ is strictly lower-triangular on a finite horizon $\Tctx$, it is nilpotent ($\Bfb^{\Tctx}=0$) and therefore
\[
(I-\Bfb)^{-1}=\sum_{k=0}^{\Tctx-1}\Bfb^k \quad\text{and hence}\quad s=\sum_{k=0}^{\Tctx-1}\Bfb^k f.
\]
The term $\Bfb^k f$ aggregates contributions that traverse $k$ internal routing steps through the feedback operator.
Thus, unlike self-attention's one-hop read, the solve realizes multi-hop routing, which can produce the heavy-tail influence regimes analyzed in Section~\ref{sec:theory_memory}.
\end{remark}

\subsection{Positional encoding}
\label{sec:sessa_posenc}

\paragraph*{RoPE in the forward path.}
In the forward attention \eqref{eq:def_f} we apply RoPE to $(q^f,k^f)$, following common practice in
decoder-only Transformers \citep{touvron2023llama,black2022gptneox}. This injects relative positional information into the attention logits while
preserving causal masking.

\paragraph*{No positional encoding in feedback.}
We do not apply RoPE, or any other positional encoding, to the feedback attention \eqref{eq:def_alpha}.
The feedback path already induces an absolute time direction: the strictly lower-triangular
feedback operator \eqref{eq:def_B} and the causal solve \eqref{eq:triangular_system} correspond to a forward
substitution recurrence \eqref{eq:mixer_recurrence}, whose output at time $t$ depends on an iterated aggregation
of the strict past. This temporal asymmetry can generate position-dependent signals even when the mixer input is
time-constant.

Corollary~\ref{cor:pos_code_padding}, proved in Appendix~\ref{sec:uat_pos_code}, shows that a single Sessa block can
produce a deterministic, position-dependent additive offset:
there exist parameters and vectors $(p_t)_{t=0}^{\Tctx-1}\subset\mathbb R^D$ such that for all inputs $x$ in any
fixed compact set $\mathcal D\subset\mathbb R^{\Bbatch\times \Tctx\times D}$,
\[
y_t = x_t + p_t,\qquad t=0,\dots,\Tctx-1.
\]
Moreover, these offsets can be chosen separated on $\mathcal D$ in the following sense:
there exist a unit direction $u\in\mathbb R^D$ and a scale $\lambda>0$ such that $p_t=c_t(\lambda u)$ with $c_t$ pairwise
distinct and the scalar ranges $\{\langle x_t+p_t,u\rangle:\ x\in\mathcal D\}$ are pairwise disjoint over $t$.
By Corollary~\ref{cor:pos_code_recovery}, the position index $t$ is recoverable by a continuous token-wise map on the set of shifted tokens,
so the feedback mechanism can supply an absolute positional signal internally.

\section{Theory}
\label{sec:theory}

This section establishes four properties of Sessa:
\begin{enumerate}[label=(\roman*), leftmargin=*, nosep]
\item stability of the feedback solve,
\item long-range memory, including flexible selective retrieval,
\item internal positional encoding,
\item universal approximation.
\end{enumerate}

\begin{remark}[LayerNorm]
All stability and Jacobian statements in this section are stated for the formulation with $\Norm=\Id$.
For the pre-norm LayerNorm extension relevant to universal approximation, we assume an explicit $\varepsilon>0$
and use the corresponding Lipschitz bounds for the normalization map; see Appendix~\ref{sec:uat_with_layernorm}.
\end{remark}

\subsection{Stability of the feedback solve}
\label{sec:theory_bibo}

We isolate the operation in Sessa that induces multi-hop behavior: the causal lower-triangular solve
\begin{equation}
(I-\Bfb(x))\,s \;=\; f(x),
\qquad
\BfbEnt{t}{j}(x)=\gamma_t(x)\,\alphafb_{t,j}(x),
\qquad
\BfbEnt{t}{j}(x)=0\ \text{for } j\ge t,
\label{eq:sessa_solve_theory}
\end{equation}
where $\alphafb_{t,\cdot}(x)$ is a convex distribution over the strict past, $j<t$, produced by the feedback attention,
and $\gamma_t(x)\in(-1,1)$ is a bounded scalar gain.
The quantity $f(x)$ is the forward aggregation defined in Section~\ref{sec:model_arch}.
\paragraph*{Scalar feedback matrix}
Throughout the stability analysis, $\Bfb(x)\in\mathbb R^{\Tctx\times \Tctx}$ is scalar-valued:
each entry $\BfbEnt{t}{j}(x)\in\mathbb R$.
The solve acts feature-wise on $s,f\in\mathbb R^{\Tctx\times r}$.
In vectorized form, $(I_r\otimes (I-\Bfb))\mathrm{vec}(s)=\mathrm{vec}(f)$.

\paragraph*{Norms}
For a finite or infinite token sequence $u=(u_t)$ with $u_t\in\mathbb R^r$, define
\[
\|u\|_{\infty,2}:=\sup_t \|u_t\|_2,
\]
and for a finite tensor $U\in\mathbb R^{\Tctx\times r}$, define $\|U\|_{\infty,2}:=\max_{0\le t\le \Tctx-1}\|U_t\|_2$.

\begin{assumption}[Uniform row contraction on the feedback margin]
\label{ass:row_contraction_main}
For every radius $R\ge 0$ there exists $\rho(R)\in[0,1)$ such that for all inputs $x$ with
$\|x\|_{\infty,2}\le R$,
\begin{equation}
\sup_t |\gamma_t(x)| \;\le\; \rho(R) \;<\; 1.
\label{eq:gamma_margin}
\end{equation}
\end{assumption}

Since each $\alphafb_{t,\cdot}(x)$ is a convex distribution over $j<t$,
Assumption~\ref{ass:row_contraction_main} implies the row-sum bound
\begin{equation}
\sup_{t\ge 1} \sum_{j<t}|\BfbEnt{t}{j}(x)|
\;\le\; \rho(R) \;<\; 1.
\label{eq:star_row_sum}
\end{equation}

\begin{lemma}[Causal lower-triangular solve is bounded on $\ell_\infty$]
\label{lem:triangular_linf}
Let $\Bfb$ be strictly lower-triangular, possibly on an infinite horizon, and define
$(\Bfb s)_t:=\sum_{j<t}\BfbEnt{t}{j}s_j$.
If $\sup_t\sum_{j<t}|\BfbEnt{t}{j}|\le \rho<1$, then for every $f\in\ell_\infty(\mathbb N,\mathbb R^r)$ there exists a unique
$s\in\ell_\infty(\mathbb N,\mathbb R^r)$ solving $(I-\Bfb)s=f$, and
\[
\|s\|_{\infty,2}\;\le\;\frac{1}{1-\rho}\,\|f\|_{\infty,2}.
\]
\end{lemma}

\begin{proof}[Proof sketch]
Forward substitution gives existence and uniqueness. The bound follows by a standard induction on the partial maxima
$\max_{k\le t}\|s_k\|_2$ using the row-sum estimate. See Appendix~\ref{app:proof_triangular_linf}.
\end{proof}

\begin{proposition}[One-block stability bound]
\label{prop:one_block_bibo_main}
Fix a Sessa block $G$ acting on finite or infinite sequences with the feedback solve \eqref{eq:sessa_solve_theory}.
Assume moreover that all tokenwise affine maps appearing in the block
(in particular, the output projection and the residual affine terms)
are fixed and have finite operator norms and finite bias magnitudes.
Assume that for every $R\ge 0$ there exist finite constants $F_R,\GR<\infty$ such that on the ball $\|x\|_{\infty,2}\le R$,
\[
\|f(x)\|_{\infty,2}\le F_R,
\qquad
\|g(x)\|_{\infty,2}\le \GR,
\qquad
\sup_t|\gamma_t(x)|\le \rho(R)<1,
\]
Here $g(x)$ denotes the tokenwise gate, the Hadamard multiplier applied to $s$ before the output projection.
Then there exists $C_R<\infty$ such that $\|G(x)\|_{\infty,2}\le C_R$ for all $\|x\|_{\infty,2}\le R$.
In particular, $G$ is BIBO-stable on $\ell_\infty(\mathbb N,\mathbb R^D)$.
\end{proposition}

\begin{proof}[Proof sketch]
By Lemma~\ref{lem:triangular_linf} and \eqref{eq:star_row_sum},
$\|s\|_{\infty,2}\le (1-\rho(R))^{-1}\|f\|_{\infty,2}$.
Then $\|s\odot g\|_{\infty,2}\le \|s\|_{\infty,2}\|g\|_{\infty,2}$.
Since bounded tokenwise affine maps send bounded sets to bounded sets,
the output projection together with the residual affine terms yields a ball-to-ball bound for $G$.
Appendix Proposition~\ref{prop:one_block_bound} strengthens this by giving an explicit ball-to-ball
constant in terms of matrix/operator norms and bias magnitudes; see Appendix~\ref{sec:sessa_bibo_infinite}.
\end{proof}

\subsection{Long-range memory}
\label{sec:theory_memory}

We compare long-range memory through Jacobian-based diagnostics that separate the memory mechanism from routing adaptation.
Let $y=G(x)$ denote the output of a causal mixer or block applied to an input token sequence
$x=(x_0,\dots,x_{\Tctx-1})$, and fix a source position $\tau\le t$ with lag
\[
\ell := t-\tau.
\]
Our analysis uses three related diagnostics.

\paragraph*{Diagnostics.}
\begin{enumerate}[label=(\roman*), leftmargin=*, nosep]
\item \emph{Fixed-routing influence Jacobians.}
We first freeze a realized routing pattern and differentiate only the induced linear map from an injected stream to the output.
This yields, for example,
\[
J^{\mathrm{attn}}=\frac{\partial y}{\partial v}\Big|_{\alphadrv},
\qquad
J^{\mathrm{sessa}}=\frac{\partial s}{\partial f}\Big|_{\Bfb},
\]
and the corresponding SSM impulse Jacobian \(J^{\mathrm{ssm}}\) induced by a realized sequence
\((\Assmt{t},\Bssmt{t},\Cssmt{t})\).
These quantities isolate the memory mechanism under a common realized routing regime.

\item \emph{End-to-end block Jacobians.}
We then return to the full input-dependent block and measure the actual sensitivity of output token $y_t$ to a past input token $x_\tau$:
\[
J^{\mathrm{e2e}}_{t,\tau}(x)
:=
\frac{\partial y_t(x)}{\partial x_\tau}.
\]
Unlike the fixed-routing Jacobians, these derivatives include both transport through the memory mechanism and the dependence of the routing coefficients on the input.
They are the relevant one-block quantities for comparing diffuse attention, failed-freeze-time Mamba, and Sessa under smooth-routing assumptions.

\item \emph{Scalar transport scores for deep retrieval.}
For selective retrieval we extract scalar scores from deep end-to-end Jacobians.
For a depth-$\Nlayer$ stack with hidden states
\[
h^{(0)}=x,\qquad h^{(1)},\dots,h^{(\Nlayer)},
\]
we write
\[
J^{\mathrm{e2e},(\Nlayer)}_{t,\tau}(x)
:=
\frac{\partial h_t^{(\Nlayer)}(x)}{\partial h_\tau^{(0)}(x)}.
\]
Later we evaluate these blocks against source and target probes to obtain scalar transport scores,
written generically as $\mathsf S$,
which are the quantities used in the selective-retrieval theorem.
\end{enumerate}

These diagnostics play complementary roles.
Fixed-routing Jacobians expose the structural difference between one-hop direct read, chain-structured feedback, and Sessa's many-path feedback solve.
End-to-end block Jacobians capture the actual behavior of the nonlinear input-dependent block.
Scalar transport scores are needed for the positive retrieval statements, since they let us compare source and distractor influence after composing end-to-end Jacobians across layers.

All decay statements in this subsection are expressed in the lag $\ell=t-\tau$, not in the context length $\Tctx$.

The key structural difference is that, for Sessa, the fixed-routing solve
\[
(I-\Bfb)^{-1}
\]
aggregates contributions over multiple hop counts and, in dense regimes, over many temporal paths.
This accumulation across hop counts and paths is the mechanism behind the polynomial tail analyzed below.

\subsubsection{Fixed-routing Jacobians}
\label{sec:fair_comparison}

We begin with realized routing patterns and isolate the induced memory operators.
Worst-case comparisons over all inputs and parameters are uninformative, since any model can suppress a token.
Instead, we compare the architectures within common diffuse-weight regimes by studying the corresponding fixed-routing influence operators.

\paragraph*{Attention value Jacobian}
For causal self-attention, for a given set of attention weights $\alphadrv_{t,\tau}$, the map from values to output is linear:
\[
y_t=\sum_{\tau\le t}\alphadrv_{t,\tau}v_\tau
\]
We define the value influence Jacobian
\begin{equation}
J^{\mathrm{attn}}_{t,\tau}:=\frac{\partial y_t}{\partial v_\tau}\Big|_{\alphadrv}=\alphadrv_{t,\tau}\,I_D.
\label{eq:def_J_attn}
\end{equation}

\paragraph*{Solve Jacobian}
In Sessa, for a given feedback matrix $\Bfb$, i.e., a given routing pattern inside the loop, the lower-triangular solve
\[
(I-\Bfb)s=f
\]
is linear in $f$. We define the solve influence Jacobian
\begin{equation}
J^{\mathrm{sessa}}:=\frac{\partial s}{\partial f}\Big|_{\Bfb}=(I-\Bfb)^{-1},
\qquad
J^{\mathrm{sessa}}_{t,\tau}=[(I-\Bfb)^{-1}]_{t,\tau}.
\label{eq:def_J_sessa}
\end{equation}
Because $\Bfb$ is scalar-valued, the solve acts identically on each feature dimension; equivalently, if \(f_t,s_t\in\mathbb R^{d_f}\), the full feature-block Jacobian is
\[
J^{\mathrm{sessa}}_{t,\tau} I_{d_f}.
\]
\paragraph*{SSM impulse Jacobian}
For a feedback recurrence $h_t=\Assmt{t} h_{t-1}+\Bssmt{t} u_t$, $y_t=\Cssmt{t} h_t$, given a realized sequence of transitions $(\Assmt{t},\Bssmt{t},\Cssmt{t})$,
the impulse influence from $u_\tau$ to $y_t$ is
\begin{equation}
J^{\mathrm{ssm}}_{t,\tau}
:=\Cssmt{t}\Big(\prod_{r=\tau+1}^{t}\Assmt{r}\Big)\Bssmt{\tau},
\qquad 0\le \tau\le t.
\label{eq:def_J_ssm}
\end{equation}
\noindent
\emph{Convention: time-ordered product.}
We interpret the matrix product in \eqref{eq:def_J_ssm} as the left-to-right time-unrolling
consistent with the recurrence $h_t=\Assmt{t}h_{t-1}+\cdots$:
\[
\prod_{r=\tau+1}^{t}\Assmt{r}
\;:=\;
\Assmt{t}\Assmt{t-1}\cdots \Assmt{\tau+1}.
\]
Equivalently, the product is time-ordered with later-time factors on the left.
For the empty product we use
\[
\prod_{r=t+1}^{t}(\cdot)\;:=\;I,
\]
so that the definition also covers the case $t=\tau$.

These Jacobians isolate the memory mechanism under a common routing regime.

\subsubsection{End-to-end Jacobians}
\label{sec:e2e_comparison}

\begin{definition}[End-to-end block Jacobian]\label{def:e2e_jac}
Let $y=G(x)$ denote the output of a single mixer/block $G$ applied to an input token sequence
$x\in(\mathbb R^D)^{\Tctx}$.
We define the end-to-end Jacobian blocks by
\[
J^{\mathrm{e2e}}_{t,\tau}(x):=\frac{\partial y_t(x)}{\partial x_\tau}\in\mathbb R^{D\times D}.
\]
For $\tau<t$, $J^{\mathrm{e2e}}_{t,\tau}(x)$ measures long-range influence without freezing routing.
\end{definition}

\begin{definition}[Diffuse attention regime]
\label{def:diffuse_regime}
We say that an attention mechanism is in a diffuse, low-separation regime on a horizon $\Tctx$ if, for each $t$,
its pre-softmax logits $\beth_{t,j}$ over the visible set satisfy a bounded spread
\[
\max_{j\in \VisSet{t}}\beth_{t,j}-\min_{j\in \VisSet{t}}\beth_{t,j}\le \Delta
\quad\text{for some finite }\Delta,
\]
uniformly over the inputs under consideration. In this regime, softmax weights are near-uniform: Appendix Lemma~\ref{lem:bounded_logits} implies that for full-prefix attention with $|\VisSet{t}|=t+1$,
\[
\alphadrv_{t,j}=\Theta(1/|\VisSet{t}|).
\]
\noindent
In particular, for full-prefix causal attention one has
\[
\VisSet{t}=\{0,\dots,t\},\qquad |\VisSet{t}|=t+1,
\]
whereas for strictly-lower attention one has
\[
\VisSet{t}=\{0,\dots,t-1\},\qquad |\VisSet{t}|=t\quad\text{for }t\ge 1.
\]
\end{definition}

We state diffuse bounds in terms of the visible-set size $|\VisSet{t}|$ to cover full-prefix and strict-past attention uniformly.

We assume diffuse attention rows $\alphadrv_{t,j}\le c_2/|\VisSet{t}|$
(Definition~\ref{def:diffuse_regime}), together with the following smooth-routing bound on the
input set of interest:
\begin{equation}\label{eq:smooth_routing_bound}
\sum_{j\in \VisSet{t}}\Big\|\frac{\partial \alphadrv_{t,j}(x)}{\partial x_\tau}\Big\|_2
\;\le\;
\frac{L_\alpha}{|\VisSet{t}|},
\qquad \tau<t.
\end{equation}
Appendix~\ref{app:jacobian_tail} derives this from standard softmax calculus under mild
logit-sensitivity control.

\begin{lemma}[Smooth-routing for standard causal attention]\label{lem:attn_smooth_routing}
Assume a single-head causal attention row is
$\alphadrv_{t,\cdot}(x)=\mathrm{softmax}(\beth_{t,0}(x),\dots,\beth_{t,t}(x))$ with logits
$\beth_{t,j}(x)=\langle q(x_t),\,k(x_j)\rangle$ where $q,k$ are tokenwise maps.
Then for every $\tau<t$,
\[
\sum_{j\le t}\Big\|\frac{\partial \alphadrv_{t,j}(x)}{\partial x_\tau}\Big\|_2
\;\le\;
2\,\alphadrv_{t,\tau}(x)\,\Big\|\frac{\partial \beth_{t,\tau}(x)}{\partial x_\tau}\Big\|_2.
\]
In particular, if $\|\partial \beth_{t,\tau}/\partial x_\tau\|_2\le L_{\beth}$ on $\mathcal X_R$, then
\[
\sum_{j\le t}\Big\|\frac{\partial \alphadrv_{t,j}(x)}{\partial x_\tau}\Big\|_2
\;\le\;
2L_{\beth}\,\alphadrv_{t,\tau}(x)
\;\lesssim\;\frac{1}{|\VisSet{t}|}
\]
in the diffuse regime of Definition~\ref{def:diffuse_regime}.
For full-prefix attention one has $|\VisSet{t}|=t+1$.
Full proof in Appendix~\ref{app:proof_attn_smooth_routing}.
\end{lemma}

\subsubsection{Exponential forgetting in LTI systems}
\label{sec:lti_exp_forgetting}

Consider a finite-dimensional linear time-invariant feedback system in state-space form:
\begin{equation}
h_t = \Assm h_{t-1} + \Bssm u_t,\qquad y_t = \Cssm h_t,
\label{eq:lti_ss}
\end{equation}
with constant matrices $(\Assm,\Bssm,\Cssm)$. Under an impulse input at time $\tau$, i.e.\ $u_\tau\neq 0$ and $u_t=0$ for $t\neq\tau$,
the contribution to $y_t$ is mediated by $\Assm^{t-\tau}=\Assm^\ell$.

\begin{proposition}[Exponential decay in BIBO-stable LTI feedback systems]
\label{prop:lti_exp_decay}
Assume \eqref{eq:lti_ss} is BIBO-stable. Then there exist constants $c>0$ and $\kappa\in(0,1)$ such that for all lags $\ell\ge 0$,
\[
\|\Cssm \Assm^{\ell} \Bssm\| \;\le\; c\,\kappa^{\ell}.
\]
Equivalently, the impulse response and long-range influence mediated by the state transition decay exponentially in the lag $\ell$.
\end{proposition}

\begin{proof}[Proof sketch]
BIBO stability implies internal stability of any minimal controllable and observable realization, hence $\rhospec{\Assm_{\mathrm{co}}}<1$ \citep{dahleh2011_chap30}.
Therefore $\|\Assm_{\mathrm{co}}^\ell\|\le c\kappa^\ell$ and $\|\Cssm\Assm^\ell \Bssm\|=\|\Cssm_{\mathrm{co}}\Assm_{\mathrm{co}}^\ell \Bssm_{\mathrm{co}}\|\le c'\kappa^\ell$.
Proof in Appendix~\ref{app:proof_lti_exp_decay}.
\end{proof}

\subsubsection{Exponential forgetting in Mamba}
\label{sec:mamba_exp}

Mamba-style layers fit Definition~\ref{def:feedback_multihop} as feedback systems.
Their update maps
$\Assmt{t}(x_{0:t}),\Bssmt{t}(x_{0:t}),\Cssmt{t}(x_{0:t})$ depend on the input. 

\paragraph*{Convention: discrete scan coefficients}
In what follows, $\Assmt{t},\Bssmt{t},\Cssmt{t}$ denote the discrete-time scan coefficients actually used in the recurrence
$h_t=\Assmt{t} h_{t-1}+\Bssmt{t} u_t$ after discretization, such as ZOH, unless stated otherwise.

Exponential forgetting is not automatic for general input-dependent feedback systems.
Section~\ref{sec:sessa_poly_tail} gives a counterexample in a diffuse feedback-routing regime.
For Mamba, the relevant condition is \emph{failed freeze time}: the model cannot sustain a long interval with $\Delta_t\approx 0$.

\paragraph*{Accumulated discretization time}
In Mamba's standard ZOH-diagonal parameterization, long-range influence is controlled by the accumulated discretization time
\[
\sum_{r=\tau+1}^{t}\Delta_r,
\]
since the transition product contains factors of the form
\[
\exp\!\Big(-a_n\sum_{r=\tau+1}^{t}\Delta_r\Big).
\]
Accordingly, failed freeze time converts control in accumulated discretization time into exponential decay in the lag.

\begin{proposition}[Mamba end-to-end Jacobian bound]
\label{prop:mamba_e2e_decay}
Consider a Mamba block with state $h_t\in\mathbb R^{\dstate}$ and output $y_t\in\mathbb R^D$:
\[
h_{-1}=0,\qquad
h_t = \Assmt{t}(x_t)\,h_{t-1}+\Gssmt{t}(x_t)\,\widetilde{\Bssmt{t}}(x_t)\,u_t(x_t),
\qquad
y_t=\Cssmt{t}(x_t)\,h_t,
\]
where the parametrization is local and ZOH-diagonal:
for each mode $n$,
\[
[\Assmt{t}(x_t)]_n=\exp(-a_n\Delta_t(x_t)),
\qquad
[\Gssmt{t}(x_t)]_n=\frac{1-\exp(-a_n\Delta_t(x_t))}{a_n},
\]
with input-independent rates satisfying
\[
a_n\ge \lambda>0
\qquad\text{for all modes }n.
\]

Assume there exist constants $U_R,\Gmax,C_R,L_A,L_B,L_u<\infty$ such that for all $x\in\mathcal X_R$ and all $t$,
\[
\|u_t(x_t)\|\le U_R,\qquad
\|\widetilde{\Bssmt{t}}(x_t)\|\le \Gmax,\qquad
\|\Cssmt{t}(x_t)\|\le C_R,
\]
\[
\Big\|\frac{\partial \Assmt{t}(x_t)}{\partial x_t}\Big\|\le L_A,\qquad
\Big\|\frac{\partial \widetilde{\Bssmt{t}}(x_t)}{\partial x_t}\Big\|\le L_B,\qquad
\Big\|\frac{\partial u_t(x_t)}{\partial x_t}\Big\|\le L_u.
\]
For $\tau<t$ with lag $\ell=t-\tau$, define
\[
\Pi_{t,\ell}(x)
:=
\exp\!\Big(-\lambda\sum_{r=\tau+1}^{t}\Delta_r(x)\Big).
\]
Then for every $x\in\mathcal X_R$ and every $\tau<t$,
\[
\Big\|\frac{\partial y_t(x)}{\partial x_\tau}\Big\|
\le
C(R)\,\Pi_{t,\ell}(x),
\]
where one may take
\[
C(R):=C_R\,J_R,
\]
with
\[
J_R:=L_A\,H_R \;+\;\frac{L_A}{\lambda}\,\Gmax U_R \;+\;\frac{1}{\lambda}\big(L_B\,U_R+\Gmax\,L_u\big),
\qquad
H_R:=\sqrt{\dstate}\,\frac{\Gmax\,U_R}{\lambda}.
\]
\end{proposition}

\begin{proof}[Proof sketch]
Differentiate the ZOH recurrence. By locality, for $t>\tau$ one has
\[
\frac{\partial h_t}{\partial x_\tau}
=
\Assmt{t}(x_t)\frac{\partial h_{t-1}}{\partial x_\tau}.
\]
Thus the long-range dependence is controlled by the transition product.
Lemma~\ref{lem:zoh_state_bound} yields the uniform state bound $H_R$, which controls the source-time injection derivative
$\partial h_\tau/\partial x_\tau$.
Since each diagonal transition satisfies
\[
\Big\|\prod_{r=\tau+1}^{t}\Assmt{r}(x_r)\Big\|
\le
\exp\!\Big(-\lambda\sum_{r=\tau+1}^{t}\Delta_r(x)\Big)
=
\Pi_{t,\ell}(x),
\]
the displayed bound follows. Proof in Appendix~\ref{app:proof_mamba_e2e_decay}.
\end{proof}

\paragraph*{ZOH discretization under freezing}
In Mamba, the discrete-time coefficients arise from a stable continuous-time diagonal kernel via ZOH \citep{gu2023mamba}.
For each mode with continuous parameter $A=-a$ with $a>0$ and step size $\Delta_t\ge 0$,
\[
\bar A_t = e^{-a\Delta_t}\in[0,1],\qquad
\bar B_t = \frac{1-e^{-a\Delta_t}}{a}\,\widetilde{\Bssmt{t}}.
\]
Here $\Assmt{t}=\bar A_t$ and $\bar B_t u_t = \Gssmt{t}\,\widetilde{\Bssmt{t}}u_t$.
In particular, when ``freezing time'' with $\Delta_t=0$ one has $\bar A_t=1$ and $\bar B_t=0$,
so the update injects no new input while holding the state.

\begin{lemma}[Bounded state for ZOH-diagonal Mamba channels]\label{lem:zoh_state_bound}
Consider the scalar ZOH recurrence
\[
h_{-1}=0,\qquad
h_t = e^{-a\Delta_t}\,h_{t-1} + \frac{1-e^{-a\Delta_t}}{a}\,b_t,
\qquad a\ge a_{\min}>0,\ \Delta_t\ge 0.
\]
If $|b_t|\le M$ for all $t$, then $\sup_t |h_t|\le M/a_{\min}$.
More generally, $\sup_t |h_t|\le \max\{|h_{-1}|,\ \sup_s |b_s|/a_{\min}\}$.
\end{lemma}

\begin{proof}[Proof sketch]
Write $h_t = \theta_t h_{t-1} + (1-\theta_t)\,\frac{b_t}{a}$ with $\theta_t:=e^{-a\Delta_t}\in[0,1]$.
Thus $h_t$ is a convex combination of $h_{t-1}$ and $\frac{b_t}{a}$, yielding
$|h_t|\le \max\{|h_{t-1}|,\ |b_t|/a\}$.
Since $a\ge a_{\min}$, we have $|b_t|/a \le |b_t|/a_{\min}$, and the claim follows by induction.
Proof in Appendix~\ref{app:proof_zoh_state_bound}.
\end{proof}

\paragraph*{Failure of freeze time}
Mamba may slow decay by keeping $\lambda_n\Delta_t\approx 0$ over selected steps.
We rule out this behavior by assuming that accumulated discretization time grows linearly on every relevant interval.

\begin{proposition}[Failed freeze time yields exponential forgetting]
\label{prop:freeze_errors_exp}
Consider a single-mode diagonal selective SSM channel with memory factor
\[
\Pi_{t,\ell}
:=
\prod_{r=t-\ell+1}^{t}\exp(-\lambda\,\Delta_r)
=
\exp\!\Big(-\lambda\sum_{r=t-\ell+1}^{t}\Delta_r\Big),
\qquad \lambda>0.
\]
Assume there exists $c_\Delta>0$ such that for every relevant pair $\tau<t$,
\[
\sum_{r=\tau+1}^{t}\Delta_r \ge c_\Delta (t-\tau).
\]
Then
\[
\Pi_{t,\ell}\le \exp\!\big(-\lambda c_\Delta \ell\big).
\]
Equivalently, once freeze time cannot be maintained over a long interval, the memory factor is exponentially small in the lag.
\end{proposition}

\begin{proof}[Proof sketch]
This is immediate from
\[
\Pi_{t,\ell}
=
\exp\!\Big(-\lambda\sum_{r=\tau+1}^{t}\Delta_r\Big)
\]
and the assumed linear lower bound on the accumulated discretization time.
Proof in Appendix~\ref{app:proof_freeze_errors_exp}.
\end{proof}

\subsubsection{Attention dilution}
\label{sec:attn_dilution_formal}

For causal self-attention, the direct contribution of token $\tau$ to $y_t$ is the one-hop weight $\alphadrv_{t,\tau}$.
In diffuse regimes this is $O(1/|\VisSet{t}|)$, hence $O(1/(t+1))$ for full-prefix attention.
For very old tokens with $\tau=O(1)$ and $t\asymp \ell$, this becomes $O(1/\ell)$.
This is a dilution phenomenon controlled primarily by the query time $t$, rather than a multi-hop forgetting mechanism.

\subsubsection{Polynomial decay in Sessa}
\label{sec:sessa_poly_tail}

We formalize a regime in which the Sessa feedback solve yields polynomial decay in the lag $\ell$.

\paragraph*{Scalar recursion}
Let $(\gamma_t)_{t\ge 0}$ be scalars and let $(\alphafb_{t,j})_{t\ge 1,\,0\le j<t}$ satisfy
$\alphafb_{t,j}\ge 0$ and $\sum_{j<t}\alphafb_{t,j}\le 1$.
Given a forward sequence $(f_t)$, define
\begin{equation}
y_0=f_0,\qquad
y_t=f_t+\gamma_t\sum_{j=0}^{t-1}\alphafb_{t,j}y_j,\quad t\ge 1.
\label{eq:scalar_feedback}
\end{equation}
For an impulse input at time $\tau$, set $f_\tau=1$ and $f_t=0$ for $t\neq\tau$.
This yields an influence profile
$y_{t}$ supported on $t\ge \tau$; the relevant memory variable is the lag $\ell=t-\tau$.

\begin{assumption}[Diffuse feedback routing envelope]
\label{ass:alpha_upper_main}
There exists $c_2\in(0,\infty)$ such that for all $t\ge 1$ and all $0\le j<t$,
\begin{equation}
\alphafb_{t,j}\le \frac{c_2}{t}.
\label{eq:alpha_upper_main}
\end{equation}
\end{assumption}

\begin{assumption}[Bounded feedback gain]
\label{ass:gamma_bound_main}
There exists $\gamma_{\max}\in[0,1)$ such that $|\gamma_t|\le \gamma_{\max}$ for all $t\ge 0$.
\end{assumption}

Define $\betatail:=1-\gamma_{\max}c_2$ and assume $\gamma_{\max}c_2<1$, so $\betatail\in(0,1]$.

\begin{theorem}[Polynomial decay of impulse influence]
\label{thm:poly_decay_main}
Under Assumptions~\ref{ass:alpha_upper_main}--\ref{ass:gamma_bound_main} and $\betatail:=1-\gamma_{\max}c_2\in(0,1]$,
the impulse influence induced by \eqref{eq:scalar_feedback} satisfies, for all lags $\ell\ge 1$,
\[
|y_{\tau+\ell}|\;\le\; C\,\ell^{-\betatail},
\qquad\text{for instance}\qquad
C=(1-\betatail)\,e^{\,1-\betatail}.
\]
uniformly over the impulse time $\tau$ (when the same constants apply).
\end{theorem}

\begin{proof}[Proof sketch]
Shift the recursion to start at $\tau$ and apply a comparison argument controlling partial sums by a harmonic-growth recursion,
yielding $\ell^{-\betatail}$. For $0<\betatail<1$, the full proof appears in
Appendix~\ref{sec:poly_decay}, Corollary~\ref{cor:impulse_j_tail} with $j=\tau$.
The endpoint case $\betatail=1$ corresponds to $\eta=\gamma_{\max}c_2=0$, hence
$\gamma_t=0$ for all $t$ and therefore $y_{\tau+\ell}=0$ for all $\ell\ge 1$;
see also Remark~\ref{rem:eta_zero}.
\end{proof}

\begin{remark}[Subcriticality]
\label{rem:poly_tail_subcriticality}
Whenever we refer in prose to a polynomial tail induced by diffuse feedback routing,
this always means the subcritical regime
\[
\alphafb_{t,j}\le \frac{c_2}{t},
\qquad
|\gamma_t|\le \gamma_{\max},
\qquad
\gamma_{\max}c_2<1.
\]
Equivalently,
\[
\betatail:=1-\gamma_{\max}c_2\in(0,1].
\]
The nontrivial heavy-tail case is \(0<\betatail<1\).
The endpoint \(\betatail=1\) corresponds to \(\gamma_{\max}c_2=0\), in which case
the post-source impulse is identically zero; see Remark~\ref{rem:eta_zero}.
Thus bounded gains alone do not suffice: the strict subcriticality condition
\(\gamma_{\max}c_2<1\) is essential in every use of Theorem~\ref{thm:poly_decay_main}.
\end{remark}

\paragraph*{Comparison and sharpness.}
Under the subcritical diffuse-routing assumptions above, Sessa yields a polynomial tail
\(
\ell^{-\betatail}
\),
unlike the exponential forgetting of stable LTI feedback systems
(Proposition~\ref{prop:lti_exp_decay}) and failed-freeze-time Mamba
(Section~\ref{sec:mamba_exp}).
The exponent is sharp:
in the explicit uniform-routing regime
$\alphafb_{t,j}=\frac{1}{t}\mathbf 1[j<t]$ with constant $\gamma\in(0,1)$,
Appendix Corollary~\ref{cor:theta_tail_uniform_shifted} gives the closed form
\[
y_{\tau+\ell}
=
\gamma\,\frac{\Gamma(\tau+1)}{\Gamma(\tau+1+\gamma)}
\cdot
\frac{\Gamma(\tau+\ell+\gamma)}{\Gamma(\tau+\ell+1)},
\]
and hence $y_{\tau+\ell}=\Theta_\tau(\ell^{-\betatail})$ with $\betatail=1-\gamma$ for every fixed $\tau$.
Appendix Corollary~\ref{cor:theta_tail_uniform_bounded_source} further gives a uniform two-sided envelope on every bounded source family for a single layer.
These one-layer statements are distinct from the deep selective-retrieval theorem below, which uses a different multi-layer construction.

\paragraph*{Connection to attention dilution}
Diffuse attention in a one-hop mixer yields per-token weights of order $O(1/t)$ and, for very old tokens, $O(1/\ell)$.
In contrast, under the diffuse-routing assumptions of Theorem~\ref{thm:poly_decay_main}, Sessa yields a tail $O(\ell^{-\betatail})$ with $\betatail<1$,
which is asymptotically slower than $1/\ell$ and therefore can mitigate dilution by sustaining longer-range influence through the
stateful feedback channel while remaining BIBO-stable under Section~\ref{sec:theory_bibo}.

\begin{proposition}[Decay envelopes in the diffuse regime]
\label{prop:shared_diffuse_decay_envelopes}
Fix a horizon $\Tctx$ and consider the fixed-routing influence Jacobians of
Section~\ref{sec:fair_comparison}.
The three items below are stated under the mechanism-specific assumptions introduced above.

\begin{enumerate}[label=(\roman*), leftmargin=*, nosep]
\item \textbf{Transformer.}
In the diffuse regime with full-prefix visibility, the value Jacobian satisfies
\[
\|J^{\mathrm{attn}}_{t,\tau}\| = \alphadrv_{t,\tau} = \Theta\!\Big(\frac{1}{t+1}\Big)
\qquad (\tau\le t),
\]
and in particular for a fixed old source $\tau=O(1)$ and lag $\ell=t-\tau$,
\[
\|J^{\mathrm{attn}}_{\tau+\ell,\tau}\|=\Theta(1/\ell).
\]

\item \textbf{Mamba.}
Assume the realized recurrence has diagonal transitions
\[
\Assmt{r}=\operatorname{diag}(\exp(-a_n\Delta_r)),
\qquad a_n\ge \lambda>0,
\]
and bounded input/output factors $\sup_r\|\Bssmt{r}\|,\sup_r\|\Cssmt{r}\|<\infty$.
If, on the region of interest,
\[
\sum_{r=\tau+1}^{t}\Delta_r\ge c_\Delta (t-\tau),
\]
then the impulse Jacobian obeys
\[
\|J^{\mathrm{ssm}}_{t,\tau}\|
\le c\,\exp\!\big(-\lambda c_\Delta (t-\tau)\big)
=
c\,e^{-\lambda c_\Delta \ell}.
\]
This expresses exponential forgetting under failed freeze time:
the model cannot maintain a long preserve corridor, so accumulated discretization time grows linearly in the lag.

\item \textbf{Sessa.}
Under the hypotheses of Theorem~\ref{thm:poly_decay_main}, the solve Jacobian column corresponding to an impulse
in $f$ obeys the polynomial envelope
\[
|J^{\mathrm{sessa}}_{\tau+\ell,\tau}|\le C\,\ell^{-\betatail},
\qquad \betatail\in(0,1],
\]
as in Theorem~\ref{thm:poly_decay_main}.
Moreover, in the explicit uniform-routing regime
$\BfbEnt{t}{j}=
\begin{cases}
0, & t=0,\\[2pt]
\dfrac{\gamma}{t}\mathbf 1[j<t], & t\ge 1,
\end{cases}$
with $\gamma\in(0,1)$ and $\betatail=1-\gamma$, this envelope is tight in the following qualified sense:
for every fixed source position $\tau$,
\[
|J^{\mathrm{sessa}}_{\tau+\ell,\tau}|=\Theta_\tau(\ell^{-\betatail}),
\]
by Corollary~\ref{cor:theta_tail_uniform_shifted}.
Moreover, for every bounded source family $0\le \tau\le \tau_{\max}$ there exist constants
$c^-_{\tau_{\max}},c^+_{\tau_{\max}}>0$ such that
\[
c^-_{\tau_{\max}}\,\ell^{-\betatail}
\le
|J^{\mathrm{sessa}}_{\tau+\ell,\tau}|
\le
c^+_{\tau_{\max}}\,\ell^{-\betatail}
\]
for all $0\le \tau\le \tau_{\max}$ and all $\ell\ge 1$,
by Corollary~\ref{cor:theta_tail_uniform_bounded_source}.
In particular, the same two-sided bound holds uniformly on every fixed finite horizon.
\end{enumerate}
Proof in Appendix~\ref{app:proof_shared_diffuse_decay_envelopes}.
\end{proposition}

\begin{proposition}[End-to-end decay envelopes]
\label{prop:shared_regime_decay_envelopes_e2e}
Fix a horizon $\Tctx$ and consider one-block end-to-end Jacobians.
In item~(i) we assume the diffuse smooth-routing regime of Section~\ref{sec:e2e_comparison}.
Assume additionally that tokenwise maps are bounded and Lipschitz on the input set:
$\|v(x_t)\|\le V_R$ and $\|\partial v(x_t)/\partial x_t\|\le L_v$.

\begin{enumerate}[label=(\roman*), leftmargin=*, nosep]
\item \textbf{Transformer.}
For $y_t=\sum_{j\le t}\alphadrv_{t,j}(x)\,v(x_j)$ and any $\tau<t$,
\[
\Big\|\frac{\partial y_t}{\partial x_\tau}\Big\|
\;\le\;
\alphadrv_{t,\tau}\,L_v
\;+\;
V_R\sum_{j\le t}\Big\|\frac{\partial \alphadrv_{t,j}}{\partial x_\tau}\Big\|
\;\lesssim\;\frac{1}{t+1}.
\]
In particular, for a fixed old source $\tau=O(1)$ and lag $\ell=t-\tau$, one gets $\|J^{\mathrm{e2e}}_{\tau+\ell,\tau}\|=O(1/\ell)$.

\item \textbf{Mamba.}
Assume the block admits a local ZOH-diagonal parametrization as in Proposition~\ref{prop:mamba_e2e_decay}.
If, on the input set of interest, there exists $c_\Delta>0$ such that for every $\tau<t$,
\[
\sum_{r=\tau+1}^{t}\Delta_r(x)\ge c_\Delta(t-\tau),
\]
then Corollary~\ref{cor:mamba_e2e_freeze} yields
\[
\Big\|\frac{\partial y_t}{\partial x_\tau}\Big\|
\le C(R)\exp\!\big(-\lambda c_\Delta (t-\tau)\big)
=
C(R)e^{-\lambda c_\Delta \ell}.
\]

\item \textbf{Sessa.}
Assume additionally the hypotheses of Corollary~\ref{cor:jacobian_tail_block}.
Under the diffuse feedback routing assumptions of Appendix~\ref{app:jacobian_tail},
\[
\Big\|\frac{\partial y_t}{\partial x_\tau}\Big\|
\;\le\; C\,\ell^{-\betatail}\big(1+\log(1+\ell)\big),
\qquad \betatail\in(0,1),
\]
via Corollary~\ref{cor:jacobian_tail_block}.
\end{enumerate}
\end{proposition}

\begin{proof}[Proof sketch]
\begin{enumerate}[label=(\roman*), leftmargin=*, nosep]
\item Differentiate \(y_t=\sum_{j\le t}\alphadrv_{t,j}(x)v(x_j)\): one term is controlled by
\(\alphadrv_{t,\tau}L_v\) and the other by
\(V_R\sum_{j\in \VisSet{t}}\|\partial \alphadrv_{t,j}/\partial x_\tau\|\).
Under the diffuse smooth-routing regime both are \(O(1/|\VisSet{t}|)\), hence \(O(1/(t+1))\)
for full-prefix attention.

\item Combine Proposition~\ref{prop:mamba_e2e_decay} with the deterministic failed-freeze-time condition
\[
\sum_{r=\tau+1}^{t}\Delta_r(x)\ge c_\Delta(t-\tau),
\]
or equivalently use Corollary~\ref{cor:mamba_e2e_freeze}.

\item This follows from Corollary~\ref{cor:jacobian_tail_block} under the additional Sessa assumptions
stated in item~(iii).
\end{enumerate}
\end{proof}

\begin{corollary}[Failed freeze time implies exponential decay of Mamba end-to-end Jacobians]
\label{cor:mamba_e2e_freeze}
Under the hypotheses of Proposition~\ref{prop:mamba_e2e_decay}, assume additionally the failed freeze-time condition of Proposition~\ref{prop:freeze_errors_exp}, namely that there exists $c_\Delta>0$ such that
\[
\sum_{r=\tau+1}^{t}\Delta_r(x)\ge c_\Delta(t-\tau)
\]
for every relevant pair $\tau<t$ and every $x\in\mathcal X_R$.
Then
\[
\Big\|\frac{\partial y_t(x)}{\partial x_\tau}\Big\|
\le
C(R)\exp\!\big(-\lambda c_\Delta (t-\tau)\big).
\]
\end{corollary}

\begin{proof}[Proof sketch]
Combine Proposition~\ref{prop:mamba_e2e_decay} with Proposition~\ref{prop:freeze_errors_exp}.
\end{proof}

\subsubsection{Deep end-to-end bounds}
\label{sec:theory_multilayer_e2e}

The fixed-routing Jacobians remain useful as mechanism diagnostics, but deep architectural
statements must be made for the \emph{end-to-end Jacobians}
\[
J^{\mathrm{e2e},(\Nlayer)}_{t,\tau}(x)
:=
\frac{\partial h_t^{(\Nlayer)}(x)}{\partial h_\tau^{(0)}(x)}\in\mathbb R^{D\times D},
\]
since these are the quantities that compose across layers by the chain rule.
The next theorem gives the corresponding deep path-sum expansion.

\begin{theorem}[Deep end-to-end aggregation]
\label{thm:e2e_deep_calculus_main}
Fix a depth $\Nlayer\ge 1$, a finite horizon $\Tctx$, and a compact input set $\mathcal X_0$.
Let
\[
h^{(0)}=x\in\mathcal X_0,\qquad
h^{(\idxlayer)} = F_{\idxlayer}\bigl(h^{(\idxlayer-1)}\bigr),
\qquad \idxlayer=1,\dots,\Nlayer,
\]
where each block $F_{\idxlayer}$ is causal and continuously differentiable on the relevant compact set
\[
\mathcal X_{\idxlayer-1}:=
F_{\idxlayer-1}\circ\cdots\circ F_1(\mathcal X_0).
\]
Assume that for each layer $\idxlayer$ there exist constants
\[
d_{\idxlayer}\ge 0,\qquad \lambda_{\idxlayer}\ge 0,
\]
and a scalar lower-triangular kernel
\[
K_{\idxlayer}: \{(t,\tau): 0\le \tau<t\le \Tctx-1\}\to [0,\infty)
\]
such that for every $u\in\mathcal X_{\idxlayer-1}$ and every $0\le \tau\le t\le \Tctx-1$,
\begin{equation}
\left\|
\frac{\partial F_{\idxlayer,t}(u)}{\partial u_\tau}
\right\|
\le
d_{\idxlayer}\,\mathbf 1[t=\tau]
+
\lambda_{\idxlayer}\,K_{\idxlayer}(t,\tau)\,\mathbf 1[\tau<t].
\label{eq:e2e_one_block_kernel_envelope_main}
\end{equation}
Then for every $x\in\mathcal X_0$ and every $0\le \tau<t\le \Tctx-1$,
\begin{align}
\left\|
J^{\mathrm{e2e},(\Nlayer)}_{t,\tau}(x)
\right\|
&\le
\sum_{k=1}^{\Nlayer}
\ \sum_{1\le \idxlayeri{1}<\cdots<\idxlayeri{k}\le \Nlayer}
\left(\prod_{m\notin\{\idxlayeri{1},\dots,\idxlayeri{k}\}} d_m\right)
\notag\\
&\qquad\qquad\qquad\qquad\cdot
\sum_{\tau=i_0<i_1<\cdots<i_k=t}
\ \prod_{r=1}^{k}
\lambda_{\idxlayeri{r}}\,K_{\idxlayeri{r}}(i_r,i_{r-1}).
\label{eq:e2e_deep_path_sum_main}
\end{align}
The same expansion also gives the diagonal bound
\[
\left\|
J^{\mathrm{e2e},(\Nlayer)}_{t,t}(x)
\right\|
\le
\prod_{\idxlayer=1}^{\Nlayer} d_{\idxlayer}.
\]
\end{theorem}

\begin{proof}[Proof sketch]
This is a direct chain-rule expansion for the full block Jacobian.
Proof in Appendix~\ref{app:deep_e2e_calculus}.
\end{proof}

Thus deep long-range memory is controlled by the path sum induced by the
one-block end-to-end Jacobian envelopes.

For the family-over-horizon comparison used below, one needs a horizon-uniform
version of this calculus, i.e., bounds whose constants are independent of the
context length \(\Tctx\). The fixed-horizon model-class estimates and the
abstract horizon-uniform lifting are recorded in
Appendix~\ref{app:deep_e2e_calculus}--\ref{app:deep_e2e_uniform_horizon}.
Here we state only the resulting horizon-uniform decay envelopes needed for the
comparison-class impossibility argument.

\begin{corollary}[Horizon-uniform deep decay envelopes]
\label{cor:e2e_deep_decay_envelopes_uniform_horizon}
Assume the hypotheses of
Appendix Theorem~\ref{thm:app_e2e_residual_calculus_uniform}.

\begin{enumerate}[label=(\roman*), leftmargin=*, nosep]
\item \textbf{Transformer.}
Assume that for each layer \(\idxlayer\) there exists \(a_{\idxlayer}>0\) such that
\[
K_{\idxlayer}(t,\tau)\le \frac{a_{\idxlayer}}{t+1},
\qquad \tau<t.
\]
Fix a bounded source family \(0\le \tau\le \tau_{\max}\).
Then for every \(\ell\ge 1\),
\[
\sup_{\Tctx\ge \tau_{\max}+\ell+1}
\ \sup_{0\le \tau\le \tau_{\max}}
\ \sup_{x\in\mathcal X_0^{(\Tctx)}}
\left\|
J^{\mathrm{e2e},(\Nlayer)}_{\tau+\ell,\tau}(x;\Tctx)
\right\|
\lesssim_{\tau_{\max},\Nlayer}
\frac{(\log(1+\ell))^{\Nlayer-1}}{1+\ell}.
\]
In particular, the right-hand side tends to \(0\) as \(\ell\to\infty\), so this is a genuine
horizon-uniform asymptotic dilution law on bounded-source families.

\item \textbf{Mamba.}
Assume that for each layer \(\idxlayer\) there exist \(a_{\idxlayer}>0\) and \(c_{\idxlayer}>0\) such that
\[
K_{\idxlayer}(t,\tau)\le a_{\idxlayer}e^{-c_{\idxlayer}(t-\tau)},
\qquad \tau<t.
\]
Set \(c_\ast:=\min_{\idxlayer}c_{\idxlayer}\).
Then for every \(\ell\ge 1\),
\[
\sup_{\Tctx\ge \ell+1}
\ \sup_{0\le \tau\le \Tctx-\ell-1}
\ \sup_{x\in\mathcal X_0^{(\Tctx)}}
\left\|
J^{\mathrm{e2e},(\Nlayer)}_{\tau+\ell,\tau}(x;\Tctx)
\right\|
\lesssim_{\Nlayer}
(1+\ell)^{\Nlayer-1}e^{-c_\ast\ell}.
\]
In particular, this yields a genuine horizon-uniform exponential forgetting law in the lag \(\ell\).

\item \textbf{Sessa.}
Assume that for each layer \(\idxlayer\) there exist \(a_{\idxlayer}>0\) and a common exponent
\(\betatail\in(0,1)\) such that
\[
K_{\idxlayer}(t,\tau)\le
a_{\idxlayer}(t-\tau)^{-\betatail}\bigl(1+\log(1+t-\tau)\bigr),
\qquad \tau<t.
\]
Then for every \(\ell\ge 1\),
\[
\sup_{\Tctx\ge \ell+1}
\ \sup_{0\le \tau\le \Tctx-\ell-1}
\ \sup_{x\in\mathcal X_0^{(\Tctx)}}
\left\|
J^{\mathrm{e2e},(\Nlayer)}_{\tau+\ell,\tau}(x;\Tctx)
\right\|
\lesssim_{\Nlayer,\betatail}
\sum_{k=1}^{\Nlayer}
\ell^{k(1-\betatail)-1}\bigl(1+\log(1+\ell)\bigr)^k.
\]
In particular, if
\[
\Nlayer(1-\betatail)<1,
\]
then the right-hand side tends to \(0\) as \(\ell\to\infty\), yielding a genuine
horizon-uniform asymptotic decay law in the lag.
Outside this subcritical regime, one still retains a controlled horizon-uniform upper envelope.
\end{enumerate}
\end{corollary}

\begin{proof}[Proof sketch]
Apply the horizon-uniform residual calculus in
Appendix Theorem~\ref{thm:app_e2e_residual_calculus_uniform}.
The Transformer, Mamba, and Sessa kernel-class estimates are proved in
Appendix Propositions~\ref{prop:app_e2e_transformer_deep},
\ref{prop:app_e2e_ssm_deep}, and \ref{prop:app_e2e_sessa_deep}, respectively.
Combining those bounds yields the stated horizon-uniform envelopes.
\end{proof}

\paragraph*{Consequence}
The fixed-horizon deep bounds are recorded in Appendix~\ref{app:deep_e2e_calculus},
whereas Corollary~\ref{cor:e2e_deep_decay_envelopes_uniform_horizon} gives lag laws uniform in \(\Tctx\).
Thus diffuse Transformers dilute like \((\log \ell)^{\Nlayer-1}/\ell\) on bounded-source families,
failed-freeze-time Mamba attenuates exponentially,
and Sessa retains the stated heavy-tail upper envelope.
These are upper-envelope results.
They are the right tool for the impossibility statements for the comparison classes,
but they do not yet yield a positive retrieval theorem for Sessa.
The next subsection does.

\subsubsection{Flexible finite-horizon selective retrieval}
\label{sec:theory_uniform_finite_horizon_retrieval}

We now state the main positive memory theorem of the section.
The point is not merely that Sessa admits a heavy-tail upper envelope,
but that on each finite-horizon family it can realize prescribed retrieval exponents
\(\nu_k(\beta)=k(1-\beta)-1\),
with constants uniform in both the horizon \(H\) and the source index \(\tau_\ast\).
For each \(H\) and \(\tau_\ast\), the realizing network may depend on \((H,\tau_\ast)\),
while the retrieval-profile constants remain uniform in both parameters.
 
\begin{definition}[Flexible finite-horizon profile realization]
\label{def:uniform_finite_horizon_profile_realization}
Fix an integer \(\tau_{\max}\ge 0\), an exponent \(\nu\in\mathbb R\), and for each \(H\ge 1\) a horizon
\[
T_H:=\tau_{\max}+H+1.
\]
Let \(\mathcal X_0^{(H)}\subset (\mathbb R^D)^{T_H}\) be compact input sets satisfying the uniform bound
\[
\sup_{H\ge 1}\ \sup_{x\in\mathcal X_0^{(H)}} \|x\|_{\infty,2}\le R<\infty.
\]

Let \(\mathfrak C\) be an architecture class.
We say that \(\mathfrak C\) realizes the profile \(\nu\) on the bounded source family
\[
0\le \tau_\ast\le \tau_{\max}
\]
if there exist constants
\[
m_->0,\qquad m_+<\infty,\qquad c_->0,
\]
independent of \(H\) and \(\tau_\ast\), such that for every \(H\ge 1\) and every source index
\(\tau_\ast\in\{0,\dots,\tau_{\max}\}\), there exist

\begin{enumerate}[label=(\roman*), leftmargin=*, nosep]
\item a network \(G_{H,\tau_\ast}\in \mathfrak C\) acting on \((\mathbb R^D)^{T_H}\),

\item a source probe
\[
c^{(H,\tau_\ast)}\in\mathbb R^D
\qquad\text{and target probes}\qquad
\rho_t^{(H,\tau_\ast)}\in\mathbb R^D,\quad 0\le t\le T_H-1,
\]
satisfying the normalization bounds
\[
\|c^{(H,\tau_\ast)}\|_2\le 1,
\qquad
\|\rho_t^{(H,\tau_\ast)}\|_2\le 1
\quad (0\le t\le T_H-1),
\]

\item the full end-to-end Jacobian blocks
\[
J^{G_{H,\tau_\ast}}_{t,\tau}(x)
:=
\frac{\partial G_{H,\tau_\ast,t}(x)}{\partial x_\tau}
\in\mathbb R^{D\times D},
\]

\item the scalar transport score
\[
\mathsf S^{(H,\tau_\ast)}_{t,\tau}(x)
:=
\bigl(\rho_t^{(H,\tau_\ast)}\bigr)^\top
J^{G_{H,\tau_\ast}}_{t,\tau}(x)\,
c^{(H,\tau_\ast)},
\]

\item and the corresponding selective margin
\[
\mathsf M^{(H,\tau_\ast)}_{t,\tau_\ast}(x)
:=
\mathsf S^{(H,\tau_\ast)}_{t,\tau_\ast}(x)
-
\sum_{\substack{0\le \tau<t\\ \tau\neq \tau_\ast}}
\bigl|\mathsf S^{(H,\tau_\ast)}_{t,\tau}(x)\bigr|.
\]
\end{enumerate}

These data are required to satisfy, for every \(x\in\mathcal X_0^{(H)}\),
\[
m_-
\le
\mathsf M^{(H,\tau_\ast)}_{\tau_\ast+1,\tau_\ast}(x)
\le
m_+,
\]
and
\[
\mathsf M^{(H,\tau_\ast)}_{\tau_\ast+\ell,\tau_\ast}(x)
\ge
c_-(1+\ell)^\nu,
\qquad
1\le \ell\le H.
\]
\end{definition}

\begin{theorem}[Flexible finite-horizon selective retrieval for deep Sessa]
\label{thm:flexible_route_B_sessa}
Work in the identity-normalized formulation with the exact GELU activation
\[
\mathrm{GELU}(z)=z\,\Phi(z),
\]
and assume
\[
D\ge 7.
\]
Fix
\[
\beta\in(0,1),\qquad k\ge 1,\qquad \tau_{\max}\ge 0,
\]
and define
\[
\nu_k(\beta):=k(1-\beta)-1.
\]
Let \(\{\mathcal X_0^{(H)}\}_{H\ge 1}\) be a uniformly bounded family of compact sets as in
Definition~\ref{def:uniform_finite_horizon_profile_realization}.
Then the class of LN-free Sessa networks realizes the profile \(\nu_k(\beta)\) on the
bounded source family \(0\le \tau_\ast\le \tau_{\max}\) in the sense of
Definition~\ref{def:uniform_finite_horizon_profile_realization}.

More precisely, there exist constants
\[
m_->0,\qquad m_+<\infty,\qquad c_->0,
\]
depending only on \((k,\beta,\tau_{\max},R)\), but independent of \(H\) and \(\tau_\ast\), such that for every \(H\ge 1\) and every \(\tau_\ast\in\{0,\dots,\tau_{\max}\}\), there exist a finite-depth LN-free Sessa network
\[
G_{H,\tau_\ast}:(\mathbb R^D)^{T_H}\to (\mathbb R^D)^{T_H}
\]
and a scalar channel score \(\mathsf S^{(H,\tau_\ast)}\) with selective margin
\(\mathsf M^{(H,\tau_\ast)}\) such that for every \(x\in\mathcal X_0^{(H)}\),
\[
m_-
\le
\mathsf M^{(H,\tau_\ast)}_{\tau_\ast+1,\tau_\ast}(x)
\le
m_+,
\]
and
\[
\mathsf M^{(H,\tau_\ast)}_{\tau_\ast+\ell,\tau_\ast}(x)
\ge
c_-(1+\ell)^{\nu_k(\beta)},
\qquad
1\le \ell\le H.
\]

Consequently:
if \(\nu_k(\beta)<0\), deep Sessa realizes a decaying profile; if \(\nu_k(\beta)=0\), it realizes a frozen profile; and if \(\nu_k(\beta)>0\), it realizes an increasing profile.
\end{theorem}

\begin{proof}[Proof sketch]
\smallskip\noindent\emph{Composite architecture.}
Fix \(H\ge 1\) and \(0\le \tau_\ast\le \tau_{\max}\).
Set
\[
L_H:=\tau_{\max}+H,
\qquad
T_H:=L_H+1.
\]
We construct
\[
G_{H,\tau_\ast}
=
M_{H,k}\circ\cdots\circ M_{H,1}\circ
S_{H,\tau_\ast,\varepsilon_H}\circ Q_H\circ P_H.
\]
Here \(P_H\) writes a strictly ordered positional code, \(Q_H\) is a signal-transparent preparatory network producing the power profile, \(S_{H,\tau_\ast,\varepsilon_H}\) selects the source \(\tau_\ast\), and \(M_{H,1},\dots,M_{H,k}\) are diffuse profile-compensated macro-layers.

By Corollaries~\ref{cor:pos_code_one_direction} and~\ref{cor:pos_code_signal_transparency}, \(P_H\) writes a strictly ordered positional code on \(e_{\mathrm{pos}}\) while remaining transparent to perturbations along \(e_{\mathrm{sig}}\). Corollary~\ref{cor:preparatory_power_profile} yields a constant-depth network \(Q_H\) that preserves the signal and positional channels and writes a profile
\[
r_t\asymp (t+1)^{1-\beta}.
\]
Lemma~\ref{lem:selector_window} yields a selector with gain \(\asymp 1\) at \(\tau_\ast\) and off-target suppression \(\varepsilon_H\asymp(H+1)^{-1}\). Lemma~\ref{lem:profile_weighted_macro} yields macro-layers whose selected-channel transport has kernel size \(\asymp (i+1)^{-\beta}\).

Appendix Lemma~\ref{lem:transparent_preprocessing_exact_transport} identifies the selected-channel transport of the post-preparatory stack with the actual Jacobian score. The desired lower bound follows by restricting to balanced \(k\)-jump paths and applying Lemma~\ref{lem:balanced_profile_paths}, while the competitor contribution is controlled by Lemma~\ref{lem:competitor_suppression_profile}. Choosing the construction constants appropriately makes the competitor mass absorbable for all \(1\le \ell\le H\), yielding the stated anchored bounds.
\end{proof}

\begin{corollary}[Flexible frozen and increasing profiles require depth]
\label{cor:flexible_route_B_depth_requirement}
Under Theorem~\ref{thm:flexible_route_B_sessa}:
\begin{enumerate}[label=(\roman*), leftmargin=*, nosep]
\item for \(k=1\), one has
\[
\nu_1(\beta)=-\beta<0,
\]
so only decaying profiles occur;
\item for \(k\ge 2\) and
\[
\beta=1-\frac1k,
\]
one gets the frozen profile \(\nu_k(\beta)=0\);
\item for \(k\ge 2\) and
\[
0<\beta<1-\frac1k,
\]
one gets the increasing profile \(\nu_k(\beta)>0\).
\end{enumerate}
\end{corollary}
\subsubsection{Impossibility for the comparison classes in the same flexible finite-horizon regime}
\label{sec:theory_uniform_finite_horizon_comparison_classes}

This is the matching negative statement in the same family-over-\(H\) regime.
By the horizon-uniform end-to-end envelopes from
Section~\ref{sec:theory_multilayer_e2e},
diffuse fixed-depth Transformers and failed-freeze-time fixed-depth Mamba admit only decaying upper bounds,
so they cannot realize frozen or increasing retrieval profiles.

\begin{proposition}[Comparison-class impossibility for flexible selective retrieval]
\label{prop:comparison_class_impossibility_same_regime}
Fix \(\tau_{\max}\ge 0\), and let
\[
T_H=\tau_{\max}+H+1.
\]

Assume we are given, for every \(H\ge 1\) and every \(\tau_\ast\in\{0,\dots,\tau_{\max}\}\), a network
\[
G^{\mathrm{comp}}_{H,\tau_\ast}
\]
from one of the following two comparison classes:
a depth-\(L\) causal Transformer in the diffuse smooth-routing regime,
or a depth-\(L\) causal Mamba stack in the failed-freeze-time regime.

Assume moreover that, in the Transformer case, the family satisfies the hypotheses of
Corollary~\ref{cor:e2e_deep_decay_envelopes_uniform_horizon}, item~(i),
with constants independent of \(H\) and \(\tau_\ast\), and that, in the Mamba case,
the family satisfies the hypotheses of
Corollary~\ref{cor:e2e_deep_decay_envelopes_uniform_horizon}, item~(ii),
with constants independent of \(H\) and \(\tau_\ast\).

Then no such comparison-class family can realize a frozen or increasing profile in the sense of
Definition~\ref{def:uniform_finite_horizon_profile_realization}. More precisely:

\begin{enumerate}[label=(\roman*), leftmargin=*, nosep]
\item \textbf{Transformer.}
There do not exist constants \(m_->0\), \(m_+<\infty\), \(c_->0\), and \(\nu\ge 0\),
independent of \(H\) and \(\tau_\ast\), such that
\[
m_-\le \mathsf M^{(H,\tau_\ast)}_{\tau_\ast+1,\tau_\ast}(x)\le m_+,
\]
and
\[
\mathsf M^{(H,\tau_\ast)}_{\tau_\ast+\ell,\tau_\ast}(x)\ge c_-(1+\ell)^\nu,
\qquad 1\le \ell\le H,
\]
hold uniformly for all \(H,\tau_\ast,x\).

\item \textbf{Mamba.}
The same impossibility holds for failed-freeze-time Mamba families.
\end{enumerate}
\end{proposition}

\begin{proof}[Proof]
Assume toward a contradiction that such a realization exists.
By Definition~\ref{def:uniform_finite_horizon_profile_realization}, the probes satisfy
\[
\|c^{(H,\tau_\ast)}\|_2\le 1,
\qquad
\|\rho_t^{(H,\tau_\ast)}\|_2\le 1.
\]
Hence for every admissible \(H,\tau_\ast,x,t,\tau\),
\[
\bigl|\mathsf S^{(H,\tau_\ast)}_{t,\tau}(x)\bigr|
=
\bigl|
(\rho_t^{(H,\tau_\ast)})^\top
J^{G^{\mathrm{comp}}_{H,\tau_\ast}}_{t,\tau}(x)\,
c^{(H,\tau_\ast)}
\bigr|
\le
\bigl\|J^{G^{\mathrm{comp}}_{H,\tau_\ast}}_{t,\tau}(x)\bigr\|.
\]
Therefore
\[
\mathsf M^{(H,\tau_\ast)}_{t,\tau_\ast}(x)
\le
\bigl|\mathsf S^{(H,\tau_\ast)}_{t,\tau_\ast}(x)\bigr|
\le
\bigl\|J^{G^{\mathrm{comp}}_{H,\tau_\ast}}_{t,\tau_\ast}(x)\bigr\|.
\]

For Transformers, Corollary~\ref{cor:e2e_deep_decay_envelopes_uniform_horizon}, item~(i),
applied to the family \(G^{\mathrm{comp}}_{H,\tau_\ast}\), gives the horizon-uniform
bounded-source-family envelope
\[
\bigl\|J^{G^{\mathrm{comp}}_{H,\tau_\ast}}_{\tau+\ell,\tau}(x)\bigr\|
\lesssim
\frac{(\log(1+\ell))^{L-1}}{1+\ell},
\]
uniformly over all admissible \(H,\tau_\ast,x\) and all \(0\le \tau\le \tau_{\max}\).
This tends to \(0\) as \(\ell\to\infty\).

For Mamba, item~(ii) gives
\[
\bigl\|J^{G^{\mathrm{comp}}_{H,\tau_\ast}}_{\tau+\ell,\tau}(x)\bigr\|
\lesssim
(1+\ell)^{L-1}e^{-c\ell},
\]
uniformly over all admissible \(H,\tau_\ast,x,\tau\).
This also tends to \(0\).

Since a frozen or increasing profile would require
\[
\mathsf M^{(H,\tau_\ast)}_{\tau_\ast+\ell,\tau_\ast}(x)\ge c_-(1+\ell)^\nu
\qquad (\nu\ge 0),
\]
uniformly in all admissible \(H,\tau_\ast,x,\ell\), this is impossible in either comparison class.
\end{proof}

\begin{corollary}[Flexible selective retrieval separates Sessa from the comparison classes]
\label{cor:uniform_finite_horizon_separation}
In the regime of Definition~\ref{def:uniform_finite_horizon_profile_realization}:
\begin{enumerate}[label=(\roman*), leftmargin=*, nosep]
\item deep identity-normalized Sessa realizes the full exponent family
\[
\nu_k(\beta)=k(1-\beta)-1;
\]
\item diffuse fixed-depth Transformers and failed-freeze-time fixed-depth Mamba do not realize
frozen or increasing profiles.
\end{enumerate}
Thus, in this uniform finite-horizon family-over-\(H\) regime, deep Sessa supports flexible selective retrieval,
whereas the two comparison classes do not.
\end{corollary}

\subsection{Internal positional encoding}
\label{sec:theory_ape}
Sessa does not require an explicit absolute positional embedding in the feedback branch.
The key point is that the feedback solve can itself write a separated absolute positional signal.
The main lemma gives this positional writer, and the corollaries record the two refinements used later:
one-directional writing with signal transparency, and continuous recovery of the position index.

\begin{lemma}[Feedback generates ordered separated positional codes]
\label{lem:pos_code_main}
Fix $\Tctx\ge 2$ and model width $m\ge 1$.
There exists a single width-$m$ Sessa block $G^{(1)}$ and vectors
$p_0,\dots,p_{\Tctx-1}\in\mathbb R^m$ such that
for all token sequences $h\in\mathbb R^{\Tctx\times m}$,
\[
G^{(1)}(h)_t = h_t + p_t,\qquad t=0,\dots,\Tctx-1.
\]
Moreover, for any compact $\Kset\subset\mathbb R^{\Tctx\times m}$ the offsets can be chosen
so that there exist a unit direction $u\in\mathbb R^m$ and pairwise disjoint compact intervals
\[
J_0<J_1<\cdots<J_{\Tctx-1}\subset(0,\infty)
\]
with
\[
\langle h_t+p_t,u\rangle \in J_t
\qquad
\text{for all }h\in\Kset,\ t=0,\dots,\Tctx-1.
\]
\end{lemma}

\begin{proof}[Proof sketch]
Choose parameters so that the mixer input is constant, the forward branch produces a constant forward
signal, and the feedback routing is chosen so that the induced scalar solve generates a deterministic
strictly increasing sequence on the finite prefix. Project that scalar sequence onto a chosen direction,
then shift and rescale it so that the resulting compact scalar ranges are pairwise disjoint, strictly
ordered, and contained in $(0,\infty)$. See Appendix~\ref{sec:uat_pos_code}.
\end{proof}
\begin{corollary}[One-directional internal positional writer]
\label{cor:pos_code_one_direction}
Under the hypotheses of Lemma~\ref{lem:pos_code_main}, the block can be chosen so that there exists
a unit direction \(e_{\mathrm{pos}}\in\mathbb R^m\) and scalars \(\lambda_0,\dots,\lambda_{\Tctx-1}\) with
\[
G^{(1)}(h)_t = h_t + \lambda_t e_{\mathrm{pos}},
\qquad t=0,\dots,\Tctx-1,
\]
for all token sequences \(h\in\mathbb R^{\Tctx\times m}\).
Moreover, for any compact \(\Kset\subset\mathbb R^{\Tctx\times m}\), the same block can be chosen so that there exist
pairwise disjoint compact intervals
\[
J_0<J_1<\cdots<J_{\Tctx-1}\subset(0,\infty)
\]
with
\[
\bigl\langle G^{(1)}(h)_t,e_{\mathrm{pos}}\bigr\rangle\in J_t
\qquad
\text{for all }h\in\Kset,\ t=0,\dots,\Tctx-1.
\]
\end{corollary}

\begin{proof}
In the construction underlying Lemma~\ref{lem:pos_code_main}, the deterministic scalar sequence generated by the
feedback solve is written onto a chosen output direction. Choosing that output direction to be \(e_{\mathrm{pos}}\)
and writing no offset on the orthogonal complement yields the form
\[
G^{(1)}(h)_t=h_t+\lambda_t e_{\mathrm{pos}}.
\]
The interval-separation conclusion is exactly the same as in Lemma~\ref{lem:pos_code_main}.
\end{proof}

\begin{corollary}[Signal transparency of the one-directional positional writer]
\label{cor:pos_code_signal_transparency}
Under the hypotheses of Corollary~\ref{cor:pos_code_one_direction}, let \(e_{\mathrm{sig}}\in\mathbb R^m\) be any unit vector with
\[
e_{\mathrm{sig}}\perp e_{\mathrm{pos}}.
\]
Then for every token sequence \(h\in\mathbb R^{\Tctx\times m}\), every source index \(\tau\in\{0,\dots,\Tctx-1\}\), and every scalar \(a\in\mathbb R\),
\[
G^{(1)}\bigl(h+a\,e_{\mathrm{sig}}\mathbf 1[\cdot=\tau]\bigr)_t
=
G^{(1)}(h)_t+a\,e_{\mathrm{sig}}\mathbf 1[t=\tau],
\qquad
t=0,\dots,\Tctx-1.
\]
In particular,
\[
\bigl\langle G^{(1)}(h+a\,e_{\mathrm{sig}}\mathbf 1[\cdot=\tau])_t,e_{\mathrm{pos}}\bigr\rangle
=
\bigl\langle G^{(1)}(h)_t,e_{\mathrm{pos}}\bigr\rangle
\qquad
\forall\,t,
\]
so perturbations along \(e_{\mathrm{sig}}\) leave the internally written positional coordinate unchanged.
\end{corollary}

\begin{proof}
By Corollary~\ref{cor:pos_code_one_direction},
\[
G^{(1)}(h)_t=h_t+\lambda_t e_{\mathrm{pos}}.
\]
Therefore
\[
G^{(1)}\bigl(h+a\,e_{\mathrm{sig}}\mathbf 1[\cdot=\tau]\bigr)_t
=
h_t+a\,e_{\mathrm{sig}}\mathbf 1[t=\tau]+\lambda_t e_{\mathrm{pos}}
=
G^{(1)}(h)_t+a\,e_{\mathrm{sig}}\mathbf 1[t=\tau].
\]
Since \(e_{\mathrm{sig}}\perp e_{\mathrm{pos}}\), taking the \(e_{\mathrm{pos}}\)-coordinate gives the second claim.
\end{proof}

\begin{corollary}[Continuous recovery of the position index]
\label{cor:pos_code_recovery}
Under the hypotheses of Corollary~\ref{cor:pos_code_one_direction}, fix a compact set
\[
\Kset\subset\mathbb R^{\Tctx\times m},
\]
and choose the block so that there exist pairwise disjoint compact intervals
\[
J_0<J_1<\cdots<J_{\Tctx-1}\subset(0,\infty)
\]
with
\[
\bigl\langle G^{(1)}(h)_t,e_{\mathrm{pos}}\bigr\rangle\in J_t
\qquad
\forall\,h\in\Kset,\ \forall\,t=0,\dots,\Tctx-1.
\]
Then there exists a continuous map
\[
\psi:\mathbb R^m\to\mathbb R
\]
such that
\[
\psi\bigl(G^{(1)}(h)_t\bigr)=t
\qquad
\forall\,h\in\Kset,\ \forall\,t=0,\dots,\Tctx-1.
\]
In particular, the position index \(t\) is recoverable by a continuous tokenwise map on the shifted-token set
\[
\bigcup_{t=0}^{\Tctx-1}\{G^{(1)}(h)_t:\ h\in\Kset\}.
\]
\end{corollary}

\begin{proof}
Write each compact interval as
\[
J_t=[a_t,b_t].
\]
Since the intervals are pairwise disjoint and ordered, one has
\[
b_t<a_{t+1}
\qquad
(t=0,\dots,\Tctx-2).
\]
Define a continuous function \(g:\mathbb R\to\mathbb R\) by requiring
\[
g(s)=t
\qquad
\text{for all } s\in J_t,
\]
interpolating linearly on each gap \([b_t,a_{t+1}]\), and extending constantly on
\((-\infty,a_0]\) and \([b_{\Tctx-1},\infty)\).
Then \(g\) is continuous on \(\mathbb R\) and satisfies \(g|_{J_t}\equiv t\) for every \(t\).

Now define
\[
\psi(z):=g\bigl(\langle z,e_{\mathrm{pos}}\rangle\bigr),
\qquad z\in\mathbb R^m.
\]
Since \(z\mapsto \langle z,e_{\mathrm{pos}}\rangle\) is continuous, \(\psi\) is continuous. Moreover, for every
\(h\in\Kset\) and every \(t\),
\[
\bigl\langle G^{(1)}(h)_t,e_{\mathrm{pos}}\bigr\rangle\in J_t,
\]
hence
\[
\psi\bigl(G^{(1)}(h)_t\bigr)
=
g\bigl(\langle G^{(1)}(h)_t,e_{\mathrm{pos}}\rangle\bigr)
=
t.
\]
\end{proof}

\paragraph*{Consequence}
Sessa can internally generate an absolute positional code through feedback, even when the forward branch uses only relative-position-aware routing such as RoPE.

\subsection{Universal approximation of causal maps}
\label{sec:theory_uat}

We state a universal approximation result for Sessa networks on compact domains, in the standard causal decoder setting.
Since intermediate constructions may require an internal width $m\ge D$, we state the result for Sessa with tokenwise
linear adapters $D\to m\to D$.

\begin{definition}[Causality]
A map $F:\mathcal D\to\mathbb R^{\Tctx\times D}$ is causal if for every $t$ and all $x,x'\in\mathcal D$,
$x_{0:t}=x'_{0:t}$ implies $F(x)_t=F(x')_t$.
\end{definition}

\begin{theorem}[Universal approximation by concrete Sessa with adapters]
\label{thm:sessa_uap_main}
Let $\mathcal D\subset\mathbb R^{\Tctx\times D}$ be compact and let $F:\mathcal D\to\mathbb R^{\Tctx\times D}$ be continuous and causal.
Then for any $\varepsilon>0$ there exist an even query/key width \(d_k\ge 2\), a model width \(m\ge D\), tokenwise adapters
\[
\mathrm{Embed}:\mathbb R^D\to\mathbb R^{m},
\qquad
\mathrm{Unembed}:\mathbb R^{m}\to\mathbb R^{D},
\]
and a finite-depth width-\(m\) concrete Sessa network \(G\) such that
\[
\sup_{x\in\mathcal D}\Big\|F(x)-\mathrm{Unembed}\big(G(\mathrm{Embed}(x))\big)\Big\|_F < \varepsilon.
\]
\end{theorem}

\begin{proof}[Proof sketch]
\begin{enumerate}[label=(\roman*), leftmargin=*, nosep]
\item Use a single Sessa block to write an internal positional code.
\item Use a finite stack of concrete Sessa blocks to encode each relevant causal prefix into dedicated internal coordinates.
\item Apply a finite tokenwise readout stack, again implemented by concrete Sessa blocks, to approximate the desired causal output on the resulting compact encoded-state set.
\end{enumerate}
Details appear in Appendix~\ref{app:sessa_uat}, in the proof of Theorem~\ref{thm:sessa_uap_main}.
\end{proof}
\section{Experiments}
\label{sec:experiments}

We compare three model variants that share the same decoder macro-architecture and training setup and differ only in the sequence mixer. The mixers are Sessa mixer, multi-head self-attention, and Mamba2 mixer. 
We match parameter count, use the same optimizer and training schedule, and train all models for the same number of optimization steps.

We do not report aggregate results on the full Long Range Arena (LRA) suite \citep{tay2021longrangearena}. 
Although LRA was originally proposed as a testbed for long-range dependencies, subsequent analyses have highlighted several issues suggesting that strong performance on LRA can be confounded by factors unrelated to robust long-context reasoning. \citep{tay2021longrangearena, miralles2025localitybiaslra} We evaluate long-context behavior on SymbolSoup and Diffuse MQAR, and short-context language modeling on SimpleStories. \citep{finke2025simplestories,simplestories2025dataset}

\subsection{Synthetic long-range tasks}

\subsubsection{Datasets and tasks}

\paragraph{SymbolSoup.}
SymbolSoup is a long-range classification dataset with two informative stylized blocks separated by label-independent noise. Each example contains three noise blocks and two stylized blocks, one from each style family. The order of the two stylized blocks is randomized.
\begin{quote}
noise \ \texttt{<sep1>} \ first/second stylized part \ \texttt{<sep2>} \ noise \ \texttt{<sep1>} \ second/first stylized part \ \texttt{<sep2>} \ noise \ \texttt{<sep>} \ \texttt{<label>}.
\end{quote}
The label is the pair of styles used in the two stylized blocks. 
Stylized blocks are generated by a Markov-like process with unigram and bigram preferences and occasional motif insertion plus small symbol noise.

\paragraph{Diffuse MQAR.}
We additionally evaluate on a modified multi-query associative recall benchmark based on MQAR \citep{arora2024zoology}. 
Relative to the original formulation, our variant uses multi-token keys, structured distractors with shared prefixes and mismatched suffixes, and explicit control of the source--query lag. 
Each example contains a prefix memory block of key--value pairs, a noise block populated with distractor key--value-like patterns, and a terminal query block. The test split includes retrieval lags up to \(4\times\) larger than those seen during training.

\begin{center}
\small
\setlength{\tabcolsep}{6pt}
{%
\captionsetup{width=0.9\linewidth}%
\captionof{table}{Long-context test results (mean \(\pm\) std over 2 seeds). For SymbolSoup we report classification accuracy; for Diffuse MQAR we report token accuracy.}%
\label{tab:long_context}%
}
\begin{tabular}{lcc}
\hline
Model & SymbolSoup Acc \(\uparrow\) & Diffuse MQAR Token Acc \(\uparrow\) \\
\hline
Sessa       & \(0.8601 \pm 0.0016\) & \(0.1541 \pm 0.0071\) \\
Transformer & \(0.7921 \pm 0.0070\) & \(0.1222 \pm 0.0003\) \\
Mamba2      & \(0.0500 \pm 0.0000\) & \(0.0021 \pm 0.0000\) \\
\hline
\end{tabular}
\end{center}
\FloatBarrier

Mamba-2 did not converge on SymbolSoup or Diffuse MQAR. 
We view this as qualitatively consistent with our selective-SSM theory: when noise makes the selection signal weakly separable, the resulting non-vanishing freeze-time errors restore exponential attenuation of long-range influence, as formalized in Proposition~\ref{prop:freeze_errors_exp} and Corollary~\ref{cor:mamba_e2e_freeze}. 
This interpretation is relevant to Mamba-2 because it is itself a selective SSM, specifically a scalar-identity restricted variant in the SSD framework \citep{dao2024transformersssms}.

\subsection{SimpleStories language modeling}

\subsubsection{Dataset and task}

For the short-context regime we use a SimpleStories corpus of short, synthetically generated stories. Each story is written in simplified English with a small vocabulary and constrained syntax.

We treat this corpus as a causal language modeling benchmark. The text is tokenized with a subword tokenizer shared across all architectures, and training sequences are formed by concatenating stories and splitting them into fixed-length segments. 
The model predicts the next token at each position using a left-to-right mask. We report validation perplexity.

\begin{table}[H]
\centering
\caption{SimpleStories test results (mean \(\pm\) std over 2 seeds).}
\label{tab:simplestories}
\begin{tabular}{lrrr}
\hline
Model & Perplexity \(\downarrow\) & Top-1 acc \(\uparrow\) & Top-5 acc \(\uparrow\) \\
\hline
Transformer & \(7.6701 \pm 0.0313\) & \(50.441 \pm 0.059\)\% & \(78.497 \pm 0.062\)\% \\
Mamba2      & \(7.7229 \pm 0.0207\) & \(50.299 \pm 0.046\)\% & \(78.302 \pm 0.043\)\% \\
Sessa       & \(8.3700 \pm 0.0482\) & \(49.144 \pm 0.081\)\% & \(77.119 \pm 0.090\)\% \\
\hline
\end{tabular}
\end{table}
\FloatBarrier

We hypothesize that the small performance drop of Sessa in the short-context regime is due to the feedback mechanism being less necessary for this task. 
Under matched parameter count, a portion of Sessa's capacity is allocated to the feedback branch, which may be weakly utilized on short-context. 
To test this interpretation, we ran a control experiment with the feedback branch removed while keeping the remainder of the architecture unchanged. 
The ablated model improves over full Sessa on SimpleStories, reducing test perplexity from \(8.3700 \pm 0.0482\) to \(8.0902 \pm 0.0192\) and increasing top-1 accuracy from \(49.144 \pm 0.081\)\% to \(49.648 \pm 0.026\)\%. 
This supports the view that feedback is less beneficial in the short-context regime, while remaining consistent with Sessa's stronger results on long-context tasks, where feedback appears to be more useful.
\section{Discussion}
\label{sec:discussion}

The main comparison in this paper is not between favorable operating regimes of Transformers, Mamba, and Sessa, but between matched regimes in which sharp retrieval is unavailable. For attention, this appears as diffuse, low-separation routing, so the selector cannot concentrate mass on a small set of relevant indices. 
For Mamba, the analogous failure is failed freeze time, so the model cannot maintain a long preserve corridor on the relevant interval. These are natural failure regimes for the respective architectures, and they provide a common basis for comparison.

In this matched setting, the difference comes from the memory mechanism rather than from access to sharp routing. Diffuse attention remains one-hop and therefore suffers dilution. 
Failed-freeze-time Mamba remains chain-structured and therefore exhibits exponential attenuation. Sessa is also studied in a diffuse regime, but its feedback solve aggregates influence over multiple hop counts and, in dense settings, over many temporal paths. 
This is the structural source of its slower long-range decay.

The main separation is not only in the polynomial tail, but in the selective-retrieval result. 
In the same family-over-\(H\) regime, deep Sessa realizes flexible selective retrieval profiles, whereas diffuse fixed-depth Transformers and failed-freeze-time fixed-depth Mamba do not realize frozen or increasing profiles. 
Thus the separation is not merely quantitative at the level of decay rates; it is qualitative at the level of what retrieval behavior the architectures can realize under the same matched breakdown of sharp retrieval.

The broader point is that long-context behavior depends not only on how routing coefficients are produced, but also on how they are composed over time. 
When sharp retrieval fails, as can become increasingly likely as context length grows, this distinction becomes decisive. In that regime, Sessa can still support flexible selective retrieval through its multi-hop feedback structure.

\FloatBarrier
\clearpage
\bibliographystyle{plainnat}
\bibliography{refs}

\clearpage
\appendix
\appendix
\section*{Appendix}
\addcontentsline{toc}{section}{Appendix}

\section{Definitions and notation}
\label{app:shared_defs}

\subsection{Sequence norms and bounded-input sets}
\label{app:shared_defs:norms}

\begin{definition}[Sup--$\ell_2$ norm and bounded-input balls]
\label{def:linf2_ball}
Fix a horizon $\Tctx\in\mathbb N^*$ and token width $D\in\mathbb N^*$.
For a finite sequence $x=(x_0,\dots,x_{\Tctx-1})\in(\mathbb R^D)^{\Tctx}$ define
\[
\|x\|_{\infty,2}:=\max_{0\le t\le \Tctx-1}\|x_t\|_2.
\]
For $R\ge 0$ define the ball
\[
\mathcal X_R:=\{x\in(\mathbb R^D)^{\Tctx}:\ \|x\|_{\infty,2}\le R\}.
\]
For infinite sequences $(x_t)_{t\ge 0}$ we use the analogous norm
$\|x\|_{\infty,2}:=\sup_{t\ge 0}\|x_t\|_2\in[0,\infty]$.
\end{definition}

\begin{equation}\label{eq:norm_compare}
\|X\|_{\infty,2}\le \|X\|_F \le \sqrt{\Tctx}\,\|X\|_{\infty,2}
\qquad\text{for }X\in\mathbb R^{\Tctx\times D}.
\end{equation}

\subsection{BIBO stability on \texorpdfstring{$\ell_\infty$}{ell_infty}}
\label{app:shared_defs:bibo}

\begin{definition}[BIBO stability on $\ell_\infty$]
\label{def:bibo_linf}
A map $\mathcal N:\ell_\infty(\mathbb N,\mathbb R^D)\to \ell_\infty(\mathbb N,\mathbb R^D)$
is BIBO-stable with respect to $\|\cdot\|_{\infty,2}$ if for every $B\ge 0$ there exists
$C_B<\infty$ such that
\[
\|x\|_{\infty,2}\le B\quad\Longrightarrow\quad \|\mathcal N(x)\|_{\infty,2}\le C_B.
\]
\end{definition}
\section{Jacobian tails under diffuse feedback routing}
\label{app:jacobian_tail}

\subsection{Sessa feedback solve as a parametric linear system}

Fix a horizon $\Tctx\in\mathbb N^*$ and token width $D\in\mathbb N^*$.
Let $x=(x_0,\dots,x_{\Tctx-1})\in(\mathbb R^D)^{\Tctx}$ be the input token sequence.
Let $f(x)=(f_0(x),\dots,f_{\Tctx-1}(x))\in(\mathbb R^r)^{\Tctx}$ be the forward sequence, where $r$ is the value space dimension,
and let $\alphafb(x)=(\alphafb_{t,j}(x))_{0\le j<t\le \Tctx-1}$ be the strictly-lower attention weights.
Let $\gamma(x)=(\gamma_0(x),\dots,\gamma_{\Tctx-1}(x))$ be the feedback gains.

Define the strictly lower-triangular matrix $\Bfb(x)\in\mathbb R^{\Tctx\times \Tctx}$ by
\[
\BfbEnt{t}{j}(x)=
\begin{cases}
\gamma_t(x)\,\alphafb_{t,j}(x), & j<t,\\
0, & j\ge t.
\end{cases}
\]
The mixer output $s(x)=(s_0(x),\dots,s_{\Tctx-1}(x))\in(\mathbb R^r)^{\Tctx}$ is defined as the unique solution to the causal solve
\begin{equation}\label{eq:app_solve}
(I-\Bfb(x))\,s(x)=f(x).
\end{equation}
Equivalently, by forward substitution,
\begin{equation}\label{eq:app_recur}
s_0=f_0,\qquad
s_t=f_t+\gamma_t\sum_{j=0}^{t-1}\alphafb_{t,j}s_j,\quad t\ge 1.
\end{equation}

We measure long-range sensitivity by the Jacobian blocks
\[
J_{t,\tau}(x):=\frac{\partial s_t(x)}{\partial x_\tau}\in\mathbb R^{r\times D},
\qquad 0\le \tau\le t\le \Tctx-1.
\]
Throughout this appendix we focus on the long-range case $\tau<t$ and lag $\ell:=t-\tau\ge 1$.

\subsection{Assumptions for diffuse routing and smoothness}

Fix a radius $R\ge 0$ and work on the ball $\mathcal X_R$ from Definition~\ref{def:linf2_ball}.
\begin{remark}[On the use of $t+1$ and $t$ in dilution bounds]\label{rem:tplus1_appendix}
In this appendix the feedback attention is strictly-lower, meaning that $j<t$, so $|\VisSet{t}|=t$ for $t\ge 1$.
We write $O(1/(t+1))$ to avoid a special case at $t=0$ and to match harmonic-series bounds;
for $t\ge 1$ this is equivalent to $O(1/t)$ up to absolute constants.
\end{remark}
\begin{assumption}[Row-stochasticity and diffuse envelope of feedback attention]\label{ass:jac_alpha}
For every $x\in\mathcal X_R$ and every $t\ge 1$,
\[
\alphafb_{t,j}(x)\ge 0,\qquad \sum_{j=0}^{t-1}\alphafb_{t,j}(x)=1,
\qquad \alphafb_{t,j}(x)\le \frac{c_2}{t}\quad\forall j<t,
\]
for some constant $c_2=c_2(R)\in(0,\infty)$.
We set $\alphafb_{0,\cdot}\equiv 0$.
\end{assumption}

\begin{assumption}[Bounded feedback gain and nontrivial diffuse regime]\label{ass:jac_gamma}
For every $x\in\mathcal X_R$ and every $t$,
\[
|\gamma_t(x)|\le \gamma_{\max}<1,
\]
and the diffuse feedback mass satisfies
\[
\eta:=\gamma_{\max}c_2<1,
\qquad \betatail:=1-\eta\in(0,1).
\]
\end{assumption}

\begin{assumption}[Token-wise local feedback gain]\label{ass:jac_gamma_local}
On $\mathcal X_R$, the feedback gain is token-wise: for each $t$ one has $\gamma_t(x)=\gamma(x_t)$.
In particular, for $\tau<t$,
\[
\frac{\partial \gamma_t(x)}{\partial x_\tau}=0.
\]
Assume additionally the token-wise Jacobian is bounded:
\[
\Big\|\frac{\partial \gamma(x_t)}{\partial x_t}\Big\|_2\le L_\gamma
\qquad\text{for all } \|x_t\|_2\le R.
\]
\end{assumption}

\begin{assumption}[Causality of forward branch and routing]\label{ass:jac_causal}
For each time $k$, the quantities $f_k(x)$, $\alphafb_{k,\cdot}(x)$, and $\gamma_k(x)$ depend only on the prefix $x_{0:k}$.
Equivalently, for any $\tau>k$,
\[
\frac{\partial f_k(x)}{\partial x_\tau}=0,\qquad
\frac{\partial \alphafb_{k,j}(x)}{\partial x_\tau}=0\ \ (\forall j<k),\qquad
\frac{\partial \gamma_k(x)}{\partial x_\tau}=0.
\]
\end{assumption}

\begin{assumption}[Local, same-token smoothness bounds]\label{ass:jac_local_bounds}
There exist finite constants $L_{f,0}=L_{f,0}(R)$ and $L_{\alpha,0}=L_{\alpha,0}(R)$ such that for all $x\in\mathcal X_R$ and all $t$,
\[
\Big\|\frac{\partial f_t(x)}{\partial x_t}\Big\|_2 \le L_{f,0},
\qquad
\sum_{j=0}^{t-1}\Big\|\frac{\partial \alphafb_{t,j}(x)}{\partial x_t}\Big\|_2 \le L_{\alpha,0}.
\]
\end{assumption}

\begin{assumption}[Bounded forward sequence]\label{ass:jac_f_bound}
There exists $F_R<\infty$ such that
\[
\|f(x)\|_{\infty,2}\le F_R\qquad\forall x\in\mathcal X_R.
\]
\end{assumption}

\begin{assumption}[Forward-branch dilution of cross-token Jacobians]\label{ass:jac_f_jac}
There exists $L_f=L_f(R)<\infty$ such that for all $x\in\mathcal X_R$,
all $t\ge \tau$, and all $\tau<t$,
\[
\Big\|\frac{\partial f_t(x)}{\partial x_\tau}\Big\|_2 \le \frac{L_f}{t+1}.
\]
Here $\|\cdot\|_2$ is the operator norm of the matrix $\mathbb R^D\to\mathbb R^r$.
\end{assumption}

\begin{assumption}[Smooth routing: $\alpha$-weighted logit sensitivity]\label{ass:jac_alpha_der}
Let $\alphafb_{t,\cdot}(x)=\mathrm{softmax}(\beth_{t,0}(x),\dots,\beth_{t,t-1}(x))$ denote the feedback-attention row
at time $t$, over $j<t$, with pre-softmax logits $\beth_{t,\idxlogit}(x)$ that may depend on the full prefix $x_{0:t}$.
There exists $L_{\mathrm{route}}=L_{\mathrm{route}}(R)<\infty$ such that for all $x\in\mathcal X_R$ and all $t>\tau\ge 0$,
\[
\sum_{\idxlogit=0}^{t-1}\alphafb_{t,\idxlogit}(x)\,
\Big\|\frac{\partial \beth_{t,\idxlogit}(x)}{\partial x_\tau}\Big\|_2
\;\le\;
\frac{L_{\mathrm{route}}}{t+1}.
\]
Consequently, by Lemma~\ref{lem:softmax_key_der},
\[
\sum_{j=0}^{t-1}\Big\|\frac{\partial \alphafb_{t,j}(x)}{\partial x_\tau}\Big\|_2
\;\le\;
\frac{2L_{\mathrm{route}}}{t+1}.
\]
\end{assumption}

\begin{remark}[When Assumption~\ref{ass:jac_alpha_der} holds]
If the feedback query is token-wise, $q_t=q(x_t)$, then for $\tau<t$ the dependence of $\alphafb_{t,\cdot}$ on $x_\tau$
typically enters only through key-side logits involving $k_\tau$, so only a small subset of logits have nonzero
$\partial \beth_{t,\idxlogit}/\partial x_\tau$. In that case, Assumption~\ref{ass:jac_alpha_der} reduces to the corresponding
localized logit-sensitivity bound.
More generally, if $q_t$, or other components upstream of logits, has cross-token sensitivity, Assumption~\ref{ass:jac_alpha_der}
requires that the resulting $\alphafb$-weighted logit sensitivities still dilute as $O(1/(t+1))$ on $\mathcal X_R$.
\end{remark}

\subsection{Auxiliary lemmas}

\begin{lemma}[Bound on the mixer state]\label{lem:jac_s_bound}
Under Assumption~\ref{ass:jac_gamma}--\ref{ass:jac_f_bound}, for all $x\in\mathcal X_R$,
\[
\|s(x)\|_{\infty,2}\le S_R:=\frac{F_R}{1-\gamma_{\max}}.
\]
\end{lemma}

\begin{proof}
Since each $\alphafb_{t,\cdot}$ is a convex distribution and $|\gamma_t|\le\gamma_{\max}$,
\[
\|s_t\|_2\le \|f_t\|_2+\gamma_{\max}\max_{j<t}\|s_j\|_2.
\]
A standard induction on $\max_{k\le t}\|s_k\|_2$ yields
$\|s\|_{\infty,2}\le (1-\gamma_{\max})^{-1}\|f\|_{\infty,2}\le (1-\gamma_{\max})^{-1}F_R$.
\end{proof}
\begin{lemma}[Softmax row derivative: total variation bound]\label{lem:softmax_key_der}
Let $\alpha=\mathrm{softmax}(\beth)\in\mathbb R^{n}$ with logits $\beth\in\mathbb R^n$ depending on a parameter $z$.
Then
\[
\sum_{j}\Big\|\frac{\partial \alpha_j}{\partial z}\Big\|
\;\le\;
2\sum_{\idxlogit}\alpha_{\idxlogit}\,\Big\|\frac{\partial \beth_{\idxlogit}}{\partial z}\Big\|_2.
\]
\end{lemma}

\begin{proof}
The softmax Jacobian satisfies $\partial\alpha_j/\partial \beth_{\idxlogit}=\alpha_j(\mathbf 1[j=\idxlogit]-\alpha_{\idxlogit})$.
Thus
\[
\sum_{j=1}^{n}\Big|\frac{\partial \alpha_j}{\partial \beth_{\idxlogit}}\Big|
=
2\alpha_{\idxlogit}(1-\alpha_{\idxlogit})\le 2\alpha_{\idxlogit}.
\]
By the chain rule,
$\sum_j\|\partial\alpha_j/\partial z\|\le \sum_{\idxlogit} (\sum_j|\partial\alpha_j/\partial \beth_{\idxlogit}|)\,\|\partial\beth_{\idxlogit}/\partial z\|
\le 2\sum_{\idxlogit} \alpha_{\idxlogit}\|\partial\beth_{\idxlogit}/\partial z\|$.
\end{proof}

\begin{lemma}[Polynomial tail of the inverse kernel entries]\label{lem:jac_kernel_tail}
Fix $x\in\mathcal X_R$ and let $K(x):=(I-\Bfb(x))^{-1}$.
Under Assumptions~\ref{ass:jac_alpha}--\ref{ass:jac_gamma}, there exists a constant
\[
C_K:=\eta e^{\eta}=(1-\betatail)e^{1-\betatail}
\]
such that for all $0\le k<t\le \Tctx-1$,
\[
|K_{t,k}(x)|\le C_K\,(t-k)^{-\betatail},
\qquad\text{and}\qquad
K_{t,t}(x)=1.
\]
\end{lemma}

\begin{proof}
Fix \(x\in\mathcal X_R\), and abbreviate
\[
\Bfb:=\Bfb(x),\qquad
\alpha_{t,j}:=\alphafb_{t,j}(x),\qquad
K:=K(x)=(I-\Bfb)^{-1}.
\]
Since \(\Bfb\) is strictly lower-triangular on the finite horizon
\(\{0,\dots,\Tctx-1\}\), one has \(\Bfb^{\Tctx}=0\), hence
\[
K=(I-\Bfb)^{-1}=\sum_{m=0}^{\Tctx-1}\Bfb^m.
\]
Therefore \(K\) is lower-triangular with unit diagonal:
\[
K_{t,t}=1,\qquad K_{t,k}=0\ \text{for }t<k.
\]
It remains to prove the off-diagonal estimate.

Fix a source index \(k\in\{0,\dots,\Tctx-1\}\), and define
\[
u_t:=|K_{t,k}|\qquad (t\ge k).
\]
Then \(u_k=|K_{k,k}|=1\). Also, since \((I-\Bfb)K=I\), equivalently
\(K=I+\Bfb K\), for every \(t>k\) we have
\[
K_{t,k}
=
\sum_{j<t}\BfbEnt{t}{j}\,K_{j,k}.
\]
Because \(K_{j,k}=0\) for \(j<k\), this reduces to
\[
K_{t,k}
=
\sum_{j=k}^{t-1}\BfbEnt{t}{j}\,K_{j,k}
=
\gamma_t(x)\sum_{j=k}^{t-1}\alpha_{t,j}\,K_{j,k}.
\]
Taking absolute values and using Assumption~\ref{ass:jac_gamma},
\[
u_t
\le
|\gamma_t(x)|\sum_{j=k}^{t-1}\alpha_{t,j}\,u_j
\le
\gamma_{\max}\sum_{j=k}^{t-1}\alpha_{t,j}\,u_j,
\qquad t>k.
\tag{1}\label{eq:kernel_tail_u_rec}
\]

We now compare \(u\) to an explicit impulse-response sequence.
Define \((v_t^{(k)})_{t\ge 0}\) by
\[
v_t^{(k)}:=
\begin{cases}
0, & t<k,\\[2mm]
1, & t=k,\\[2mm]
\gamma_{\max}\displaystyle\sum_{j=0}^{t-1}\widetilde\alpha_{t,j}\,v_j^{(k)}, & t>k,
\end{cases}
\]
where the coefficients \(\widetilde\alpha_{t,j}\) are the following extension of the
finite-horizon row weights:
\[
\widetilde\alpha_{t,j}:=
\begin{cases}
\alpha_{t,j}, & 0\le j<t\le \Tctx-1,\\[2mm]
0, & t\ge \Tctx,\ 0\le j<t.
\end{cases}
\]
Then \(\widetilde\alpha_{t,j}\ge 0\), \(\sum_{j<t}\widetilde\alpha_{t,j}\le 1\) for every
\(t\ge 1\), and by Assumption~\ref{ass:jac_alpha},
\[
\widetilde\alpha_{t,j}\le \frac{c_2}{t}\qquad (t\ge 1,\ 0\le j<t).
\]
Thus the scalar recursion defining \(v^{(k)}\) satisfies the hypotheses of
Corollary~\ref{cor:impulse_j_tail} with impulse position \(j=k\), attention envelope
constant \(c_2\), and feedback bound \(\gamma_{\max}\). In particular, with
\[
\eta:=\gamma_{\max}c_2,\qquad \betatail:=1-\eta\in(0,1),
\]
that corollary yields
\[
v_t^{(k)}\le \eta e^{\eta}\,(t-k)^{-\betatail}
\qquad\text{for all }t>k.
\tag{2}\label{eq:kernel_tail_v_bound}
\]

It remains to show that \(u_t\le v_t^{(k)}\) for all \(t\in\{k,\dots,\Tctx-1\}\).
We prove this by induction on \(t\).

For \(t=k\), one has \(u_k=1=v_k^{(k)}\).

Now let \(t>k\), and assume \(u_j\le v_j^{(k)}\) for every \(j\in\{k,\dots,t-1\}\).
Using \eqref{eq:kernel_tail_u_rec}, the nonnegativity of the coefficients \(\alpha_{t,j}\),
and the induction hypothesis, we obtain
\[
u_t
\le
\gamma_{\max}\sum_{j=k}^{t-1}\alpha_{t,j}\,u_j
\le
\gamma_{\max}\sum_{j=k}^{t-1}\alpha_{t,j}\,v_j^{(k)}.
\]
Since \(v_j^{(k)}=0\) for \(j<k\) and \(\widetilde\alpha_{t,j}=\alpha_{t,j}\) for
\(t\le \Tctx-1\), this is exactly
\[
u_t
\le
\gamma_{\max}\sum_{j=0}^{t-1}\widetilde\alpha_{t,j}\,v_j^{(k)}
=
v_t^{(k)}.
\]
This closes the induction.

Combining the comparison \(u_t\le v_t^{(k)}\) with
\eqref{eq:kernel_tail_v_bound}, we conclude that for every \(0\le k<t\le \Tctx-1\),
\[
|K_{t,k}(x)|=u_t\le v_t^{(k)}\le \eta e^{\eta}\,(t-k)^{-\betatail}.
\]
Thus the claim holds with
\[
C_K:=\eta e^{\eta}=(1-\betatail)e^{\,1-\betatail}.
\]
Together with \(K_{t,t}=1\), this proves the lemma.
\end{proof}

\begin{lemma}[A convolution bound]\label{lem:jac_convolution}
Let $\betatail\in(0,1)$. There exists $C_{\betatail}<\infty$ such that for all integers $\ell\ge 1$ and all $\tau\ge 0$,
\[
\sum_{k=\tau}^{\tau+\ell-1}\frac{1}{(\tau+\ell-k)^{\betatail}}\cdot \frac{1}{k+1}
\ \le\
C_{\betatail}\,\ell^{-\betatail}\big(1+\log(1+\ell)\big).
\]
One may take, for instance,
\[
C_{\betatail}:=2^{\betatail}\,+\,\frac{2^{\betatail}}{1-\betatail}.
\]
\end{lemma}

\begin{proof}
Write $k=\tau+m$ where $m=0,\dots,\ell-1$:
\[
\sum_{m=0}^{\ell-1}\frac{1}{(\ell-m)^{\betatail}}\cdot\frac{1}{\tau+m+1}.
\]
Split into $m\le \lfloor \ell/2\rfloor$ and $m> \lfloor \ell/2\rfloor$.

If $m\le \ell/2$, then $(\ell-m)^{-\betatail}\le (\ell/2)^{-\betatail}=2^{\betatail}\ell^{-\betatail}$ and
\[
\sum_{m=0}^{\lfloor \ell/2\rfloor}\frac{1}{\tau+m+1}
\le 1+\int_{0}^{\ell/2}\frac{dm}{\tau+m+1}
\le 1+\log(1+\ell).
\]
Thus this part is $\le 2^{\betatail}\ell^{-\betatail}(1+\log(1+\ell))$.

If $m>\ell/2$, then $\tau+m+1\ge \ell/2$, so $(\tau+m+1)^{-1}\le 2/\ell$, hence
\[
\sum_{m>\ell/2}\frac{1}{(\ell-m)^{\betatail}}\cdot\frac{1}{\tau+m+1}
\le \frac{2}{\ell}\sum_{r=1}^{\lfloor \ell/2\rfloor}\frac{1}{r^{\betatail}}
\le \frac{2}{\ell}\Big(1+\int_{1}^{\ell/2}r^{-\betatail}\,dr\Big)
\le \frac{2}{\ell}\cdot \frac{1}{1-\betatail}\left(\frac{\ell}{2}\right)^{1-\betatail}
= \frac{2^{\betatail}}{1-\betatail}\,\ell^{-\betatail}.
\]
Combine the two bounds.
\end{proof}

\subsection{Polynomial Jacobian tail}

\begin{theorem}[Polynomial Jacobian tail under diffuse routing]\label{thm:jacobian_tail}
Assume Assumptions~\ref{ass:jac_alpha}--\ref{ass:jac_alpha_der}, \ref{ass:jac_gamma_local},
\ref{ass:jac_causal}, and \ref{ass:jac_local_bounds} hold on $\mathcal X_R$, and let
$\betatail:=1-\gamma_{\max}c_2\in(0,1)$ as in Assumption~\ref{ass:jac_gamma}.
Then there exists a constant $C(R)<\infty$ such that for every $x\in\mathcal X_R$ and every pair $\tau<t$ with lag $\ell=t-\tau\ge 1$,
\[
\Big\|\frac{\partial s_t(x)}{\partial x_\tau}\Big\|_2
\ \le\
C(R)\,\ell^{-\betatail}\big(1+\log(1+\ell)\big).
\]
In particular, long-range sensitivity decays at least polynomially in the lag, up to a logarithmic factor.

One may take explicitly
\[
C(R):=\widetilde C_K\Big(A_0(R)+(1+C_{\betatail})\,A_1(R)\Big),
\quad
\widetilde C_K:=\max\{1,C_K\},
\quad
C_K=\eta e^{\eta},
\quad
\eta=\gamma_{\max}c_2,
\]
where $C_{\betatail}$ is as in Lemma~\ref{lem:jac_convolution} and
\[
A_1(R):=L_f + 2\gamma_{\max}\,S_R\,L_{\mathrm{route}},
\qquad
A_0(R):=L_{f,0} + L_\gamma\,S_R + \gamma_{\max}\,S_R\,L_{\alpha,0},
\qquad S_R=\frac{F_R}{1-\gamma_{\max}}.
\]
\end{theorem}

\begin{proof}
Fix $x\in\mathcal X_R$ and a source index $\tau$.
Differentiate the solve \eqref{eq:app_solve} with respect to $x_\tau$:
\[
(I-\Bfb)\,\frac{\partial s}{\partial x_\tau}
-\frac{\partial \Bfb}{\partial x_\tau}\,s
=
\frac{\partial f}{\partial x_\tau}.
\]
Multiplying by $K=(I-\Bfb)^{-1}$ gives
\begin{equation}\label{eq:app_J_rep}
\frac{\partial s}{\partial x_\tau}
=
K\Big(\frac{\partial f}{\partial x_\tau}+\frac{\partial \Bfb}{\partial x_\tau}\,s\Big).
\end{equation}
Taking the $t$-th row and operator norms yields
\begin{equation}\label{eq:app_J_row}
\Big\|\frac{\partial s_t}{\partial x_\tau}\Big\|_2
\le
\sum_{k=0}^{t}|K_{t,k}|\cdot
\Big\|\frac{\partial f_k}{\partial x_\tau}+\big(\frac{\partial \Bfb}{\partial x_\tau}s\big)_k\Big\|_2.
\end{equation}
By Assumption~\ref{ass:jac_causal}, if $k<\tau$ then $\partial f_k/\partial x_\tau=0$ and
$\partial \Bfb_{k,\cdot}/\partial x_\tau=0$, hence the sum starts at $k=\tau$.

\paragraph{Bounding the forcing term.}
We treat the single index $k=\tau$ separately from the range $k>\tau$.

\emph{Case 1: $k>\tau$.}
For $k>\tau$, Assumption~\ref{ass:jac_f_jac} gives
\[
\Big\|\frac{\partial f_k}{\partial x_\tau}\Big\|_2\le \frac{L_f}{k+1}.
\]
It remains to bound $\|(\partial \Bfb/\partial x_\tau)s\|$.
For $k>\tau$ we use the full decomposition
\[
\frac{\partial \BfbEnt{k}{j}}{\partial x_\tau}
=
\frac{\partial \gamma_k}{\partial x_\tau}\,\alphafb_{k,j}
+\gamma_k\,\frac{\partial\alphafb_{k,j}}{\partial x_\tau}.
\]
By Assumption~\ref{ass:jac_gamma_local}, $\partial\gamma_k/\partial x_\tau=0$ for $k>\tau$, so only the second term remains.
Therefore, using Lemma~\ref{lem:jac_s_bound} and Assumption~\ref{ass:jac_alpha_der},
\[
\Big\|\big(\frac{\partial \Bfb}{\partial x_\tau}s\big)_k\Big\|_2
\le
|\gamma_k|
\sum_{j<k}\Big\|\frac{\partial\alphafb_{k,j}}{\partial x_\tau}\Big\|_2\cdot \|s_j\|_2
\le
\gamma_{\max}\,S_R\sum_{j<k}\Big\|\frac{\partial\alphafb_{k,j}}{\partial x_\tau}\Big\|_2
\le
\gamma_{\max}\,S_R\cdot \frac{2L_{\mathrm{route}}}{k+1}.
\]
Thus for all $k>\tau$,
\[
\Big\|\frac{\partial f_k}{\partial x_\tau}+\big(\frac{\partial \Bfb}{\partial x_\tau}s\big)_k\Big\|_2
\le \frac{A_1(R)}{k+1},
\qquad
A_1(R):=L_f + 2\gamma_{\max}\,S_R\,L_{\mathrm{route}}.
\]

\emph{Case 2: $k=\tau$.}
Using Assumption~\ref{ass:jac_local_bounds} and Lemma~\ref{lem:jac_s_bound}, we bound
\[
\Big\|\frac{\partial f_\tau}{\partial x_\tau}\Big\|_2 \le L_{f,0}.
\]
Moreover, since $\BfbEnt{\tau}{j}=\gamma_\tau\,\alphafb_{\tau,j}$ for $j<\tau$,
\[
\Big\|\big(\tfrac{\partial \Bfb}{\partial x_\tau}s\big)_\tau\Big\|_2
\le
\Big\|\tfrac{\partial\gamma_\tau}{\partial x_\tau}\Big\|_2\cdot
\sum_{j<\tau}\alphafb_{\tau,j}\|s_j\|_2
\;+\;
|\gamma_\tau|\sum_{j<\tau}\Big\|\tfrac{\partial\alphafb_{\tau,j}}{\partial x_\tau}\Big\|_2\cdot \|s_j\|_2
\le
L_\gamma\,S_R+\gamma_{\max}\,L_{\alpha,0}\,S_R.
\]
Hence
\[
\Big\|\frac{\partial f_\tau}{\partial x_\tau}+\big(\frac{\partial \Bfb}{\partial x_\tau}s\big)_\tau\Big\|_2
\le A_0(R),
\qquad
A_0(R):=L_{f,0} + L_\gamma\,S_R + \gamma_{\max}\,S_R\,L_{\alpha,0}.
\]
\paragraph{Kernel tail and convolution.}
Plugging the forcing bound into \eqref{eq:app_J_row} and using Lemma~\ref{lem:jac_kernel_tail} yields
\[
\Big\|\frac{\partial s_t}{\partial x_\tau}\Big\|_2
\le
\ |K_{t,\tau}|\,A_0(R)\ +\ \sum_{k=\tau+1}^{t}|K_{t,k}|\cdot \frac{A_1(R)}{k+1}
\le
|K_{t,\tau}|\,A_0(R)\ +\ A_1(R)\Big(\frac{1}{t+1}+\sum_{k=\tau+1}^{t-1} C_K\,(t-k)^{-\betatail}\cdot \frac{1}{k+1}\Big).
\]
Let $\ell=t-\tau\ge 1$. 

We keep the $k=t$ term explicit and show it can be absorbed into the final tail factor:
\[
\frac{1}{t+1}\le \frac{1}{\tau+\ell+1}\le \frac{1}{\ell+1}\le \ell^{-1}.
\]
Since $\betatail\in(0,1)$ and $\ell\ge 1$, we have $\ell^{1-\betatail}\ge 1$, hence
\[
\ell^{-\betatail}=\ell^{1-\betatail}\,\ell^{-1}\ \ge\ \ell^{-1}.
\]
Therefore,
\begin{equation}\label{eq:app_diag_absorb}
\frac{1}{t+1}\ \le\ \ell^{-1}\ \le\ \ell^{-\betatail}\ \le\ \ell^{-\betatail}\big(1+\log(1+\ell)\big),
\end{equation}
so the $k=t$ contribution $\frac{A_1(R)}{t+1}$ is dominated by the same
$\ell^{-\betatail}(1+\log(1+\ell))$ envelope, with constant $1$.

For the isolated term, Lemma~\ref{lem:jac_kernel_tail} gives $|K_{t,\tau}|\le C_K\,\ell^{-\betatail}$.
For the remaining sum, apply Lemma~\ref{lem:jac_convolution} 
Note that $\sum_{k=\tau+1}^{t-1}\le\sum_{k=\tau}^{t-1}$:
\[
\sum_{k=\tau}^{t-1}(t-k)^{-\betatail}\cdot \frac{1}{k+1}
\le
C_{\betatail}\,\ell^{-\betatail}(1+\log(1+\ell)).
\]
Therefore
\[
\Big\|\frac{\partial s_t}{\partial x_\tau}\Big\|_2
\le
C_K\,A_0(R)\,\ell^{-\betatail}
\ +\ A_1(R)\,\ell^{-\betatail}(1+\log(1+\ell))
\ +\ C_K\,C_{\betatail}\,A_1(R)\,\ell^{-\betatail}(1+\log(1+\ell)).
\]
Since $\ell^{-\betatail}\le \ell^{-\betatail}(1+\log(1+\ell))$ for $\ell\ge 1$ and
$\widetilde C_K=\max\{1,C_K\}\ge 1$ and $\widetilde C_K\ge C_K$, we obtain
\[
\Big\|\frac{\partial s_t}{\partial x_\tau}\Big\|_2
\le
\widetilde C_K\Big(A_0(R)+(1+C_{\betatail})A_1(R)\Big)\,
\ell^{-\betatail}(1+\log(1+\ell)),
\]
which is the claim with the stated $C(R)$.
\end{proof}

\subsection{Jacobian tail for block outputs}

Consider the simplified block output of the form
\[
y_t = x_t + W^{\mathrm{out}}\,(s_t\odot g_t) + b^{\mathrm{out}},
\]
where $g_t=g_t(x_t)$ is token-wise and serves as a gate, and $W^{\mathrm{out}}$ is a fixed matrix.

\begin{corollary}[Jacobian tail for block outputs]\label{cor:jacobian_tail_block}
Under the assumptions of Theorem~\ref{thm:jacobian_tail}, suppose additionally that
$\|g(x)\|_{\infty,2}\le \GR$ for all $x\in\mathcal X_R$. Then for every $\tau<t$ with lag $\ell=t-\tau\ge 1$,
\[
\Big\|\frac{\partial y_t(x)}{\partial x_\tau}\Big\|_2
\ \le\
\|W^{\mathrm{out}}\|_2\,\GR\cdot C(R)\,\ell^{-\betatail}(1+\log(1+\ell)),
\qquad \forall x\in\mathcal X_R.
\]
\end{corollary}

\begin{proof}
For $\tau<t$, $\partial x_t/\partial x_\tau=0$, and since $g_t$ is token-wise, $\partial g_t/\partial x_\tau=0$.
Thus
\[
\frac{\partial y_t}{\partial x_\tau}
=
W^{\mathrm{out}}\,\mathrm{Diag}(g_t)\,\frac{\partial s_t}{\partial x_\tau}.
\]
Taking operator norms and using $\|\mathrm{Diag}(g_t)\|_2\le \|g_t\|_2\le \GR$ plus Theorem~\ref{thm:jacobian_tail}
gives the result.
\end{proof}

\section{Proofs for Section~\ref{sec:theory_memory}}
\label{app:sec4_aux_proofs}

\begin{lemma}[Bounded logit spread implies near-uniform softmax weights]
\label{lem:bounded_logits}
Let $\mathcal I$ be a finite index set with $n:=|\mathcal I|$, and let $(\beth_j)_{j\in \mathcal I}\subset\mathbb R$ be logits.
Define the softmax weights
\[
\alpha_j \;=\; \frac{e^{\beth_j}}{\sum_{i\in \mathcal I} e^{\beth_i}},\qquad j\in \mathcal I.
\]
If the logit spread is bounded by
\[
\Delta \;:=\; \max_{i\in \mathcal I}\beth_i \;-\; \min_{i\in \mathcal I}\beth_i \;\le\; \Delta_0<\infty,
\]
then for every $j\in \mathcal I$,
\begin{equation}
\frac{e^{-\Delta_0}}{n}\;\le\;\alpha_j\;\le\;\frac{e^{\Delta_0}}{n}.
\label{eq:softmax_near_uniform}
\end{equation}
Equivalently, for all $i,j\in \mathcal I$ one has $e^{-\Delta_0}\le \alpha_i/\alpha_j \le e^{\Delta_0}$.
In particular, if $\Delta_0$ is uniformly bounded while $n$ grows, then $\alpha_j=\Theta(1/n)$ uniformly over $j\in \mathcal I$.
\end{lemma}

\begin{proof}
Let $\beth_{\min}:=\min_{i\in \mathcal I}\beth_i$. Then $\beth_{\min}\le \beth_j\le \beth_{\min}+\Delta_0$ for all $j\in\mathcal I$, hence
$e^{\beth_{\min}}\le e^{\beth_j}\le e^{\beth_{\min}+\Delta_0}$ and
\[
n\,e^{\beth_{\min}} \;\le\; \sum_{i\in \mathcal I} e^{\beth_i} \;\le\; n\,e^{\beth_{\min}+\Delta_0}.
\]
Dividing $e^{\beth_j}$ by these bounds yields \eqref{eq:softmax_near_uniform}.
\end{proof}

\subsection{Proof of Lemma~\ref{lem:attn_smooth_routing}}
\label{app:proof_attn_smooth_routing}
\begin{proof}[Proof of Lemma~\ref{lem:attn_smooth_routing}]
Fix a time $t$ and an index $\tau<t$.  Write
\[
\alphadrv_{t,\cdot}(x)=\Softmax(\beth_{t,0}(x),\dots,\beth_{t,t}(x)),
\qquad
\alpha_j:=\alphadrv_{t,j}(x),\quad
\beta_j:=\beth_{t,j}(x),\qquad 0\le j\le t.
\]
Thus $\alpha=\Softmax(\beta)\in\R^{t+1}$ and $\sum_{j\le t}\alpha_j=1$.

Recall the standard softmax Jacobian identity: for all $j,\idxlogit\in\{0,\dots,t\}$, the softmax partial derivatives satisfy
\begin{equation}\label{eq:softmax_jac_attn_smooth}
\frac{\partial \alpha_j}{\partial \beta_{\idxlogit}}
=\alpha_j(\mathbf 1[j=\idxlogit]-\alpha_{\idxlogit}).
\end{equation}

By assumption, for each $j\le t$,
\[
\beta_j=\beth_{t,j}(x)=\langle q(x_t),\,k(x_j)\rangle,
\]
where $q,k$ are token-wise maps.  Since $\tau<t$, the quantity $q(x_t)$ depends only on $x_t$, hence
$\partial q(x_t)/\partial x_\tau=0$.  Similarly, $k(x_j)$ depends only on $x_j$, hence
$\partial k(x_j)/\partial x_\tau=0$ unless $j=\tau$.  Therefore,
\begin{equation}\label{eq:only_tau_logit_depends}
\frac{\partial \beta_{\idxlogit}}{\partial x_\tau}=0\quad\text{for all }\idxlogit\neq \tau,
\qquad\text{and potentially}\qquad
\frac{\partial \beta_\tau}{\partial x_\tau}\neq 0.
\end{equation}

Consequently, by the chain rule and \eqref{eq:only_tau_logit_depends},
\[
\frac{\partial \alpha_j}{\partial x_\tau}
=\sum_{\idxlogit\le t}\frac{\partial \alpha_j}{\partial \beta_{\idxlogit}}\frac{\partial \beta_{\idxlogit}}{\partial x_\tau}
=\frac{\partial \alpha_j}{\partial \beta_\tau}\frac{\partial \beta_\tau}{\partial x_\tau}
=\alpha_j(\mathbf 1[j=\tau]-\alpha_\tau)\,\frac{\partial \beta_\tau}{\partial x_\tau},
\]
where we used \eqref{eq:softmax_jac_attn_smooth} in the last step.  Taking operator norms gives
\begin{equation}\label{eq:dalpha_norm_pointwise}
\Big\|\frac{\partial \alpha_j}{\partial x_\tau}\Big\|_2
=
\big|\alpha_j(\mathbf 1[j=\tau]-\alpha_\tau)\big|
\cdot
\Big\|\frac{\partial \beta_\tau}{\partial x_\tau}\Big\|_2.
\end{equation}
Summing \eqref{eq:dalpha_norm_pointwise} over $j\le t$ yields
\[
\sum_{j\le t}\Big\|\frac{\partial \alpha_j}{\partial x_\tau}\Big\|_2
=
\Big(\sum_{j\le t}\big|\alpha_j(\mathbf 1[j=\tau]-\alpha_\tau)\big|\Big)
\Big\|\frac{\partial \beta_\tau}{\partial x_\tau}\Big\|_2.
\]
To evaluate the scalar sum, note that
\[
\sum_{j\le t}\big|\alpha_j(\mathbf 1[j=\tau]-\alpha_\tau)\big|
=
\underbrace{\alpha_\tau(1-\alpha_\tau)}_{j=\tau}
+
\underbrace{\sum_{j\neq \tau}\alpha_j\alpha_\tau}_{j\neq\tau}
=
\alpha_\tau(1-\alpha_\tau)+\alpha_\tau\sum_{j\neq\tau}\alpha_j
=
2\alpha_\tau(1-\alpha_\tau)
\le 2\alpha_\tau,
\]
since $\sum_{j\neq\tau}\alpha_j=1-\alpha_\tau$ and $1-\alpha_\tau\le 1$.
Therefore,
\[
\sum_{j\le t}\Big\|\frac{\partial \alphadrv_{t,j}(x)}{\partial x_\tau}\Big\|_2
\le
2\,\alphadrv_{t,\tau}(x)\,\Big\|\frac{\partial \beth_{t,\tau}(x)}{\partial x_\tau}\Big\|_2,
\]
which is the first claim.

\paragraph{In particular.}
If $\|\partial\beth_{t,\tau}(x)/\partial x_\tau\|_2\le L_{\beth}$ on $\mathcal X_R$, then
\[
\sum_{j\le t}\Big\|\frac{\partial \alphadrv_{t,j}(x)}{\partial x_\tau}\Big\|_2
\le
2L_{\beth}\,\alphadrv_{t,\tau}(x).
\]
In the diffuse regime of Definition~\ref{def:diffuse_regime}, Lemma~\ref{lem:bounded_logits} implies
$\alphadrv_{t,\tau}(x)=\Theta(1/|\VisSet{t}|)$ uniformly over $\tau\in\VisSet{t}$, hence the right-hand side is
$\lesssim 1/|\VisSet{t}|$.  For full-prefix attention $|\VisSet{t}|=t+1$.
\end{proof}

\subsection{Proof of Proposition~\ref{prop:shared_diffuse_decay_envelopes}}
\label{app:proof_shared_diffuse_decay_envelopes}
\begin{proof}[Proof of Proposition~\ref{prop:shared_diffuse_decay_envelopes}]
Fix a horizon $\Tctx$ and work with the fixed-routing Jacobians from
Section~\ref{sec:fair_comparison}.

\paragraph*{(1) Transformer: attention one-hop dilution.}
By definition of the value influence Jacobian under realized attention weights, by Eq.~\eqref{eq:def_J_attn},
\[
J^{\mathrm{attn}}_{t,\tau}
=\frac{\partial y_t}{\partial v_\tau}\Big|_{\alphadrv}
=\alphadrv_{t,\tau}\,I.
\]
Taking operator norms and using $\|I\|=1$ gives
\[
\|J^{\mathrm{attn}}_{t,\tau}\|=\|\alphadrv_{t,\tau}I\|=\alphadrv_{t,\tau}.
\]
Assume the shared diffuse (low-separation) regime of Definition~\ref{def:diffuse_regime} with
full-prefix visibility $\VisSet{t}=\{0,\dots,t\}$, so $|\VisSet{t}|=t+1$.
The bounded logit spread over $\VisSet{t}$ implies, by Lemma~\ref{lem:bounded_logits}, that for every $\tau\le t$,
\[
\frac{e^{-\Delta}}{t+1}\le \alphadrv_{t,\tau}\le \frac{e^{\Delta}}{t+1},
\]
hence $\alphadrv_{t,\tau}=\Theta(1/(t+1))$ and therefore
\[
\|J^{\mathrm{attn}}_{t,\tau}\|=\Theta\!\Big(\frac{1}{t+1}\Big)\qquad(\tau\le t).
\]
For a fixed old source $\tau=O(1)$ and lag $\ell=t-\tau$, we have
\[
\|J^{\mathrm{attn}}_{\tau+\ell,\tau}\|
=\alphadrv_{\tau+\ell,\tau}
=\Theta\!\Big(\frac{1}{\tau+\ell+1}\Big)
=\Theta(1/\ell),
\]
since $\tau$ is fixed and $\ell\to\infty$.

\paragraph*{(2) Mamba under failed freeze time.}
By definition of the fixed-routing impulse Jacobian for an SSM, by Eq.~\eqref{eq:def_J_ssm},
\[
J^{\mathrm{ssm}}_{t,\tau}
=
\Cssmt{t}\Big(\prod_{r=\tau+1}^{t}\Assmt{r}\Big)\Bssmt{\tau},
\qquad 0\le \tau\le t.
\]
Assume the realized recurrence has diagonal transitions
\[
\Assmt{r}=\operatorname{diag}(\exp(-a_n\Delta_r)),
\qquad a_n\ge \lambda>0,
\]
and bounded input/output factors
\[
\sup_r\|\Bssmt{r}\|\le B_{\max},
\qquad
\sup_r\|\Cssmt{r}\|\le C_{\max}.
\]
Then
\[
\Big\|\prod_{r=\tau+1}^{t}\Assmt{r}\Big\|
=
\max_n \exp\!\Big(-a_n\sum_{r=\tau+1}^{t}\Delta_r\Big)
\le
\exp\!\Big(-\lambda\sum_{r=\tau+1}^{t}\Delta_r\Big).
\]
Under the failed-freeze-time condition
\[
\sum_{r=\tau+1}^{t}\Delta_r\ge c_\Delta (t-\tau),
\]
it follows that
\[
\|J^{\mathrm{ssm}}_{t,\tau}\|
\le
C_{\max}B_{\max}\exp\!\big(-\lambda c_\Delta (t-\tau)\big).
\]
Setting $c:=C_{\max}B_{\max}$ and $\ell:=t-\tau$ gives
\[
\|J^{\mathrm{ssm}}_{t,\tau}\|\le c\,e^{-\lambda c_\Delta \ell}.
\]

\paragraph*{(3) Sessa: diffuse feedback routing.}
For a realized feedback matrix $\Bfb$, the solve Jacobian is the resolvent given by Eq.~\eqref{eq:def_J_sessa}
\[
J^{\mathrm{sessa}}=(I-\Bfb)^{-1},
\qquad
J^{\mathrm{sessa}}_{t,\tau}=[(I-\Bfb)^{-1}]_{t,\tau}.
\]
Since $\Bfb$ is scalar-valued, $J^{\mathrm{sessa}}_{t,\tau}\in\R$ is a scalar coefficient shared across features.

Fix $\tau$ and consider the impulse in the forward stream $f$ at time $\tau$:
$f_\tau=1$ and $f_t=0$ for $t\neq \tau$.  Let $s$ be the solution to $(I-\Bfb)s=f$.
By linearity, $s_t=J^{\mathrm{sessa}}_{t,\tau}$ for all $t$.
Moreover, by forward substitution (equivalently \eqref{eq:scalar_feedback}), $s_\tau=1$ and for $t>\tau$,
\[
s_t
=
f_t+\gamma_t\sum_{j=0}^{t-1}\alphafb_{t,j}s_j
=
\gamma_t\sum_{j=\tau}^{t-1}\alphafb_{t,j}s_j,
\]
since $f_t=0$ for $t\neq\tau$ and $s_j=0$ for $j<\tau$ in a strictly causal solve.

Under Assumptions~\ref{ass:alpha_upper_main}--\ref{ass:gamma_bound_main} we have
$\alphafb_{t,j}\le c_2/t$ for all $j<t$ and $|\gamma_t|\le \gamma_{\max}<1$, and defining
$\betatail:=1-\gamma_{\max}c_2\in(0,1]$, with $\gamma_{\max}c_2<1$,
Theorem~\ref{thm:poly_decay_main} applies to this impulse recursion, shifted to start at $\tau$, and yields that for all lags $\ell\ge 1$,
\[
|J^{\mathrm{sessa}}_{\tau+\ell,\tau}|
=
|s_{\tau+\ell}|
\le
C\,\ell^{-\betatail},
\]
for an explicit constant $C$, e.g.\ $C=(1-\betatail)e^{\,1-\betatail}$.

\paragraph{Tightness.}
In the explicit uniform-routing regime
\[
\BfbEnt{t}{j}=
\begin{cases}
0, & t=0,\\[2pt]
\dfrac{\gamma}{t}\mathbf 1[j<t], & t\ge 1,
\end{cases}
\qquad \gamma\in(0,1),
\]
one has $\alphafb_{t,j}=t^{-1}\mathbf 1[j<t]$ and constant gain
$\gamma_t\equiv\gamma$, hence $\betatail=1-\gamma$.
Appendix Corollary~\ref{cor:theta_tail_uniform_shifted} gives, for every fixed source position $\tau$,
\[
|J^{\mathrm{sessa}}_{\tau+\ell,\tau}|=\Theta_\tau(\ell^{-\betatail}).
\]
Moreover, Appendix Corollary~\ref{cor:theta_tail_uniform_bounded_source} yields the stronger
uniform statement that for every $\tau_{\max}<\infty$ there exist constants
$c^-_{\tau_{\max}},c^+_{\tau_{\max}}>0$ such that
\[
c^-_{\tau_{\max}}\,\ell^{-\betatail}
\le
|J^{\mathrm{sessa}}_{\tau+\ell,\tau}|
\le
c^+_{\tau_{\max}}\,\ell^{-\betatail}
\]
for all $0\le \tau\le \tau_{\max}$ and all $\ell\ge 1$.
Thus the one-layer envelope is tight for each fixed source and uniformly on every bounded source family,
in particular on every fixed finite horizon.
\end{proof}

\subsection{Proof of Proposition~\ref{prop:lti_exp_decay}}
\label{app:proof_lti_exp_decay}
\begin{proof}
The claim is about the input--output map and is independent of the chosen realization.
By the controllable and observable decomposition, also known as the Kalman decomposition~\citep{antsaklis2006linear}, there exists a similarity transform that isolates the
controllable and observable subsystem $(\Assm_{\mathrm{co}},\Bssm_{\mathrm{co}},\Cssm_{\mathrm{co}})$ such that for all $\ell\ge 0$,
\[
\Cssm \Assm^{\ell} \Bssm \;=\; \Cssm_{\mathrm{co}}\,\Assm_{\mathrm{co}}^{\ell}\,\Bssm_{\mathrm{co}}.
\]
Moreover, $(\Assm_{\mathrm{co}},\Bssm_{\mathrm{co}},\Cssm_{\mathrm{co}})$ is a minimal realization of the same transfer function, so it admits no pole--zero cancellations and its poles coincide with the reachable and observable eigenvalues of $\Assm_{\mathrm{co}}$~\citep{dahleh2011_chap27}. Since the transfer function is BIBO stable, all its poles lie strictly inside the unit disk (DT case)~\citep{dahleh2011_chap15}; hence $\rhospec{\Assm_{\mathrm{co}}}<1$.
It follows from standard finite-dimensional matrix power bounds that there exist $c>0$ and $\kappa\in(0,1)$ such that
$\|\Assm_{\mathrm{co}}^{\ell}\|\le c\,\kappa^{\ell}$ for all $\ell$, and therefore
$\|\Cssm \Assm^{\ell} \Bssm\|=\|\Cssm_{\mathrm{co}}\Assm_{\mathrm{co}}^{\ell}\Bssm_{\mathrm{co}}\|\le c'\kappa^{\ell}$.
\end{proof}

\subsection{Proof of Proposition~\ref{prop:mamba_e2e_decay}}
\label{app:proof_mamba_e2e_decay}

The key point is that, under ZOH discretization, the state-transition product is controlled by the
accumulated discretization time
\[
\sum_{r=\tau+1}^{t}\Delta_r(x),
\]
since each channel contributes a factor \(\exp(-a_n\Delta_r(x))\). Accordingly, the proof first obtains
an end-to-end Jacobian bound in terms of
\[
\Pi_{t,\ell}(x)=\exp\!\Big(-\lambda\sum_{r=\tau+1}^{t}\Delta_r(x)\Big),
\]
and only then converts this into exponential-in-lag decay under failed freeze time.

\begin{proof}
Fix $x\in\mathcal X_R$ and indices $\tau<t$, and set $\ell:=t-\tau\ge 1$.
Write $J^h_{t,\tau}:=\partial h_t(x)/\partial x_\tau$ and $J^{\mathrm{e2e}}_{t,\tau}:=\partial y_t(x)/\partial x_\tau$.
We use the product convention
\[
\prod_{r=\tau+1}^{t}\Assmt{r}:=\Assmt{t}\Assmt{t-1}\cdots \Assmt{\tau+1},
\qquad
\prod_{r=t+1}^{t}(\cdot):=I.
\]

\paragraph{State bound via ZOH convexity.}
In a ZOH-diagonal channel, each mode $n$ evolves as the scalar recursion
\[
(h_t)_n = e^{-a_n\Delta_t}\,(h_{t-1})_n + \frac{1-e^{-a_n\Delta_t}}{a_n}\,(b_t)_n,
\qquad a_n\ge \lambda,\ \Delta_t\ge 0,
\]
where we take
\[
b_t:=\widetilde{\Bssmt{t}}(x_t)u_t(x_t).
\]
By the bounds on $\widetilde{\Bssmt{t}}$ and $u_t$ on $\mathcal X_R$, we have
\[
\|b_t\|\le \Gmax U_R,
\]
and hence $|(b_t)_n|\le \Gmax U_R$ for each mode.
Since \(h_{-1}=0\), Lemma~\ref{lem:zoh_state_bound} applied componentwise with $a_{\min}=\lambda$ gives
\[
\sup_t |(h_t)_n|\le \frac{\Gmax U_R}{\lambda}\qquad\text{for every mode }n.
\]
Therefore
\[
\|h_t\|_2\le \sqrt{\dstate}\,\|h_t\|_\infty
\le \sqrt{\dstate}\,\frac{\Gmax U_R}{\lambda}
=:H_R.
\]

\paragraph{Jacobian recursion for $t>\tau$.}
For $t>\tau$, locality implies
\[
\frac{\partial \Assmt{t}(x_t)}{\partial x_\tau}
=
\frac{\partial \widetilde{\Bssmt{t}}(x_t)}{\partial x_\tau}
=
\frac{\partial u_t(x_t)}{\partial x_\tau}
=
\frac{\partial \Gssmt{t}(x_t)}{\partial x_\tau}
=0.
\]
Differentiating
\[
h_t=\Assmt{t}(x_t)\,h_{t-1}+\Gssmt{t}(x_t)\,\widetilde{\Bssmt{t}}(x_t)\,u_t(x_t)
\]
with respect to $x_\tau$ yields
\[
J^h_{t,\tau}=\Assmt{t}(x_t)\,J^h_{t-1,\tau},
\qquad t>\tau.
\]
Iterating gives
\[
J^h_{t,\tau}=\Big(\prod_{r=\tau+1}^{t}\Assmt{r}(x_r)\Big)\,J^h_{\tau,\tau}.
\]

\paragraph{Source-time derivative bound.}
At $t=\tau$, write $b_\tau:=\widetilde{\Bssmt{\tau}}(x_\tau)u_\tau(x_\tau)$ and differentiate the ZOH update:
\[
J^h_{\tau,\tau}
=
\Big(\frac{\partial \Assmt{\tau}(x_\tau)}{\partial x_\tau}\Big)h_{\tau-1}
+\Big(\frac{\partial \Gssmt{\tau}(x_\tau)}{\partial x_\tau}\Big)b_\tau
+\Gssmt{\tau}(x_\tau)\,\frac{\partial b_\tau}{\partial x_\tau}.
\]
Moreover,
\[
\frac{\partial b_\tau}{\partial x_\tau}
=
\Big(\frac{\partial \widetilde{\Bssmt{\tau}}(x_\tau)}{\partial x_\tau}\Big)u_\tau
+
\widetilde{\Bssmt{\tau}}(x_\tau)\Big(\frac{\partial u_\tau(x_\tau)}{\partial x_\tau}\Big).
\]
Since
\[
\Gssmt{\tau}(x_\tau)=\operatorname{diag}\!\Big(\frac{1-[\Assmt{\tau}(x_\tau)]_n}{a_n}\Big)_n,
\]
we have the operator bounds
\[
\|\Gssmt{\tau}(x_\tau)\|\le \frac{1}{\lambda},
\qquad
\Big\|\frac{\partial \Gssmt{\tau}(x_\tau)}{\partial x_\tau}\Big\|
\le \frac{1}{\lambda}\Big\|\frac{\partial \Assmt{\tau}(x_\tau)}{\partial x_\tau}\Big\|.
\]
Using
\[
\|h_{\tau-1}\|\le H_R,\qquad
\|b_\tau\|\le \Gmax U_R,
\]
together with the derivative bounds gives
\[
\|J^h_{\tau,\tau}\|
\le
L_A\,H_R \;+\;\frac{L_A}{\lambda}\,\Gmax U_R \;+\;\frac{1}{\lambda}\big(L_B\,U_R+\Gmax\,L_u\big)
=:J_R.
\]

\paragraph{Transition product bound by accumulated discretization time.}
Since each $\Assmt{r}$ is diagonal with entries $\exp(-a_n\Delta_r)$ and $a_n\ge \lambda$,
\[
\Big\|\prod_{r=\tau+1}^{t}\Assmt{r}(x_r)\Big\|
=
\max_n \exp\!\Big(-a_n\sum_{r=\tau+1}^{t}\Delta_r(x)\Big)
\le
\exp\!\Big(-\lambda\sum_{r=\tau+1}^{t}\Delta_r(x)\Big)
=:\Pi_{t,\ell}(x).
\]
Therefore
\[
\|J^h_{t,\tau}\|
\le
\Pi_{t,\ell}(x)\,\|J^h_{\tau,\tau}\|
\le
J_R\,\Pi_{t,\ell}(x).
\]

\paragraph{Output Jacobian.}
For $\tau<t$, locality implies $\partial \Cssmt{t}(x_t)/\partial x_\tau=0$, so
\[
\frac{\partial y_t}{\partial x_\tau}
=
\Cssmt{t}(x_t)\,J^h_{t,\tau}.
\]
Hence
\[
\Big\|\frac{\partial y_t(x)}{\partial x_\tau}\Big\|
\le
\|\Cssmt{t}(x_t)\|\,\|J^h_{t,\tau}\|
\le
C_R\,J_R\,\Pi_{t,\ell}(x).
\]
Thus the claim holds with
\[
C(R):=C_RJ_R.
\]
\end{proof}

\subsection{Proof of Lemma~\ref{lem:zoh_state_bound}}
\label{app:proof_zoh_state_bound}
\begin{proof}[Proof of Lemma~\ref{lem:zoh_state_bound}]
Fix $t\ge 0$ and define $\theta_t:=e^{-a\Delta_t}\in[0,1]$, since $a>0$ and $\Delta_t\ge 0$.
Then $1-\theta_t=1-e^{-a\Delta_t}\in[0,1]$, and the update can be rewritten as
\[
h_t
=\theta_t\,h_{t-1}+(1-\theta_t)\,\frac{b_t}{a}.
\]
Taking absolute values and using the triangle inequality yields
\[
|h_t|
\le
\theta_t|h_{t-1}|+(1-\theta_t)\frac{|b_t|}{a}.
\]
Since $\theta_t\in[0,1]$, for any $u,v\ge 0$ one has $\theta_t u+(1-\theta_t)v\le \max\{u,v\}$, hence
\[
|h_t|
\le
\max\Big\{|h_{t-1}|,\ \frac{|b_t|}{a}\Big\}
\le
\max\Big\{|h_{t-1}|,\ \frac{|b_t|}{a_{\min}}\Big\},
\]
using $a\ge a_{\min}$.

Define
\[
B_t:=\max\Big\{|h_{-1}|,\ \max_{0\le s\le t}\frac{|b_s|}{a_{\min}}\Big\}.
\]
We claim by induction that $|h_t|\le B_t$ for all $t\ge 0$.
For $t=0$ this follows from the previous inequality. If $|h_{t-1}|\le B_{t-1}$, then
\[
|h_t|
\le
\max\Big\{|h_{t-1}|,\ \frac{|b_t|}{a_{\min}}\Big\}
\le
\max\Big\{B_{t-1},\ \frac{|b_t|}{a_{\min}}\Big\}
=B_t,
\]
proving the induction. Taking $\sup_{t\ge 0}$ gives
\[
\sup_{t\ge 0}|h_t|
\le
\max\Big\{|h_{-1}|,\ \sup_{s\ge 0}\frac{|b_s|}{a_{\min}}\Big\},
\]
which is the general bound.

If additionally $|b_t|\le M$ for all $t$ and $h_{-1}=0$, then the right-hand side is at most $M/a_{\min}$,
proving $\sup_t|h_t|\le M/a_{\min}$.
\end{proof}

\begin{remark}[Vector and diagonal case]
For diagonal $A=-\mathrm{diag}(a_n)$ with $\min_n a_n\ge a_{\min}$, the bound holds componentwise for each mode and channel,
and hence yields the uniform bound $\|h_t\|_\infty\le \sup_s\|b_s\|_\infty/a_{\min}$.
More generally, for any monotone norm $\|\cdot\|$ one has
$\|h_t\|\le \|\mathbf 1\|\,\sup_s\|b_s\|_\infty/a_{\min}$.
\end{remark}

\subsection{Proof of Corollary~\ref{cor:mamba_e2e_freeze}}
\label{app:proof_mamba_e2e_freeze}
\begin{proof}
Proposition~\ref{prop:mamba_e2e_decay} gives
\[
\Big\|\frac{\partial y_t(x)}{\partial x_\tau}\Big\|
\le
C(R)\,\Pi_{t,\ell}(x),
\qquad
\Pi_{t,\ell}(x)
=
\exp\!\Big(-\lambda\sum_{r=\tau+1}^{t}\Delta_r(x)\Big).
\]
Under failed freeze time,
\[
\sum_{r=\tau+1}^{t}\Delta_r(x)\ge c_\Delta(t-\tau).
\]
Applying Proposition~\ref{prop:freeze_errors_exp} yields
\[
\Pi_{t,\ell}(x)\le \exp\!\big(-\lambda c_\Delta (t-\tau)\big),
\]
and therefore
\[
\Big\|\frac{\partial y_t(x)}{\partial x_\tau}\Big\|
\le
C(R)\exp\!\big(-\lambda c_\Delta (t-\tau)\big).
\]
\end{proof}

\begin{remark}[Local windows]
If $\Assmt{t},\widetilde{\Bssmt{t}},\Cssmt{t},u_t$ depend on a fixed window $x_{t-K:t}$, the same argument yields
\[
\Big\|\frac{\partial y_t}{\partial x_\tau}\Big\|
\le
C(R)\,\exp\!\Big(-\lambda\sum_{r=\tau+K+1}^{t}\Delta_r(x)\Big)
\qquad (t>\tau+K),
\]
so the same failed-freeze-time conclusion holds up to a finite-window slack.
\end{remark}

\subsection{Proof of Proposition~\ref{prop:freeze_errors_exp}}
\label{app:proof_freeze_errors_exp}
\begin{proof}
By definition,
\[
\Pi_{t,\ell}
=
\exp\!\Big(-\lambda\sum_{r=\tau+1}^{t}\Delta_r\Big).
\]
Under the failed-freeze-time condition
\[
\sum_{r=\tau+1}^{t}\Delta_r\ge c_\Delta(t-\tau)=c_\Delta \ell,
\]
we obtain
\[
\Pi_{t,\ell}
\le
\exp\!\big(-\lambda c_\Delta \ell\big).
\]
This is exactly the claim.
\end{proof}

\subsection{Details for Proposition~\ref{prop:shared_regime_decay_envelopes_e2e}}
\label{app:proof_shared_regime_decay_envelopes_e2e_items12}
\begin{proof}
\textbf{(1) Transformer attention in the no-freeze setting.}
Let $y_t(x)=\sum_{j\in \VisSet{t}}\alpha_{t,j}(x)\,v(x_j)$.
For $\tau<t$, differentiate:
\[
\frac{\partial y_t}{\partial x_\tau}
=
\alpha_{t,\tau}\frac{\partial v(x_\tau)}{\partial x_\tau}
+\sum_{j\in \VisSet{t}}\frac{\partial \alpha_{t,j}(x)}{\partial x_\tau}\,v(x_j).
\]
Taking operator norms and using $\|\partial v(x_\tau)/\partial x_\tau\|\le L_v$ and $\|v(x_j)\|\le V_R$ yields
\[
\Big\|\frac{\partial y_t}{\partial x_\tau}\Big\|
\le
\alpha_{t,\tau}L_v
+
V_R\sum_{j\in \VisSet{t}}\Big\|\frac{\partial \alpha_{t,j}}{\partial x_\tau}\Big\|.
\]
Under the shared regime in Section~\ref{sec:e2e_comparison}, $\alpha_{t,\tau}\le c_2/|\VisSet{t}|$ and
$\sum_{j\in \VisSet{t}}\|\partial \alpha_{t,j}/\partial x_\tau\|\le L_\alpha/|\VisSet{t}|$, hence
$\|\partial y_t/\partial x_\tau\|\lesssim 1/|\VisSet{t}|$.
For full-prefix attention $|\VisSet{t}|=t+1$, recovering $\|\partial y_t/\partial x_\tau\|\lesssim 1/(t+1)$.

\textbf{(2) Mamba under failed freeze time.}
Item~(2) follows by combining Proposition~\ref{prop:mamba_e2e_decay} with failed freeze time, namely
\[
\sum_{r=\tau+1}^{t}\Delta_r(x)\ge c_\Delta(t-\tau),
\]
that is, by Corollary~\ref{cor:mamba_e2e_freeze}.
\end{proof}
\section{BIBO stability on infinite horizons and uniform-in-\texorpdfstring{$\Tctx$}{Tctx} bounds}
\label{sec:sessa_bibo_infinite}

We extend the finite-horizon BIBO statement to infinite sequences under an explicit
row-contraction condition, and to uniform-in-$\Tctx$ bounds for truncated length-$\Tctx$ networks
without appealing to compactness.

\subsection{Sequence norms and stability definition}

We use the norm $\|\cdot\|_{\infty,2}$ and balls from Definition~\ref{def:linf2_ball}.
For finite tensors we also use the comparison \eqref{eq:norm_compare}.

\subsection{Feedback matrix and row-contraction condition}

Fix a causal width-$m$ Sessa block $G$ as in Section~\ref{sec:sessa_block}, but now acting on
infinite sequences in $\ell_\infty(\mathbb N,\mathbb R^m)$.
We emphasize that the block input and output live in $\mathbb R^m$, while the triangular solve
$(I-\Bfb)s=f$ is performed in a value space $\mathbb R^r$: in our definition,
$s_t\in\mathbb R^r$, $f_t\in\mathbb R^r$, $g_t\in\mathbb R^r$, and $z_t=s_t\odot g_t\in\mathbb R^r$,
and the output projection is token-wise affine $o:\mathbb R^r\to\mathbb R^m$.

\paragraph{Causal feedback-attention weights.}
For each input $x$, the masked softmax in the feedback branch defines strictly lower-triangular weights
$(\alphafb_{t\tau}(x))_{t,\tau\ge 0}$ with
\begin{equation}\label{eq:attention_props}
\alphafb_{t\tau}(x)\ge 0,\qquad \alphafb_{t\tau}(x)=0\ \text{for }\tau\ge t,\qquad
\sum_{\tau<t}\alphafb_{t\tau}(x)=1\ \text{for }t\ge 1,
\end{equation}
with the empty sum $=0$ for $t=0$. These properties hold as follows: for $t\ge 1$ each row $t$ is a softmax over the finite set
$\{0,\dots,t-1\}$, hence $\alphafb_{t\tau}\ge 0$ and $\sum_{\tau<t}\alphafb_{t\tau}=1$; for $t=0$ we set
$\alphafb_{0\tau}=0$ for all $\tau$, i.e.\ the context is empty, so the empty sum equals $0$.
\paragraph{Feedback attention matrix.}
Define $\Alphafb(x):=(\alphafb_{t\tau}(x))_{t,\tau\ge 0}$.

\paragraph{Feedback coefficient and the Sessa matrix $\Bfb$.}
By definition of the Sessa block, the feedback coefficient is
\[
\gamma_t(x)=\tanh(u_t(x))\in(-1,1),
\]
computed token-wise from the block input, via affine maps and element-wise nonlinearities.
Define the diagonal operator $\Gammafb{x}:=\mathrm{diag}(\gamma_t(x))_{t\ge 0}$ and the
strictly lower-triangular matrix
\begin{equation}\label{eq:Bfb_def}
\Bfb(x):=\Gammafb{x}\,\Alphafb(x)
\qquad\Longleftrightarrow\qquad
\BfbEnt{t}{\tau}(x)=\gamma_t(x)\,\alphafb_{t\tau}(x).
\end{equation}

\begin{assumption}[Uniform feedback margin and row contraction]\label{ass:row_contraction}
For every radius $R\ge 0$ there exists $\rho(R)\in[0,1)$ such that for all inputs
$x\in\ell_\infty(\mathbb N,\mathbb R^m)$ with $\|x\|_{\infty,2}\le R$,
\begin{equation}\label{eq:rho_condition}
\sup_{t\ge 0}|\gamma_t(x)|\le \rho(R).
\end{equation}
In particular, using \eqref{eq:attention_props}--\eqref{eq:Bfb_def}, for every $x$,
\begin{equation}\tag{$\star$}\label{eq:star_condition}
\sup_{t\ge 0}\sum_{\tau<t}|\BfbEnt{t}{\tau}(x)|
=\sup_{t\ge 1}\sum_{\tau<t}|\BfbEnt{t}{\tau}(x)|
=\sup_{t\ge 1}|\gamma_t(x)|
\le \sup_{t\ge 0}|\gamma_t(x)|
\le \rho(R)<1.
\end{equation}
\end{assumption}

\begin{remark}[An explicit choice of \rho(R)]
If $u_t(x)$ is produced by a token-wise feedforward stack of affine maps and element-wise
nonlinearities $\sigma$ satisfying $|\sigma(z)|\le |z|$ coordinate-wise; this holds for $\mathrm{GELU}$.
Affine and linear maps are handled separately via spectral norms as in Lemma~\ref{lem:affine_bound}. Then for some explicit constants $c_\gamma\ge 0$, $L_{\gamma,\mathrm{pre}}\ge 0$
depending only on the block parameters,
\begin{equation}\label{eq:u_bound}
\sup_{t\ge 0}|u_t(x)|\ \le\ c_\gamma + L_{\gamma,\mathrm{pre}} \|x\|_{\infty,2}.
\end{equation}
Hence on the ball $\|x\|_{\infty,2}\le R$ one can take
\begin{equation}\label{eq:rho_explicit}
\rho(R):=\tanh(c_\gamma+L_{\gamma,\mathrm{pre}} R)\ <\ 1.
\end{equation}
The strict inequality holds since $c_\gamma+L_{\gamma,\mathrm{pre}} R<\infty$ and $\tanh(\cdot)<1$ for finite
arguments.
\end{remark}

\subsection{Causal triangular solve on \texorpdfstring{$\ell_\infty$}{ell_infty}}

The only operation that truly changes nature at $\Tctx=\infty$ is the lower-triangular solve.
We treat it as a causal linear system.

\subsection{Proof of Lemma~\ref{lem:triangular_linf}}
\label{app:proof_triangular_linf}

\begin{proof}
Let $\Bfb=(\BfbEnt{t}{\tau})_{t,\tau\ge 0}$ be strictly lower-triangular and define the causal operator
$(\Bfb s)_t:=\sum_{\tau<t}\BfbEnt{t}{\tau}s_\tau$, a finite sum for each fixed $t$, acting on $\mathbb R^r$-valued sequences.
Here $\BfbEnt{t}{\tau}\in\mathbb R$ is scalar and multiplies $s_\tau\in\mathbb R^r$, i.e.\ scalar--vector multiplication.
Assume
\[
\sup_{t\ge 0}\sum_{\tau<t}|\BfbEnt{t}{\tau}|\ \le\ \rho\ <\ 1.
\]
Then for every bounded input $f\in\ell_\infty(\mathbb N,\mathbb R^r)$ there exists a unique
bounded solution $s\in\ell_\infty(\mathbb N,\mathbb R^r)$ to
\[
s = f + \Bfb s \qquad\text{equivalently, $(I-\Bfb)s=f$},
\]
and it satisfies the explicit bound
\begin{equation}\label{eq:triangular_s_bound}
\|s\|_{\infty,2}\ \le\ \frac{1}{1-\rho}\,\|f\|_{\infty,2}.
\end{equation}

Existence and uniqueness follow by forward substitution: for $t=0$, $s_0=f_0$; for $t\ge 1$,
\[
s_t = f_t + \sum_{\tau<t} \BfbEnt{t}{\tau}s_\tau
\]
depends only on previously defined $(s_\tau)_{\tau<t}$. Thus a unique sequence $s$ exists.

For the bound, define the partial maxima
\[
M_t := \max_{0\le k\le t}\|s_k\|_2 \qquad (t\ge 0).
\]
For $t=0$ we have $s_0=f_0$, hence $M_0=\|s_0\|_2\le \|f\|_{\infty,2}$.
For $t\ge 1$, using the row-sum estimate and $M_{t-1}\ge \|s_\tau\|_2$ for all $\tau<t$,
\[
\|s_t\|_2 \le \|f_t\|_2 + \sum_{\tau<t}|\BfbEnt{t}{\tau}|\|s_\tau\|_2
\le \|f\|_{\infty,2} + \rho\, M_{t-1}.
\]
We now prove by induction that for all $t\ge 0$,
\[
M_t \le \frac{1}{1-\rho}\,\|f\|_{\infty,2}.
\]
The base case $t=0$ holds since $M_0\le \|f\|_{\infty,2}\le \frac{1}{1-\rho}\|f\|_{\infty,2}$.
Assume the claim holds for $t-1$ with some $t\ge 1$. Then the previous estimate gives
\[
\|s_t\|_2 \le \|f\|_{\infty,2} + \rho\,M_{t-1}
\le \|f\|_{\infty,2} + \rho\,\frac{1}{1-\rho}\|f\|_{\infty,2}
= \frac{1}{1-\rho}\|f\|_{\infty,2}.
\]
Hence $M_t=\max\{M_{t-1},\|s_t\|_2\}\le \frac{1}{1-\rho}\|f\|_{\infty,2}$, completing the induction.
Taking $\sup_{t\ge 0}$ gives $\|s\|_{\infty,2}=\sup_t\|s_t\|_2=\sup_t M_t
\le \frac{1}{1-\rho}\|f\|_{\infty,2}$, which is \eqref{eq:triangular_s_bound}.
\end{proof}

\subsection{Explicit one-block bound without compactness}

We now bound one Sessa block on $\ell_\infty$ balls by tracking constants explicitly.

\begin{lemma}[Token-wise affine bound]\label{lem:affine_bound}
Let $y_t=x_tW+b$ with $W\in\mathbb R^{d\times d'}$ and $b\in\mathbb R^{d'}$, where the same $W$ and $b$ are used for all tokens.
Then for any sequence $x$, finite or infinite,
\[
\|y\|_{\infty,2}\le \|W\|_2\,\|x\|_{\infty,2}+\|b\|_2,
\]
where $\|\cdot\|_2$ is the spectral norm for matrices and Euclidean norm for vectors.
\end{lemma}

\begin{lemma}[Causal attention is $\ell_\infty$-nonexpansive]\label{lem:attn_nonexpansive}
Let $\Alphafb=(\alphafb_{t\tau})$ satisfy \eqref{eq:attention_props}. Then for any value sequence $v$,
\noindent the sequence $y$ defined by $y_t:=\sum_{\tau<t}\alphafb_{t\tau}v_\tau$ satisfies
\[
\|y\|_{\infty,2}\le \|v\|_{\infty,2}.
\]
\end{lemma}

\begin{proof}
For $t\ge 1$, $y_t$ is a convex combination of $\{v_\tau\}_{\tau<t}$, hence
\[
\|y_t\|_2 \le \sup_{\tau<t}\|v_\tau\|_2 \le \|v\|_{\infty,2}.
\]
For $t=0$ the sum is empty, hence $y_0=0$ and $\|y_0\|_2\le \|v\|_{\infty,2}$ as well.
Taking the supremum over $t\ge 0$ gives $\|y\|_{\infty,2}\le \|v\|_{\infty,2}$.
\end{proof}

\begin{proposition}[One Sessa block: explicit ball-to-ball bound]\label{prop:one_block_bound}
Consider one width-$m$ Sessa block $G:\ell_\infty(\mathbb N,\mathbb R^m)\to\ell_\infty(\mathbb N,\mathbb R^m)$.
Assume:
\begin{itemize}
\item the feedback matrix is $\Bfb(x)=\Gammafb{x}\Alphafb(x)$ with $\Alphafb(x)$ satisfying \eqref{eq:attention_props} and
$\gamma_t(x)=\tanh(u_t(x))$ as above;
\item the block produces sequences $f(x),g(x)\in\ell_\infty(\mathbb N,\mathbb R^r)$ and an output projection
$o:\mathbb R^r\to\mathbb R^m$ given token-wise by
\[
o(z)_t = z_t W^{\mathrm{out}} + b^{\mathrm{out}},\qquad
W^{\mathrm{out}}\in\mathbb R^{r\times m},\ \ b^{\mathrm{out}}\in\mathbb R^{m};
\]
\item the block output is $G(x)=x+o(z)$ with $z_t=s_t\odot g_t\in\mathbb R^r$ and the solve is in value space:
\[
z_t=s_t\odot g_t\in\mathbb R^r,\qquad (I-\Bfb(x))s=f(x),\qquad s\in\ell_\infty(\mathbb N,\mathbb R^r).
\]
\end{itemize}
Suppose there exist explicit constants $c_f,c_g,c_\gamma\ge 0$ and $L_f,L_g,L_{\gamma,\mathrm{pre}}\ge 0$, depending only on the block parameters, such that for all inputs $x$,
\begin{equation}\label{eq:branch_bounds}
\|f(x)\|_{\infty,2}\le c_f+L_f\|x\|_{\infty,2},\quad
\|g(x)\|_{\infty,2}\le c_g+L_g\|x\|_{\infty,2},\quad
\sup_t|u_t(x)|\le c_\gamma+L_{\gamma,\mathrm{pre}}\|x\|_{\infty,2}.
\end{equation}
Define, for $R\ge 0$,
\[
\rho_R:=\tanh(c_\gamma+L_{\gamma,\mathrm{pre}} R)\in[0,1),
\qquad
F_R:=c_f+L_f R,
\qquad
G_R:=c_g+L_g R.
\]
Then for all $x$ with $\|x\|_{\infty,2}\le R$, the block output satisfies the explicit bound
\begin{equation}\label{eq:block_bound}
\|G(x)\|_{\infty,2}
\;\le\;
R
\;+\;
\|W^{\mathrm{out}}\|_2\,\frac{F_R\,G_R}{1-\rho_R}
\;+\;
\|b^{\mathrm{out}}\|_2.
\end{equation}
\end{proposition}

\begin{proof}
On $\|x\|_{\infty,2}\le R$, \eqref{eq:branch_bounds} gives $\|f\|_{\infty,2}\le F_R$ and
$\|g\|_{\infty,2}\le G_R$. Also $\sup_t|u_t(x)|\le c_\gamma+L_{\gamma,\mathrm{pre}} R$, hence
$\sup_t|\gamma_t(x)|\le \rho_R$.
Using \eqref{eq:star_condition}, we get $\sup_t\sum_{\tau<t}|\BfbEnt{t}{\tau}(x)|\le \rho_R<1$.
Lemma~\ref{lem:triangular_linf} then yields
\[
\|s\|_{\infty,2}\le \frac{1}{1-\rho_R}\|f\|_{\infty,2}\le \frac{F_R}{1-\rho_R}.
\]

For the element-wise product in $\mathbb R^r$, for each $t$,
\[
\|z_t\|_2=\|s_t\odot g_t\|_2\le \|s_t\|_2\,\|g_t\|_2,
\]
since
\[
\|s_t\odot g_t\|_2^2
=\sum_i s_{ti}^2 g_{ti}^2
\le \sum_i s_{ti}^2 \Big(\sum_j g_{tj}^2\Big)
= \|s_t\|_2^2\|g_t\|_2^2.
\]
Hence
\[
\|z\|_{\infty,2}\le \|s\|_{\infty,2}\,\|g\|_{\infty,2}
\le \frac{F_R}{1-\rho_R}\,G_R.
\]

Finally, by Lemma~\ref{lem:affine_bound} for $o(z)=zW_o+b_o$ and the residual $G(x)=x+o(z)$,
\[
\|G(x)\|_{\infty,2}\le \|x\|_{\infty,2}+\|o(z)\|_{\infty,2}
\le R+\|W^{\mathrm{out}}\|_2\|z\|_{\infty,2}+\|b^{\mathrm{out}}\|_2,
\]
which gives \eqref{eq:block_bound}.
\end{proof}

\begin{remark}[Explicit dependence of the constants in \eqref{eq:branch_bounds}]
Each branch, including the query, key, and value maps and the MLPs producing $f$, $g$, and $u$ and related components,
is a finite composition of token-wise affine maps, $\mathrm{RoPE}_t$ rotations that are orthogonal and norm-preserving,
masked softmax attention as in Lemma~\ref{lem:attn_nonexpansive}, and element-wise nonlinearities whose growth is at most linear on bounded sets.
The solve $(I-\Bfb)s=f$ and the Hadamard product $z=s\odot g$ take place in the value space $\mathbb R^r$,
while the output projection $o:\mathbb R^r\to\mathbb R^m$ is token-wise affine.
Thus one can always choose $c_\bullet$ and $L_\bullet$ explicitly from the operator norms of the weight matrices involved
and the norms of the biases, by repeated use of Lemma~\ref{lem:affine_bound} and the inequality $\|\mathrm{GELU}(v)\|_2\le \|v\|_2$.
\end{remark}

\section{Polynomial decay of token influence in the feedback recursion}
\label{sec:poly_decay}

\subsection{Scalar recursion and impulse response}

We work on discrete time $t\in\mathbb N=\{0,1,2,\dots\}$.
Let $(\gamma_t)_{t\ge 0}$ be a sequence in $\mathbb R$, and let
$\{\alphafb_{t,j}\}_{t\ge 1,\ 0\le j<t}$ be nonnegative weights such that, for every $t\ge 1$,
\begin{equation}\label{eq:alpha_simplex}
\alphafb_{t,j}\ge 0,\qquad \sum_{j=0}^{t-1}\alphafb_{t,j}\le 1.
\end{equation}
Given an input sequence $(f_t)_{t\ge 0}$, consider the recursion
\begin{equation}\label{eq:feedback_general}
y_0=f_0,\qquad
y_t=f_t+\gamma_t\sum_{j=0}^{t-1}\alphafb_{t,j}y_j,\quad t\ge 1.
\end{equation}

To isolate the influence of a single token, we consider the impulse input at time $0$:
\[
f_0=1,\qquad f_t=0\ \text{for }t\ge 1,
\]
so that \eqref{eq:feedback_general} reduces to the impulse response recursion
\begin{equation}\label{eq:impulse}
\begin{cases}
y_0=1,\\[2mm]
y_t=\gamma_t\displaystyle\sum_{j=0}^{t-1}\alphafb_{t,j}y_j,\quad t\ge 1.
\end{cases}
\end{equation}
In the full vector model, $y_t$ can be interpreted as a scalar influence coefficient, e.g.\ an entry of $(I-\Bfb)^{-1}$.

\subsection{Assumptions}

\begin{assumption}[Upper envelope on attention]\label{ass:alpha_upper_tail}
There exists a constant $c_2\in(0,\infty)$ such that for all $t\ge 1$ and all $0\le j<t$,
\begin{equation}\label{eq:alpha_upper}
\alphafb_{t,j}\le \frac{c_2}{t},
\qquad\text{and \eqref{eq:alpha_simplex} holds.}
\end{equation}
\end{assumption}

\begin{remark}[On the size of $c_2$]
Under \eqref{eq:alpha_simplex} with $\sum_{j=0}^{t-1}\alphafb_{t,j}\le 1$, the conclusion $c_2\ge 1$ no longer follows.
If one additionally has $\sum_{j=0}^{t-1}\alphafb_{t,j}=1$ for all $t$, then $c_2\ge 1$ is necessary.
\end{remark}

\begin{assumption}[Bounded feedback]\label{ass:gamma_bound_tail}
There exists $\gamma_{\max}\in[0,1)$ such that for all $t\ge 0$,
\begin{equation}\label{eq:gamma_bound}
|\gamma_t|\le \gamma_{\max}.
\end{equation}
\end{assumption}

Define the feedback mass parameter
\begin{equation}\label{eq:eta_def}
\eta:=\gamma_{\max}c_2,
\end{equation}
and assume the nontrivial feedback regime
\begin{equation}\label{eq:eta_lt_1}
0<\eta<1.
\end{equation}
Equivalently, define the tail exponent
\begin{equation}\label{eq:betatail_def}
\betatail:=1-\eta=1-\gamma_{\max}c_2\in(0,1],
\end{equation}
so that $\eta=1-\betatail$.

\begin{remark}[Degenerate case $\eta=0$]\label{rem:eta_zero}
If $\eta=0$ then $\gamma_{\max}=0$ and hence $\gamma_t=0$ for all $t$.
The recursion \eqref{eq:impulse} has no feedback and the impulse response is trivial:
$y_0=1$ and $y_t=0$ for all $t\ge 1$.
We therefore focus on $0<\eta<1$ when stating a genuine power-law tail.
\end{remark}

\subsection[Bounded logits imply near-uniform softmax weights]
{Bounded logits imply near-uniform softmax weights}

\paragraph{Bounded logits imply near-uniform softmax weights.}
This is an immediate specialization of Lemma~\ref{lem:bounded_logits}.
Indeed, fix $t\ge 1$ and take the index set $\mathcal I=\{0,\dots,t-1\}$ with $n=|\mathcal I|=t$.
If the logits satisfy $\beth_{\min}\le \beth_{t,j}\le \beth_{\max}$ for all $j\in\mathcal I$, then the spread is
$\Delta_0=\beth_{\max}-\beth_{\min}$, and Lemma~\ref{lem:bounded_logits} gives, for all $j<t$,
\begin{equation}\label{eq:alpha_bounds_logits}
\frac{e^{\beth_{\min}-\beth_{\max}}}{t}\ \le\ \alphafb_{t,j}\ \le\ \frac{e^{\beth_{\max}-\beth_{\min}}}{t}.
\end{equation}
In particular, Assumption~\ref{ass:alpha_upper_tail} holds with $c_2=e^{\beth_{\max}-\beth_{\min}}$.

\subsection{Polynomial decay theorem}

\begin{theorem}[Polynomial decay of the impulse response]\label{thm:poly_decay_tail}
Consider the impulse recursion \eqref{eq:impulse}. Suppose Assumptions~\ref{ass:alpha_upper_tail}
and~\ref{ass:gamma_bound_tail} hold and $0<\eta=\gamma_{\max}c_2<1$; equivalently $\betatail=1-\eta\in(0,1)$.
Then for all $t\ge 1$,
\begin{equation}\label{eq:poly_decay_bound}
|y_t|\le C\,t^{-\betatail},
\qquad\text{where one may take}\qquad
C:=(1-\betatail)\,e^{\,1-\betatail}=\eta e^{\eta}.
\end{equation}
In particular, since $\betatail>0$, we have $\lim_{t\to\infty}y_t=0$.
\end{theorem}

\begin{proof}
Assume $0<\eta<1$. The degenerate case $\eta=0$ is covered by Remark~\ref{rem:eta_zero}.
Let $z_t:=|y_t|$. From \eqref{eq:impulse} and Assumptions~\ref{ass:alpha_upper_tail}--\ref{ass:gamma_bound_tail},
for $t\ge 1$,
\[
z_t
=\Bigl|\gamma_t\sum_{j=0}^{t-1}\alphafb_{t,j}y_j\Bigr|
\le |\gamma_t|\sum_{j=0}^{t-1}\alphafb_{t,j}|y_j|
\le \gamma_{\max}\sum_{j=0}^{t-1}\alphafb_{t,j}z_j.
\]
Define the comparison sequence $(\tilde y_t)_{t\ge 0}$ by
\begin{equation}\label{eq:tilde_def}
\tilde y_0=1,\qquad
\tilde y_t=\gamma_{\max}\sum_{j=0}^{t-1}\alphafb_{t,j}\tilde y_j,\quad t\ge 1.
\end{equation}
By induction on $t$, using $\alphafb_{t,j}\ge 0$, we have $z_t\le \tilde y_t$ for all $t$, hence
\begin{equation}\label{eq:zt_le_tilde}
|y_t|=z_t\le \tilde y_t\qquad\forall t.
\end{equation}

Let $s_t:=\sum_{k=0}^t\tilde y_k$. Since $\tilde y_k\ge 0$, the sequence $s_t$ is increasing and $s_t\ge 1$.
Using \eqref{eq:tilde_def} and $\alphafb_{t,j}\le c_2/t$ we obtain, for $t\ge 1$,
\[
\tilde y_t
\;=\gamma_{\max}\sum_{j=0}^{t-1}\alphafb_{t,j}\tilde y_j
\le \gamma_{\max}\sum_{j=0}^{t-1}\frac{c_2}{t}\tilde y_j
=\frac{\eta}{t}\,s_{t-1}.
\]
Therefore,
\begin{equation}\label{eq:s_rec}
s_t=s_{t-1}+\tilde y_t
\le s_{t-1}+\frac{\eta}{t}s_{t-1}
=s_{t-1}\Bigl(1+\frac{\eta}{t}\Bigr),\qquad t\ge 1.
\end{equation}
Taking logarithms and using $\log(1+x)\le x$ for $x>-1$,
\[
\log s_n
\le \log s_0 + \sum_{t=1}^n\log\Bigl(1+\frac{\eta}{t}\Bigr)
\le \sum_{t=1}^n\frac{\eta}{t}
=\eta H_n,
\]
where $H_n=\sum_{t=1}^n\frac{1}{t}$ is the $n$-th harmonic number. 
Using $H_n\le 1+\log n$ for $n\ge 1$ gives
\begin{equation}\label{eq:s_bound_poly}
s_n\le e^{\eta}n^{\eta}\qquad\forall n\ge 1.
\end{equation}

Finally, for $t\ge 1$ we use $s_{t-1}\le s_t$ and \eqref{eq:s_bound_poly}:
\[
\tilde y_t\le \frac{\eta}{t}s_{t-1}\le \frac{\eta}{t}s_t
\le \frac{\eta}{t}\,e^{\eta}t^{\eta}
=\eta e^{\eta}\,t^{\eta-1}.
\]
Since $\eta-1=-(1-\eta)=-\betatail$, we obtain $\tilde y_t\le \eta e^{\eta}\,t^{-\betatail}$.
Combining with \eqref{eq:zt_le_tilde} yields \eqref{eq:poly_decay_bound} with $C=\eta e^{\eta}$.
\end{proof}

\subsection{Finite-horizon formulation}

\begin{corollary}[Finite-horizon bound]\label{cor:finite_horizon_tail}
Fix $\Tctx\in\mathbb N^*$ and consider \eqref{eq:impulse} only for $t\in\{0,1,\dots,\Tctx-1\}$.
Assume that Assumptions~\ref{ass:alpha_upper_tail} and~\ref{ass:gamma_bound_tail} hold for all $1\le t\le \Tctx-1$
with the same constants $c_2$ and $\gamma_{\max}$, and $0<\eta=\gamma_{\max}c_2<1$; equivalently $\betatail=1-\eta\in(0,1)$.
Then \eqref{eq:poly_decay_bound} holds for all $t\in\{1,\dots,\Tctx-1\}$ with the same constant $C=\eta e^{\eta}=(1-\betatail)e^{\,1-\betatail}$.
\end{corollary}

\begin{proof}
This is an immediate restriction of Theorem~\ref{thm:poly_decay_tail} to $1\le t\le \Tctx-1$.
\end{proof}

\subsection{Impulse at an arbitrary position $j$}

\begin{corollary}[Decay from an impulse at position $j$]\label{cor:impulse_j_tail}
Fix an index $j\ge 0$. Consider \eqref{eq:feedback_general} with the impulse input at $j$:
\[
f_j=1,\qquad f_t=0\ \text{for }t\neq j,
\]
and with $y_t=0$ for $t<j$. Equivalently, $y_0=0$ if $j>0$ and the recursion is started from $t=j$.
Assume Assumptions~\ref{ass:alpha_upper_tail}--\ref{ass:gamma_bound_tail} and $0<\eta=\gamma_{\max}c_2<1$; equivalently $\betatail=1-\eta\in(0,1)$.
Then for all $t>j$,
\[
|y_t|\le C\,(t-j)^{-\betatail},
\qquad\text{where one may take}\qquad
C:=\eta e^{\eta}=(1-\betatail)e^{\,1-\betatail}.
\]
\end{corollary}

\begin{proof}
Define $u_n:=|y_{j+n}|$ for $n\ge 0$. Then $u_0=|y_j|=1$.
For $n\ge 1$, since $y_k=0$ for $k<j$ and $\alphafb_{j+n,k}\ge 0$,
\[
u_n
=|y_{j+n}|
=\Bigl|\gamma_{j+n}\sum_{k=0}^{j+n-1}\alphafb_{j+n,k}y_k\Bigr|
\le |\gamma_{j+n}|\sum_{k=j}^{j+n-1}\alphafb_{j+n,k}|y_k|
\le \gamma_{\max}\sum_{r=0}^{n-1}\alphafb_{j+n,j+r}\,u_r.
\]
Moreover, by Assumption~\ref{ass:alpha_upper_tail},
\[
\alphafb_{j+n,j+r}\le \frac{c_2}{j+n}\le \frac{c_2}{n}\qquad(n\ge 1),
\]
since $j+n\ge n$. 
Thus the sequence $u_n$ satisfies the same comparison inequality as in the proof of
Theorem~\ref{thm:poly_decay_tail}, with the same $\gamma_{\max}$ and the envelope $c_2/n$,
so repeating that argument yields
\[
u_n\le \eta e^{\eta}\,n^{\eta-1} = (\eta e^{\eta})\,n^{-\betatail}.
\]
Substituting $n=t-j$ yields the claim.
\end{proof}

\section[Tightness of the polynomial tail in a realizable regime]
{Tightness of the polynomial tail in a realizable regime}
\label{app:tightness_powerlaw}

This section complements Theorem~\ref{thm:poly_decay_tail} with the upper bound $O(\ell^{-\betatail})$
by exhibiting a concrete diffuse routing regime in which the impulse influence is
exactly polynomial, that is, $\Theta(\ell^{-\betatail})$.
This eliminates the semantic ambiguity that an upper bound alone does not preclude
faster decay, for instance exponential decay.

\subsection{Gamma-ratio inequality of Gautschi}
\label{app:gautschi}

\begin{lemma}[Gautschi inequality for $0<\gamma<1$]
\label{lem:gautschi}
Let $\gamma\in(0,1)$ and $t\ge 1$ be an integer. Then
\begin{equation}
(t+1)^{\gamma-1}
\;\le\;
\frac{\Gamma(t+\gamma)}{\Gamma(t+1)}
\;\le\;
t^{\gamma-1}.
\label{eq:gautschi_ratio}
\end{equation}
Equivalently, with $\betatail:=1-\gamma\in(0,1)$,
\[
(t+1)^{-\betatail}\le \frac{\Gamma(t+\gamma)}{\Gamma(t+1)}\le t^{-\betatail}.
\]
\end{lemma}

\begin{proof}
By Gautschi's inequality \citep{gautschi1959}, for $x>0$ and $0<\gamma<1$,
\[
x^{1-\gamma} < \frac{\Gamma(x+1)}{\Gamma(x+\gamma)} < (x+1)^{1-\gamma}.
\]
Setting $x=t$ and taking reciprocals yields
\[
(t+1)^{\gamma-1} \le \frac{\Gamma(t+\gamma)}{\Gamma(t+1)} \le t^{\gamma-1},
\]
which is \eqref{eq:gautschi_ratio}.
\end{proof}

\subsection{Uniform routing yields a \texorpdfstring{$\Theta(\ell^{-\betatail})$}{Theta(l^{-betatail})} tail}
\label{app:uniform_routing_theta}

We consider the scalar impulse recursion from Section~\ref{sec:poly_decay}:
\begin{equation}
y_0=f_0,\qquad
y_t=f_t+\gamma_t\sum_{j=0}^{t-1}\alphafb_{t,j}y_j,\quad t\ge 1.
\label{eq:app_scalar_rec}
\end{equation}

\begin{proposition}[Tightness under uniform routing]
\label{prop:theta_tail_uniform}
Assume uniform routing, which is maximally diffuse, and constant positive feedback:
\[
\alphafb_{t,j}=\frac{1}{t}\mathbf 1[j<t],\qquad \gamma_t\equiv \gamma\in(0,1).
\]
Consider an impulse at time $0$: $f_0=1$ and $f_t=0$ for all $t\ge 1$.
Then for every $t\ge 1$ the impulse influence admits the closed form
\begin{equation}
y_t
=
\frac{\gamma}{\Gamma(1+\gamma)}\cdot \frac{\Gamma(t+\gamma)}{\Gamma(t+1)}.
\label{eq:yt_closed_form}
\end{equation}
Consequently, letting $\betatail:=1-\gamma\in(0,1)$, one has the two-sided bound
\begin{equation}
\frac{\gamma}{\Gamma(1+\gamma)}\,(t+1)^{-\betatail}
\;\le\;
y_t
\;\le\;
\frac{\gamma}{\Gamma(1+\gamma)}\,t^{-\betatail},
\qquad t\ge 1,
\label{eq:yt_theta}
\end{equation}
and in particular
\[
y_t=\Theta(t^{-\betatail})\quad\text{and hence}\quad y_t=\Omega(t^{-\betatail}).
\]
\end{proposition}

\begin{proof}
Define partial sums $S_t:=\sum_{k=0}^t y_k$. Under the stated assumptions and for $t\ge 1$,
\[
y_t=\frac{\gamma}{t}\sum_{j=0}^{t-1}y_j=\frac{\gamma}{t}S_{t-1},
\qquad
S_t=S_{t-1}+y_t=S_{t-1}\Big(1+\frac{\gamma}{t}\Big),
\]
with $S_0=y_0=f_0=1$. Thus
\[
S_t=\prod_{i=1}^t\Big(1+\frac{\gamma}{i}\Big)=\prod_{i=1}^t\frac{i+\gamma}{i}
=\frac{\Gamma(t+1+\gamma)}{\Gamma(1+\gamma)\Gamma(t+1)}.
\]
Using $y_t=\frac{\gamma}{t}S_{t-1}$ and $\Gamma(t+1)=t\,\Gamma(t)$ gives
\[
y_t=\frac{\gamma}{t}\cdot \frac{\Gamma(t+\gamma)}{\Gamma(1+\gamma)\Gamma(t)}
=\frac{\gamma}{\Gamma(1+\gamma)}\cdot \frac{\Gamma(t+\gamma)}{\Gamma(t+1)},
\]
which is \eqref{eq:yt_closed_form}. The two-sided bound \eqref{eq:yt_theta} follows directly
from Lemma~\ref{lem:gautschi} with $\betatail=1-\gamma$.
\end{proof}

\begin{corollary}[Uniform routing with an impulse at an arbitrary source position]
\label{cor:theta_tail_uniform_shifted}
Assume the same explicit uniform-routing regime as in Proposition~\ref{prop:theta_tail_uniform}:
\[
\alphafb_{t,j}=\frac{1}{t}\mathbf 1[j<t],\qquad \gamma_t\equiv \gamma\in(0,1).
\]
Consider an impulse at time $\tau\ge 0$, i.e.
\[
f_\tau=1,\qquad f_t=0\ \text{for }t\neq \tau,
\]
with $y_t=0$ for $t<\tau$. Then for every $\ell\ge 1$,
\begin{equation}
y_{\tau+\ell}
=
\gamma\,\frac{\Gamma(\tau+1)}{\Gamma(\tau+1+\gamma)}
\cdot
\frac{\Gamma(\tau+\ell+\gamma)}{\Gamma(\tau+\ell+1)}.
\label{eq:shifted_uniform_routing_exact}
\end{equation}
Consequently, with $\betatail:=1-\gamma\in(0,1)$, for every fixed source position $\tau$,
\[
y_{\tau+\ell}=\Theta_\tau(\ell^{-\betatail})
\qquad (\ell\to\infty).
\]
Moreover, the prefactor depends on $\tau$ and satisfies
\[
\gamma\,\frac{\Gamma(\tau+1)}{\Gamma(\tau+1+\gamma)}
\asymp \tau^{-\gamma}
\qquad (\tau\to\infty).
\]
In particular, there is no positive lower constant $c_->0$ such that
\[
y_{\tau+\ell}\ge c_-\,\ell^{-\betatail}
\]
for all source positions $\tau$ and all $\ell\ge 1$ on an unbounded horizon.
\end{corollary}

\begin{proof}
Define partial sums
\[
S_t:=\sum_{k=\tau}^{t}y_k,\qquad t\ge \tau.
\]
Then $S_\tau=y_\tau=1$. For $t\ge \tau+1$, the recursion gives
\[
y_t=\frac{\gamma}{t}\sum_{j=\tau}^{t-1}y_j=\frac{\gamma}{t}S_{t-1},
\qquad
S_t=S_{t-1}+y_t=S_{t-1}\Bigl(1+\frac{\gamma}{t}\Bigr).
\]
Hence, for every $t\ge \tau+1$,
\[
S_t
=
\prod_{i=\tau+1}^{t}\Bigl(1+\frac{\gamma}{i}\Bigr)
=
\prod_{i=\tau+1}^{t}\frac{i+\gamma}{i}
=
\frac{\Gamma(t+1+\gamma)\Gamma(\tau+1)}{\Gamma(\tau+1+\gamma)\Gamma(t+1)}.
\]
Using $y_t=\frac{\gamma}{t}S_{t-1}$ and $\Gamma(t+1)=t\,\Gamma(t)$ yields
\[
y_t
=
\frac{\gamma}{t}\cdot
\frac{\Gamma(t+\gamma)\Gamma(\tau+1)}{\Gamma(\tau+1+\gamma)\Gamma(t)}
=
\gamma\,\frac{\Gamma(\tau+1)}{\Gamma(\tau+1+\gamma)}
\cdot
\frac{\Gamma(t+\gamma)}{\Gamma(t+1)}.
\]
Setting $t=\tau+\ell$ gives \eqref{eq:shifted_uniform_routing_exact}.

For fixed $\tau$, the factor
\[
\gamma\,\frac{\Gamma(\tau+1)}{\Gamma(\tau+1+\gamma)}
\]
is a positive constant depending only on $\tau$, while Lemma~\ref{lem:gautschi} gives
\[
(\tau+\ell+1)^{-\betatail}
\le
\frac{\Gamma(\tau+\ell+\gamma)}{\Gamma(\tau+\ell+1)}
\le
(\tau+\ell)^{-\betatail}.
\]
Since $\tau$ is fixed, this implies
\[
\frac{\Gamma(\tau+\ell+\gamma)}{\Gamma(\tau+\ell+1)}
=
\Theta_\tau(\ell^{-\betatail}),
\]
hence $y_{\tau+\ell}=\Theta_\tau(\ell^{-\betatail})$.

The Gamma-ratio asymptotic gives
\[
\frac{\Gamma(\tau+1)}{\Gamma(\tau+1+\gamma)}\asymp (\tau+1)^{-\gamma}
\qquad(\tau\to\infty),
\]
so the source-dependent prefactor decays polynomially with $\tau$.

Moreover, taking $\ell=1$ in \eqref{eq:shifted_uniform_routing_exact} gives
\[
y_{\tau+1}
=
\gamma\,\frac{\Gamma(\tau+1)}{\Gamma(\tau+1+\gamma)}
\cdot
\frac{\Gamma(\tau+1+\gamma)}{\Gamma(\tau+2)}
=
\frac{\gamma}{\tau+1}.
\]
Hence $y_{\tau+1}\to 0$ as $\tau\to\infty$. Therefore no positive lower constant
independent of $\tau$ can satisfy
\[
y_{\tau+\ell}\ge c_-\,\ell^{-\betatail}
\]
for all source positions $\tau$ and all $\ell\ge 1$ on an unbounded horizon.
\end{proof}

\begin{corollary}[Uniform two-sided heavy-tail envelope on a bounded source family]
\label{cor:theta_tail_uniform_bounded_source}
Fix $\tau_{\max}\in\mathbb N$. Under the regime of
Corollary~\ref{cor:theta_tail_uniform_shifted}, there exist constants
$c^-_{\tau_{\max},\gamma},c^+_{\tau_{\max},\gamma}>0$ such that for every
source position $0\le \tau\le \tau_{\max}$ and every $\ell\ge 1$,
\[
c^-_{\tau_{\max},\gamma}\,\ell^{-\betatail}
\le
y_{\tau+\ell}
\le
c^+_{\tau_{\max},\gamma}\,\ell^{-\betatail},
\qquad \betatail:=1-\gamma.
\]
In particular, the explicit uniform-routing regime realizes a uniform two-sided
heavy-tail envelope on every bounded source family, and hence on every fixed finite horizon.
\end{corollary}

\begin{proof}
Write
\[
a_\tau:=\gamma\,\frac{\Gamma(\tau+1)}{\Gamma(\tau+1+\gamma)}.
\]
Since the set $\{0,\dots,\tau_{\max}\}$ is finite and each $a_\tau$ is positive,
\[
m_{\tau_{\max},\gamma}:=\min_{0\le \tau\le \tau_{\max}} a_\tau>0,
\qquad
M_{\tau_{\max},\gamma}:=\max_{0\le \tau\le \tau_{\max}} a_\tau<\infty.
\]
By Corollary~\ref{cor:theta_tail_uniform_shifted},
\[
y_{\tau+\ell}=a_\tau\,\frac{\Gamma(\tau+\ell+\gamma)}{\Gamma(\tau+\ell+1)}.
\]
Lemma~\ref{lem:gautschi} yields
\[
(\tau+\ell+1)^{-\betatail}
\le
\frac{\Gamma(\tau+\ell+\gamma)}{\Gamma(\tau+\ell+1)}
\le
(\tau+\ell)^{-\betatail}.
\]
Therefore
\[
y_{\tau+\ell}\le M_{\tau_{\max},\gamma}\,(\tau+\ell)^{-\betatail}\le M_{\tau_{\max},\gamma}\,\ell^{-\betatail}.
\]
Also, since $0\le \tau\le \tau_{\max}$ and $\ell\ge 1$,
\[
\tau+\ell+1\le \tau_{\max}+\ell+1\le (\tau_{\max}+2)\ell,
\]
hence
\[
(\tau+\ell+1)^{-\betatail}\ge (\tau_{\max}+2)^{-\betatail}\,\ell^{-\betatail}.
\]
Thus
\[
y_{\tau+\ell}
\ge
m_{\tau_{\max},\gamma}\,(\tau+\ell+1)^{-\betatail}
\ge
m_{\tau_{\max},\gamma}\,(\tau_{\max}+2)^{-\betatail}\,\ell^{-\betatail}.
\]
So one may take
\[
c^-_{\tau_{\max},\gamma}:=
m_{\tau_{\max},\gamma}\,(\tau_{\max}+2)^{-\betatail},
\qquad
c^+_{\tau_{\max},\gamma}:=
M_{\tau_{\max},\gamma}.
\]
\end{proof}

\paragraph{Consequence for the influence kernel}
In the lower-triangular solve $s=Kf$ with $K=(I-\Bfb)^{-1}$, choosing
\[
\BfbEnt{t}{j}=\gamma\,\alphafb_{t,j}
=
\begin{cases}
0, & t=0,\\[2pt]
\dfrac{\gamma}{t}\,\mathbf 1[j<t], & t\ge 1,
\end{cases}
\]
yields that the column $K_{\cdot,0}$ is precisely the impulse response $(y_t)_{t\ge 0}$ above.
Hence,
\[
|K_{t,0}|=\Theta(t^{-\betatail}),
\]
so the polynomial envelope in Theorem~\ref{thm:poly_decay_main} is sharp, and the rate is attained by a concrete heavy-tailed memory mode.

\begin{remark}[Impulse at time $\tau$]
Assume $\gamma\in(0,1)$.
The same computation applies to an impulse at time $\tau$. If $f_\tau=1$, $f_t=0$ for $t\neq\tau$, and $y_t=0$ for $t<\tau$, then for $t\ge \tau+1$
\[
y_t
=
\gamma\,
\frac{\Gamma(t+\gamma)\Gamma(\tau+1)}{\Gamma(t+1)\Gamma(\tau+1+\gamma)}
=
C(\tau,\gamma)\cdot \frac{\Gamma(t+\gamma)}{\Gamma(t+1)},
\]
with $C(\tau,\gamma):=\gamma\,\Gamma(\tau+1)/\Gamma(\tau+1+\gamma)>0$.
Hence, for $\ell=t-\tau$, the lag-$\ell$ tail is again $\Theta(\ell^{-\betatail})$ by Lemma~\ref{lem:gautschi}, in agreement with Corollary~\ref{cor:impulse_j_tail}.
\end{remark}

\section{Heavy-tail convolution estimates}
\label{app:aux_heavytail_convolution}
\begin{definition}[Discrete convolution on positive lags]
\label{def:positive_lag_convolution}
For nonnegative sequences $a,b:\mathbb N^*\to[0,\infty)$, define
\[
(a*b)(n):=\sum_{m=1}^{n-1}a(n-m)b(m),\qquad n\ge 2,
\]
and $(a*b)(1):=0$.  
Inductively define $a^{(*1)}:=a$ and $a^{(*k)}:=a^{(*(k-1))}*a$ for $k\ge 2$.
\end{definition}

\begin{lemma}[Discrete power convolution]
\label{lem:discrete_power_convolution}
Let $\sigma,\rho>0$, and define
\[
u_\sigma(n):=n^{\sigma-1},\qquad u_\rho(n):=n^{\rho-1},\qquad n\in\mathbb N^*.
\]
Then there exist constants $c_{\sigma,\rho},C_{\sigma,\rho}\in(0,\infty)$ such that
\[
c_{\sigma,\rho}\,n^{\sigma+\rho-1}
\le
(u_\sigma*u_\rho)(n)
\le
C_{\sigma,\rho}\,n^{\sigma+\rho-1},
\qquad n\ge 2.
\]
\end{lemma}
\begin{proof}
Fix $n\ge 2$.

For the upper bound, split the sum into the two regions
\[
1\le m\le \left\lfloor \frac n2\right\rfloor
\qquad\text{and}\qquad
\left\lfloor \frac n2\right\rfloor+1\le m\le n-1.
\]

If $1\le m\le n/2$, then $n-m\in[n/2,n-1]$, hence
\[
(n-m)^{\sigma-1}\le C_\sigma\,n^{\sigma-1},
\qquad
C_\sigma:=\max\{1,2^{1-\sigma}\}.
\]
Therefore
\[
\sum_{m=1}^{\lfloor n/2\rfloor}(n-m)^{\sigma-1}m^{\rho-1}
\le
C_\sigma n^{\sigma-1}\sum_{m=1}^{\lfloor n/2\rfloor}m^{\rho-1}.
\]
Since $\rho>0$, the standard integral comparison gives
\[
\sum_{m=1}^{\lfloor n/2\rfloor}m^{\rho-1}
\le
1+\int_{1}^{n/2}x^{\rho-1}\,dx
\le C'_\rho\,n^\rho
\]
for some constant $C'_\rho$ depending only on $\rho$. Hence
\[
\sum_{m=1}^{\lfloor n/2\rfloor}(n-m)^{\sigma-1}m^{\rho-1}
\le
C_\sigma C'_\rho\,n^{\sigma+\rho-1}.
\]

If $\lfloor n/2\rfloor+1\le m\le n-1$, then $m\in[n/2,n-1]$, hence
\[
m^{\rho-1}\le C_\rho\,n^{\rho-1},
\qquad
C_\rho:=\max\{1,2^{1-\rho}\}.
\]
Therefore
\[
\sum_{m=\lfloor n/2\rfloor+1}^{n-1}(n-m)^{\sigma-1}m^{\rho-1}
\le
C_\rho n^{\rho-1}\sum_{m=\lfloor n/2\rfloor+1}^{n-1}(n-m)^{\sigma-1}.
\]
After the change of variable $r=n-m$, the inner sum becomes
\[
\sum_{r=1}^{\lceil n/2\rceil-1} r^{\sigma-1}\le C'_\sigma n^\sigma
\]
for some constant $C'_\sigma$ depending only on $\sigma$. Hence
\[
\sum_{m=\lfloor n/2\rfloor+1}^{n-1}(n-m)^{\sigma-1}m^{\rho-1}
\le
C_\rho C'_\sigma\,n^{\sigma+\rho-1}.
\]
Adding the two estimates proves the upper bound.

For the lower bound, restrict the sum to the central block
\[
\left\lfloor \frac n4\right\rfloor \le m\le \left\lfloor \frac{3n}{4}\right\rfloor.
\]
For every such $m$ and every $n\ge 4$ one has
\[
\frac n4\le m\le \frac{3n}{4},
\qquad
\frac n4\le n-m\le \frac{3n}{4}.
\]
Hence
\[
m^{\rho-1}\ge c_\rho\,n^{\rho-1},
\qquad
(n-m)^{\sigma-1}\ge c_\sigma\,n^{\sigma-1},
\]
where one may take
\[
c_\rho:=\min\{1,4^{1-\rho}\},
\qquad
c_\sigma:=\min\{1,4^{1-\sigma}\}.
\]
Indeed, if $\rho\le 1$, then $m\le n$ implies $m^{\rho-1}\ge n^{\rho-1}$; if $\rho\ge 1$, then
$m\ge n/4$ implies $m^{\rho-1}\ge 4^{1-\rho}n^{\rho-1}$.
The same argument applies to $(n-m)^{\sigma-1}$.

Therefore every summand in the central block is bounded below by
\[
c_\sigma c_\rho\,n^{\sigma+\rho-2}.
\]
The number of integers in the central block is at least $n/2-2$. Consequently, for all $n\ge 8$,
\[
(u_\sigma*u_\rho)(n)
\ge
\left(\frac n2-2\right)c_\sigma c_\rho\,n^{\sigma+\rho-2}
\ge
\frac{c_\sigma c_\rho}{4}\,n^{\sigma+\rho-1}.
\]
Since only finitely many values $2\le n<8$ remain, their minimum ratio to
$n^{\sigma+\rho-1}$ is positive. Adjusting the constant completes the proof.
\end{proof}

\begin{theorem}[Heavy-tail convolution class]
\label{thm:heavy_tail_convolution_class}
Fix $\betatail\in(0,1)$ and define
\[
f_{\betatail}(n):=n^{-\betatail},\qquad n\in\mathbb N^*.
\]
Then, for every fixed $k\ge 1$, there exist constants
$c_{k,\betatail},C_{k,\betatail}\in(0,\infty)$ such that
\begin{equation}
c_{k,\betatail}\,n^{k(1-\betatail)-1}
\le
f_{\betatail}^{(*k)}(n)
\le
C_{k,\betatail}\,n^{k(1-\betatail)-1},
\qquad n\ge k.
\label{eq:heavy_tail_convolution_class}
\end{equation}
\end{theorem}
\begin{proof}
Set
\[
\sigma:=1-\betatail\in(0,1).
\]
Then
\[
f_{\betatail}(n)=n^{-\betatail}=n^{\sigma-1}=u_\sigma(n).
\]
We prove by induction on $k$ that there exist constants
$a_k,b_k>0$ such that
\begin{equation}
a_k\,n^{k\sigma-1}\le u_\sigma^{(*k)}(n)\le b_k\,n^{k\sigma-1},
\qquad n\ge k.
\label{eq:heavy_tail_induction_claim}
\end{equation}

For $k=1$, this is exactly
\[
u_\sigma(n)=n^{\sigma-1}.
\]

Assume now that \eqref{eq:heavy_tail_induction_claim} holds for some $k\ge 1$.

Fix $n\ge k+1$. By definition,
\[
u_\sigma^{(*(k+1))}(n)
=
\sum_{m=1}^{n-1}u_\sigma^{(*k)}(n-m)u_\sigma(m).
\]

For the upper bound, note that $u_\sigma^{(*k)}(r)=0$ for $r<k$, since it is a
$k$-fold convolution of positive-lag sequences. Hence, after enlarging $b_k$
if necessary, we may write
\[
u_\sigma^{(*k)}(r)\le b_k\,r^{k\sigma-1}
\qquad\text{for every }r\ge 1.
\]
Therefore
\[
u_\sigma^{(*(k+1))}(n)
\le
b_k\sum_{m=1}^{n-1}(n-m)^{k\sigma-1}m^{\sigma-1}.
\]
Applying Lemma~\ref{lem:discrete_power_convolution} with exponents
$k\sigma$ and $\sigma$ yields
\[
u_\sigma^{(*(k+1))}(n)\le b_{k+1}\,n^{(k+1)\sigma-1}
\]
for some constant $b_{k+1}>0$.

For the lower bound, rewrite the sum using $r:=n-m$:
\[
u_\sigma^{(*(k+1))}(n)
=
\sum_{r=1}^{n-1}u_\sigma^{(*k)}(r)\,u_\sigma(n-r).
\]
Since $u_\sigma^{(*k)}(r)=0$ for $r<k$, this becomes
\[
u_\sigma^{(*(k+1))}(n)
=
\sum_{r=k}^{n-1}u_\sigma^{(*k)}(r)\,(n-r)^{\sigma-1}.
\]
Applying the lower induction hypothesis on the range $r\ge k$ gives
\[
u_\sigma^{(*(k+1))}(n)
\ge
a_k\sum_{r=k}^{n-1} r^{k\sigma-1}(n-r)^{\sigma-1}.
\]
Now write
\[
\sum_{r=k}^{n-1} r^{k\sigma-1}(n-r)^{\sigma-1}
=
\sum_{r=1}^{n-1} r^{k\sigma-1}(n-r)^{\sigma-1}
-
\sum_{r=1}^{k-1} r^{k\sigma-1}(n-r)^{\sigma-1}.
\]
By Lemma~\ref{lem:discrete_power_convolution}, the full sum is bounded below by
\[
c\,n^{(k+1)\sigma-1}
\]
for some constant $c>0$ depending only on $k$ and $\sigma$.

On the other hand, since $k-1$ is fixed,
\[
\sum_{r=1}^{k-1} r^{k\sigma-1}(n-r)^{\sigma-1}
\le
C\,n^{\sigma-1}
\]
for some constant $C>0$ depending only on $k$ and $\sigma$.
Because $k\sigma>0$, one has
\[
n^{\sigma-1}=o\!\left(n^{(k+1)\sigma-1}\right)
\qquad\text{as }n\to\infty.
\]
Hence there exist constants $c'>0$ and $N_k$ such that, for all $n\ge N_k$,
\[
\sum_{r=k}^{n-1} r^{k\sigma-1}(n-r)^{\sigma-1}
\ge
c'\,n^{(k+1)\sigma-1}.
\]
Therefore, for all $n\ge N_k$,
\[
u_\sigma^{(*(k+1))}(n)\ge a_k c'\,n^{(k+1)\sigma-1}.
\]

It remains to treat the finitely many values $k+1\le n<N_k$.
For each such $n$, one has $u_\sigma^{(*(k+1))}(n)>0$ because $n$ can be written
as a sum of $k+1$ positive integers. Hence the ratio
\[
\frac{u_\sigma^{(*(k+1))}(n)}{n^{(k+1)\sigma-1}}
\]
is positive for each of those finitely many $n$. Taking the minimum of these
finitely many positive ratios and $a_k c'$ gives a constant $a_{k+1}>0$ such that
\[
u_\sigma^{(*(k+1))}(n)\ge a_{k+1}\,n^{(k+1)\sigma-1}
\qquad\text{for all }n\ge k+1.
\]
This closes the induction.

Since $f_{\betatail}=u_\sigma$ with $\sigma=1-\betatail$, we obtain
\[
f_{\betatail}^{(*k)}(n)\asymp n^{k(1-\betatail)-1},
\qquad n\ge k.
\]
This is \eqref{eq:heavy_tail_convolution_class}
\end{proof}

\section{Deep Jacobian estimates}
\label{app:deep_e2e_calculus}

\subsection{Setup}

Fix a depth $\Nlayer\ge 1$, a finite horizon $\Tctx$, and a compact input set $\mathcal X_0$.
Let
\[
h^{(0)}=x\in\mathcal X_0,\qquad
h^{(\idxlayer)} = F_{\idxlayer}\bigl(h^{(\idxlayer-1)}\bigr),
\qquad \idxlayer=1,\dots,\Nlayer,
\]
where each $F_{\idxlayer}$ is causal and continuously differentiable on the relevant compact set
\[
\mathcal X_{\idxlayer-1}
:=
F_{\idxlayer-1}\circ\cdots\circ F_1(\mathcal X_0).
\]
For each layer $\idxlayer$ and each $0\le \tau\le t\le \Tctx-1$, define the one-block Jacobian block
\[
J^{(\idxlayer)}_{t,\tau}(u)
:=
\frac{\partial F_{\idxlayer,t}(u)}{\partial u_\tau}
\in\mathbb R^{D\times D},
\qquad u\in\mathcal X_{\idxlayer-1}.
\]
Define also the full end-to-end Jacobian blocks
\[
J^{\mathrm{e2e},(\Nlayer)}_{t,\tau}(x)
:=
\frac{\partial h_t^{(\Nlayer)}(x)}{\partial h_\tau^{(0)}(x)}
\in\mathbb R^{D\times D}.
\]

For scalar lower-triangular kernels $\mathcal A,\mathcal B$ on
\[
\{(t,\tau): 0\le \tau\le t\le \Tctx-1\},
\]
we use the standard kernel product
\[
(\mathcal A\mathcal B)(t,\tau)
:=
\sum_{j=\tau}^{t}\mathcal A(t,j)\mathcal B(j,\tau).
\]

\subsection{Residual calculus}

\begin{theorem}[Residual calculus]
\label{thm:app_e2e_residual_calculus}
Assume that for each layer $\idxlayer$ there exist constants
\[
d_{\idxlayer}\ge 0,\qquad \lambda_{\idxlayer}\ge 0,
\]
and a scalar lower-triangular kernel
\[
K_{\idxlayer}:\{(t,\tau):0\le \tau<t\le \Tctx-1\}\to [0,\infty)
\]
such that for every $u\in\mathcal X_{\idxlayer-1}$ and every $0\le \tau\le t\le \Tctx-1$,
\begin{equation}
\|J^{(\idxlayer)}_{t,\tau}(u)\|
\le
d_{\idxlayer}\,\mathbf 1[t=\tau]
+
\lambda_{\idxlayer}\,K_{\idxlayer}(t,\tau)\,\mathbf 1[\tau<t].
\label{eq:app_one_block_e2e_envelope}
\end{equation}
Then, for every $x\in\mathcal X_0$, every $0\le \tau<t\le \Tctx-1$, and every depth $\Nlayer\ge 1$,
\begin{align}
\|J^{\mathrm{e2e},(\Nlayer)}_{t,\tau}(x)\|
&\le
\sum_{k=1}^{\Nlayer}
\ \sum_{1\le \idxlayeri{1}<\cdots<\idxlayeri{k}\le \Nlayer}
\left(\prod_{m\notin\{\idxlayeri{1},\dots,\idxlayeri{k}\}} d_m\right)
\notag\\
&\qquad\qquad\qquad\qquad\cdot
\sum_{\tau=i_0<i_1<\cdots<i_k=t}
\ \prod_{r=1}^{k}
\lambda_{\idxlayeri{r}}\,K_{\idxlayeri{r}}(i_r,i_{r-1}).
\label{eq:app_full_e2e_path_sum}
\end{align}
Moreover, for the diagonal blocks one has
\[
\|J^{\mathrm{e2e},(\Nlayer)}_{t,t}(x)\|
\le
\prod_{\idxlayer=1}^{\Nlayer} d_{\idxlayer}.
\]
\end{theorem}

\begin{proof}
For each layer $\idxlayer$, define the scalar diagonal kernel
\[
\mathcal D_{\idxlayer}(t,\tau):=
d_{\idxlayer}\,\mathbf 1[t=\tau],
\]
and the scalar strictly lower-triangular kernel
\[
\mathcal G_{\idxlayer}(t,\tau):=
\lambda_{\idxlayer}K_{\idxlayer}(t,\tau)\,\mathbf 1[\tau<t].
\]
Then \eqref{eq:app_one_block_e2e_envelope} says precisely that
\[
\|J^{(\idxlayer)}_{t,\tau}(u)\|
\le
\mathcal D_{\idxlayer}(t,\tau)+\mathcal G_{\idxlayer}(t,\tau)
\qquad
\forall u\in\mathcal X_{\idxlayer-1}.
\]

We prove by induction on the depth $p\in\{1,\dots,\Nlayer\}$ that
\begin{equation}
\left\|
\frac{\partial h_t^{(p)}(x)}{\partial h_\tau^{(0)}(x)}
\right\|
\le
\bigl[(\mathcal D_p+\mathcal G_p)\cdots(\mathcal D_1+\mathcal G_1)\bigr](t,\tau)
\qquad (0\le \tau\le t\le \Tctx-1).
\label{eq:app_inductive_e2e_bound}
\end{equation}

For $p=1$, \eqref{eq:app_inductive_e2e_bound} is exactly
\eqref{eq:app_one_block_e2e_envelope} evaluated at $u=x\in\mathcal X_0$.

Assume now that \eqref{eq:app_inductive_e2e_bound} holds for some $p-1\ge 1$.
By the chain rule,
\[
\frac{\partial h_t^{(p)}(x)}{\partial h_\tau^{(0)}(x)}
=
\sum_{j=\tau}^{t}
\frac{\partial F_{p,t}(h^{(p-1)}(x))}{\partial h_j^{(p-1)}(x)}
\cdot
\frac{\partial h_j^{(p-1)}(x)}{\partial h_\tau^{(0)}(x)}.
\]
Taking operator norms and using submultiplicativity gives
\[
\left\|
\frac{\partial h_t^{(p)}(x)}{\partial h_\tau^{(0)}(x)}
\right\|
\le
\sum_{j=\tau}^{t}
\left\|
\frac{\partial F_{p,t}(h^{(p-1)}(x))}{\partial h_j^{(p-1)}(x)}
\right\|
\cdot
\left\|
\frac{\partial h_j^{(p-1)}(x)}{\partial h_\tau^{(0)}(x)}
\right\|.
\]
Since $h^{(p-1)}(x)\in\mathcal X_{p-1}$, the one-block bound
\eqref{eq:app_one_block_e2e_envelope} applies:
\[
\left\|
\frac{\partial F_{p,t}(h^{(p-1)}(x))}{\partial h_j^{(p-1)}(x)}
\right\|
\le
\mathcal D_p(t,j)+\mathcal G_p(t,j).
\]
Using the induction hypothesis for the second factor, we get
\[
\left\|
\frac{\partial h_t^{(p)}(x)}{\partial h_\tau^{(0)}(x)}
\right\|
\le
\sum_{j=\tau}^{t}
(\mathcal D_p+\mathcal G_p)(t,j)\,
\bigl[(\mathcal D_{p-1}+\mathcal G_{p-1})\cdots(\mathcal D_1+\mathcal G_1)\bigr](j,\tau).
\]
This is exactly
\[
\bigl[(\mathcal D_p+\mathcal G_p)\cdots(\mathcal D_1+\mathcal G_1)\bigr](t,\tau),
\]
which proves \eqref{eq:app_inductive_e2e_bound} for depth $p$.

Taking $p=\Nlayer$ yields
\[
\|J^{\mathrm{e2e},(\Nlayer)}_{t,\tau}(x)\|
\le
\bigl[(\mathcal D_{\Nlayer}+\mathcal G_{\Nlayer})\cdots(\mathcal D_1+\mathcal G_1)\bigr](t,\tau).
\]

We now expand the right-hand side.
Since each $\mathcal D_{\idxlayer}$ is diagonal and equals $d_{\idxlayer}I$ as a kernel,
one has the exact product expansion
\[
(\mathcal D_{\Nlayer}+\mathcal G_{\Nlayer})\cdots(\mathcal D_1+\mathcal G_1)
=
\sum_{S\subseteq\{1,\dots,\Nlayer\}}
\left(\prod_{m\notin S} d_m\right)
\prod_{\idxlayer\in S}^{\rightarrow}\mathcal G_{\idxlayer},
\]
where the ordered product is taken in increasing layer order.
For $\tau<t$, the empty-set term vanishes because it is purely diagonal.
Thus
\[
\|J^{\mathrm{e2e},(\Nlayer)}_{t,\tau}(x)\|
\le
\sum_{k=1}^{\Nlayer}
\ \sum_{1\le \idxlayeri{1}<\cdots<\idxlayeri{k}\le \Nlayer}
\left(\prod_{m\notin\{\idxlayeri{1},\dots,\idxlayeri{k}\}} d_m\right)
(\mathcal G_{\idxlayeri{k}}\cdots \mathcal G_{\idxlayeri{1}})(t,\tau).
\]
Finally, by repeated expansion of the kernel product,
\[
(\mathcal G_{\idxlayeri{k}}\cdots \mathcal G_{\idxlayeri{1}})(t,\tau)
=
\sum_{\tau=i_0<i_1<\cdots<i_k=t}
\ \prod_{r=1}^{k}
\lambda_{\idxlayeri{r}}\,K_{\idxlayeri{r}}(i_r,i_{r-1}),
\]
which gives \eqref{eq:app_full_e2e_path_sum}.

For the diagonal blocks $\tau=t$, only the empty-set term survives, hence
\[
\|J^{\mathrm{e2e},(\Nlayer)}_{t,t}(x)\|
\le
\prod_{\idxlayer=1}^{\Nlayer} d_{\idxlayer}.
\]
\end{proof}

\subsection{A harmonic-kernel bound}

For diffuse Transformer blocks the one-block kernel depends on the query time $t$.
The next lemma gives the corresponding convolution bound for
\[
\mathcal H(t,\tau):=\frac{1}{t+1}\mathbf 1[\tau<t].
\]

\begin{lemma}[Nested harmonic bound]
\label{lem:app_nested_harmonic}
Fix $k\ge 1$ and define
\[
\mathcal H(t,\tau):=\frac{1}{t+1}\mathbf 1[\tau<t].
\]
Then for every $0\le \tau<t\le \Tctx-1$,
\begin{equation}
(\mathcal H^k)(t,\tau)
\le
\frac{1}{t+1}\cdot \frac{H_t^{\,k-1}}{(k-1)!},
\label{eq:app_nested_harmonic_bound}
\end{equation}
where
\[
H_t:=\sum_{m=1}^{t}\frac{1}{m}
\]
is the $t$-th harmonic number, with the convention $H_0:=0$.
Consequently, for every fixed $k$,
\[
(\mathcal H^k)(t,\tau)\lesssim_k \frac{(\log(1+t))^{k-1}}{t+1}.
\]
\end{lemma}

\begin{proof}
For $k=1$ the claim is immediate:
\[
\mathcal H(t,\tau)=\frac{1}{t+1}\mathbf 1[\tau<t]
\le \frac{1}{t+1}.
\]

Assume now $k\ge 2$.
By the kernel-product expansion,
\[
(\mathcal H^k)(t,\tau)
=
\sum_{\tau=i_0<i_1<\cdots<i_k=t}
\ \prod_{r=1}^{k}\frac{1}{i_r+1}.
\]
Since $i_k=t$, the last factor is exactly $\frac{1}{t+1}$, hence
\[
(\mathcal H^k)(t,\tau)
=
\frac{1}{t+1}
\sum_{\tau<i_1<\cdots<i_{k-1}<t}
\ \prod_{r=1}^{k-1}\frac{1}{i_r+1}.
\]
Dropping the lower bound $\tau$ only enlarges the sum, so
\[
(\mathcal H^k)(t,\tau)
\le
\frac{1}{t+1}
\sum_{0<i_1<\cdots<i_{k-1}<t}
\ \prod_{r=1}^{k-1}\frac{1}{i_r+1}.
\]
Now expand
\[
\left(\sum_{m=1}^{t-1}\frac{1}{m+1}\right)^{k-1}.
\]
Every strictly increasing $(k-1)$-tuple
\[
0<i_1<\cdots<i_{k-1}<t
\]
appears exactly $(k-1)!$ times among the ordered monomials in this expansion.
Therefore
\[
\sum_{0<i_1<\cdots<i_{k-1}<t}
\ \prod_{r=1}^{k-1}\frac{1}{i_r+1}
\le
\frac{1}{(k-1)!}
\left(\sum_{m=1}^{t-1}\frac{1}{m+1}\right)^{k-1}
\le
\frac{H_t^{\,k-1}}{(k-1)!}.
\]
Substituting this into the previous display gives
\[
(\mathcal H^k)(t,\tau)
\le
\frac{1}{t+1}\cdot \frac{H_t^{\,k-1}}{(k-1)!},
\]
which is \eqref{eq:app_nested_harmonic_bound}.

Since $H_t\lesssim \log(1+t)$, the logarithmic form follows.
\end{proof}

\subsection{Model-specific bounds}

\begin{proposition}[Deep Transformer bound]
\label{prop:app_e2e_transformer_deep}
Assume the hypotheses of Theorem~\ref{thm:app_e2e_residual_calculus}.
Assume in addition that for each layer $\idxlayer$ there exists $a_{\idxlayer}>0$ such that
\[
K_{\idxlayer}(t,\tau)\le \frac{a_{\idxlayer}}{t+1},
\qquad \tau<t.
\]
Fix a bounded source family $0\le \tau\le \tau_{\max}$.
Then for every $x\in\mathcal X_0$ and every $\ell\ge 1$ with $\tau+\ell\le \Tctx-1$,
\[
\left\|
J^{\mathrm{e2e},(\Nlayer)}_{\tau+\ell,\tau}(x)
\right\|
\lesssim_{\tau_{\max},\Nlayer}
\frac{(\log(1+\ell))^{\Nlayer-1}}{1+\ell}.
\]
\end{proposition}

\begin{proof}
Fix an ordered layer subset
\[
1\le \idxlayeri{1}<\cdots<\idxlayeri{k}\le \Nlayer.
\]
Define
\[
\mathcal H(t,\tau):=\frac{1}{t+1}\mathbf 1[\tau<t].
\]
By the assumption on $K_{\idxlayer}$,
\[
K_{\idxlayeri{r}}(i_r,i_{r-1})\le a_{\idxlayeri{r}}\mathcal H(i_r,i_{r-1})
\qquad \forall r.
\]
Therefore
\[
\sum_{\tau=i_0<\cdots<i_k=t}
\ \prod_{r=1}^{k}\lambda_{\idxlayeri{r}}K_{\idxlayeri{r}}(i_r,i_{r-1})
\le
\left(\prod_{r=1}^{k}\lambda_{\idxlayeri{r}}a_{\idxlayeri{r}}\right)
(\mathcal H^k)(t,\tau).
\]
By Lemma~\ref{lem:app_nested_harmonic},
\[
(\mathcal H^k)(t,\tau)\lesssim_k \frac{(\log(1+t))^{k-1}}{t+1}.
\]
Insert this estimate into Theorem~\ref{thm:app_e2e_residual_calculus}:
\[
\|J^{\mathrm{e2e},(\Nlayer)}_{t,\tau}(x)\|
\lesssim_{\Nlayer}
\sum_{k=1}^{\Nlayer}
\ \sum_{1\le \idxlayeri{1}<\cdots<\idxlayeri{k}\le \Nlayer}
\left(\prod_{m\notin\{\idxlayeri{1},\dots,\idxlayeri{k}\}} d_m\right)
\left(\prod_{r=1}^{k}\lambda_{\idxlayeri{r}}a_{\idxlayeri{r}}\right)
\frac{(\log(1+t))^{k-1}}{t+1}.
\]
Since $\Nlayer$ is fixed, the finite sum is bounded by
\[
C_{\Nlayer}\frac{(\log(1+t))^{\Nlayer-1}}{t+1}.
\]
Now restrict to the bounded source family $0\le \tau\le \tau_{\max}$ and set $t=\tau+\ell$.
Then
\[
t+1=\tau+\ell+1\asymp_{\tau_{\max}} 1+\ell,
\qquad
\log(1+t)\asymp_{\tau_{\max}}\log(1+\ell),
\]
uniformly for $0\le \tau\le \tau_{\max}$.
Hence
\[
\left\|
J^{\mathrm{e2e},(\Nlayer)}_{\tau+\ell,\tau}(x)
\right\|
\lesssim_{\tau_{\max},\Nlayer}
\frac{(\log(1+\ell))^{\Nlayer-1}}{1+\ell}.
\]
\end{proof}

\begin{proposition}[Deep Mamba bound under failed freeze time]
\label{prop:app_e2e_ssm_deep}
Assume the hypotheses of Theorem~\ref{thm:app_e2e_residual_calculus}.
Assume in addition that for each layer $\idxlayer$ there exist $a_{\idxlayer}>0$ and $c_{\idxlayer}>0$ such that
\[
K_{\idxlayer}(t,\tau)\le a_{\idxlayer}e^{-c_{\idxlayer}(t-\tau)},
\qquad \tau<t.
\]
Set
\[
c_\ast:=\min_{1\le \idxlayer\le \Nlayer} c_{\idxlayer}.
\]
Then for every $x\in\mathcal X_0$ and every $\tau<t$,
\[
\left\|
J^{\mathrm{e2e},(\Nlayer)}_{t,\tau}(x)
\right\|
\lesssim_{\Nlayer}
(1+t-\tau)^{\Nlayer-1}e^{-c_\ast(t-\tau)}.
\]
\end{proposition}

\begin{proof}
Fix an ordered layer subset
\[
1\le \idxlayeri{1}<\cdots<\idxlayeri{k}\le \Nlayer
\]
and write $\ell:=t-\tau$.
For every temporal path $\tau=i_0<\cdots<i_k=t$, one has
\[
\prod_{r=1}^{k}K_{\idxlayeri{r}}(i_r,i_{r-1})
\le
\left(\prod_{r=1}^{k}a_{\idxlayeri{r}}\right)
\exp\left(-\sum_{r=1}^{k}c_{\idxlayeri{r}}(i_r-i_{r-1})\right)
\le
\left(\prod_{r=1}^{k}a_{\idxlayeri{r}}\right)e^{-c_\ast\ell}.
\]
The number of strictly increasing temporal paths
\[
\tau=i_0<i_1<\cdots<i_k=t
\]
is the number of compositions of $\ell$ into $k$ positive integers, namely
\[
\binom{\ell-1}{k-1},
\]
with the convention that this is $0$ if $\ell<k$.
Therefore
\[
\sum_{\tau=i_0<\cdots<i_k=t}
\ \prod_{r=1}^{k}\lambda_{\idxlayeri{r}}K_{\idxlayeri{r}}(i_r,i_{r-1})
\le
\left(\prod_{r=1}^{k}\lambda_{\idxlayeri{r}}a_{\idxlayeri{r}}\right)
\binom{\ell-1}{k-1}e^{-c_\ast\ell}.
\]
Insert this estimate into Theorem~\ref{thm:app_e2e_residual_calculus}:
\[
\|J^{\mathrm{e2e},(\Nlayer)}_{t,\tau}(x)\|
\le
\sum_{k=1}^{\Nlayer}
\ \sum_{1\le \idxlayeri{1}<\cdots<\idxlayeri{k}\le \Nlayer}
\left(\prod_{m\notin\{\idxlayeri{1},\dots,\idxlayeri{k}\}} d_m\right)
\left(\prod_{r=1}^{k}\lambda_{\idxlayeri{r}}a_{\idxlayeri{r}}\right)
\binom{\ell-1}{k-1}e^{-c_\ast\ell}.
\]
Since $\Nlayer$ is fixed and
\[
\binom{\ell-1}{k-1}\lesssim_k (1+\ell)^{k-1},
\]
the finite sum is bounded by a constant multiple of
\[
(1+\ell)^{\Nlayer-1}e^{-c_\ast\ell}.
\]
\end{proof}

\begin{proposition}[Deep Sessa bound]
\label{prop:app_e2e_sessa_deep}
Assume the hypotheses of Theorem~\ref{thm:app_e2e_residual_calculus}.
Assume in addition that for each layer $\idxlayer$ there exist $a_{\idxlayer}>0$ and a common exponent
$\betatail\in(0,1)$ such that
\[
K_{\idxlayer}(t,\tau)\le
a_{\idxlayer}(t-\tau)^{-\betatail}\bigl(1+\log(1+t-\tau)\bigr),
\qquad \tau<t.
\]
Then for every $x\in\mathcal X_0$ and every $\tau<t$,
\[
\left\|
J^{\mathrm{e2e},(\Nlayer)}_{t,\tau}(x)
\right\|
\lesssim_{\Nlayer,\betatail}
\sum_{k=1}^{\Nlayer}
(t-\tau)^{k(1-\betatail)-1}\bigl(1+\log(1+t-\tau)\bigr)^k.
\]
In particular, since $\Nlayer$ is fixed,
\[
\left\|
J^{\mathrm{e2e},(\Nlayer)}_{t,\tau}(x)
\right\|
\lesssim_{\Nlayer,\betatail}
(t-\tau)^{\Nlayer(1-\betatail)-1}\bigl(1+\log(1+t-\tau)\bigr)^{\Nlayer}.
\]
\end{proposition}

\begin{proof}
Fix $\tau<t$ and write $\ell:=t-\tau$.
Fix an ordered layer subset
\[
1\le \idxlayeri{1}<\cdots<\idxlayeri{k}\le \Nlayer.
\]
For every temporal path $\tau=i_0<\cdots<i_k=t$, set
\[
m_r:=i_r-i_{r-1}\in\mathbb N^\ast.
\]
Then
\[
m_1+\cdots+m_k=\ell.
\]
Using the bound on $K_{\idxlayer}$,
\[
\prod_{r=1}^{k}K_{\idxlayeri{r}}(i_r,i_{r-1})
\le
\left(\prod_{r=1}^{k}a_{\idxlayeri{r}}\right)
\prod_{r=1}^{k}m_r^{-\betatail}\bigl(1+\log(1+m_r)\bigr).
\]
Since every $m_r\le \ell$, one has
\[
1+\log(1+m_r)\le 1+\log(1+\ell).
\]
Therefore
\[
\prod_{r=1}^{k}K_{\idxlayeri{r}}(i_r,i_{r-1})
\le
\left(\prod_{r=1}^{k}a_{\idxlayeri{r}}\right)
\bigl(1+\log(1+\ell)\bigr)^k
\prod_{r=1}^{k}m_r^{-\betatail}.
\]
Summing over all temporal paths from $\tau$ to $t$ gives
\begin{align*}
&\sum_{\tau=i_0<\cdots<i_k=t}
\ \prod_{r=1}^{k}\lambda_{\idxlayeri{r}}K_{\idxlayeri{r}}(i_r,i_{r-1}) \\
&\qquad\le
\left(\prod_{r=1}^{k}\lambda_{\idxlayeri{r}}a_{\idxlayeri{r}}\right)
\bigl(1+\log(1+\ell)\bigr)^k
\sum_{\substack{m_1,\dots,m_k\ge 1\\ m_1+\cdots+m_k=\ell}}
m_1^{-\betatail}\cdots m_k^{-\betatail}.
\end{align*}
The remaining sum is exactly the $k$-fold positive-lag convolution
\[
f_{\betatail}^{(*k)}(\ell),
\qquad
f_{\betatail}(n):=n^{-\betatail}.
\]
By Theorem~\ref{thm:heavy_tail_convolution_class},
\[
f_{\betatail}^{(*k)}(\ell)\lesssim_{k,\betatail}\ell^{k(1-\betatail)-1}.
\]
Hence
\[
\sum_{\tau=i_0<\cdots<i_k=t}
\ \prod_{r=1}^{k}\lambda_{\idxlayeri{r}}K_{\idxlayeri{r}}(i_r,i_{r-1})
\lesssim_{k,\betatail}
\left(\prod_{r=1}^{k}\lambda_{\idxlayeri{r}}a_{\idxlayeri{r}}\right)
\ell^{k(1-\betatail)-1}\bigl(1+\log(1+\ell)\bigr)^k.
\]
Insert this estimate into Theorem~\ref{thm:app_e2e_residual_calculus} and sum over
\[
k=1,\dots,\Nlayer.
\]
Since $\Nlayer$ is fixed, the finite sum yields the stated bound.

The final simplified estimate follows because, for $\betatail\in(0,1)$, the exponent
\[
k(1-\betatail)-1
\]
is increasing in $k$, so the $k=\Nlayer$ term dominates the smaller-$k$ terms up to a constant.
\end{proof}

\subsection{Horizon-uniform bounds}
\label{app:deep_e2e_uniform_horizon}

We now state the horizon-uniform version used in
Section~\ref{sec:theory_multilayer_e2e}.

\begin{theorem}[Horizon-uniform residual calculus]
\label{thm:app_e2e_residual_calculus_uniform}
Fix a depth \(\Nlayer\ge 1\).
For each horizon \(\Tctx\ge 1\), let
\[
h^{(0,\Tctx)}=x\in\mathcal X_0^{(\Tctx)},
\qquad
h^{(\idxlayer,\Tctx)} = F_{\idxlayer}^{(\Tctx)}\bigl(h^{(\idxlayer-1,\Tctx)}\bigr),
\qquad \idxlayer=1,\dots,\Nlayer,
\]
where \(\mathcal X_0^{(\Tctx)}\subset(\mathbb R^D)^{\Tctx}\) is compact and each
\(F_{\idxlayer}^{(\Tctx)}\) is causal and continuously differentiable on the relevant compact set
\[
\mathcal X_{\idxlayer-1}^{(\Tctx)}
:=
F_{\idxlayer-1}^{(\Tctx)}\circ\cdots\circ F_1^{(\Tctx)}(\mathcal X_0^{(\Tctx)}).
\]
Define the full end-to-end Jacobian blocks by
\[
J^{\mathrm{e2e},(\Nlayer)}_{t,\tau}(x;\Tctx)
:=
\frac{\partial h_t^{(\Nlayer,\Tctx)}(x)}{\partial h_\tau^{(0,\Tctx)}(x)}
\in\mathbb R^{D\times D},
\qquad 0\le \tau\le t\le \Tctx-1.
\]

Assume that for each layer \(\idxlayer\) there exist constants
\[
d_{\idxlayer}\ge 0,\qquad \lambda_{\idxlayer}\ge 0,
\]
independent of \(\Tctx\), and a scalar lower-triangular kernel
\[
K_{\idxlayer}:\{(t,\tau):0\le \tau<t<\infty\}\to[0,\infty)
\]
independent of \(\Tctx\), such that for every horizon \(\Tctx\ge 1\),
every \(u\in\mathcal X_{\idxlayer-1}^{(\Tctx)}\), and every
\(0\le \tau\le t\le \Tctx-1\),
\[
\left\|
\frac{\partial F_{\idxlayer,t}^{(\Tctx)}(u)}{\partial u_\tau}
\right\|
\le
d_{\idxlayer}\,\mathbf 1[t=\tau]
+
\lambda_{\idxlayer}\,K_{\idxlayer}(t,\tau)\,\mathbf 1[\tau<t].
\]
Then for every horizon \(\Tctx\ge 1\), every \(x\in\mathcal X_0^{(\Tctx)}\), and every
\(0\le \tau<t\le \Tctx-1\),
\begin{align}
\left\|
J^{\mathrm{e2e},(\Nlayer)}_{t,\tau}(x;\Tctx)
\right\|
&\le
\sum_{k=1}^{\Nlayer}
\ \sum_{1\le \idxlayeri{1}<\cdots<\idxlayeri{k}\le \Nlayer}
\left(\prod_{m\notin\{\idxlayeri{1},\dots,\idxlayeri{k}\}} d_m\right)
\notag\\
&\qquad\qquad\qquad\qquad\cdot
\sum_{\tau=i_0<i_1<\cdots<i_k=t}
\ \prod_{r=1}^{k}
\lambda_{\idxlayeri{r}}\,K_{\idxlayeri{r}}(i_r,i_{r-1}).
\label{eq:app_full_e2e_path_sum_uniform_horizon}
\end{align}
Moreover,
\[
\left\|
J^{\mathrm{e2e},(\Nlayer)}_{t,t}(x;\Tctx)
\right\|
\le
\prod_{\idxlayer=1}^{\Nlayer} d_{\idxlayer}.
\]
In particular, the right-hand side of
\eqref{eq:app_full_e2e_path_sum_uniform_horizon} is independent of \(\Tctx\).
\end{theorem}

\begin{proof}
Fix a horizon \(\Tctx\ge 1\).
Apply Theorem~\ref{thm:app_e2e_residual_calculus} to the horizon-\(\Tctx\) stack
\[
F_1^{(\Tctx)},\dots,F_{\Nlayer}^{(\Tctx)}
\]
on the compact input set \(\mathcal X_0^{(\Tctx)}\).
The hypotheses of Theorem~\ref{thm:app_e2e_residual_calculus} are satisfied with the same
layerwise constants \(d_{\idxlayer},\lambda_{\idxlayer}\) and the same kernels \(K_{\idxlayer}\),
because these are assumed to be independent of \(\Tctx\).
Therefore, for this fixed horizon \(\Tctx\), Theorem~\ref{thm:app_e2e_residual_calculus}
gives exactly the path-sum bound
\eqref{eq:app_full_e2e_path_sum_uniform_horizon} and the same diagonal estimate.

Since the displayed right-hand side contains no dependence on \(\Tctx\), the same bound
holds verbatim for every horizon \(\Tctx\ge 1\).
\end{proof}

\begin{corollary}[Horizon-uniform decay bounds]
\label{cor:app_e2e_deep_decay_uniform_horizon}
Assume the hypotheses of Theorem~\ref{thm:app_e2e_residual_calculus_uniform}.

\begin{enumerate}[label=(\roman*), leftmargin=*, nosep]
\item \textbf{Transformer.}
Assume that for each layer \(\idxlayer\) there exists \(a_{\idxlayer}>0\) such that
\[
K_{\idxlayer}(t,\tau)\le \frac{a_{\idxlayer}}{t+1},
\qquad \tau<t.
\]
Fix a bounded source family \(0\le \tau\le \tau_{\max}\).
Then
\[
\sup_{\Tctx\ge \tau_{\max}+\ell+1}
\ \sup_{0\le \tau\le \tau_{\max}}
\ \sup_{x\in\mathcal X_0^{(\Tctx)}}
\left\|
J^{\mathrm{e2e},(\Nlayer)}_{\tau+\ell,\tau}(x;\Tctx)
\right\|
\lesssim_{\tau_{\max},\Nlayer}
\frac{(\log(1+\ell))^{\Nlayer-1}}{1+\ell}.
\]

\item \textbf{Mamba.}
Assume that for each layer \(\idxlayer\) there exist \(a_{\idxlayer}>0\) and \(c_{\idxlayer}>0\) such that
\[
K_{\idxlayer}(t,\tau)\le a_{\idxlayer}e^{-c_{\idxlayer}(t-\tau)},
\qquad \tau<t.
\]
Set \(c_\ast:=\min_{\idxlayer}c_{\idxlayer}\).
Then
\[
\sup_{\Tctx\ge \ell+1}
\ \sup_{0\le \tau\le \Tctx-\ell-1}
\ \sup_{x\in\mathcal X_0^{(\Tctx)}}
\left\|
J^{\mathrm{e2e},(\Nlayer)}_{\tau+\ell,\tau}(x;\Tctx)
\right\|
\lesssim_{\Nlayer}
(1+\ell)^{\Nlayer-1}e^{-c_\ast\ell}.
\]

\item \textbf{Sessa.}
Assume that for each layer \(\idxlayer\) there exist \(a_{\idxlayer}>0\) and a common exponent
\(\betatail\in(0,1)\) such that
\[
K_{\idxlayer}(t,\tau)\le
a_{\idxlayer}(t-\tau)^{-\betatail}\bigl(1+\log(1+t-\tau)\bigr),
\qquad \tau<t.
\]
Then
\[
\sup_{\Tctx\ge \ell+1}
\ \sup_{0\le \tau\le \Tctx-\ell-1}
\ \sup_{x\in\mathcal X_0^{(\Tctx)}}
\left\|
J^{\mathrm{e2e},(\Nlayer)}_{\tau+\ell,\tau}(x;\Tctx)
\right\|
\lesssim_{\Nlayer,\betatail}
\sum_{k=1}^{\Nlayer}
\ell^{k(1-\betatail)-1}\bigl(1+\log(1+\ell)\bigr)^k.
\]
In particular, if \(\Nlayer(1-\betatail)<1\), then the right-hand side tends to \(0\) as \(\ell\to\infty\).
\end{enumerate}
\end{corollary}

\begin{proof}
Apply Theorem~\ref{thm:app_e2e_residual_calculus_uniform} and then repeat exactly the
kernel-class estimates used in the proofs of
Propositions~\ref{prop:app_e2e_transformer_deep},
\ref{prop:app_e2e_ssm_deep}, and \ref{prop:app_e2e_sessa_deep}.
Because the layerwise envelope parameters are horizon-uniform, the resulting constants
are independent of \(\Tctx\).
Taking the indicated suprema over all admissible horizons therefore leaves the bounds unchanged.
For the Transformer case, the passage from \(t=\tau+\ell\) to \(1+\ell\) is uniform on bounded-source
families \(0\le \tau\le \tau_{\max}\).
For the Sessa case, the final asymptotic decay to \(0\) occurs exactly when the largest power
\[
\ell^{\Nlayer(1-\betatail)-1}
\]
has negative exponent, i.e.\ when \(\Nlayer(1-\betatail)<1\).
\end{proof}

\section{Universal approximation for Sessa with adapters}
\label{app:sessa_uat}

\subsection{Preliminaries and notation}

Fix $\Tctx\ge 3$ and $d_{\mathrm{ext}}\in\mathbb N^*$.
Inputs are
\[
x=(x_0,\dots,x_{\Tctx-1})\in(\mathbb R^{d_{\mathrm{ext}}})^{\Tctx}\cong \mathbb R^{\Tctx\times d_{\mathrm{ext}}},
\]
and outputs are in $\mathbb R^{\Tctx\times d_{\mathrm{ext}}}$.
For $X\in\mathbb R^{\Tctx\times d_{\mathrm{ext}}}$ define
\[
\|X\|_F^2=\sum_{t=0}^{\Tctx-1}\|X_t\|_2^2.
\]

Let $\mathcal D\subset \mathbb R^{\Tctx\times d_{\mathrm{ext}}}$ be compact and
\[
\Mdom:=\sup_{x\in\mathcal D}\|x\|_F <\infty.
\]
Hence $\|x_t\|_2\le \Mdom$ for all $x\in\mathcal D$ and all $t$.

\begin{definition}[Causality]
$F:\mathcal D\to \mathbb R^{\Tctx\times d_{\mathrm{ext}}}$ is causal if for every $t$ and all $x,x'\in\mathcal D$,
$x_{0:t}=x'_{0:t}$ implies $F(x)_t=F(x')_t$.
\end{definition}

\begin{lemma}[Prefix factorization of continuous causal maps]
\label{lem:causal_prefix_factorization}
Let
\[
\mathcal D\subset \mathbb R^{\Tctx\times d_{\mathrm{ext}}}
\]
be compact and let
\[
F:\mathcal D\to\mathbb R^{\Tctx\times d_{\mathrm{ext}}}
\]
be continuous and causal.
For each \(t\in\{0,\dots,\Tctx-1\}\), define
\[
p_t:\mathcal D\to (\mathbb R^{d_{\mathrm{ext}}})^{t+1},
\qquad
p_t(x):=x_{0:t},
\]
and
\[
\Pref_t:=p_t(\mathcal D).
\]
Then there exists a unique continuous map
\[
\widehat F_t:\Pref_t\to\mathbb R^{d_{\mathrm{ext}}}
\]
such that
\[
\widehat F_t(x_{0:t})=F(x)_t
\qquad \forall x\in\mathcal D.
\]
\end{lemma}

\begin{proof}
Uniqueness is immediate because \(p_t\) is surjective onto \(\Pref_t\).

Causality ensures that \(\widehat F_t\) is well defined: if \(p_t(x)=p_t(x')\), then \(x_{0:t}=x'_{0:t}\), hence
\[
F(x)_t=F(x')_t.
\]

Let
\[
\operatorname{pr}_t:\mathbb R^{\Tctx\times d_{\mathrm{ext}}}\to\mathbb R^{d_{\mathrm{ext}}},
\qquad
\operatorname{pr}_t(y):=y_t,
\]
and define
\[
g_t:=\operatorname{pr}_t\circ F:\mathcal D\to\mathbb R^{d_{\mathrm{ext}}}.
\]
Then
\[
g_t=\widehat F_t\circ p_t.
\]

Let \(C\subset\mathbb R^{d_{\mathrm{ext}}}\) be closed.
Since \(g_t\) is continuous, \(g_t^{-1}(C)\) is closed in the compact set \(\mathcal D\), hence compact.
Applying \(p_t\), the image
\[
p_t\big(g_t^{-1}(C)\big)
\]
is compact in \(\Pref_t\), hence closed because \(\Pref_t\) is Hausdorff.
Moreover,
\[
\widehat F_t^{-1}(C)=p_t\big(g_t^{-1}(C)\big).
\]
Therefore \(\widehat F_t\) is continuous.
\end{proof}

\subsection{Architecture and function classes}

\paragraph{Sessa blocks of width $m$}

Fix an even query--key width $d_k\in 2\mathbb N$, a model width $m\in\mathbb N^*$,
and a tokenwise pre-normalization map
\[
\Norm:\mathbb R^m\to\mathbb R^m
\]
applied independently to each token.
We consider two choices:
\[
\Norm=\Id
\qquad\text{and}\qquad
\Norm=\LN_{\varepsilon_{\ln}}\ \ (\varepsilon_{\ln}>0).
\]

A width-$m$ Sessa block is the block of Section~\ref{sec:model_arch} specialized to model width $m$,
and we use the following RoPE convention throughout this section.

Write every $z\in\mathbb R^{d_k}$ as
\[
z=(z^{(0)},z^{(1)},\dots,z^{(d_k/2-1)}),
\qquad z^{(r)}\in\mathbb R^2.
\]
Fix a RoPE base $\vartheta>1$ and define the standard pairwise frequencies
\[
\omega_r:=\vartheta^{-2r/d_k},
\qquad r=0,\dots,d_k/2-1.
\]
In particular,
\[
\omega_0=1.
\]
For every $\tau\in\mathbb R$ define
\[
\mathrm{RoPE}_\tau(z)
:=
\big(
R_{\omega_0 \tau}z^{(0)},
R_{\omega_1 \tau}z^{(1)},
\dots,
R_{\omega_{d_k/2-1} \tau}z^{(d_k/2-1)}
\big),
\]
where $R_\theta$ denotes the planar rotation by angle $\theta$.
In the architecture, $\tau=t\in\{0,\dots,\Tctx-1\}$; in the constructions below we also allow shifts such as $\tau=-\ell$.
All diagonalization arguments use only the first rotary pair.
Hence, whenever $q,k\in\mathbb R^{d_k}$ are supported on that first pair,
\[
\langle \mathrm{RoPE}_t(q),\mathrm{RoPE}_j(k)\rangle
=
\langle R_t q_{1:2},\,R_j k_{1:2}\rangle.
\]
The comparison RoPE-Transformer class uses the same convention.
\paragraph{Parameters and dimensions}
\[
W^{\mathrm{in}}\in\mathbb R^{m\times 2m},\qquad
b^{\mathrm{in}}\in\mathbb R^{2m},\qquad
W^{\mathrm{out}}\in\mathbb R^{m\times m},\qquad
b^{\mathrm{out}}\in\mathbb R^{m},
\]
\[
W_{Qf},W_{Kf},W_{Qb},W_{Kb}\in\mathbb R^{m\times d_k},\qquad
W_V\in\mathbb R^{m\times m},
\]
\[
w^\gamma\in\mathbb R^m,\qquad b^\gamma\in\mathbb R.
\]

\paragraph{Tokenwise preprocessing}
Given $x\in\mathbb R^{\Tctx\times m}$:
\begin{align*}
\xnorm_t &= \Norm(x_t)\in\mathbb R^m,\\
u_t &= \xnorm_t W^{\mathrm{in}} + b^{\mathrm{in}} \in\mathbb R^{2m},\\
u_t &= (a_t,g_t),\qquad a_t,g_t\in\mathbb R^m,\\
\bar a_t &= \mathrm{GELU}(a_t)\in\mathbb R^m.
\end{align*}

\paragraph{Attention-feedback operator}
We fix the attention scale to
\[
\sigk:=d_k^{-1/2}.
\]
Define
\[
q^f_t=\bar a_t W_{Qf},\qquad
k^f_t=\bar a_t W_{Kf},\qquad
v_t=\bar a_t W_V,\qquad
q^b_t=\bar a_t W_{Qb},\qquad
k^b_t=\bar a_t W_{Kb},
\]
with
\[
q^f_t,k^f_t,q^b_t,k^b_t\in\mathbb R^{d_k},
\qquad
v_t\in\mathbb R^m.
\]

For the causal forward branch $(j\le t)$, define
\[
\tilde q^f_t=\mathrm{RoPE}_t(q^f_t),
\qquad
\tilde k^f_j=\mathrm{RoPE}_j(k^f_j),
\]
and define
\[
\alphadrv_{t,j}
=
\frac{
\exp\!\Big(\sigk\langle \tilde q^f_t,\tilde k^f_j\rangle\Big)\mathbf 1[j\le t]
}{
\sum_{\tau\le t}\exp\!\Big(\sigk\langle \tilde q^f_t,\tilde k^f_\tau\rangle\Big)
},
\qquad
f_t=\sum_{j\le t}\alphadrv_{t,j}v_j.
\]

For the strictly lower feedback branch $(j<t)$, define
\[
\alphafb_{t,j}
=
\frac{
\exp\!\Big(\sigk\langle q^b_t,k^b_j\rangle\Big)\mathbf 1[j<t]
}{
\sum_{\tau<t}\exp\!\Big(\sigk\langle q^b_t,k^b_\tau\rangle\Big)
},
\qquad
\alphafb_{0,\cdot}=0.
\]
\[
\gamma_t=\tanh\!\big(\langle \bar a_t,w^\gamma\rangle+b^\gamma\big)\in(-1,1).
\]
\[
\BfbEnt{t}{j}=\gamma_t\alphafb_{t,j},\qquad \BfbEnt{t}{j}=0\ \text{for }j\ge t.
\]

The mixer output is defined by the exact solve
\[
(I-\Bfb)s=f.
\]
Since $\Bfb$ is strictly lower triangular, the system has a unique solution.

\paragraph{Residual update}
\[
y_t=x_t+\big((s_t\odot g_t)W^{\mathrm{out}}+b^{\mathrm{out}}\big).
\]

\paragraph{Function classes}
Let
\[
\mathrm{ConcreteSessaBlocks}_{\Norm}(d_k,m)
\]
denote the set of all width-$m$ concrete Sessa blocks above with the chosen pre-normalization map $\Norm$.
Define
\[
\Omega_{\mathrm{Sessa},\Norm}^{d_k}(m)
:=
\Big\{
G_{\Nlayer}\circ\cdots\circ G_1:\ 
G_{\idxlayer}\in \mathrm{ConcreteSessaBlocks}_{\Norm}(d_k,m)\ \text{for all }\idxlayer,\ 
\Nlayer\in\mathbb N^*
\Big\}.
\]

\paragraph{Tokenwise input and output adapters}\label{app:uat_adapters}
Fix the external data dimension $d_{\mathrm{ext}}$ and a model width $m\ge d_{\mathrm{ext}}$.
Define tokenwise affine adapters
\[
\mathrm{Embed}(x)_t := x_t W^{\mathrm{emb}} + b^{\mathrm{emb}}\in\mathbb R^{m},
\qquad
\mathrm{Unembed}(h)_t := h_t W^{\mathrm{un}} + b^{\mathrm{un}}\in\mathbb R^{d_{\mathrm{ext}}}.
\]

\paragraph{Parameters and dimensions}
\[
W^{\mathrm{emb}}\in\mathbb R^{d_{\mathrm{ext}}\times m},\quad b^{\mathrm{emb}}\in\mathbb R^{m},\qquad
W^{\mathrm{un}}\in\mathbb R^{m\times d_{\mathrm{ext}}},\quad b^{\mathrm{un}}\in\mathbb R^{d_{\mathrm{ext}}}.
\]

\[
\mathrm{Unembed}\circ \mathrm{Embed}=\Id
\qquad\text{on }\mathbb R^{\Tctx\times d_{\mathrm{ext}}}.
\]

We consider Sessa networks of the form
\[
x\ \mapsto\ \mathrm{Unembed}\big( G(\mathrm{Embed}(x))\big),
\]
with
\[
G\in\Omega_{\mathrm{Sessa},\Id}^{d_k}(m)
\]
in the main LN-free theorem, and
\[
G\in\Omega_{\mathrm{Sessa},\LN_{\varepsilon_{\ln}}}^{d_k}(m)
\]
in the LayerNorm extension.

\paragraph{Causal RoPE-Transformer class}

We also define a causal decoder-only RoPE-Transformer class of functions from
$\mathbb R^{\Tctx\times d_{\mathrm{ext}}}\to \mathbb R^{\Tctx\times d_{\mathrm{ext}}}$,
with internal model width $m$ and adapters.

A width-$m$ RoPE-Transformer block is a standard decoder block operating on $\mathbb R^{\Tctx\times m}$:
it consists of causal self-attention with $j\le t$, RoPE applied to queries and keys in the logits, and fixed scale $\sigk=d_k^{-1/2}$, together with a tokenwise FFN of hidden width $r$ and residual connections in $\mathbb R^m$.
An absolute positional embedding $E\in\mathbb R^{\Tctx\times m}$ is added once at the network input.
Let $\Omega^{H,d_k,r}_{\mathrm{RoPETr,cau}}(m)$ be the set of finite compositions of such blocks on $\mathbb R^{\Tctx\times m}$.

Finally define the adapted function class
\[
\Omega^{H,d_k,r}_{\mathrm{RoPETr,cau}}(d_{\mathrm{ext}}\to m\to d_{\mathrm{ext}})
:=
\Big\{
x\mapsto \mathrm{Unembed}\big(\widetilde g(\mathrm{Embed}(x)+E)\big)
:\ \widetilde g\in \Omega^{H,d_k,r}_{\mathrm{RoPETr,cau}}(m),\ E\in\mathbb R^{\Tctx\times m}
\Big\}.
\]

\subsection{Softmax lemmas}

\begin{lemma}[Softmax concentration]\label{lem:softmax_onehot}
Let $v\in\mathbb R^n$ and let $i^*=\arg\max_i v_i$ be unique.
Let $\Delta=v_{i^*}-\max_{i\ne i^*}v_i>0$ and fix $\delta\in(0,1)$.
For $\sigk>0$, define $p=\mathrm{softmax}(\sigk v)$.
\[
p_{i^*}\ge 1-\delta
\quad\text{whenever}\quad
\sigk\Delta \ge \log\frac{n-1}{\delta}.
\]
\end{lemma}

\begin{proof}
\[
1-p_{i^*}
=
\frac{\sum_{i\ne i^*}e^{\sigk v_i}}{\sum_{i}e^{\sigk v_i}}
\le
\frac{(n-1)e^{\sigk(v_{i^*}-\Delta)}}{e^{\sigk v_{i^*}}}
=
(n-1)e^{-\sigk\Delta}.
\]
Thus $1-p_{i^*}\le\delta$ if $\sigk\Delta\ge \log\frac{n-1}{\delta}$.
\end{proof}

\begin{corollary}[Sharpening at fixed attention scale]
\label{cor:softmax_onehot_fixed_scale}
Let $v\in\mathbb R^n$ and let $i^*=\arg\max_i v_i$ be unique.
Let $\Delta=v_{i^*}-\max_{i\ne i^*}v_i>0$, fix $\delta\in(0,1)$, and fix the attention scale
$\sigk>0$.
For $c>0$, define
\[
p^{(c)}:=\mathrm{softmax}(\sigk\,c^2 v).
\]
Then
\[
p^{(c)}_{i^*}\ge 1-\delta
\quad\text{whenever}\quad
\sigk\,c^2\Delta \ge \log\frac{n-1}{\delta}.
\]
Thus, in the concrete architecture where $\sigk=d_k^{-1/2}$ is fixed,
arbitrarily sharp softmax rows are obtained by scaling the query and key vectors by a common factor $c$.
\end{corollary}

\begin{proof}
Apply Lemma~\ref{lem:softmax_onehot} to the logits $c^2 v$.
\end{proof}

\begin{lemma}[Error of an almost one-hot mixture]\label{lem:convex_error}
Let $(w_j)_{j\in J}\subset\mathbb R^m$ and let $p_j\ge 0$, $\sum_{j\in J}p_j=1$.
If $p_{j^*}\ge 1-\delta$ then
\[
\Big\|\sum_{j\in J}p_j w_j - w_{j^*}\Big\|_2
\le
2\delta\cdot \Vmax,
\]
where $\Vmax := \max_{j\in J}\|w_j\|_2$.
\end{lemma}

\begin{proof}
\[
\sum_{j}p_j w_j - w_{j^*}
=
(p_{j^*}-1)w_{j^*} + \sum_{j\ne j^*}p_j w_j.
\]
Since $\sum_{j\ne j^*}p_j=1-p_{j^*}\le\delta$,
\[
\Big\|\sum_j p_j w_j - w_{j^*}\Big\|_2
\le
|1-p_{j^*}|\|w_{j^*}\|_2 + \sum_{j\ne j^*}p_j\|w_j\|_2
\le 2\delta\,\Vmax,
\]
where $\Vmax:=\max_j\|w_j\|_2$.
\end{proof}

\subsection{RoPE diagonalization and triangular solve}

\begin{lemma}[RoPE-diagonalization]\label{lem:rope_diag}
Fix $\Tctx\ge 2$ and an even query--key width $d_k\in 2\mathbb N$.
For any $\delta\in(0,1)$ there exists a parameter choice with one head and this $d_k$
such that for all $t$,
\[
\alphadrv_{t,t}\ge 1-\delta,\qquad \sum_{\substack{j\le t\\ j\ne t}}\alphadrv_{t,j}\le \delta.
\]
At the architectural scale $\sigk=d_k^{-1/2}$, it suffices to scale the
active query/key pair by a common factor $c_{\mathrm{diag}}>0$ such that
\[
\sigk\,c_{\mathrm{diag}}^{2}\Delta_{\Tctx} \ge \log\frac{\Tctx-1}{\delta},
\qquad
\Delta_{\Tctx}:=1-\max_{s\in\{1,\dots,\Tctx-1\}}\cos(s) \ >\ 0.
\]
\end{lemma}

\begin{proof}
Under the RoPE convention above, $\mathrm{RoPE}_t$ acts pairwise on consecutive
$2$-dimensional coordinates with frequencies $(\omega_r)_{r=0}^{d_k/2-1}$ and
\[
\omega_0=1.
\]
Activate only the first $2$-dimensional pair by choosing
\[
q_0=(1,0,0,\dots,0)\in\mathbb R^{d_k},\qquad
k_0=(1,0,0,\dots,0)\in\mathbb R^{d_k},
\]
and then setting
\[
q=c_{\mathrm{diag}}q_0,\qquad
k=c_{\mathrm{diag}}k_0.
\]
With RoPE, $\tilde q_t=\mathrm{RoPE}_t(q)$ and $\tilde k_j=\mathrm{RoPE}_j(k)$ satisfy
\[
\langle \tilde q_t,\tilde k_j\rangle=c_{\mathrm{diag}}^2\cos(t-j),
\]
since all coordinate pairs except the first are identically zero, and the first pair rotates with
frequency $\omega_0=1$.
For fixed $t$ and $j\le t$, the unique maximum equals $c_{\mathrm{diag}}^2$ at $j=t$.
For $j\ne t$, $s=t-j\in\{1,\dots,\Tctx-1\}$ so $\cos(s)\le 1-\Delta_{\Tctx}$.
Hence the logit gap is at least $c_{\mathrm{diag}}^2\Delta_{\Tctx}$.
Apply Corollary~\ref{cor:softmax_onehot_fixed_scale}.
\end{proof}

\begin{lemma}[Mixing error under diagonalization]\label{lem:mixing_error}
Assume $\|v_j\|_2\le \Vmax$. If $\alphadrv_{t,t}\ge 1-\delta$, then
\[
\Big\|\sum_{j\le t}\alphadrv_{t,j}v_j - v_t\Big\|_2 \le 2\delta \Vmax,
\qquad
\|f-v\|_F \le 2\delta \Vmax\sqrt{\Tctx}.
\]
\end{lemma}

\begin{proof}
Lemma~\ref{lem:convex_error} with $j^*=t$, then sum over $t$.
\end{proof}

\begin{lemma}[Lower-triangular inversion]\label{lem:triangular_inverse}
For every input $x$, $\Bfb(x)\in\mathbb R^{\Tctx\times \Tctx}$ is strictly lower-triangular.
Hence $\Bfb(x)$ is nilpotent, with $\Bfb(x)^{\Tctx}=0$.
\[
(I-\Bfb(x))^{-1}=\sum_{k=0}^{\Tctx-1}\Bfb(x)^k.
\]
\end{lemma}

\begin{proof}
A strictly lower-triangular $\Tctx\times\Tctx$ matrix is nilpotent of index at most $\Tctx$. Hence $\Bfb^{\Tctx}=0$, and the Neumann series terminates after $\Tctx-1$ terms.
\end{proof}

\subsection{Generating positional codes via feedback}
\label{sec:uat_pos_code}

\begin{corollary}[A Sessa block can generate separated positional codes]
\label{cor:pos_code_padding}
Fix any tokenwise pre-normalization map
\[
\Norm:\mathbb R^m\to\mathbb R^m
\]
(applied independently to each token), any even query/key width $d_k\ge 2$, and any model width $m\ge 1$.
Then there exists a single width-$m$ concrete Sessa block
\[
G^{\mathrm{pos}}\in \mathrm{ConcreteSessaBlocks}_{\Norm}(d_k,m)
\]
and vectors $p_0,\dots,p_{\Tctx-1}\in\mathbb R^m$ such that:
\begin{enumerate}[label=(\roman*), leftmargin=*, nosep]
\item for all $h\in\mathbb R^{\Tctx\times m}$ and all $t$,
\[
G^{\mathrm{pos}}(h)_t = h_t+p_t;
\]
\item for any prescribed unit vector $u\in\mathbb R^m$, one may choose
\[
p_t=(\lambda c_t)u
\]
with pairwise distinct scalars $(c_t)_{t=0}^{\Tctx-1}$ and some $\lambda>0$,
so that on any compact $\Kset\subset\mathbb R^{\Tctx\times m}$ the scalar sets
\[
\mathcal I_t:=\{\langle h_t+p_t,u\rangle:\ h\in\Kset\}
\]
are pairwise disjoint after choosing $\lambda$ large enough.
\end{enumerate}
\end{corollary}

\begin{proof}
Fix a prescribed unit vector $u\in\mathbb R^m$.

The construction does not depend on $\Norm$: setting $W^{\mathrm{in}}=0$ gives
\[
u_t=\xnorm_t W^{\mathrm{in}}+b^{\mathrm{in}}=b^{\mathrm{in}},
\]
for all $t$.

Choose $W^{\mathrm{in}}=0$ and choose $b^{\mathrm{in}}$ so that for every token
\[
a_t\equiv a_\ast e_1,
\qquad
g_t\equiv e_1,
\]
for some $a_\ast>0$.
Set
\[
A:=\GELU(a_\ast)>0.
\]
Then
\[
\bar a_t=Ae_1
\qquad
\forall t.
\]

Choose
\[
W_{Qf}=0,\qquad W_{Kf}=0.
\]
Then all forward logits vanish, so each forward row is a causal probability vector.
Choose $W_V$ so that
\[
v_t=e_1
\qquad
\forall t.
\]
Therefore
\[
f_t=\sum_{j\le t}\alphadrv_{t,j}v_j=e_1
\qquad
\forall t.
\]

Choose
\[
W_{Qb}=0,\qquad W_{Kb}=0.
\]
Then for $t\ge 1$,
\[
\alphafb_{t,j}=\frac1t\mathbf 1[j<t],
\qquad
\alphafb_{0,\cdot}=0.
\]

Fix any constant $\gamma\in(0,1)$, and choose
\[
w^\gamma=0,\qquad b^\gamma=\operatorname{arctanh}(\gamma).
\]
Then
\[
\gamma_t\equiv \gamma,
\qquad
\BfbEnt{t}{j}
=
\begin{cases}
0,& t=0,\\[2pt]
\dfrac{\gamma}{t}\mathbf 1[j<t],& t\ge 1.
\end{cases}
\]

Since $f_t=e_1$, we have
\[
s_t=c_t e_1,
\]
where
\[
c_0=1,\qquad
c_t=1+\frac{\gamma}{t}\sum_{j=0}^{t-1}c_j
\qquad (t\ge 1).
\]
Let
\[
S_t:=\sum_{j=0}^{t}c_j,
\qquad
\mu_t:=\frac{S_t}{t+1}.
\]
Then
\[
S_t=\Bigl(1+\frac{\gamma}{t}\Bigr)S_{t-1}+1,
\]
hence
\[
\mu_t=\frac{t+\gamma}{t+1}\mu_{t-1}+\frac1{t+1},
\qquad
\mu_t-\mu_{t-1}=\frac{1-(1-\gamma)\mu_{t-1}}{t+1}.
\]
Since $\mu_0=1<\frac1{1-\gamma}$, an induction gives
\[
\mu_t<\frac1{1-\gamma}
\qquad
\forall t,
\]
so
\[
\mu_t-\mu_{t-1}>0
\qquad
\forall t\ge 1.
\]
Now
\[
c_1=1+\gamma>1=c_0,
\]
and for $t\ge 1$,
\[
c_{t+1}-c_t
=
\gamma\Bigl(\frac{S_t}{t+1}-\frac{S_{t-1}}{t}\Bigr)
=
\gamma(\mu_t-\mu_{t-1})>0.
\]
Therefore $(c_t)$ is strictly increasing.

Choose $W^{\mathrm{out}}$ so that its first row is $\lambda u^\top$ and all other rows are zero, and set
$b^{\mathrm{out}}=0$.
Since
\[
s_t\odot g_t=(c_t e_1)\odot e_1=c_t e_1,
\]
the residual update equals
\[
(s_t\odot g_t)W^{\mathrm{out}}=c_t(\lambda u)=:p_t.
\]
Hence
\[
G^{\mathrm{pos}}(h)_t=h_t+p_t.
\]

Let $\Kset\subset\mathbb R^{\Tctx\times m}$ be compact and set
\[
R:=\sup_{h\in\Kset}\max_t\|h_t\|_2<\infty.
\]
Then
\[
|\langle h_t,u\rangle|\le R
\qquad
\forall h\in\Kset,\ \forall t.
\]
Since the $c_t$ are pairwise distinct, let
\[
\Delta_c:=\min_{s\neq t}|c_s-c_t|>0.
\]
Choose
\[
\lambda>\frac{2R}{\Delta_c}.
\]
Then the shifted scalar sets
\[
\mathcal I_t
=
\{\langle h_t+p_t,u\rangle:\ h\in\Kset\}
=
\{\langle h_t,u\rangle+\lambda c_t:\ h\in\Kset\}
\]
are pairwise disjoint.
\end{proof}

\subsection{Composition error control}

\begin{lemma}[Composition error on thickened compacts]\label{lem:composition_error}
Let $(X,d)$ be a metric space such that closed neighborhoods of compact sets are compact,
for example, $X=\mathbb R^n$ with the Euclidean metric.
Fix a compact $\Kset_1\subset X$ and continuous maps $f_i:X\to X$ for $i=1,\dots,L$.

Fix $\rhonbhd>0$ and define recursively
\[
\widetilde{\Kset}_1:=\Kset_1,\qquad
\Kset_{i+1}:=f_i(\widetilde{\Kset}_i),\qquad
\widetilde{\Kset}_{i+1}:=\overline{\mathcal N}_{\rhonbhd}(\Kset_{i+1})=\{x\in X:\ d(x,\Kset_{i+1})\le \rhonbhd\}.
\]
Then each $\widetilde{\Kset}_i$ is compact.

For every $\varepsilon>0$ there exist tolerances $\delta_1,\dots,\delta_L>0$ such that:
for any continuous maps $g_i:\widetilde{\Kset}_i\to X$ satisfying, for each $i$,
\[
\sup_{x\in \widetilde{\Kset}_i} d\big(f_i(x),g_i(x)\big)\le \delta_i
\quad\text{and}\quad
\delta_i\le \rhonbhd,
\]
the compositions $F:=f_L\circ\cdots\circ f_1$ and $G:=g_L\circ\cdots\circ g_1$ are well-defined on $\Kset_1$
(and in fact $g_i(\widetilde{\Kset}_i)\subset \widetilde{\Kset}_{i+1}$), and
\[
\sup_{x\in \Kset_1} d\big(F(x),G(x)\big)\le \varepsilon.
\]
\end{lemma}

\begin{proof}
Well-definedness is immediate.
Fix $i$ and $x\in \widetilde{\Kset}_i$. By definition, $f_i(x)\in \Kset_{i+1}=f_i(\widetilde{\Kset}_i)$, hence
$d\big(f_i(x),\Kset_{i+1}\big)=0$. Therefore
\[
d\big(g_i(x),\Kset_{i+1}\big)
\le d\big(g_i(x),f_i(x)\big)+d\big(f_i(x),\Kset_{i+1}\big)
\le \delta_i\le \rhonbhd,
\]
so $g_i(x)\in \widetilde{\Kset}_{i+1}$. Thus $g_i(\widetilde{\Kset}_i)\subset \widetilde{\Kset}_{i+1}$ and all compositions are defined.

The remainder of the proof is by induction on $L$.
For $L=1$ it is immediate.

Assume the claim holds for $L-1$. Let
\[
F_{<L}:=f_{L-1}\circ\cdots\circ f_1,\qquad G_{<L}:=g_{L-1}\circ\cdots\circ g_1.
\]
Since $\widetilde{\Kset}_L$ is compact and $f_L$ is continuous, $f_L$ is uniformly continuous on $\widetilde{\Kset}_L$.
Pick $\eta>0$ such that
\[
d(u,v)\le \eta \ \Rightarrow\ d\big(f_L(u),f_L(v)\big)\le \varepsilon/2
\qquad \forall u,v\in \widetilde{\Kset}_L.
\]
Set $\delta_L:=\min(\rhonbhd,\varepsilon/2)$.
By the inductive hypothesis applied with target accuracy $\eta$, choose $\delta_1,\dots,\delta_{L-1}>0$ so that
\[
\sup_{x\in \Kset_1} d\big(F_{<L}(x),G_{<L}(x)\big)\le \eta.
\]
Then for $x\in \Kset_1$, noting that $G_{<L}(x)\in \widetilde{\Kset}_L$ by well-definedness,
\[
d\big(F(x),G(x)\big)
\le
d\big(f_L(F_{<L}(x)),f_L(G_{<L}(x))\big)
+
d\big(f_L(G_{<L}(x)),g_L(G_{<L}(x))\big)
\le
\varepsilon/2+\delta_L
\le \varepsilon.
\]
\end{proof}

\begin{lemma}[Tokenwise GELU approximation]
\label{lem:tokenwise_gelu_vector_uap}
Let \(S\subset\mathbb R^m\) be compact and let \(\Theta:S\to\mathbb R^p\) be continuous.
Then for every \(\eta>0\) there exist \(r\in\mathbb N^*\) and affine maps
\[
A:\mathbb R^m\to\mathbb R^r,
\qquad
B:\mathbb R^r\to\mathbb R^p
\]
such that
\[
\sup_{z\in S}\|B(\mathrm{GELU}(A(z)))-\Theta(z)\|_2\le \eta.
\]
Moreover, if a larger width \(r'\ge r\) is prescribed in advance, the same conclusion still holds with \(r'\) in place of \(r\),
by padding the hidden layer with unused coordinates.
\end{lemma}

\begin{proof}
Apply the standard one-hidden-layer universal approximation theorem for non-polynomial activations coordinatewise to the components of \(\Theta\), and concatenate the resulting hidden units into a single hidden layer. Since \(\mathrm{GELU}\) is continuous and non-polynomial, the theorem applies; see, e.g., \citet{hornik1989universal,leshno1993nonpolynomial}. The padding claim is immediate by adding hidden coordinates with zero incoming and outgoing weights.
\end{proof}

\begin{lemma}[Tokenwise GELU approximation with zero-padding]
\label{lem:tokenwise_gelu_vector_uap_pad}
Let \(S\subset\mathbb R^m\) be compact, let \(\Theta:S\to\mathbb R^{p_0}\) be continuous,
let \(\eta>0\), and let \(r_0\in\mathbb N^*\).
For each \(r\ge r_0\), let
\[
E_r:\mathbb R^{p_0}\hookrightarrow \mathbb R^{p(r)}
\]
be a coordinate zero-padding embedding, where \(p(r)\) may depend on \(r\).
Then there exist \(r\ge r_0\) and affine maps
\[
A:\mathbb R^m\to\mathbb R^r,
\qquad
B:\mathbb R^r\to\mathbb R^{p(r)}
\]
such that
\[
\sup_{z\in S}\big\|B(\mathrm{GELU}(A(z))) - E_r(\Theta(z))\big\|_2\le \eta.
\]
\end{lemma}

\begin{proof}
By Lemma~\ref{lem:tokenwise_gelu_vector_uap}, there exist
\(s\in\mathbb N^*\) and affine maps
\[
\bar A:\mathbb R^m\to\mathbb R^s,
\qquad
\bar B:\mathbb R^s\to\mathbb R^{p_0}
\]
such that
\[
\sup_{z\in S}\big\|\bar B(\mathrm{GELU}(\bar A(z))) - \Theta(z)\big\|_2\le \eta.
\]
Set
\[
r:=\max\{r_0,s\}.
\]
Let
\[
I_{s\to r}:\mathbb R^s\hookrightarrow\mathbb R^r
\]
be the coordinate zero-padding inclusion into the first \(s\) coordinates, and let
\[
\Pi_{r\to s}:\mathbb R^r\to\mathbb R^s
\]
be the projection onto those first \(s\) coordinates.
Define
\[
A:=I_{s\to r}\circ \bar A,
\qquad
B:=E_r\circ \bar B\circ \Pi_{r\to s}.
\]
Then \(A\) is affine and \(B\) is affine.
Since \(\mathrm{GELU}(0)=0\) and \(\mathrm{GELU}\) acts coordinatewise,
\[
\Pi_{r\to s}\big(\mathrm{GELU}(A(z))\big)
=
\Pi_{r\to s}\big(\mathrm{GELU}(I_{s\to r}\bar A(z))\big)
=
\mathrm{GELU}(\bar A(z)).
\]
Hence
\[
B(\mathrm{GELU}(A(z)))
=
E_r\!\big(\bar B(\mathrm{GELU}(\bar A(z)))\big).
\]
Because \(E_r\) is coordinate zero-padding, it is an isometric embedding for the Euclidean norm, so
\[
\big\|B(\mathrm{GELU}(A(z))) - E_r(\Theta(z))\big\|_2
=
\big\|\bar B(\mathrm{GELU}(\bar A(z))) - \Theta(z)\big\|_2.
\]
Taking the supremum over \(z\in S\) gives the claim.
\end{proof}

\subsection{Stability of finite-horizon RoPE attention}
\label{sec:uat_attn_stability}

For fixed $\Tctx$, causal RoPE attention depends continuously on the query, key, and value arrays.
The next two lemmas collect the continuity and near-diagonal transport estimates used below.

\begin{lemma}[Stability of finite-horizon RoPE attention]
\label{lem:rope_attn_tokenwise_stability}
Fix a horizon \(\Tctx\ge 1\), number of heads \(H\ge 1\), even key/query width \(d_k\ge 2\), value width \(d_v\ge 1\),
attention scale \(\sigk>0\), and an output matrix
\[
W^O\in\mathbb R^{H d_v\times m}.
\]
Let \(\Kset\subset\mathbb R^{\Tctx\times m}\) be compact, and define the compact token set
\[
S_{\Kset}:=\{u_t:\ u\in\Kset,\ 0\le t\le \Tctx-1\}\subset\mathbb R^m.
\]

For each head \(a=1,\dots,H\), let
\[
q^a,k^a,\widehat q^a,\widehat k^a:S_{\Kset}\to\mathbb R^{d_k},
\qquad
v^a,\widehat v^a:S_{\Kset}\to\mathbb R^{d_v}
\]
be continuous.
Let \(A,\widehat A:\Kset\to\mathbb R^{\Tctx\times m}\) be the corresponding causal RoPE-attention maps:
for \(u\in\Kset\),
\[
A(u)_t=\Big(\operatorname{concat}_{a=1}^{H} z_t^a(u)\Big)W^O,
\qquad
z_t^a(u):=\sum_{j\le t}\alpha_{t,j}^a(u)\,v^a(u_j),
\]
where
\[
\alpha_{t,j}^a(u)
=
\frac{
\exp\!\Big(\sigk\big\langle \mathrm{RoPE}_t(q^a(u_t)),\,\mathrm{RoPE}_j(k^a(u_j))\big\rangle\Big)\mathbf 1[j\le t]
}{
\sum_{\tau\le t}
\exp\!\Big(\sigk\big\langle \mathrm{RoPE}_t(q^a(u_t)),\,\mathrm{RoPE}_\tau(k^a(u_\tau))\big\rangle\Big)
},
\]
and similarly \(\widehat A\) is defined from \((\widehat q^a,\widehat k^a,\widehat v^a)\).

Then for every \(\varepsilon>0\) there exists \(\eta>0\) such that
\[
\sup_{z\in S_{\Kset}}
\max_{1\le a\le H}
\Big(
\|q^a(z)-\widehat q^a(z)\|_2
+
\|k^a(z)-\widehat k^a(z)\|_2
+
\|v^a(z)-\widehat v^a(z)\|_2
\Big)
\le \eta
\]
implies
\[
\sup_{u\in\Kset}\|A(u)-\widehat A(u)\|_F\le \varepsilon.
\]
\end{lemma}

\begin{proof}
Define the finite-dimensional array space
\[
\mathcal X
:=
\Big((\mathbb R^{d_k})^H\Big)^{\Tctx}
\times
\Big((\mathbb R^{d_k})^H\Big)^{\Tctx}
\times
\Big((\mathbb R^{d_v})^H\Big)^{\Tctx},
\]
and equip it with the max norm
\[
\|(Q,K,V)\|_{\max}
:=
\max\Big\{
\max_{t,a}\|q_t^a\|_2,\ 
\max_{t,a}\|k_t^a\|_2,\ 
\max_{t,a}\|v_t^a\|_2
\Big\}.
\]

Let
\[
\mathcal A:\mathcal X\to\mathbb R^{\Tctx\times m}
\]
denote the finite-horizon causal RoPE-attention operator defined by the displayed formulas above.
RoPE attention is continuous as a composition of continuous finite-dimensional operations.

Now define continuous maps
\[
\Xi,\widehat\Xi:\Kset\to\mathcal X
\]
by collecting the tokenwise arrays:
\[
\Xi(u):=\big((q^a(u_t))_{t,a},\ (k^a(u_t))_{t,a},\ (v^a(u_t))_{t,a}\big),
\]
\[
\widehat\Xi(u):=\big((\widehat q^a(u_t))_{t,a},\ (\widehat k^a(u_t))_{t,a},\ (\widehat v^a(u_t))_{t,a}\big).
\]
Then
\[
A=\mathcal A\circ \Xi,\qquad \widehat A=\mathcal A\circ \widehat\Xi.
\]

The image \(\Xi(\Kset)\subset\mathcal X\) is compact. Fix \(\eta_0>0\); then its closed \(\eta_0\)-neighborhood
\[
\overline{\mathcal N}_{\eta_0}(\Xi(\Kset))
\]
is compact as well. Hence \(\mathcal A\) is uniformly continuous on this neighborhood.
Therefore, for the given \(\varepsilon>0\), there exists \(\delta>0\) such that
\[
x,x'\in \overline{\mathcal N}_{\eta_0}(\Xi(\Kset)),
\qquad
\|x-x'\|_{\max}\le \delta
\quad\Longrightarrow\quad
\|\mathcal A(x)-\mathcal A(x')\|_F\le \varepsilon.
\]

Set \(\eta:=\min\{\eta_0,\delta\}\).
If the stated tokenwise bound holds, then for every \(u\in\Kset\),
\[
\|\Xi(u)-\widehat\Xi(u)\|_{\max}\le \eta,
\]
because each of the three summands is individually bounded by \(\eta\).
In particular,
\[
\widehat\Xi(u)\in \overline{\mathcal N}_{\eta_0}(\Xi(\Kset)).
\]
Applying the uniform continuity estimate to \(\Xi(u)\) and \(\widehat\Xi(u)\) gives
\[
\|A(u)-\widehat A(u)\|_F
=
\|\mathcal A(\Xi(u))-\mathcal A(\widehat\Xi(u))\|_F
\le \varepsilon
\qquad \forall u\in\Kset.
\]
Taking the supremum over \(u\in\Kset\) proves the claim.
\end{proof}

\begin{lemma}[Near-diagonal attention transports values]
\label{lem:near_diag_attn_transports_values}
Fix a horizon \(\Tctx\ge 1\), an output width \(s\ge 1\), and a compact set
\[
\Kset'\subset\mathbb R^{\Tctx\times m}.
\]
Let
\[
S_{\Kset'}:=\{u_t:\ u\in\Kset',\ 0\le t\le \Tctx-1\}\subset\mathbb R^m.
\]
Let \(\phi,v:S_{\Kset'}\to\mathbb R^s\) be continuous, and define
\[
M_\phi:=\sup_{z\in S_{\Kset'}}\|\phi(z)\|_2<\infty.
\]

Suppose a one-head causal attention mechanism on \(\Kset'\) produces weights \(\alpha_{t,j}(u)\) and outputs
\[
f_t(u):=\sum_{j\le t}\alpha_{t,j}(u)\,v(u_j),
\qquad u\in\Kset'.
\]
Assume that for some \(\delta\in(0,1)\) and \(\eta\ge 0\),
\[
\alpha_{t,t}(u)\ge 1-\delta
\qquad
\forall\,u\in\Kset',\ \forall\,t\in\{0,\dots,\Tctx-1\},
\]
and
\[
\sup_{z\in S_{\Kset'}}\|v(z)-\phi(z)\|_2\le \eta.
\]
Then
\[
\sup_{u\in\Kset'}\max_{0\le t\le \Tctx-1}\|f_t(u)-\phi(u_t)\|_2
\le
2\delta(M_\phi+\eta)+\eta.
\]
\end{lemma}

\begin{proof}
Fix \(u\in\Kset'\) and \(t\in\{0,\dots,\Tctx-1\}\).
Set
\[
w_j:=v(u_j)\in\mathbb R^s,\qquad 0\le j\le t.
\]
Then \((\alpha_{t,j}(u))_{j\le t}\) is a convex distribution and
\[
f_t(u)=\sum_{j\le t}\alpha_{t,j}(u)\,w_j.
\]
Moreover,
\[
\|w_j\|_2
\le
\|\phi(u_j)\|_2+\|v(u_j)-\phi(u_j)\|_2
\le
M_\phi+\eta
\qquad \forall j\le t.
\]

Since \(\alpha_{t,t}(u)\ge 1-\delta\), Lemma~\ref{lem:convex_error} yields
\[
\|f_t(u)-w_t\|_2
\le
2\delta(M_\phi+\eta).
\]
Also,
\[
\|w_t-\phi(u_t)\|_2\le \eta.
\]
Hence
\[
\|f_t(u)-\phi(u_t)\|_2
\le
\|f_t(u)-w_t\|_2+\|w_t-\phi(u_t)\|_2
\le
2\delta(M_\phi+\eta)+\eta.
\]
Since this bound is uniform in \(u\) and \(t\), the claim follows.
\end{proof}

\subsection{Universal approximation for causal RoPE-Transformers with adapters}
\begin{lemma}[Universality of causal RoPE-Transformers with adapters]
\label{lem:ropeTr_uap}
Let
\[
\mathcal D\subset\mathbb R^{\Tctx\times d_{\mathrm{ext}}}
\]
be compact and let
\[
F:\mathcal D\to\mathbb R^{\Tctx\times d_{\mathrm{ext}}}
\]
be continuous and causal.
Then for any $\varepsilon>0$ there exist finite $(H,d_k,r,m)$ and
\[
g\in \Omega^{H,d_k,r}_{\mathrm{RoPETr,cau}}(d_{\mathrm{ext}}\to m\to d_{\mathrm{ext}})
\]
such that
\[
\sup_{x\in\mathcal D}\|F(x)-g(x)\|_F<\varepsilon.
\]
Moreover, the construction in the proof allows an arbitrary choice of distinct scalars $(c_t)_{t=0}^{\Tctx-1}$
in Paragraph~3, hence an arbitrary absolute embedding $E$ supported on the pos-scalar coordinate of slice $h=1$
with distinct entries.
\end{lemma}

\begin{proof}
Fix $\varepsilon>0$.

\paragraph{0. Causal factorization}
For each \(t\in\{0,\dots,\Tctx-1\}\), define the compact set of attainable prefixes
\[
\Pref_t:=\{(x_0,\dots,x_t):x\in\mathcal D\}\subset (\mathbb R^{d_{\mathrm{ext}}})^{t+1}.
\]
By Lemma~\ref{lem:causal_prefix_factorization}, there exists a unique continuous map
\[
\widehat F_t:\Pref_t\to\mathbb R^{d_{\mathrm{ext}}},
\qquad
\widehat F_t(x_0,\dots,x_t):=F(x)_t
\quad (x\in\mathcal D).
\]
Since \(\Pref_t\) is compact in Euclidean space, it is closed in \((\mathbb R^{d_{\mathrm{ext}}})^{t+1}\).
By Tietze extension applied coordinatewise \citep{tietze1915extension}, extend \(\widehat F_t\) to a continuous map
\[
F_t:(\mathbb R^{d_{\mathrm{ext}}})^{t+1}\to\mathbb R^{d_{\mathrm{ext}}}
\]
such that \(F(x)_t=F_t(x_0,\dots,x_t)\) for all \(x\in\mathcal D\).
Let \(\Mdom:=\sup_{x\in\mathcal D}\|x\|_F\).

\paragraph{1. Model width}
Set the number of heads to be
\[
H:=\Tctx+1,\qquad d_k:=2,
\]
and choose the per-head value width
\[
d_v:=d_{\mathrm{ext}}+2.
\]
Define
\[
m:=H\,d_v=(\Tctx+1)(d_{\mathrm{ext}}+2).
\]
We index coordinates of $\mathbb R^m$ by head-slices:
\[
\mathbb R^m \cong \bigoplus_{h=1}^{H}\mathbb R^{d_v},
\]
and within each slice $\mathbb R^{d_v}$ we separate content coordinates, the first $d_{\mathrm{ext}}$ coordinates,
a constant coordinate with index $d_{\mathrm{ext}}+1$, and a pos-scalar coordinate with index $d_{\mathrm{ext}}+2$.

\paragraph{2. Adapters}
We now fix concrete adapters $\mathrm{Embed},\mathrm{Unembed}$ of the form introduced in Paragraph~\ref{app:uat_adapters}.
This choice satisfies $\mathrm{Unembed}\circ\mathrm{Embed}=\mathrm{Id}$ on $\mathbb R^{\Tctx\times d_{\mathrm{ext}}}$.
Define the sequence-level affine adapter
\[
\mathrm{Embed}:\mathbb R^{\Tctx\times d_{\mathrm{ext}}}\to\mathbb R^{\Tctx\times m}
\]
tokenwise by placing $x_t$ into the content coordinates of slice $h=1$,
setting the constant coordinate to $1$, and all other coordinates to $0$:
\[
\mathrm{Embed}(x)_t
=
\Big( (x_t,\ 1,\ 0)\ ;\ 0\ ;\ 0\ ;\ \cdots\ ;\ 0 \Big)\in\bigoplus_{h=1}^{H}\mathbb R^{d_{\mathrm{ext}}+2}.
\]
This is an affine map $x_t\mapsto x_tW^{\mathrm{emb}}+b^{\mathrm{emb}}$ for suitable $W^{\mathrm{emb}}$ and $b^{\mathrm{emb}}$.

Define $\mathrm{Unembed}:\mathbb R^{\Tctx\times m}\to\mathbb R^{\Tctx\times d_{\mathrm{ext}}}$ tokenwise by reading out the content coordinates of slice $h=1$:
\[
\mathrm{Unembed}(h)_t := \big(h_t^{(h=1)}\big)_{1:d_{\mathrm{ext}}}\in\mathbb R^{d_{\mathrm{ext}}},
\]
which is exactly a coordinate projection (equivalently, an affine map with $b^{\mathrm{un}}=0$) and satisfies
$\mathrm{Unembed}\circ\mathrm{Embed}=\mathrm{Id}$ on $\mathbb R^{\Tctx\times d_{\mathrm{ext}}}$.
Thus $\mathrm{Unembed}$ is linear and non-expansive in Frobenius norm:
\[
\|\mathrm{Unembed}(U)-\mathrm{Unembed}(U')\|_F\le \|U-U'\|_F
\qquad \forall U,U'\in\mathbb R^{\Tctx\times m}.
\]

Let $\bar x:=\mathrm{Embed}(x)\in\mathbb R^{\Tctx\times m}$. The set $\bar{\mathcal D}:=\mathrm{Embed}(\mathcal D)$ is compact.

\paragraph{3. Absolute positional code}
Choose distinct scalars $c_0,\dots,c_{\Tctx-1}\in\mathbb R$ and define $E\in\mathbb R^{\Tctx\times m}$ by:
\[
E_t \ \text{is zero in all coordinates except the pos-scalar coordinate of slice }h=1,\ \text{where it equals }c_t.
\]
Thus for all $x\in\mathcal D$ and all $t$,
\[
(\bar x_t+E_t)^{(h=1)}_{d_{\mathrm{ext}}+2} = c_t,
\]
i.e. the pos-scalar is exactly $c_t$, independent of $x$.

\paragraph{4. Prefix encoding}
Fix a diagonalization tolerance \(\delta\in(0,1)\), to be chosen sufficiently small later.
Under the standing RoPE convention fixed above, when $d_k=2$ there is only one rotary pair and
$\omega_0=1$, so
\[
\mathrm{RoPE}_t(z)=R_t z
\]
with $R_t$ the planar rotation by angle $t$ radians \citep{su2021roformer}.
Construct a single causal RoPE-attention sublayer whose output at time $t$ stores
\[
x_t,\ x_{t-1},\ \dots,\ x_0
\]
in the content coordinates of slices $h=2,3,\dots,t+2$, respectively. Equivalently, lag $\ell=0,\dots,t$ is stored in slice $h=\ell+2$, and all active slices $h=2,\dots,H$ are controlled uniformly via the one-hot estimates below.
 
Because slice $h=1$ has a constant coordinate equal to $1$, we may choose the linear maps $W^Q_h,W^K_h$
so that for every token representation $u$:
\[
q_t^{(h)} = (u_t^{(h=1)})_{d_{\mathrm{ext}}+1}\,\bar q^{(h)}=\bar q^{(h)}\in\mathbb R^{2},\qquad
k_j^{(h)} = (u_j^{(h=1)})_{d_{\mathrm{ext}}+1}\,\bar k=\bar k\in\mathbb R^{2},
\]
for fixed vectors $\bar q^{(h)},\bar k\in\mathbb R^{2}$.
Fix a scaling factor $c_{\mathrm{pack}}>0$.
We set $\bar k=c_{\mathrm{pack}}(1,0)$ and for head $h\in\{2,\dots,H\}$ set
\[
\bar q^{(h)} := c_{\mathrm{pack}}\mathrm{RoPE}_{-(h-2)}(1,0)\in\mathbb R^{2}.
\]
Under RoPE inside logits, for $j\le t$,
\[
\left\langle \mathrm{RoPE}_t(\bar q^{(h)}),\ \mathrm{RoPE}_j(\bar k)\right\rangle
=
c_{\mathrm{pack}}^2\cos\big((t-(h-2))-j\big).
\]
Define for each \((t,h)\) the maximizer
\[
j^*(t,h)\in\arg\max_{0\le j\le t}\cos\big((t-(h-2))-j\big).
\]
For \(h\le t+2\), the unique maximizer is \(j^*(t,h)=t-(h-2)\), since the maximum value \(1\) is attained only at argument \(0\).
For \(h>t+2\), all arguments \((t-(h-2))-j\) are distinct negative integers, and the corresponding cosine values are pairwise distinct
(since \(\cos(a)=\cos(b)\) implies \(a=\pm b+2\pi k\) for some \(k\in\mathbb Z\), and for integers \(a,b\) this forces \(k=0\) because \(2\pi\) is irrational, hence \(a=\pm b\)).
Thus the maximizer is unique for every \((t,h)\).

Let
\[
v_{t,h}(j):=\cos\big((t-(h-2))-j\big),\qquad j\in\{0,\dots,t\},
\]
and for \(t\ge 1\) define
\[
\Delta_{t,h}:=
v_{t,h}\big(j^*(t,h)\big)
-
\max_{j\in\{0,\dots,t\}\setminus\{j^*(t,h)\}} v_{t,h}(j)
>0.
\]
Since the set of pairs \((t,h)\) is finite, the uniform gap
\[
\Delta_*:=\min_{\substack{t\in\{1,\dots,\Tctx-1\}\\ h\in\{2,\dots,H\}}}\Delta_{t,h}
\]
is strictly positive.
For \(t=0\), the row is exactly one-hot.

Choose $c_{\mathrm{pack}}$ such that
\[
\sigk c_{\mathrm{pack}}^2\Delta_* \ge \log\frac{\Tctx-1}{\delta}.
\]
Then by Corollary~\ref{cor:softmax_onehot_fixed_scale}, for every $x\in\mathcal D$, every $t\ge 1$, and every head $h\in\{2,\dots,H\}$,
\[
\alphadrv^{(h)}_{t,\,j^*(t,h)} \ge 1-\delta.
\]
For $t=0$ the distribution is exactly one-hot on $j=0$.

For heads $h=2,\dots,H$, choose $W^V_h$ so that the value vector copies the content coordinates of slice $h=1$
(and has zeros in the last two coordinates of the head output):
\[
v_j^{(h)} = \big( x_j,\ 0,\ 0 \big)\in\mathbb R^{d_{\mathrm{ext}}+2}.
\]
For head $h=1$, set $W^V_1\equiv 0$, so head $1$ contributes $0$.

Let $f_t\in\mathbb R^m$ denote the concatenation of head outputs.
Choose $W^O=I_m$.
Since slices $h\ge 2$ are initially zero, the residual update
\[
h_t \leftarrow h_t + f_t
\]
injects the head outputs directly into these slices.

Let $\Vmax:=\sup_{x\in\mathcal D}\max_{j}\|x_j\|_2 \le \Mdom$.
For each $t$ and each head $h\in\{2,\dots,H\}$, by Lemma~\ref{lem:convex_error},
\[
\left\|\big(f_t^{(h)}\big)_{1:d_{\mathrm{ext}}} - x_{j^*(t,h)}\right\|_2 \le 2\delta\,\Vmax \le 2\delta\,\Mdom.
\]
In particular, for $h\le t+2$ we have $j^*(t,h)=t-(h-2)$, hence slices $h=2,\dots,t+2$ recover
$(x_t,x_{t-1},\dots,x_0)$ with per-slice content error at most $2\delta\,\Vmax$.

\paragraph{5. Ideal encoded state and target map}
Fix $H:=\Tctx+1$ heads indexed by $h=1,\dots,H$, with head $h=1$ unused as before.
For each $(t,h)$ with $t\in\{0,\dots,\Tctx-1\}$ and $h\in\{2,\dots,H\}$ define the deterministic index
\[
j^*(t,h)\in\arg\max_{0\le j\le t}\ \cos\big((t-(h-2))-j\big).
\]

With the same $c_{\mathrm{pack}}$ chosen in Paragraph~4 so that
\[
\sigk\,c_{\mathrm{pack}}^2\Delta_* \ \ge\ \log\frac{\Tctx-1}{\delta},
\]
Corollary~\ref{cor:softmax_onehot_fixed_scale} gives, for every $x\in\mathcal D$, every $t\ge 1$, and every head $h\in\{2,\dots,H\}$,
the causal attention distribution over $j\le t$ satisfies
\[
\alphadrv^{(h)}_{t,\,j^*(t,h)} \ \ge\ 1-\delta.
\]
For $t=0$ the attention is exactly one-hot.

Define $\widehat h_t(x)\in\mathbb R^m$, where $m=(\Tctx+1)(d_{\mathrm{ext}}+2)$, by
letting slice $h=1$ equal $(x_t,1,c_t)$ in coordinates $(1{:}d_{\mathrm{ext}},\,d_{\mathrm{ext}}{+}1,\,d_{\mathrm{ext}}{+}2)$ and zero elsewhere, and for each slice $h=2,\dots,H$ placing $x_{j^*(t,h)}$ in the first $d_{\mathrm{ext}}$ coordinates and zeros in the last two;
and set
\[
\widehat S := \{\widehat h_t(x): x\in\mathcal D,\ t\in\{0,\dots,\Tctx-1\}\}\subset\mathbb R^m.
\]
Then $\widehat S$ is compact as a continuous image of a compact set.

For each fixed $t\in\{0,\dots,\Tctx-1\}$, define the affine map, in fact linear,
\[
\mathrm{Read}_t:\mathbb R^m\to (\mathbb R^{d_{\mathrm{ext}}})^{t+1}
\]
by reading the content coordinates of slices $h=2,\dots,t+2$ in reverse order:
\[
\mathrm{Read}_t(u)
:=
\Big(
\big(u^{(t+2)}\big)_{1:d_{\mathrm{ext}}},\,
\big(u^{(t+1)}\big)_{1:d_{\mathrm{ext}}},\,
\dots,\,
\big(u^{(2)}\big)_{1:d_{\mathrm{ext}}}
\Big).
\]
Equivalently, for $\ell=0,\dots,t$,
\[
\big(\mathrm{Read}_t(u)\big)_\ell = \big(u^{(t-\ell+2)}\big)_{1:d_{\mathrm{ext}}}.
\]
By construction of the ideal encoded state and because $j^*(t,h)=t-(h-2)$ for $h\le t+2$,
\[
\mathrm{Read}_t\big(\widehat h_t(x)\big)=(x_0,\dots,x_t)
\qquad \forall x\in\mathcal D.
\]

Thus the pos-scalar coordinate identifies $t$, while the encoded slices determine the prefix $(x_0,\dots,x_t)$.

Decompose $\widehat S$ as the finite disjoint union $\widehat S=\bigsqcup_{t=0}^{\Tctx-1}\widehat S_t$ where
$\widehat S_t:=\{\widehat h_t(x):x\in\mathcal D\}$.
Each $\widehat S_t$ is compact and contained in the affine hyperplane
$\{u\in\mathbb R^m:\ (u^{(h=1)})_{d_{\mathrm{ext}}+2}=c_t\}$.
Since the scalars $c_t$ are distinct, the sets $\widehat S_t$ are pairwise separated. Therefore \(\widehat\Phi\) is continuous on \(\widehat S\) once each restriction \(\widehat\Phi|_{\widehat S_t}\) is continuous.
Now fix $t$. For every $u=\widehat h_t(x)\in\widehat S_t$, by the defining property of $F_t$ from Paragraph~0 and by the
readout identity above,
\[
\widehat\Phi(u)
=F(x)_t
=F_t(x_0,\dots,x_t)
=F_t\!\big(\mathrm{Read}_t(u)\big).
\]
Therefore
\[
\widehat\Phi|_{\widehat S_t} = F_t\circ \mathrm{Read}_t|_{\widehat S_t}.
\]
$\mathrm{Read}_t$ is a linear map, and $F_t:(\mathbb R^{d_{\mathrm{ext}}})^{t+1}\to\mathbb R^{d_{\mathrm{ext}}}$ is continuous, so
$\widehat\Phi|_{\widehat S_t}$ is continuous.
Thus $\widehat\Phi$ is continuous on $\widehat S$.

By Tietze extension applied coordinatewise, extend $\widehat\Phi$ to a continuous $\widetilde\Phi:\mathbb R^m\to\mathbb R^{d_{\mathrm{ext}}}$.

\paragraph{6. FFN approximation}
Let $h^{\mathrm{enc}}_t(x)\in\mathbb R^m$ denote the token state after the first RoPE-attention block,
constructed in Paragraph~4, with $W^O=I_m$, head $h=1$ set to zero, and the FFN set to zero.
Slice $h=1$ is unchanged by the residual, since the concatenated head output has zero slice $h=1$,
so $(h^{\mathrm{enc}}_t(x))^{(h=1)}=(x_t,1,c_t)$ exactly.

For each head slice $h\in\{2,\dots,H\}$, by the encoding construction in Paragraph~4 we have
$\|v_j^{(h)}\|_2\le \Vmax$ and $\alphadrv^{(h)}_{t,\,j^*(t,h)}\ge 1-\delta$.
Therefore Lemma~\ref{lem:convex_error} gives, for each $x\in\mathcal D$, each $t$, each $h\in\{2,\dots,H\}$,
\[
\left\|\big(h^{\mathrm{enc}}_t(x)\big)^{(h)}_{1:d_{\mathrm{ext}}} - x_{j^*(t,h)}\right\|_2 \ \le\ 2\delta \Vmax,
\]
and the last two coordinates of each slice are exactly zero on both sides.
Therefore, for each $(x,t)$,
\[
\big\|h^{\mathrm{enc}}_t(x)-\widehat h_t(x)\big\|_2
\ \le\ \sqrt{\sum_{h=2}^{H} (2\delta \Vmax)^2}
\ =\ 2\delta \Vmax\sqrt{\Tctx}.
\]
In particular,
\[
\sup_{x\in\mathcal D}\max_{t}\ \big\|h^{\mathrm{enc}}_t(x)-\widehat h_t(x)\big\|_2
\ \le\ 2\delta \Vmax\sqrt{\Tctx}.
\]

Let
\[
S_{\mathrm{enc}}:=\{h^{\mathrm{enc}}_t(x):x\in\mathcal D,\ t=0,\dots,\Tctx-1\}\subset\mathbb R^m
\]
(compact).
Since \(\widehat S\) is compact, for every radius \(r_{\mathrm{nbhd}}>0\) the closed neighborhood
\[
\overline{\mathcal N}_{r_{\mathrm{nbhd}}}(\widehat S)
:=
\{u\in\mathbb R^m:\ \mathrm{dist}(u,\widehat S)\le r_{\mathrm{nbhd}}\}
\]
is compact. Fix such an \(r_{\mathrm{nbhd}}>0\).

By uniform continuity of \(\widetilde\Phi\) on the compact set \(\overline{\mathcal N}_{r_{\mathrm{nbhd}}}(\widehat S)\), there exists a continuity tolerance
\[
\delta_{\mathrm{UC}}>0
\]
such that
\[
u,v\in \overline{\mathcal N}_{r_{\mathrm{nbhd}}}(\widehat S),
\qquad
\|u-v\|_2\le \delta_{\mathrm{UC}}
\quad\Longrightarrow\quad
\|\widetilde\Phi(u)-\widetilde\Phi(v)\|_2\le \varepsilon/(3\sqrt{\Tctx}).
\]

Now choose the diagonalization parameter \(\delta\in(0,1)\) above small enough so that
\[
2\delta \Vmax\sqrt{\Tctx}\le \min\{r_{\mathrm{nbhd}},\delta_{\mathrm{UC}}\}.
\]
Then \(S_{\mathrm{enc}}\subset \overline{\mathcal N}_{r_{\mathrm{nbhd}}}(\widehat S)\), and for all \(x\in\mathcal D\) and all \(t\),
\[
\|h^{\mathrm{enc}}_t(x)-\widehat h_t(x)\|_2\le \delta_{\mathrm{UC}}.
\]
Hence
\[
\|\widetilde\Phi(h^{\mathrm{enc}}_t(x))-\widetilde\Phi(\widehat h_t(x))\|_2\le \varepsilon/(3\sqrt{\Tctx}).
\]
Since $\widetilde\Phi(\widehat h_t(x))=\widehat\Phi(\widehat h_t(x))=F(x)_t$ by construction, it follows that
\[
\|\widetilde\Phi(h^{\mathrm{enc}}_t(x)) - F(x)_t\|_2 \le \varepsilon/(3\sqrt{\Tctx}).
\]

Define the continuous map \(\Psi:S_{\mathrm{enc}}\to\mathbb R^{d_{\mathrm{ext}}}\) by
\[
\Psi(u) := \widetilde\Phi(u) - \big(u^{(h=1)}\big)_{1:d_{\mathrm{ext}}},
\]
i.e. the increment needed (in slice $h=1$ content) to turn the current content into $\widetilde\Phi(u)$.
By the universal approximation theorem for tokenwise GELU FFNs \citep{leshno1993nonpolynomial,hornik1989universal}, there exists a tokenwise FFN (hidden width $r$ large enough)
whose output $\mathrm{FFN}(h)_t\in\mathbb R^m$ is supported only on slice $h=1$ content coordinates and satisfies
\[
\sup_{u\in S_{\mathrm{enc}}}
\left\|\big(\mathrm{FFN}(u)\big)^{(h=1)}_{1:d_{\mathrm{ext}}} - \Psi(u)\right\|_2
\ \le\ \varepsilon/(3\sqrt{\Tctx}),
\]
and $\mathrm{FFN}(u)$ equals $0$ on all other coordinates.
Applying this tokenwise, define the sequence-level FFN by $\mathrm{FFN}(h)_t:=\mathrm{FFN}(h_t)$.
Using the residual connection in the second block (with its attention set to zero), the slice $h=1$ content becomes
\[
\big(h^{\mathrm{enc}}_t(x)\big)^{(h=1)}_{1:d_{\mathrm{ext}}} + \big(\mathrm{FFN}(h^{\mathrm{enc}}_t(x))\big)^{(h=1)}_{1:d_{\mathrm{ext}}}
\approx \widetilde\Phi\big(h^{\mathrm{enc}}_t(x)\big)
\approx F(x)_t.
\]
Combining the encoding and FFN errors yields for each $t$
\[
\left\| \big(h^{\mathrm{out}}_t(x)\big)^{(h=1)}_{1:d_{\mathrm{ext}}} - F(x)_t \right\|_2
\le \varepsilon/\sqrt{\Tctx},
\]
hence $\|F(x)-g(x)\|_F\le \varepsilon$ uniformly on $\mathcal D$ after applying $\mathrm{Unembed}$.
\end{proof}

\subsection{Direct Sessa building blocks}
\paragraph{Storage decomposition}
Fix a model width
\[
m=(\Tctx+1)d_{\mathrm{ext}}+2.
\]
Write $\mathbb R^m$ as the orthogonal direct sum of coordinate subspaces
\[
\mathbb R^m
=
U_0\oplus U_1\oplus\cdots\oplus U_{\Tctx-1}\oplus U_{\mathrm{out}}\oplus \mathrm{span}\{e_{\mathrm{const}},e_{\mathrm{pos}}\},
\]
where each $U_\ell$ is a coordinate copy of $\mathbb R^{d_{\mathrm{ext}}}$ and
$U_{\mathrm{out}}$ is a coordinate copy of $\mathbb R^{d_{\mathrm{ext}}}$.

Fix linear isometries
\[
J_\ell:\mathbb R^{d_{\mathrm{ext}}}\to U_\ell
\qquad (\ell=0,\dots,\Tctx-1),
\qquad
J_{\mathrm{out}}:\mathbb R^{d_{\mathrm{ext}}}\to U_{\mathrm{out}},
\]
and let
\[
R_\ell:=J_\ell^{-1}:U_\ell\to\mathbb R^{d_{\mathrm{ext}}},
\qquad
R_{\mathrm{out}}:=J_{\mathrm{out}}^{-1}:U_{\mathrm{out}}\to\mathbb R^{d_{\mathrm{ext}}}.
\]

Let $\pi_\ell:\mathbb R^m\to U_\ell$ denote the projection onto $U_\ell$, let $\pi_{\mathrm{out}}:\mathbb R^m\to U_{\mathrm{out}}$ denote the projection onto $U_{\mathrm{out}}$, and let
\[
\pi_{\mathrm{st}}:\mathbb R^m\to
U_0\oplus\cdots\oplus U_{\Tctx-1}\oplus \mathrm{span}\{e_{\mathrm{const}},e_{\mathrm{pos}}\}
\]
denote the projection onto the storage slice.

For each \(\ell\in\{1,\dots,\Tctx-1\}\), let
\[
T_{0\to \ell}:=J_\ell\circ R_0:U_0\to U_\ell
\]
denote the fixed coordinate-copy isomorphism, and let
\[
T_{0\to \mathrm{out}}:=J_{\mathrm{out}}\circ R_0:U_0\to U_{\mathrm{out}}
\]
denote the corresponding copy map into the output slice.

Let
\[
\iota_{\mathrm{st}}:\pi_{\mathrm{st}}(\mathbb R^m)\to \mathbb R^m
\]
denote the linear lift obtained by restoring the output slice as the copy of $U_0$, i.e.
\[
\pi_{\mathrm{st}}(\iota_{\mathrm{st}}(z))=z,
\qquad
\pi_{\mathrm{out}}(\iota_{\mathrm{st}}(z))=T_{0\to \mathrm{out}}(\pi_0(z)).
\]

\begin{lemma}[Uniform small-signal linearization of GELU]
\label{lem:gelu_small_signal}
Let $K\subset\mathbb R^q$ be compact.
Then
\[
\sup_{u\in K}
\left\|
\frac{2}{\varepsilon}\,\mathrm{GELU}(\varepsilon u)-u
\right\|_2
\longrightarrow 0
\qquad\text{as }\varepsilon\downarrow 0.
\]
Consequently, for every compact $K\subset\mathbb R^p$, every linear map $L:\mathbb R^p\to\mathbb R^q$, and every $\eta>0$,
there exists $\varepsilon>0$ such that
\[
\sup_{z\in K}
\left\|
\frac{2}{\varepsilon}\,\mathrm{GELU}(\varepsilon Lz)-Lz
\right\|_2
\le \eta.
\]
\end{lemma}

\begin{proof}
$\mathrm{GELU}$ is $C^1$ and $\mathrm{GELU}'(0)=1/2$. Hence
\[
\mathrm{GELU}(u)=\frac12 u+r(u),
\qquad
\frac{\|r(u)\|_2}{\|u\|_2}\to 0
\quad\text{as }u\to 0.
\]
Apply this uniformly on the compact set $\varepsilon K$.
The second statement follows by substituting $u=Lz$.
\end{proof}

\begin{lemma}[A single Sessa block copies one lag into a dedicated slice]
\label{lem:one_lag_pack}
Fix $\ell\in\{1,\dots,\Tctx-1\}$ and a compact set $\Kset\subset\mathbb R^{\Tctx\times m}$.
Define the compact source-token set
\[
S_0:=\{\pi_0(h_t):\ h\in\Kset,\ 0\le t\le \Tctx-1\}\subset U_0.
\]
Then for every $\eta>0$ there exists a width-$m$ concrete Sessa block
\[
G^{\mathrm{lag}}_\ell\in \mathrm{ConcreteSessaBlocks}_{\Id}(2,m)
\]
such that:
\begin{enumerate}[label=(\roman*), leftmargin=*, nosep]
\item feedback is turned off identically, i.e. $\gamma_t\equiv 0$;

\item for every $h\in\Kset$ and every $t$, the block can be chosen so that its input projection depends only on the source slice \(U_0\) (and fixed biases), i.e. it ignores all coordinates in \(U_r\) for \(r\neq 0\), as well as \(U_{\mathrm{out}}\), \(e_{\mathrm{const}}\), and \(e_{\mathrm{pos}}\);
\[
\pi_r(G^{\mathrm{lag}}_\ell(h)_t)=\pi_r(h_t)
\qquad
\text{for all }r\in\{0,\dots,\Tctx-1\}\setminus\{\ell\},
\]
and the coordinates in $U_{\mathrm{out}}$, $e_{\mathrm{const}}$, and $e_{\mathrm{pos}}$ are unchanged;

\item if
\[
j^*(t,\ell)\in\arg\max_{0\le j\le t}\cos\big((t-\ell)-j\big),
\]
then
\[
\sup_{h\in\Kset}\max_{0\le t\le \Tctx-1}
\left\|
\pi_\ell\big(G^{\mathrm{lag}}_\ell(h)_t\big)-\pi_\ell(h_t)-T_{0\to\ell}\!\big(\pi_0(h_{j^*(t,\ell)})\big)
\right\|_2
\le \eta.
\]
In particular, for $t\ge \ell$ one has $j^*(t,\ell)=t-\ell$.
\end{enumerate}
\end{lemma}

\begin{proof}
Reserve one coordinate of $a_t$ for a constant bias so that the corresponding coordinate of $\bar a_t$ is strictly positive.
Fix a diagonalization tolerance \(\delta\in(0,1)\), to be chosen sufficiently small later.
Choose \(W_{Qf},W_{Kf}\) so that only the designated constant coordinate of \(\bar a_t\) contributes to the forward queries and keys, and set
\[
q^f_t\equiv q^{(\ell)}_{\mathrm{diag}}:=c_\ell\,\mathrm{RoPE}_{-\ell}(1,0),
\qquad
k^f_t\equiv k_{\mathrm{diag}}:=c_\ell(1,0)\in\mathbb R^{2},
\]
for some scale \(c_\ell>0\).
Then for \(j\le t\),
\[
\big\langle \mathrm{RoPE}_t(q^f_t),\,\mathrm{RoPE}_j(k^f_j)\big\rangle
=
c_\ell^2\cos\big((t-\ell)-j\big).
\]
For each \(t\), the maximizer of \(j\mapsto \cos((t-\ell)-j)\) on \(\{0,\dots,t\}\) is unique; denote it by \(j^*(t,\ell)\).
Uniqueness is proved as in Lemma~\ref{lem:rope_diag}: for \(t\ge \ell\), the maximizer is \(j=t-\ell\), while for \(t<\ell\) the arguments are distinct negative integers and therefore yield distinct cosine values.
Hence, by the proof of Lemma~\ref{lem:rope_diag} together with Corollary~\ref{cor:softmax_onehot_fixed_scale}, after choosing \(c_\ell\) large enough we obtain
\[
\alpha_{t,j^*(t,\ell)}\ge 1-\delta
\qquad
\forall t=0,\dots,\Tctx-1.
\]

Use \(d_{\mathrm{ext}}\) further coordinates of \(a_t\) to encode the source slice via
\[
a^{\mathrm{src}}_t=\varepsilon\,\pi_0(h_t)\in U_0.
\]
By Lemma~\ref{lem:gelu_small_signal}, after choosing \(\varepsilon>0\) small enough, these coordinates of
\[
\bar a_t=\mathrm{GELU}(a_t)
\]
can be linearly mapped by \(W_V\) to approximate \(T_{0\to\ell}(\pi_0(h_t))\) uniformly on the compact source-token set \(S_0\).
Choose \(W_V\) so that the resulting value vector lives only in the destination slice \(U_\ell\).
Set \(g\equiv \mathbf 1\), set \(W^{\mathrm{out}}\) to be the identity on \(U_\ell\) and zero on all other coordinates,
and set \(b^{\mathrm{out}}=0\).
Choose the feedback branch identically zero.

Define the compact token set
\[
S_{\Kset}:=\{h_t:\ h\in\Kset,\ 0\le t\le \Tctx-1\}\subset\mathbb R^m,
\]
and let
\[
\phi(z):=T_{0\to \ell}(\pi_0(z)),
\qquad z\in S_{\Kset}.
\]
Set
\[
M_\ell:=\sup_{z\in S_{\Kset}}\|\phi(z)\|_2<\infty.
\]
Choose the small-signal approximation so that the induced value map \(v:S_{\Kset}\to U_\ell\) satisfies
\[
\sup_{z\in S_{\Kset}}\|v(z)-\phi(z)\|_2\le \eta_{\mathrm{val}}.
\]
Then for every \(h\in\Kset\) and every \(t\), Lemma~\ref{lem:convex_error} applied to
\[
f_t(h)=\sum_{j\le t}\alpha_{t,j}\,v(h_j)
\]
with distinguished index \(j^*(t,\ell)\) gives
\[
\big\|f_t(h)-v(h_{j^*(t,\ell)})\big\|_2
\le
2\delta\,(M_\ell+\eta_{\mathrm{val}}).
\]
Therefore
\[
\big\|f_t(h)-\phi(h_{j^*(t,\ell)})\big\|_2
\le
2\delta\,(M_\ell+\eta_{\mathrm{val}})+\eta_{\mathrm{val}}.
\]
Since
\[
\phi(h_{j^*(t,\ell)})=T_{0\to\ell}\big(\pi_0(h_{j^*(t,\ell)})\big),
\]
choosing \(\delta\) and \(\eta_{\mathrm{val}}\) sufficiently small makes the total error at most \(\eta\).
All remaining coordinates are unchanged by construction.
\end{proof}

\begin{lemma}[A diagonal Sessa block realizes a block of tokenwise GELU units]
\label{lem:one_batch_tokenwise}
Let
\[
A:\pi_{\mathrm{st}}(\mathbb R^m)\to\mathbb R^q
\]
be affine, with
\[
q\in\{1,\dots,m-1\},
\]
and let
\[
B:\mathbb R^q\to U_{\mathrm{out}}
\]
be linear.
Fix a compact set
\[
S\subset \pi_{\mathrm{st}}(\mathbb R^m).
\]
Then for every $\eta>0$ there exists a width-$m$ concrete Sessa block
\[
G^{\mathrm{batch}}\in \mathrm{ConcreteSessaBlocks}_{\Id}(2,m)
\]
such that:
\begin{enumerate}[label=(\roman*), leftmargin=*, nosep]
\item feedback is turned off identically;
\item the storage coordinates are preserved exactly:
\[
\pi_{\mathrm{st}}(G^{\mathrm{batch}}(h)_t)=\pi_{\mathrm{st}}(h_t)
\qquad\forall h,\forall t;
\]
\item the input projection ignores the current output slice, i.e. it depends only on $\pi_{\mathrm{st}}(h_t)$;
\item for every sequence $h$ whose tokenwise storage states lie in $S$,
\[
\sup_t
\left\|
\pi_{\mathrm{out}}(G^{\mathrm{batch}}(h)_t)-\pi_{\mathrm{out}}(h_t)-B(\mathrm{GELU}(A(\pi_{\mathrm{st}}(h_t))))
\right\|_2
\le \eta.
\]
\end{enumerate}
\end{lemma}

\begin{proof}
Let the first \(q\) coordinates of \(a_t\) encode the affine preactivations
\[
A(\pi_{\mathrm{st}}(h_t)).
\]
Reserve one additional coordinate of $a_t$ for a constant bias so that the corresponding coordinate of $\bar a_t$ is strictly positive.
Choose $W_{Qf},W_{Kf}$ so that only that coordinate contributes to the forward queries and keys, yielding constant queries and keys that make the forward attention arbitrarily close to diagonal uniformly in $t$ by Lemma~\ref{lem:rope_diag}.
 
Choose \(W_V\) so that the resulting value vector equals
\[
B(\bar a_{1:q})\in U_{\mathrm{out}}
\]
in the output slice and is zero on the storage slice.
Choose \(g\equiv \mathbf 1\), choose \(W^{\mathrm{out}}\) to be the identity on \(U_{\mathrm{out}}\) and zero on the storage slice,
set \(b^{\mathrm{out}}=0\), and set the columns of the input projection corresponding to the current output slice \(U_{\mathrm{out}}\) to zero.
Choose the feedback branch identically zero.

Let
\[
\phi(z):=B(\mathrm{GELU}(A(z))),
\qquad z\in S,
\]
and set
\[
M_\phi:=\sup_{z\in S}\|\phi(z)\|_2<\infty.
\]
Because the input projection ignores the current output slice, the preactivations \(a_t\) depend only on \(\pi_{\mathrm{st}}(h_t)\), hence for every sequence \(h\) whose tokenwise storage states lie in \(S\), the resulting value vector is exactly
\[
v_t=\phi(\pi_{\mathrm{st}}(h_t))\in U_{\mathrm{out}}.
\]
By the diagonal forward-attention construction, after choosing the diagonalization tolerance \(\delta\in(0,1)\) sufficiently small we have
\[
\alpha_{t,t}\ge 1-\delta
\qquad\forall t=0,\dots,\Tctx-1.
\]
Therefore, for every such sequence \(h\) and every \(t\), Lemma~\ref{lem:convex_error} applied to
\[
f_t=\sum_{j\le t}\alpha_{t,j}v_j
\]
with distinguished index \(j^*=t\) gives
\[
\|f_t-v_t\|_2\le 2\delta M_\phi.
\]
Choosing \(\delta\) so that \(2\delta M_\phi\le \eta\) (trivial if \(M_\phi=0\)) yields
\[
\sup_t\left\|f_t-\phi(\pi_{\mathrm{st}}(h_t))\right\|_2\le \eta.
\]
Since the residual update is added only in \(U_{\mathrm{out}}\), this gives the desired conclusion.
\end{proof}

\begin{corollary}[Tokenwise GELU approximation by stacked Sessa blocks]
\label{cor:serialized_tokenwise_concrete}
Let
\[
S\subset \pi_{\mathrm{st}}(\mathbb R^m)
\]
be compact and let
\[
\Theta:S\to U_{\mathrm{out}}
\]
be continuous.
Then for every $\eta>0$ there exists a finite composition
\[
G^{\mathrm{tok}}=G_M^{\mathrm{batch}}\circ\cdots\circ G_1^{\mathrm{batch}},
\qquad
G_b^{\mathrm{batch}}\in \mathrm{ConcreteSessaBlocks}_{\Id}(2,m),
\]
such that:
\begin{enumerate}[label=(\roman*), leftmargin=*, nosep]
\item every $G_b^{\mathrm{batch}}$ preserves the storage slice exactly and ignores the current output slice in its input projection;
\item for every sequence $h$ whose tokenwise storage states lie in $S$,
\[
\pi_{\mathrm{st}}(G^{\mathrm{tok}}(h)_t)=\pi_{\mathrm{st}}(h_t)\qquad\forall t,
\]
and
\[
\sup_t
\left\|
\pi_{\mathrm{out}}(G^{\mathrm{tok}}(h)_t)-\pi_{\mathrm{out}}(h_t)-\Theta(\pi_{\mathrm{st}}(h_t))
\right\|_2
\le \eta.
\]
\end{enumerate}
\end{corollary}

\begin{proof}
By Lemma~\ref{lem:tokenwise_gelu_vector_uap}, for every $\eta'>0$ there exist
a width $R\in\mathbb N^*$, an affine map
\[
A_{\mathrm{tot}}:\pi_{\mathrm{st}}(\mathbb R^m)\to\mathbb R^R,
\]
and an affine map
\[
B_{\mathrm{tot}}:\mathbb R^R\to U_{\mathrm{out}}
\]
such that
\[
\sup_{z\in S}
\left\|
B_{\mathrm{tot}}(\mathrm{GELU}(A_{\mathrm{tot}}(z))) - \Theta(z)
\right\|_2
\le \eta'.
\]

Write
\[
B_{\mathrm{tot}}(u)=L_{\mathrm{tot}}u+b_{\mathrm{tot}},
\]
where
\[
L_{\mathrm{tot}}:\mathbb R^R\to U_{\mathrm{out}}
\]
is linear and
\[
b_{\mathrm{tot}}\in U_{\mathrm{out}}.
\]

Partition the $R$ hidden units into batches of size at most $m-1$:
\[
R=q_1+\cdots+q_M,
\qquad
1\le q_b\le m-1.
\]
Write accordingly
\[
A_{\mathrm{tot}}=(A_1,\dots,A_M),
\]
with each
\[
A_b:\pi_{\mathrm{st}}(\mathbb R^m)\to\mathbb R^{q_b}
\]
affine, and decompose the linear map $L_{\mathrm{tot}}$ as
\[
L_{\mathrm{tot}}(u^{(1)},\dots,u^{(M)})=\sum_{b=1}^M L_b u^{(b)},
\]
where each
\[
L_b:\mathbb R^{q_b}\to U_{\mathrm{out}}
\]
is linear.

Choose $\eta'>0$ so that
\[
\eta'\le \eta/2
\]
and
\[
\sup_{z\in S}
\left\|
B_{\mathrm{tot}}(\mathrm{GELU}(A_{\mathrm{tot}}(z))) - \Theta(z)
\right\|_2
\le \eta'.
\]

Apply Lemma~\ref{lem:one_batch_tokenwise} to each pair $(A_b,L_b)$ with accuracy
\[
\frac{\eta}{2(M+1)}.
\]
This yields concrete Sessa batch blocks
\[
G^{\mathrm{batch}}_b\in \mathrm{ConcreteSessaBlocks}_{\Id}(2,m),
\qquad b=1,\dots,M,
\]
such that each block preserves the storage slice exactly, ignores the current output slice in its input projection,
and contributes
\[
L_b(\mathrm{GELU}(A_b(\cdot)))
\]
to the output slice up to error at most $\eta/(2(M+1))$.

It remains to represent the constant term $b_{\mathrm{tot}}$.
Choose the scalar constant hidden map
\[
A_{\mathrm{const}}:\pi_{\mathrm{st}}(\mathbb R^m)\to\mathbb R,
\qquad
A_{\mathrm{const}}(z)\equiv 1,
\]
and the linear map
\[
L_{\mathrm{const}}:\mathbb R\to U_{\mathrm{out}},
\qquad
L_{\mathrm{const}}(\xi):=\frac{\xi}{\mathrm{GELU}(1)}\,b_{\mathrm{tot}}.
\]
Then
\[
L_{\mathrm{const}}(\mathrm{GELU}(A_{\mathrm{const}}(z)))=b_{\mathrm{tot}}
\qquad\forall z\in S.
\]
Apply Lemma~\ref{lem:one_batch_tokenwise} once more to $(A_{\mathrm{const}},L_{\mathrm{const}})$,
again with accuracy
\[
\frac{\eta}{2(M+1)}.
\]

Since each batch block preserves storage exactly and ignores the current output slice in its input projection,
all blocks act on the same storage input and their contributions add in $U_{\mathrm{out}}$.
Hence the cumulative implementation error of the $M$ linear batches together with the one constant batch is at most
\[
(M+1)\cdot \frac{\eta}{2(M+1)}=\frac{\eta}{2}.
\]
Combining this with the approximation error $\eta'\le \eta/2$ gives the total error bound $\eta$.
\end{proof}

\subsection{Sessa universality for causal maps}

\begin{theorem*}[Universal approximation for Sessa with adapters]
Let $\mathcal{D}\subset\mathbb{R}^{\Tctx\times d_{\mathrm{ext}}}$ be compact and let
\[
F:\mathcal{D}\to\mathbb{R}^{\Tctx\times d_{\mathrm{ext}}}
\]
be continuous and causal.
Then for any $\varepsilon>0$ there exist a model width $m\in\mathbb N^*$,
an even key/query width $d_k$ (in fact $d_k=2$ suffices),
tokenwise adapters
\[
\mathrm{Embed}:\mathbb R^{d_{\mathrm{ext}}}\to\mathbb R^m,
\qquad
\mathrm{Unembed}:\mathbb R^m\to\mathbb R^{d_{\mathrm{ext}}},
\]
and a finite-depth network
\[
G\in \Omega_{\mathrm{Sessa},\Id}^{d_k}(m)
\]
consisting only of the concrete Sessa blocks from Section~\ref{sec:model_arch},
such that
\[
\sup_{x\in\mathcal{D}}
\Big\|
F(x)-\mathrm{Unembed}\big(G(\mathrm{Embed}(x))\big)
\Big\|_F
<\varepsilon.
\]
\end{theorem*}

\begingroup\makeatletter
\@ifundefined{hyper@linkstart}
  {\def\@currentlabel{\ref{thm:sessa_uap_main}}}
  {\def\@currentlabel{\ref*{thm:sessa_uap_main}}}
\label{app:proof_sessa_uap}
\endgroup
\begin{proof}[Proof of Theorem~\ref{thm:sessa_uap_main}]
Fix $\varepsilon>0$.

\paragraph{Step 0: causal factorization.}
For each \(t\in\{0,\dots,\Tctx-1\}\), define the compact set of attainable prefixes
\[
\Pref_t:=\{(x_0,\dots,x_t):x\in\mathcal D\}\subset (\mathbb R^{d_{\mathrm{ext}}})^{t+1}.
\]
By Lemma~\ref{lem:causal_prefix_factorization}, there exists a unique continuous map
\[
\widehat F_t:\Pref_t\to\mathbb R^{d_{\mathrm{ext}}},
\qquad
\widehat F_t(x_0,\dots,x_t):=F(x)_t
\quad (x\in\mathcal D).
\]
Since \(\Pref_t\) is compact in Euclidean space, it is closed in \((\mathbb R^{d_{\mathrm{ext}}})^{t+1}\).
By Tietze extension applied coordinatewise, extend \(\widehat F_t\) to a continuous map
\[
F_t:(\mathbb R^{d_{\mathrm{ext}}})^{t+1}\to\mathbb R^{d_{\mathrm{ext}}}
\]
such that
\[
F(x)_t=F_t(x_0,\dots,x_t)
\qquad\forall x\in\mathcal D.
\]

\paragraph{Step 1: width and adapters}
Set
\[
m:=(\Tctx+1)d_{\mathrm{ext}}+2.
\]
Use the storage decomposition introduced above.

Define the tokenwise embedding by
\[
\mathrm{Embed}(x)_t
=
J_0(x_t)+J_{\mathrm{out}}(x_t)+e_{\mathrm{const}},
\]
that is, place $x_t$ in both $U_0$ and $U_{\mathrm{out}}$, set the constant coordinate to $1$, and set all other coordinates to $0$.

Define $\mathrm{Unembed}$ tokenwise by
\[
\mathrm{Unembed}(h)_t:=R_{\mathrm{out}}(\pi_{\mathrm{out}}(h_t))\in\mathbb R^{d_{\mathrm{ext}}}.
\]
Then
\[
\mathrm{Unembed}(\mathrm{Embed}(x))=x
\qquad \forall x\in\mathbb R^{\Tctx\times d_{\mathrm{ext}}},
\]
$\mathrm{Embed}(\mathcal D)$ is compact, and $\mathrm{Unembed}$ is linear and non-expansive in Frobenius norm.

\paragraph{Step 2: positional code}
Apply Corollary~\ref{cor:pos_code_padding} with $u=e_{\mathrm{pos}}$ to obtain a block
\[
G^{\mathrm{pos}}\in \mathrm{ConcreteSessaBlocks}_{\Id}(2,m)
\]
and pairwise distinct scalars $c_0,\dots,c_{\Tctx-1}$ such that
\[
G^{\mathrm{pos}}(h)_t = h_t + c_t e_{\mathrm{pos}}
\qquad\forall h,\forall t.
\]
By construction, $G^{\mathrm{pos}}$ leaves $U_0,\dots,U_{\Tctx-1}$ and $U_{\mathrm{out}}$ unchanged.

\paragraph{Step 3: prefix encoding}
Fix a packing tolerance
\[
\delta_{\mathrm{pack}}>0,
\]
to be specified later in Step~4.
For each lag $\ell=1,\dots,\Tctx-1$, apply Lemma~\ref{lem:one_lag_pack} successively on the compact set obtained after the previous blocks to construct a concrete Sessa block
\[
G^{\mathrm{lag}}_\ell\in \mathrm{ConcreteSessaBlocks}_{\Id}(2,m)
\]
that preserves all coordinates except $U_\ell$ and writes an approximation of the lag-$\ell$ token from $U_0$ into $U_\ell$.

For $t\in\{0,\dots,\Tctx-1\}$ and $\ell\in\{1,\dots,\Tctx-1\}$, define
\[
j^*(t,\ell)\in\arg\max_{0\le j\le t}\cos\big((t-\ell)-j\big).
\]
For $t\ge \ell$ one has $j^*(t,\ell)=t-\ell$.

Define the ideal encoded state $\widehat h_t(x)\in\mathbb R^m$ by:
\[
\pi_0(\widehat h_t(x))=J_0(x_t),
\qquad
\pi_\ell(\widehat h_t(x))=J_\ell\!\big(x_{j^*(t,\ell)}\big)\quad (1\le \ell\le \Tctx-1),
\]
\[
\pi_{\mathrm{out}}(\widehat h_t(x))=J_{\mathrm{out}}(x_t),
\qquad
\langle \widehat h_t(x),e_{\mathrm{const}}\rangle=1,
\qquad
\langle \widehat h_t(x),e_{\mathrm{pos}}\rangle=c_t.
\]

Since each lag block depends only on the exact source slice \(U_0\) and fixed biases, while writing only to its own destination slice and preserving all previously written slices, the packing errors do not propagate to later lag blocks.
Hence, choosing per-lag accuracies $\eta_\ell>0$ with
\[
\sum_{\ell=1}^{\Tctx-1}\eta_\ell^2\le \delta_{\mathrm{pack}}^2,
\]
we obtain for
\[
G^{\mathrm{pack}}
:=
G^{\mathrm{lag}}_{\Tctx-1}\circ\cdots\circ G^{\mathrm{lag}}_1\circ G^{\mathrm{pos}}
\]
that
\[
\sup_{x\in\mathcal D}\max_{0\le t\le \Tctx-1}
\left\|
G^{\mathrm{pack}}(\mathrm{Embed}(x))_t - \widehat h_t(x)
\right\|_2
\le \delta_{\mathrm{pack}}.
\]

\paragraph{Step 4: target map}
For each $t$, let
\[
\widehat S_t:=\{\widehat h_t(x):x\in\mathcal D\}\subset\mathbb R^m,
\qquad
\widehat S:=\bigcup_{t=0}^{\Tctx-1}\widehat S_t.
\]
Each $\widehat S_t$ is compact.
Since the $e_{\mathrm{pos}}$-coordinate equals $c_t$ on $\widehat S_t$ and the scalars $c_t$ are distinct, the sets $\widehat S_t$ are pairwise disjoint and positively separated.

Define the linear readout
\[
\mathrm{Read}_t:\mathbb R^m\to (\mathbb R^{d_{\mathrm{ext}}})^{t+1}
\]
by
\[
\mathrm{Read}_t(u):=\big(R_t\pi_t(u),\,R_{t-1}\pi_{t-1}(u),\,\dots,\,R_0\pi_0(u)\big).
\]
For $u=\widehat h_t(x)$, one has
\[
R_0\pi_0(\widehat h_t(x))=x_t,
\]
and for \(1\le \ell\le t\),
\[
R_\ell\pi_\ell(\widehat h_t(x))=x_{j^*(t,\ell)}.
\]
Since \(j^*(t,\ell)=t-\ell\) for \(1\le \ell\le t\), it follows that
\[
\mathrm{Read}_t(\widehat h_t(x))=(x_0,\dots,x_t).
\]

Define
\[
\widehat\Phi:\widehat S\to U_{\mathrm{out}}
\]
by
\[
\widehat\Phi(u):=J_{\mathrm{out}}\big(F_t(\mathrm{Read}_t(u))\big)
\qquad\text{for }u\in\widehat S_t.
\]
This is well defined because the index $t$ is uniquely determined by the $e_{\mathrm{pos}}$-coordinate of $u$, and if
\[
u=\widehat h_t(x)=\widehat h_t(x'),
\]
then
\[
\mathrm{Read}_t(u)=(x_0,\dots,x_t)=(x'_0,\dots,x'_t),
\]
so the value of $J_{\mathrm{out}}(F_t(\mathrm{Read}_t(u)))$ does not depend on the choice of $x$.

Moreover, on each $\widehat S_t$ one has
\[
\widehat\Phi|_{\widehat S_t}
=
J_{\mathrm{out}}\circ F_t\circ \mathrm{Read}_t|_{\widehat S_t},
\]
hence $\widehat\Phi$ is continuous on each $\widehat S_t$, and therefore continuous on $\widehat S$.

Apply Tietze extension coordinatewise to the $\mathbb R^{d_{\mathrm{ext}}}$-valued map
\[
R_{\mathrm{out}}\circ \widehat\Phi:\widehat S\to\mathbb R^{d_{\mathrm{ext}}}.
\]
This yields a continuous extension
\[
\bar\Phi:\mathbb R^m\to\mathbb R^{d_{\mathrm{ext}}}
\]
of $R_{\mathrm{out}}\circ \widehat\Phi$. Set
\[
\widetilde\Phi:=J_{\mathrm{out}}\circ \bar\Phi:\mathbb R^m\to U_{\mathrm{out}}.
\]
Then $\widetilde\Phi$ extends $\widehat\Phi$.

Fix $\rho>0$ and let
\[
N:=\overline{\mathcal N}_{\rho}(\widehat S)\subset\mathbb R^m.
\]
Then $N$ is compact, so $\widetilde\Phi$ is uniformly continuous on $N$.
Choose $\delta_{\mathrm{UC}}>0$ such that
\[
u,v\in N,\ \|u-v\|_2\le \delta_{\mathrm{UC}}
\quad\Longrightarrow\quad
\|\widetilde\Phi(u)-\widetilde\Phi(v)\|_2\le \frac{\varepsilon}{2\sqrt{\Tctx}}.
\]

Choose $\delta_{\mathrm{pack}}>0$ small enough that
\[
\delta_{\mathrm{pack}}\le \min\{\rho,\delta_{\mathrm{UC}}\}
\]
and that the encoding construction of Step~3 yields
\[
\sup_{x\in\mathcal D}\max_{0\le t\le \Tctx-1}
\left\|
G^{\mathrm{pack}}(\mathrm{Embed}(x))_t - \widehat h_t(x)
\right\|_2
\le \delta_{\mathrm{pack}}.
\]
Then for every $x\in\mathcal D$ and every $t$ one has
\[
G^{\mathrm{pack}}(\mathrm{Embed}(x))_t\in N,
\]
and
\[
\left\|
\widetilde\Phi(G^{\mathrm{pack}}(\mathrm{Embed}(x))_t)-J_{\mathrm{out}}(F(x)_t)
\right\|_2
\le \frac{\varepsilon}{2\sqrt{\Tctx}}.
\]

\paragraph{Step 5: tokenwise readout}
Define the compact storage-token set
\[
S_{\mathrm{st}}
:=
\left\{
\pi_{\mathrm{st}}\big(G^{\mathrm{pack}}(\mathrm{Embed}(x))_t\big)
:\ x\in\mathcal D,\ 0\le t\le \Tctx-1
\right\}.
\]
Define
\[
\Theta:S_{\mathrm{st}}\to U_{\mathrm{out}},
\qquad
\Theta(z):=\widetilde\Phi(\iota_{\mathrm{st}}(z)) - T_{0\to \mathrm{out}}(\pi_0(z)).
\]
Since $\iota_{\mathrm{st}}$ is linear and $\widetilde\Phi$ is continuous, $\Theta$ is continuous.
Moreover, for every $x\in\mathcal D$ and every $t$,
\[
\iota_{\mathrm{st}}\!\big(\pi_{\mathrm{st}}(G^{\mathrm{pack}}(\mathrm{Embed}(x))_t)\big)
=
G^{\mathrm{pack}}(\mathrm{Embed}(x))_t,
\]
since $\mathrm{Embed}$ initializes the output slice as a copy of $U_0$ and $G^{\mathrm{pack}}$ preserves $U_{\mathrm{out}}$.
Hence
\[
\Theta\!\big(\pi_{\mathrm{st}}(G^{\mathrm{pack}}(\mathrm{Embed}(x))_t)\big)
=
\widetilde\Phi(G^{\mathrm{pack}}(\mathrm{Embed}(x))_t)
-
\pi_{\mathrm{out}}(G^{\mathrm{pack}}(\mathrm{Embed}(x))_t),
\]
so $\Theta$ is exactly the tokenwise increment that must be added in $U_{\mathrm{out}}$.
Apply Corollary~\ref{cor:serialized_tokenwise_concrete} to $S_{\mathrm{st}}$ and $\Theta$.
This yields a finite composition
\[
G^{\mathrm{tok}}
=
G^{\mathrm{batch}}_M\circ\cdots\circ G^{\mathrm{batch}}_1
\]
of concrete Sessa blocks such that every batch block preserves the storage coordinates exactly, every batch block ignores the current output slice in its input projection, and for all $x\in\mathcal D$ and all $t$,
\[
\left\|
\pi_{\mathrm{out}}\big(G^{\mathrm{tok}}(G^{\mathrm{pack}}(\mathrm{Embed}(x)))_t\big)
-
\widetilde\Phi\big(G^{\mathrm{pack}}(\mathrm{Embed}(x))_t\big)
\right\|_2
\le \frac{\varepsilon}{2\sqrt{\Tctx}}.
\]

\paragraph{Step 6: conclusion}
Set
\[
G:=G^{\mathrm{tok}}\circ G^{\mathrm{pack}}\in \Omega_{\mathrm{Sessa},\Id}^{2}(m).
\]
Since
\[
\mathrm{Unembed}(h)_t=R_{\mathrm{out}}(\pi_{\mathrm{out}}(h_t)),
\]
combining the two error bounds and using that $R_{\mathrm{out}}$ is an isometry gives
\[
\left\|
\mathrm{Unembed}(G(\mathrm{Embed}(x)))_t - F(x)_t
\right\|_2
=
\left\|
R_{\mathrm{out}}(\pi_{\mathrm{out}}(G(\mathrm{Embed}(x))_t)) - F(x)_t
\right\|_2
\le \frac{\varepsilon}{\sqrt{\Tctx}}
\qquad
\forall x\in\mathcal D,\ \forall t.
\]
Hence
\[
\sup_{x\in\mathcal D}
\Big\|
\mathrm{Unembed}(G(\mathrm{Embed}(x)))-F(x)
\Big\|_F
<\varepsilon.
\]
\end{proof}

\section{Universal approximation in the pre-norm LayerNorm setting}
\label{sec:uat_with_layernorm}

We now extend Theorem~\ref{thm:sessa_uap_main} from $\Norm=\Id$
to the pre-norm LayerNorm case $\Norm=\LN_{\varepsilon_{\ln}}$ with $\varepsilon_{\ln}>0$ \citep{xiong2020layernormtransformer},
after a width expansion via a fixed scaffold.

\subsection{Tokenwise LayerNorm}

Fix a width $m\ge 2$ and $\varepsilon_{\ln}>0$.
For $z\in\mathbb R^m$, define
\[
\mu_{\ln}(z):=\frac{1}{m}\langle z,\mathbf 1\rangle,
\qquad
\sigma_{\ln}(z):=\sqrt{\frac{1}{m}\|z-\mu_{\ln}(z)\mathbf 1\|_2^2+\varepsilon_{\ln}},
\qquad
\LN_{\varepsilon_{\ln}}(z):=\frac{z-\mu_{\ln}(z)\mathbf 1}{\sigma_{\ln}(z)}.
\]
With $\varepsilon_{\ln}>0$, $\LN_{\varepsilon_{\ln}}$ is well-defined and continuous on all of $\mathbb R^m$,
in particular, there is no singularity at nearly-constant tokens.

\subsection{Zero-mean scaffold embedding}

Fix a ``dynamic'' width $m_0\ge 1$ and let $\msc\ge 2$ be an even scaffold width.
Let $m:=m_0+\msc$ and define, for $c>0$, the fixed zero-mean scaffold vector
\[
s_{c,\msc}:=(\underbrace{c,\dots,c}_{\msc/2},\underbrace{-c,\dots,-c}_{\msc/2})\in\mathbb R^{\msc},
\qquad
\langle s_{c,\msc},\mathbf 1_{\msc}\rangle =0.
\]
Define the scaffold embedding
\[
\Phi_{c,\msc}:\mathbb R^{m_0}\to\mathbb R^{m},\qquad
\Phi_{c,\msc}(u):=(u,\ s_{c,\msc}).
\]
Let $\pi_{\mathrm{dyn}}:\mathbb R^{m}\to\mathbb R^{m_0}$ be the projection onto the first $m_0$ coordinates, and
let $\pi_{\mathrm{sc}}:\mathbb R^{m}\to\mathbb R^{\msc}$ be the projection onto the last $\msc$ coordinates:
\[
\pi_{\mathrm{dyn}}(z_1,\dots,z_{m_0+\msc})=(z_1,\dots,z_{m_0}),
\qquad
\pi_{\mathrm{sc}}(z_1,\dots,z_{m_0+\msc})=(z_{m_0+1},\dots,z_{m_0+\msc}).
\]

\begin{lemma}[Approximate linearity of LayerNorm on scaffold sets]
\label{lem:ln_almost_linear_scaffold}
Fix $m_0\ge 1$, $\varepsilon_{\ln}>0$, a compact set $\Kset\subset\mathbb R^{m_0}$, and $\delta>0$.
Then there exist an even $\msc\ge 2$, a scalar $c>0$, and a constant $a>0$ such that
\[
\sup_{u\in \Kset}
\Big\|
\pi_{\mathrm{dyn}}\!\big(\LN_{\varepsilon_{\ln}}(\Phi_{c,\msc}(u))\big) - a\,u
\Big\|_2
\le \delta.
\]
Moreover, $\pi_{\mathrm{sc}}(\Phi_{c,\msc}(u))\equiv s_{c,\msc}$ is constant on $\Kset$.
\end{lemma}

\begin{proof}
Let $R:=\sup_{u\in \Kset}\|u\|_2<\infty$ and fix an even $\msc\ge 2$. Set $m:=m_0+\msc$.
For $u\in \Kset$ write
\[
z:=\Phi_{c,\msc}(u)=(u,s_{c,\msc})\in\mathbb R^{m}.
\]
Since $\langle s_{c,\msc},\mathbf 1_{\msc}\rangle=0$, we have
\[
\mu_{\ln}(z)=\frac{1}{m}\sum_{i=1}^{m_0}u_i =:\mu_u,
\qquad
|\mu_u|\le \frac{1}{m}\Big|\sum_{i=1}^{m_0}u_i\Big|\le \frac{\sqrt{m_0}}{m}\,\|u\|_2
\le \frac{\sqrt{m_0}}{m}R.
\]
Define the mean-centered dynamic vector $\bar u:=u-\mu_u\mathbf 1_{m_0}$.
Then the dynamic slice of LayerNorm equals
\[
\pi_{\mathrm{dyn}}(\LN_{\varepsilon_{\ln}}(z))=\frac{\bar u}{\sigma_{\ln}(z)}.
\]

Define the reference scale
\[
\sigma_0:=\sigma_{\ln}(\Phi_{c,\msc}(0))=\sqrt{\frac{1}{m}\|s_{c,\msc}\|_2^2+\varepsilon_{\ln}}
=\sqrt{\frac{1}{m}\,(\msc c^2)+\varepsilon_{\ln}},
\qquad
a:=\frac{1}{\sigma_0}.
\]

We estimate
\[
\left\|\frac{\bar u}{\sigma_{\ln}(z)}-a u\right\|_2
\le
\left\|\frac{\bar u-u}{\sigma_{\ln}(z)}\right\|_2
+
\left\|u\left(\frac{1}{\sigma_{\ln}(z)}-\frac{1}{\sigma_0}\right)\right\|_2
=:T_1+T_2.
\]

For the term $T_1$ (mean leakage),
Since $\bar u-u=-\mu_u\mathbf 1_{m_0}$,
\[
T_1
=
\frac{\|\mu_u\mathbf 1_{m_0}\|_2}{\sigma_{\ln}(z)}
\le
\frac{\sqrt{m_0}|\mu_u|}{\sqrt{\varepsilon_{\ln}}}
\le
\frac{\sqrt{m_0}}{\sqrt{\varepsilon_{\ln}}}\cdot \frac{\sqrt{m_0}}{m}R
=
\frac{m_0R}{m\sqrt{\varepsilon_{\ln}}}.
\]

for the term $T_2$ (variance perturbation),
Note that $\sigma_{\ln}(z)^2=\frac{1}{m}\|z-\mu_u\mathbf 1\|_2^2+\varepsilon_{\ln}$ and, because
$\langle s_{c,\msc},\mathbf 1_{\msc}\rangle=0$, we have the exact decomposition
\[
\|z-\mu_u\mathbf 1\|_2^2
=
\|u-\mu_u\mathbf 1_{m_0}\|_2^2 + \|s_{c,\msc}-\mu_u\mathbf 1_{\msc}\|_2^2
=
\|\bar u\|_2^2 + \|s_{c,\msc}\|_2^2 + \msc\mu_u^2,
\]
and the cross term vanishes since $\langle s_{c,\msc},\mathbf 1_{\msc}\rangle=0$.
Therefore
\[
\sigma_{\ln}(z)^2-\sigma_0^2
=
\frac{1}{m}\big(\|\bar u\|_2^2+\msc\mu_u^2\big)
\le
\frac{1}{m}\big(\|u\|_2^2+\msc\mu_u^2\big)
\le
\frac{1}{m}\left(R^2+\msc\cdot \frac{m_0R^2}{m^2}\right)
\le
\frac{2R^2}{m},
\]
since $\msc\le m$ implies $\msc m_0/m^2\le m_0/m\le 1$ for $m\ge m_0$.

Using $|\sqrt{A}-\sqrt{B}|\le |A-B|/( \sqrt{A}+\sqrt{B})$ and $\sigma_{\ln}(z),\sigma_0\ge \sqrt{\varepsilon_{\ln}}$,
\[
|\sigma_{\ln}(z)-\sigma_0|
\le
\frac{|\sigma_{\ln}(z)^2-\sigma_0^2|}{\sigma_{\ln}(z)+\sigma_0}
\le
\frac{(2R^2/m)}{2\sqrt{\varepsilon_{\ln}}}
=
\frac{R^2}{m\sqrt{\varepsilon_{\ln}}}.
\]
Hence
\[
\left|\frac{1}{\sigma_{\ln}(z)}-\frac{1}{\sigma_0}\right|
=
\frac{|\sigma_{\ln}(z)-\sigma_0|}{\sigma_{\ln}(z)\sigma_0}
\le
\frac{R^2}{m\sqrt{\varepsilon_{\ln}}}\cdot \frac{1}{\varepsilon_{\ln}}
=
\frac{R^2}{m\,\varepsilon_{\ln}^{3/2}}.
\]
Therefore
\[
T_2
\le
\|u\|_2\left|\frac{1}{\sigma_{\ln}(z)}-\frac{1}{\sigma_0}\right|
\le
R\cdot \frac{R^2}{m\,\varepsilon_{\ln}^{3/2}}
=
\frac{R^3}{m\,\varepsilon_{\ln}^{3/2}}.
\]

Combining,
\[
\sup_{u\in \Kset}\left\|\pi_{\mathrm{dyn}}(\LN_{\varepsilon_{\ln}}(\Phi_{c,\msc}(u))) - a u\right\|_2
\le
\frac{m_0R}{m\sqrt{\varepsilon_{\ln}}}+\frac{R^3}{m\,\varepsilon_{\ln}^{3/2}}.
\]
Choose $\msc$ (hence $m=m_0+\msc$) large enough so that the right-hand side is $\le \delta$.
This proves the claim; note that $c>0$ can be arbitrary and only changes the scaling $a$.
\end{proof}

\subsection{Simulating identity-normalized Sessa blocks with pre-norm LN-Sessa blocks}

We call a pre-norm LN-Sessa block a Sessa block with $\Norm=\LN_{\varepsilon_{\ln}}$ in the tokenwise preprocessing stage, i.e. $\xnorm_t=\LN_{\varepsilon_{\ln}}(x_t)$, and residual $y_t=x_t+o_t$.

\begin{lemma}[Simulation of an identity-normalized block by a pre-norm LN block on a scaffold]
\label{lem:ln_block_simulates_lnfree}
Let $G:\mathbb R^{\Tctx\times m_0}\to\mathbb R^{\Tctx\times m_0}$ be a width-$m_0$ concrete Sessa block
from Section~\ref{sec:model_arch}, with $\Norm=\Id$.
Fix a compact set $\Kset\subset\mathbb R^{\Tctx\times m_0}$ and $\varepsilon_{\mathrm{sim}}>0$.
Then there exist an even $\msc\ge 2$, a scalar $c>0$, and a width-$m$ pre-norm LN-Sessa block
$\widetilde G:\mathbb R^{\Tctx\times (m_0+\msc)}\to\mathbb R^{\Tctx\times (m_0+\msc)}$ with $\Norm=\LN_{\varepsilon_{\ln}}$
such that, with $m:=m_0+\msc$,
\[
\sup_{x\in \Kset}
\Big\|
\pi_{\mathrm{dyn}}\big(\widetilde G(\Phi_{c,\msc}(x))\big)-G(x)
\Big\|_F
\le \varepsilon_{\mathrm{sim}},
\qquad\text{and}\qquad
\pi_{\mathrm{sc}}\big(\widetilde G(\Phi_{c,\msc}(x))\big)\equiv s_{c,\msc}.
\]
Here $\Phi_{c,\msc}(x)$ denotes the tokenwise application of $\Phi_{c,\msc}$.
\end{lemma}

\begin{proof}
Define the compact set of attainable tokens
\[
S_{\Kset}:=\{x_t:\ x\in \Kset,\ t=0,\dots,\Tctx-1\}\subset\mathbb R^{m_0}.
\]
Choose once and for all
\[
a\in \bigl(0,\varepsilon_{\ln}^{-1/2}\bigr).
\]
Define the continuous map
\[
\Delta:\mathbb R^{\Tctx\times m_0}\to\mathbb R^{\Tctx\times m_0},
\]
i.e. given $v\in\mathbb R^{\Tctx\times m_0}$, run the Sessa block from the stage after normalization, with the dynamic weights scaled by $1/a$, i.e.\ with first input projection on the dynamic slice
$\widetilde W^{\mathrm{in}}_{\mathrm{dyn}}:=a^{-1}W^{\mathrm{in}}$, $\widetilde b^{\mathrm{in}}:=b^{\mathrm{in}}$,
and all other dynamic parameters copied from $G$.
Then, by construction,
\[
G(x)=x+\Delta(ax)\qquad \forall x\in\mathbb R^{\Tctx\times m_0}.
\]

Since $\Kset$ is compact, so is $a\Kset$, and $\Delta$ is uniformly continuous on a compact neighborhood of $a\Kset$.
Choose $\eta_{\mathrm{UC}}>0$ such that
\[
\|v-v'\|_F\le \eta_{\mathrm{UC}} \ \Rightarrow\ \|\Delta(v)-\Delta(v')\|_F\le \varepsilon_{\mathrm{sim}}
\quad \text{for all }v,v' \text{ in that neighborhood.}
\]
\[
\eta_{\mathrm{LN}}:=\eta_{\mathrm{UC}}/\sqrt{\Tctx}.
\]

Fix an even $\msc\ge 2$ (to be chosen large enough), set $m:=m_0+\msc$, and define
\[
c:=\sqrt{\frac{m}{\msc}\bigl(a^{-2}-\varepsilon_{\ln}\bigr)} \ >\ 0.
\]
Then the reference scale in Lemma~\ref{lem:ln_almost_linear_scaffold} equals exactly
\[
\sigma_0=\sqrt{\frac{\msc c^2}{m}+\varepsilon_{\ln}}=a^{-1},
\qquad\text{hence}\qquad
\frac{1}{\sigma_0}=a.
\]
Inspecting the proof of Lemma~\ref{lem:ln_almost_linear_scaffold}, the approximation bound depends on $m=m_0+\msc$
(and on $S_{\Kset},\varepsilon_{\ln}$) and tends to $0$ as $m\to\infty$; therefore, after increasing the even $\msc$ if needed,
we obtain
\[
\sup_{u\in S_{\Kset}}
\Big\|
\pi_{\mathrm{dyn}}\!\big(\LN_{\varepsilon_{\ln}}(\Phi_{c,\msc}(u))\big)-a u
\Big\|_2
\le \eta_{\mathrm{LN}}.
\]

Write the width-\(m_0\) input projection of \(G\) as
\[
W^{\mathrm{in}}=[W_a\ \ W_g],
\qquad
b^{\mathrm{in}}=(b_a,b_g),
\]
with
\[
W_a,W_g\in\mathbb R^{m_0\times m_0},
\qquad
b_a,b_g\in\mathbb R^{m_0}.
\]
Decompose the widened coordinates as
\[
\mathbb R^m=\mathbb R^{m_0}\oplus\mathbb R^{\msc},
\]
where the first summand is the dynamic slice and the second is the scaffold slice.

Define
\[
\widetilde W_a=
\begin{bmatrix}
a^{-1}W_a & 0\\
0 & 0
\end{bmatrix},
\qquad
\widetilde W_g=
\begin{bmatrix}
a^{-1}W_g & 0\\
0 & 0
\end{bmatrix}
\in\mathbb R^{m\times m},
\]
and
\[
\widetilde W^{\mathrm{in}}=[\widetilde W_a\ \widetilde W_g]\in\mathbb R^{m\times 2m},
\qquad
\widetilde b^{\mathrm{in}}=(b_a,0_{\msc},\,b_g,0_{\msc})\in\mathbb R^{2m}.
\]

For the mixer parameters define
\[
\widetilde W_{Qf}=\begin{bmatrix}W_{Qf}\\0\end{bmatrix},\qquad
\widetilde W_{Kf}=\begin{bmatrix}W_{Kf}\\0\end{bmatrix},\qquad
\widetilde W_{Qb}=\begin{bmatrix}W_{Qb}\\0\end{bmatrix},\qquad
\widetilde W_{Kb}=\begin{bmatrix}W_{Kb}\\0\end{bmatrix}
\in\mathbb R^{m\times d_k},
\]
\[
\widetilde W_V=
\begin{bmatrix}
W_V & 0\\
0 & 0
\end{bmatrix}\in\mathbb R^{m\times m},
\qquad
\widetilde w^\gamma=(w^\gamma,0_{\msc})\in\mathbb R^m,
\qquad
\widetilde b^\gamma:=b^\gamma.
\]

For the output map define
\[
\widetilde W^{\mathrm{out}}=
\begin{bmatrix}
W^{\mathrm{out}} & 0\\
0 & 0
\end{bmatrix}\in\mathbb R^{m\times m},
\qquad
\widetilde b^{\mathrm{out}}=(b^{\mathrm{out}},0_{\msc})\in\mathbb R^m.
\]
All remaining scaffold rows and columns are set to zero.

Thus, once the pre-norm token
\[
z_t:=\LN_{\varepsilon_{\ln}}(X_t)
\]
is formed, every learned linear map in \(\widetilde G\) reads only \(\pi_{\mathrm{dyn}}(z_t)\), while the residual increment has zero scaffold coordinates.

For \(X=\Phi_{c,\msc}(x)\), define
\[
v_t:=\pi_{\mathrm{dyn}}\!\big(\LN_{\varepsilon_{\ln}}(X_t)\big)\in\mathbb R^{m_0}.
\]
Then the widened block has
\[
\widetilde a_t=(a^{-1}v_tW_a+b_a,\ 0_{\msc}),
\qquad
\widetilde g_t=(a^{-1}v_tW_g+b_g,\ 0_{\msc}),
\]
hence
\[
\mathrm{GELU}(\widetilde a_t)
=
\bigl(\mathrm{GELU}(a^{-1}v_tW_a+b_a),\,0_{\msc}\bigr).
\]
Therefore the forward logits, feedback logits, gains, dynamic mixer output, and dynamic residual increment of \(\widetilde G\) coincide exactly with those of the width-\(m_0\) block defining \(\Delta(v)\), whereas the scaffold part of \(f\), \(s\), and of the residual increment is identically zero. Consequently
\[
\pi_{\mathrm{dyn}}\big(\widetilde G(\Phi_{c,\msc}(x))\big)
=
x+\Delta(v),
\qquad
\pi_{\mathrm{sc}}\big(\widetilde G(\Phi_{c,\msc}(x))\big)
=
s_{c,\msc}.
\]

For $x\in \Kset$, the tokenwise bound above implies
\[
\left\|\pi_{\mathrm{dyn}}(\LN_{\varepsilon_{\ln}}(\Phi_{c,\msc}(x)))-ax\right\|_F\le \eta_{\mathrm{LN}}\sqrt{\Tctx}=\eta_{\mathrm{UC}},
\]
hence
\[
\left\|\pi_{\mathrm{dyn}}\big(\widetilde G(\Phi_{c,\msc}(x))\big)-G(x)\right\|_F
=
\left\|\Delta\!\big(\pi_{\mathrm{dyn}}(\LN_{\varepsilon_{\ln}}(\Phi_{c,\msc}(x)))\big)-\Delta(ax)\right\|_F
\le \varepsilon_{\mathrm{sim}}.
\]
Finally, since the increment has zero scaffold coordinates, the scaffold stays constant:
$\pi_{\mathrm{sc}}(\widetilde G(\Phi_{c,\msc}(x)))\equiv s_{c,\msc}$.
\end{proof}

\subsection{Universal approximation for pre-norm LN-Sessa}

\begin{corollary}[Universal approximation for pre-norm LN-Sessa]
\label{cor:uat_prenorm_ln_sessa}
Let $\mathcal{D}\subset\mathbb{R}^{\Tctx\times d_{\mathrm{ext}}}$ be compact and let
\[
F:\mathcal{D}\to\mathbb{R}^{\Tctx\times d_{\mathrm{ext}}}
\]
be continuous and causal. Fix $\varepsilon_{\ln}>0$ for tokenwise LayerNorm.
Then for any $\varepsilon>0$ there exist a model width $m\in\mathbb N^*$,
an even key/query width $d_k$,
tokenwise adapters
\[
\mathrm{Embed}:\mathbb R^{d_{\mathrm{ext}}}\to\mathbb R^{m},
\qquad
\mathrm{Unembed}:\mathbb R^{m}\to\mathbb R^{d_{\mathrm{ext}}},
\]
and a finite-depth pre-norm LN-Sessa network
\[
G_{\ln}\in \Omega_{\mathrm{Sessa},\LN_{\varepsilon_{\ln}}}^{d_k}(m),
\]
such that
\[
\sup_{x\in\mathcal D}\Big\|F(x)-\mathrm{Unembed}\big(G_{\ln}(\mathrm{Embed}(x))\big)\Big\|_F < \varepsilon.
\]
\end{corollary}

\begin{proof}
By Theorem~\ref{thm:sessa_uap_main} for $\Norm=\Id$, choose adapters
\[
\mathrm{Embed}_0:\mathbb R^{d_{\mathrm{ext}}}\to\mathbb R^{m_0},
\qquad
\mathrm{Unembed}_0:\mathbb R^{m_0}\to\mathbb R^{d_{\mathrm{ext}}},
\]
and a concrete Sessa network with $\Norm=\Id$
\[
G^\star\in \Omega_{\mathrm{Sessa},\Id}^{d_{k,0}}(m_0)
\]
of depth $\Nlayer$ such that
\[
\sup_{x\in\mathcal D}\Big\|F(x)-\mathrm{Unembed}_0\big(G^\star(\mathrm{Embed}_0(x))\big)\Big\|_F < \varepsilon/2.
\]
Write
\[
G^\star=G_{\Nlayer}\circ\cdots\circ G_1
\]
as a composition of concrete Sessa blocks with $\Norm=\Id$ on $\mathbb R^{\Tctx\times m_0}$.

Let $\Kset_1:=\mathrm{Embed}_0(\mathcal D)$ (compact).
Fix $\rhonbhd>0$ and define the thickened compacts recursively as in
Lemma~\ref{lem:composition_error}:
\[
\widetilde{\Kset}_1:=\Kset_1,\qquad
\Kset_{\idxlayer+1}:=G_{\idxlayer}(\widetilde{\Kset}_{\idxlayer}),\qquad
\widetilde{\Kset}_{\idxlayer+1}:=\overline{\mathcal N}_{\rhonbhd}(\Kset_{\idxlayer+1})
\quad\text{for }\idxlayer=1,\dots,\Nlayer.
\]

Since $\Nlayer$ is finite, the union of attainable token sets
\[
S:=\bigcup_{\idxlayer=1}^{\Nlayer} \{u_t:\ u\in \widetilde{\Kset}_{\idxlayer},\ t=0,\dots,\Tctx-1\}\subset\mathbb R^{m_0}.
\]
is a finite union of compact sets and hence compact.

By Lemma~\ref{lem:composition_error}, choose tolerances $\varepsilon^{\mathrm{sim}}_{\idxlayer}>0$ such that if each block $G_{\idxlayer}$
is approximated on $\widetilde{\Kset}_{\idxlayer}$ within $\varepsilon^{\mathrm{sim}}_{\idxlayer}$, then the composed approximation error on $\Kset_1$ is at most $\varepsilon/2$.

Moreover, by the same lemma we may (and do) choose them so that
\[
\varepsilon^{\mathrm{sim}}_{\idxlayer} \le \rhonbhd,\qquad \idxlayer=1,\dots,\Nlayer.
\]

Fix once and for all a scale
\[
a\in \bigl(0,\varepsilon_{\ln}^{-1/2}\bigr).
\]
For each layer $\idxlayer$, apply the construction from the proof of Lemma~\ref{lem:ln_block_simulates_lnfree} with target accuracy $\varepsilon^{\mathrm{sim}}_{\idxlayer}$ and prescribed scale $a$.
This yields a required tokenwise LN-approximation tolerance $\eta^{(\idxlayer)}_{\mathrm{LN}}>0$
such that the layer simulation error is $\le \varepsilon^{\mathrm{sim}}_{\idxlayer}$ whenever
\[
\sup_{u\in \{v_t:\,v\in \widetilde{\Kset}_{\idxlayer},\ t=0,\dots,\Tctx-1\}}
\Big\|
\pi_{\mathrm{dyn}}\!\big(\LN_{\varepsilon_{\ln}}(\Phi_{c,\msc}(u))\big)-a u
\Big\|_2
\le \eta^{(\idxlayer)}_{\mathrm{LN}}.
\]
Set
\[
\eta_{\mathrm{LN}}:=\min_{1\le \idxlayer\le \Nlayer}\eta^{(\idxlayer)}_{\mathrm{LN}}.
\]
Applying the proof of Lemma~\ref{lem:ln_almost_linear_scaffold} to the compact token set $S$,
choose one even $\msc\ge 2$ and one $c>0$ such that:
\begin{itemize}
\item the induced reference scale equals the prescribed $a$, and
\item
\[
\sup_{u\in S}
\Big\|
\pi_{\mathrm{dyn}}\!\big(\LN_{\varepsilon_{\ln}}(\Phi_{c,\msc}(u))\big)-a u
\Big\|_2
\le \eta_{\mathrm{LN}}.
\]
\end{itemize}
Let $m:=m_0+\msc$ and write $\Phi:=\Phi_{c,\msc}$.

For each $\idxlayer$, apply the construction of Lemma~\ref{lem:ln_block_simulates_lnfree} with this common scaffold $(\msc,c)$
to obtain a pre-norm LN concrete Sessa block
\[
\widetilde G_{\idxlayer}\in \mathrm{ConcreteSessaBlocks}_{\LN_{\varepsilon_{\ln}}}(d_{k,0},m)
\]
viewed as a map
\[
\widetilde G_{\idxlayer}:\mathbb R^{\Tctx\times m}\to\mathbb R^{\Tctx\times m}
\]
such that
\[
\sup_{h\in \widetilde{\Kset}_{\idxlayer}}
\big\|\pi_{\mathrm{dyn}}(\widetilde G_{\idxlayer}(\Phi(h)))-G_{\idxlayer}(h)\big\|_F
\le \varepsilon^{\mathrm{sim}}_{\idxlayer}.
\]
and
\[
\pi_{\mathrm{sc}}(\widetilde G_{\idxlayer}(\Phi(h)))\equiv s_{c,\msc}\qquad \forall h\in \widetilde{\Kset}_{\idxlayer}.
\]
Define the induced dynamic maps
\[
G_{\idxlayer}^{\mathrm{dyn}}:\widetilde{\Kset}_{\idxlayer}\to \mathbb R^{\Tctx\times m_0},
\qquad
G_{\idxlayer}^{\mathrm{dyn}}(h):=\pi_{\mathrm{dyn}}(\widetilde G_{\idxlayer}(\Phi(h))).
\]
Then
\[
\sup_{h\in \widetilde{\Kset}_{\idxlayer}}\|G_{\idxlayer}^{\mathrm{dyn}}(h)-G_{\idxlayer}(h)\|_F\le \varepsilon^{\mathrm{sim}}_{\idxlayer}.
\]
Moreover, by scaffold invariance,
\[
\widetilde G_{\idxlayer}(\Phi(h))=\Phi(G_{\idxlayer}^{\mathrm{dyn}}(h))\qquad \forall h\in \widetilde{\Kset}_{\idxlayer}.
\]

Applying Lemma~\ref{lem:composition_error} to the maps $G_{\idxlayer}$ and $G_{\idxlayer}^{\mathrm{dyn}}$ on the dynamic space
$\mathbb R^{\Tctx\times m_0}$ yields
\[
\sup_{x\in\mathcal D}
\Big\|
G^\star(\mathrm{Embed}_0(x))
-
\big(G_{\Nlayer}^{\mathrm{dyn}}\circ\cdots\circ G_1^{\mathrm{dyn}}\big)(\mathrm{Embed}_0(x))
\Big\|_F
\le \varepsilon/2.
\]

Define
\[
G_{\ln}:=\widetilde G_{\Nlayer}\circ\cdots\circ \widetilde G_1
\in \Omega_{\mathrm{Sessa},\LN_{\varepsilon_{\ln}}}^{d_{k,0}}(m).
\]

Finally, define new adapters
\[
\mathrm{Embed}(x):=\Phi(\mathrm{Embed}_0(x))\in\mathbb R^{\Tctx\times m},
\qquad
\mathrm{Unembed}(u):=\mathrm{Unembed}_0(\pi_{\mathrm{dyn}}(u)).
\]
Since
\[
\mathrm{Unembed}_0(h)_t = R_{\mathrm{out}}(\pi_{\mathrm{out}}(h_t)),
\]
with \(\pi_{\mathrm{out}}\) an orthogonal projection and \(R_{\mathrm{out}}\) an isometry, \(\mathrm{Unembed}_0\) is non-expansive in Frobenius norm.
\[
\mathrm{Unembed}\big(G_{\ln}(\mathrm{Embed}(x))\big)
=
\mathrm{Unembed}_0\big((G_{\Nlayer}^{\mathrm{dyn}}\circ\cdots\circ G_1^{\mathrm{dyn}})(\mathrm{Embed}_0(x))\big)
\qquad \forall x\in\mathcal D.
\]
Therefore,
\[
\sup_{x\in\mathcal D}
\Big\|
\mathrm{Unembed}_0\big(G^\star(\mathrm{Embed}_0(x))\big)
-
\mathrm{Unembed}\big(G_{\ln}(\mathrm{Embed}(x))\big)
\Big\|_F
\le \varepsilon/2.
\]
Combining this with the approximation error $\varepsilon/2$ from the $\Norm=\Id$ case gives the claim.
\end{proof}

\section{Proofs for flexible finite-horizon selective retrieval}
\label{app:flexible_finite_horizon_retrieval}

\begin{lemma}[Predecessor focusing from ordered codes]
\label{lem:predecessor_focus}
Fix $T\ge 1$ and $\mu\in(0,1)$.
Let $I_0< I_1<\cdots < I_T$ be pairwise disjoint compact intervals in $\mathbb R$, and assume
all of them lie in $(0,\infty)$.
Then there exist scalar linear feedback-query/key maps on a single coordinate such that for every
token sequence $u$ satisfying
\[
\langle u_t,e_{\mathrm{pos}}\rangle \in I_t,\qquad 0\le t\le T,
\]
the resulting strict-past feedback attention row satisfies
\[
\alpha^{b}_{t,t-1}\ge 1-\mu,
\qquad
\sum_{j=0}^{t-2}\alpha^b_{t,j}\le \mu,
\qquad
1\le t\le T.
\]
\end{lemma}

\begin{proof}
If \(T=1\), the claim is trivial, since the strict past of \(t=1\) contains only the index \(0\).
Assume henceforth that \(T\ge 2\).
Let
\[
z_t:=\langle u_t,e_{\mathrm{pos}}\rangle,\qquad 0\le t\le T.
\]
By assumption,
\[
z_t\in I_t,\qquad I_0<I_1<\cdots<I_T\subset(0,\infty).
\]

To implement the focusing inside an actual LN-free Sessa block, we first realize a
single dedicated post-GELU scalar coordinate carrying a strictly ordered positive code.
Choose one $a$-branch coordinate to be
\[
a_t^{\mathrm{pos}}=c\,z_t
\]
with some fixed $c>0$.
Since $z_t>0$ on all intervals and the exact GELU satisfies
\[
\GELU'(x)=\Phi(x)+x\phi(x)>0\qquad (x>0),
\]
the scalar map $x\mapsto \GELU(c x)$ is strictly increasing on $(0,\infty)$.
Hence the post-GELU coordinate
\[
\xi_t:=\GELU(c z_t)
\]
ranges in compact intervals
\[
J_t:=\GELU(c I_t)
\]
satisfying
\[
J_0<J_1<\cdots<J_T\subset(0,\infty).
\]

Now define scalar feedback queries and keys from that post-GELU coordinate:
\[
q_t^b=\Lambda\,\xi_t,\qquad k_j^b=\Lambda\,\xi_j,
\]
with $\Lambda>0$ to be chosen.
All unused heads and coordinates are set to zero.

Let
\[
m_t:=\inf J_t,\qquad M_t:=\sup J_t.
\]
For $2\le t\le T$, compactness and strict ordering give
\[
\Delta_t:=m_{t-1}-M_{t-2}>0.
\]
Set
\[
\Delta:=\min_{2\le t\le T}\Delta_t>0,
\qquad
m_*:=\min_{0\le t\le T} m_t>0.
\]

For every $2\le t\le T$, every $j\le t-2$, and every admissible input $u$,
\[
q_t^b k_{t-1}^b-q_t^b k_j^b
=
\Lambda^2 \xi_t(\xi_{t-1}-\xi_j)
\ge
\Lambda^2 m_* \Delta.
\]
Hence each non-predecessor strict-past logit is smaller than the predecessor logit by at least
\[
\Lambda^2 m_*\Delta.
\]
Therefore
\[
\sum_{j=0}^{t-2}
\exp\!\Bigl(
\langle q_t^b,k_j^b\rangle-\langle q_t^b,k_{t-1}^b\rangle
\Bigr)
\le
T\,e^{-\Lambda^2 m_*\Delta}.
\]
Choose $\Lambda$ so large that
\[
T\,e^{-\Lambda^2 m_*\Delta}\le \frac{\mu}{1-\mu}.
\]
Then the softmax formula yields
\[
\alpha^b_{t,t-1}
=
\frac{1}{1+\sum_{j=0}^{t-2}e^{\langle q_t^b,k_j^b\rangle-\langle q_t^b,k_{t-1}^b\rangle}}
\ge 1-\mu,
\]
and consequently
\[
\sum_{j=0}^{t-2}\alpha^b_{t,j}\le \mu.
\]
For $t=1$ the strict past contains only the predecessor $0$, so the claim is trivial.
\end{proof}

\begin{lemma}[RoPE self-focusing]
\label{lem:self_focus_forward}
Fix $T\ge 0$ and $\mu\in(0,1)$.
Let $I_0<I_1<\cdots<I_T$ be pairwise disjoint compact intervals in $(0,\infty)$.
Then there exist forward query/key maps realized inside a single actual RoPE forward branch of an LN-free Sessa block such that for every
token sequence $u$ satisfying
\[
\langle u_t,e_{\mathrm{pos}}\rangle\in I_t,\qquad 0\le t\le T,
\]
the resulting full-prefix forward attention row satisfies
\[
\alpha^f_{t,t}\ge 1-\mu,
\qquad
\sum_{j=0}^{t-1}\alpha^f_{t,j}\le \mu,
\qquad 0\le t\le T.
\]
\end{lemma}

\begin{proof}
If \(T=0\), the statement is trivial.
Assume henceforth that \(T\ge 1\).
Let
\[
z_t:=\langle u_t,e_{\mathrm{pos}}\rangle,\qquad z_t\in I_t.
\]
As in the proof of Lemma~\ref{lem:predecessor_focus}, choose one dedicated $a$-branch coordinate
\[
a_t^{\mathrm{pos}}=c\,z_t
\]
with $c>0$, and let
\[
\xi_t:=\GELU(c z_t).
\]
Because $z_t>0$ and GELU is strictly increasing on $(0,\infty)$, the ranges
\[
J_t:=\GELU(c I_t)
\]
are compact, strictly ordered, and positive:
\[
J_0<J_1<\cdots<J_T\subset(0,\infty).
\]

Let
\[
m_t:=\inf J_t,\qquad M_t:=\sup J_t.
\]
Since the intervals are strictly ordered and compact,
\[
\delta_t:=m_t-M_{t-1}>0,\qquad 1\le t\le T.
\]
Set
\[
\delta:=\min_{1\le t\le T}\delta_t>0,
\qquad
m_*:=\min_{0\le t\le T} m_t>0.
\]

Now realize the forward query/key pair on a single RoPE plane by setting, before RoPE,
\[
q_t^f=\Lambda\,\xi_t\,e_1,\qquad
k_j^f=\Lambda\,\xi_j\,e_1
\]
inside the first $2$-dimensional RoPE plane, with all other coordinates and heads set to zero.
Let
\[
\ell_{t,j}
:=
\sigma_k\big\langle \mathrm{RoPE}(q_t^f),\mathrm{RoPE}(k_j^f)\big\rangle.
\]
Then for every $j\le t$,
\[
\ell_{t,j}
=
\sigma_k \Lambda^2 \xi_t\xi_j \cos(\vartheta_t-\vartheta_j)
\]
for the corresponding RoPE phases $\vartheta_t,\vartheta_j$ on that plane.
Hence for every $j<t$,
\begin{align*}
\ell_{t,t}-\ell_{t,j}
&=
\sigma_k \Lambda^2 \xi_t\Big(\xi_t-\xi_j\cos(\vartheta_t-\vartheta_j)\Big)\\
&\ge
\sigma_k \Lambda^2 \xi_t(\xi_t-\xi_j)
\qquad\text{since }\cos(\cdot)\le 1\\
&\ge
\sigma_k \Lambda^2 m_* \delta.
\end{align*}

Therefore, for every $1\le t\le T$,
\[
\sum_{j=0}^{t-1}\exp(\ell_{t,j}-\ell_{t,t})
\le
T\,e^{-\sigma_k \Lambda^2 m_* \delta}.
\]
Choose $\Lambda$ so large that
\[
T\,e^{-\sigma_k \Lambda^2 m_* \delta}\le \frac{\mu}{1-\mu}.
\]
Then the softmax formula gives
\[
\alpha^f_{t,t}
=
\frac{1}{1+\sum_{j=0}^{t-1}e^{\ell_{t,j}-\ell_{t,t}}}
\ge 1-\mu,
\]
and consequently
\[
\sum_{j=0}^{t-1}\alpha^f_{t,j}\le \mu.
\]
For $t=0$ the statement is trivial.
\end{proof}
\begin{lemma}[Scaled GELU uniformly approximates ReLU]
\label{lem:scaled_gelu_relu}
Assume the exact GELU activation
\[
\GELU(x)=x\,\Phi(x).
\]
For $L>0$, define
\[
R_L(u):=\frac{1}{L}\GELU(Lu).
\]
Then
\[
\sup_{u\in\mathbb R}\bigl|R_L(u)-u_+\bigr|
\le
\frac{1}{L\sqrt{2\pi}},
\qquad
u_+:=\max\{u,0\}.
\]
\end{lemma}

\begin{proof}
Since $\GELU(x)=x\Phi(x)$,
\[
R_L(u)=u\,\Phi(Lu).
\]

If $u\ge 0$, then
\[
R_L(u)-u_+
=
u\Phi(Lu)-u
=
-u(1-\Phi(Lu)).
\]
By the Mills bound
\[
1-\Phi(v)\le \frac{\phi(v)}{v}\qquad (v>0),
\]
we obtain for $u>0$,
\[
|R_L(u)-u_+|
=
u(1-\Phi(Lu))
\le
\frac{\phi(Lu)}{L}
\le
\frac{1}{L\sqrt{2\pi}}.
\]
The same bound is trivial at $u=0$.

If $u<0$, then $u_+=0$ and
\[
|R_L(u)|
=
|u|\,\Phi(Lu)
=
|u|\,(1-\Phi(-Lu)).
\]
Applying the same Mills bound with $v=-Lu>0$ yields
\[
|R_L(u)|
\le
\frac{\phi(-Lu)}{L}
=
\frac{\phi(Lu)}{L}
\le
\frac{1}{L\sqrt{2\pi}}.
\]
Combining the two cases proves the claim.
\end{proof}

\begin{lemma}[Symmetrized scaled GELU equals the identity]
\label{lem:scaled_gelu_identity}
Assume the exact GELU activation
\[
\GELU(x)=x\,\Phi(x).
\]
For \(L>0\), define
\[
R_L(x):=\frac1L\GELU(Lx),
\qquad
\Id_L(x):=R_L(x)-R_L(-x).
\]
Then
\[
\Id_L(x)=x
\qquad
\forall x\in\mathbb R.
\]
In particular,
\[
\sup_{x\in\mathbb R}|\Id_L(x)-x|=0
\le \frac{2}{L\sqrt{2\pi}}.
\]
\end{lemma}

\begin{proof}
Since \(\GELU(x)=x\Phi(x)\),
\[
R_L(x)=x\Phi(Lx).
\]
Hence
\[
\Id_L(x)
=
x\Phi(Lx)-(-x)\Phi(-Lx)
=
x\bigl(\Phi(Lx)+\Phi(-Lx)\bigr)
=
x,
\]
because \(\Phi(-z)=1-\Phi(z)\).
\end{proof}

\begin{corollary}[Exact channel read on the \(a\)-branch]
\label{cor:exact_channel_read}
Fix a unit vector \(e\in\mathbb R^m\) and \(L>0\).
In an LN-free concrete Sessa block, if two \(a\)-coordinates are chosen as
\[
a_t^{(+)}=L\langle u_t,e\rangle,
\qquad
a_t^{(-)}=-L\langle u_t,e\rangle,
\]
then the corresponding post-GELU coordinates satisfy
\[
\frac1L\Bigl(\bar a_t^{(+)}-\bar a_t^{(-)}\Bigr)=\langle u_t,e\rangle
\qquad
\forall\,t.
\]
Hence any scalar input channel can be read exactly by a linear value projection from two \(a\)-slots.
\end{corollary}

\begin{proof}
Apply Lemma~\ref{lem:scaled_gelu_identity} pointwise with \(x=\langle u_t,e\rangle\).
\end{proof}

\begin{lemma}[Plateau window from four scaled GELUs]
\label{lem:plateau_window_gelu}
Fix $T\ge 0$ and pairwise disjoint compact intervals
\[
I_0<I_1<\cdots<I_T\subset(0,\infty).
\]
Fix a target index $\tau_\ast\in\{0,\dots,T\}$ and an accuracy parameter $\eta\in(0,1)$.
Then there exist real numbers
\[
a_-<a_+<b_-<b_+
\]
and a scalar function $W_\eta:\mathbb R\to\mathbb R$ of the form
\[
W_\eta(x)=
\frac{R_L(x-a_-)-R_L(x-a_+)}{a_+-a_-}
-
\frac{R_L(x-b_-)-R_L(x-b_+)}{b_+-b_-}
\]
for some $L>0$, such that
\[
|W_\eta(x)-1|\le \eta\quad\text{for }x\in I_{\tau_\ast},
\]
\[
|W_\eta(x)|\le \eta\quad\text{for }x\in \bigcup_{t\neq \tau_\ast} I_t,
\]
and
\[
\sup_{x\in\mathbb R}|W_\eta(x)|\le 1+\eta.
\]
Moreover, $W_\eta$ is realizable exactly as a linear combination of four $a$-branch GELU coordinates inside a single LN-free Sessa block.
\end{lemma}

\begin{proof}
Because the intervals are pairwise disjoint, compact, and strictly ordered, one can choose
\[
a_-<a_+<\inf I_{\tau_\ast}\le \sup I_{\tau_\ast}<b_-<b_+
\]
such that
\[
I_{\tau_\ast}\subset [a_+,b_-],
\qquad
\bigcup_{t\neq \tau_\ast} I_t \subset (-\infty,a_-]\cup [b_+,\infty).
\]

Define the exact piecewise-linear plateau window
\[
w(x):=
\frac{(x-a_-)_{+}-(x-a_+)_{+}}{a_+-a_-}
-
\frac{(x-b_-)_{+}-(x-b_+)_{+}}{b_+-b_-}.
\]
By construction,
\[
w(x)=1\quad\text{on }[a_+,b_-]\supset I_{\tau_\ast},
\]
\[
w(x)=0\quad\text{on }(-\infty,a_-]\cup [b_+,\infty)\supset \bigcup_{t\neq \tau_\ast} I_t,
\]
and
\[
0\le w(x)\le 1\qquad\forall x\in\mathbb R.
\]

Now replace each ReLU ramp by the scaled-GELU ramp from Lemma~\ref{lem:scaled_gelu_relu}:
\[
R_L(u)=\frac1L\GELU(Lu).
\]
Set
\[
W_L(x):=
\frac{R_L(x-a_-)-R_L(x-a_+)}{a_+-a_-}
-
\frac{R_L(x-b_-)-R_L(x-b_+)}{b_+-b_-}.
\]
Using Lemma~\ref{lem:scaled_gelu_relu} on each of the four ramp terms,
\[
\|W_L-w\|_{\infty}
\le
\frac{2}{L\sqrt{2\pi}}
\left(
\frac1{a_+-a_-}
+
\frac1{b_+-b_-}
\right).
\]
Choose $L$ so large that the right-hand side is at most $\eta$.
Then on $I_{\tau_\ast}$, where $w\equiv 1$,
\[
|W_L-1|\le \eta,
\]
and on $\bigcup_{t\neq\tau_\ast} I_t$, where $w\equiv 0$,
\[
|W_L|\le \eta.
\]
Also, since $0\le w\le 1$,
\[
|W_L(x)|\le |w(x)|+\eta\le 1+\eta
\qquad\forall x.
\]
Set $W_\eta:=W_L$.

Finally, $W_\eta$ is realizable exactly inside one LN-free Sessa block because each term
\[
R_L(x-c)=\frac1L\GELU(L(x-c))
\]
is one $a$-branch GELU coordinate applied to an affine function of the tokenwise scalar $x$, and the displayed linear combination is absorbed into the value projection.
\end{proof}

\begin{lemma}[Writing a window into an auxiliary channel]
\label{lem:window_writer}
Fix $T\ge 0$, $\tau_\ast\in\{0,\dots,T\}$, and $\varepsilon\in(0,1)$.
Let $\Kset\subset (\mathbb R^m)^{T+1}$ be compact.
Assume that for some unit vector $e_{\mathrm{pos}}\in\mathbb R^m$,
\[
I_t:=\{\langle u_t,e_{\mathrm{pos}}\rangle:\ u\in\Kset\},
\qquad 0\le t\le T,
\]
are compact and strictly ordered:
\[
I_0<I_1<\cdots<I_T\subset(0,\infty).
\]
Fix orthonormal directions
\[
e_{\mathrm{pos}},\ e_{\mathrm{sig}},\ e_{\mathrm{aux}}
\]
and let $E_{\mathrm{carry}}\subset\mathbb R^m$ be any fixed subspace orthogonal to all three.
Assume moreover that \(m\ge 6\).
Then there exists a single LN-free Sessa block
\[
W^{\mathrm{write}}_{T,\tau_\ast,\varepsilon}:(\mathbb R^m)^{T+1}\to(\mathbb R^m)^{T+1}
\]

such that the feedback branch is switched off, the
$e_{\mathrm{pos}}$-, $e_{\mathrm{sig}}$-, and $E_{\mathrm{carry}}$-channels are preserved exactly, and,
writing
\[
a_t(u):=\big\langle W^{\mathrm{write}}_{T,\tau_\ast,\varepsilon}(u)_t,\ e_{\mathrm{aux}}\big\rangle,
\]
one has uniformly on $\Kset$,
\[
|a_{\tau_\ast}(u)-1|\le \varepsilon,
\qquad
|a_t(u)|\le \varepsilon\quad (t\neq \tau_\ast),
\]
and
\[
\sup_{u\in\Kset}\sup_{0\le t\le T}|a_t(u)|\le 2.
\]
\end{lemma}

\begin{proof}
Choose $\eta\in(0,\varepsilon)$ so small that
\[
\eta+\eta(1+\eta)\le \varepsilon.
\]
Apply Lemma~\ref{lem:plateau_window_gelu} to obtain a scalar function $W_\eta$ satisfying
\[
|W_\eta(x)-1|\le \eta\quad (x\in I_{\tau_\ast}),
\qquad
|W_\eta(x)|\le \eta\quad (x\in \bigcup_{t\neq\tau_\ast} I_t),
\qquad
\sup_x |W_\eta(x)|\le 1+\eta.
\]

Next apply Lemma~\ref{lem:self_focus_forward} with parameter $\mu:=\eta$.
This gives a forward branch whose full-prefix row satisfies
\[
\alpha^f_{t,t}\ge 1-\eta,
\qquad
\sum_{j<t}\alpha^f_{t,j}\le \eta
\qquad (0\le t\le T).
\]

We now build the block.

\emph{Values.}
Choose a positive constant $c_1$ such that
\[
\GELU(c_1)=1.
\]
Realize the first value coordinate by a constant $a$-branch coordinate equal to $c_1$, so that
\[
v_t^{(0)}\equiv 1.
\]
Realize the second value coordinate as
\[
v_t^{(1)}=W_\eta(\langle u_t,e_{\mathrm{pos}}\rangle),
\]
using Lemma~\ref{lem:plateau_window_gelu}.

\emph{Gate and output on the auxiliary channel.}
Choose two gate coordinates
\[
g_t^{(0)}=\langle u_t,e_{\mathrm{aux}}\rangle,
\qquad
g_t^{(1)}\equiv 1.
\]
Choose the output projection on the $e_{\mathrm{aux}}$-channel with coefficients $(-1,+1)$ on the two gated coordinates and zero on all other channels.
Because the row sum of attention is exactly $1$,
\[
s_t^{(0)}=\sum_{j\le t}\alpha^f_{t,j}\cdot 1 =1.
\]
Hence the auxiliary output becomes
\[
a_t(u)=\langle u_t,e_{\mathrm{aux}}\rangle - s_t^{(0)}\langle u_t,e_{\mathrm{aux}}\rangle + s_t^{(1)} = s_t^{(1)},
\]
where
\[
s_t^{(1)}=\sum_{j\le t}\alpha^f_{t,j}\,W_\eta(\langle u_j,e_{\mathrm{pos}}\rangle).
\]
Thus the block overwrites the auxiliary channel by the forward average of $W_\eta$.

All other output columns are zero, so the $e_{\mathrm{pos}}$-, $e_{\mathrm{sig}}$-, and $E_{\mathrm{carry}}$-channels are preserved exactly.

It remains to bound $a_t=s_t^{(1)}$.

\emph{Target time $t=\tau_\ast$.}
All indices $j<\tau_\ast$ are off-target, hence
\[
|W_\eta(\langle u_j,e_{\mathrm{pos}}\rangle)|\le \eta.
\]
At the target index,
\[
W_\eta(\langle u_{\tau_\ast},e_{\mathrm{pos}}\rangle)\in [1-\eta,1+\eta].
\]
Therefore
\[
a_{\tau_\ast}(u)
\ge
(1-\eta)(1-\eta)-\eta\cdot \eta
\ge 1-2\eta,
\]
and
\[
a_{\tau_\ast}(u)
\le
(1-\eta)(1+\eta)+\eta\cdot \eta
\le 1+\eta.
\]
Hence
\[
|a_{\tau_\ast}(u)-1|\le 2\eta\le \varepsilon.
\]

\emph{Off-target times $t<\tau_\ast$.}
Then all visible indices $j\le t$ are off-target, so
\[
|a_t(u)|\le \eta\le \varepsilon.
\]

\emph{Off-target times $t>\tau_\ast$.}
Then self-mass is on an off-target index, so the self contribution is at most $\eta$ in magnitude, while all nonself mass is at most $\eta$ and every visible value has magnitude at most $1+\eta$.
Thus
\[
|a_t(u)|
\le
\eta+\eta(1+\eta)
\le \varepsilon.
\]

Finally, from
\[
|a_t(u)|=|s_t^{(1)}|
\le
\sum_{j\le t}\alpha^f_{t,j}\sup_x|W_\eta(x)|
\le 1+\eta\le 2,
\]
we obtain the uniform bound.
\end{proof}

\begin{definition}[Signal-fiber saturation]
\label{def:bounded_signal_fiber_saturation}
Fix \(T\ge 0\), a unit signal direction \(e_{\mathrm{sig}}\in\mathbb R^m\), and a set
\(\Kset\subset (\mathbb R^m)^{T+1}\).
For \(\delta\ge 0\), define
\[
\Sat^{\mathrm{sig}}_\delta(\Kset)
:=
\left\{
u+z:\ u\in\Kset,\ z_t=a_t e_{\mathrm{sig}},\ \max_{0\le t\le T}|a_t|\le \delta
\right\}.
\]
Equivalently,
\[
\Sat^{\mathrm{sig}}_\delta(\Kset)
=
\left\{
u+\sum_{t=0}^{T} a_t e_{\mathrm{sig}}\mathbf 1[\cdot=t]:
u\in\Kset,\ \max_t |a_t|\le \delta
\right\}.
\]
If \(\Kset\) is compact, then \(\Sat^{\mathrm{sig}}_\delta(\Kset)\) is compact.
\end{definition}

\begin{definition}[Exact signal transport]
\label{def:signal_blind_exact_transport}
Fix \(T\ge 0\), a unit signal direction \(e_{\mathrm{sig}}\in\mathbb R^m\), and a control subspace
\(E_{\mathrm{ctrl}}\subset\mathbb R^m\) with \(e_{\mathrm{sig}}\perp E_{\mathrm{ctrl}}\).
Let \(\Pi_{\mathrm{ctrl}}\) denote the orthogonal projection onto \(E_{\mathrm{ctrl}}\), and let
\[
\pi_{\mathrm{sig}}(v):=\langle v,e_{\mathrm{sig}}\rangle.
\]

For \(u=(u_t)_{t=0}^{T}\in(\mathbb R^m)^{T+1}\), write
\[
c^u_t:=\Pi_{\mathrm{ctrl}}u_t,\qquad x^u_t:=\pi_{\mathrm{sig}}(u_t).
\]

A causal map
\[
B:(\mathbb R^m)^{T+1}\to (\mathbb R^m)^{T+1}
\]
is said to have \emph{exact signal transport along \(e_{\mathrm{sig}}\) over \(E_{\mathrm{ctrl}}\)}
on a set \(\Kset\subset(\mathbb R^m)^{T+1}\) if:

\begin{enumerate}[label=(\roman*), leftmargin=*, nosep]
\item \(B\) preserves the control channels exactly:
\[
\Pi_{\mathrm{ctrl}} B(u)_t = c^u_t
\qquad
\forall\,u\in\Kset,\ \forall\,0\le t\le T;
\]

\item there exists a scalar lower-triangular kernel
\[
\mathcal T_B^u(i,j),\qquad 0\le j\le i\le T,
\]
depending only on the control stream \(c^u=(c^u_t)_{t=0}^{T}\), such that
\[
\pi_{\mathrm{sig}}(B(u)_i)
=
\sum_{j=0}^{i}\mathcal T_B^u(i,j)\,x^u_j
\qquad
\forall\,u\in\Kset,\ \forall\,0\le i\le T.
\]
\end{enumerate}
\end{definition}

\begin{lemma}[Transport calculus on signal fibers]
\label{lem:signal_blind_transport_calculus}
Fix \(T\ge 0\), \(e_{\mathrm{sig}}\), \(E_{\mathrm{ctrl}}\), and a compact set
\(\Kset\subset(\mathbb R^m)^{T+1}\).
Fix \(\delta>0\).

\begin{enumerate}[label=(\roman*), leftmargin=*, nosep]
\item \textbf{Jacobian extraction.}
Assume \(B\) is continuously differentiable on a neighborhood of
\(\Sat^{\mathrm{sig}}_\delta(\Kset)\), and that \(B\) has
signal-blind exact scalar transport along \(e_{\mathrm{sig}}\) over
\(E_{\mathrm{ctrl}}\) on \(\Sat^{\mathrm{sig}}_\delta(\Kset)\), with kernel
\(\mathcal T_B^u\).
Then for every \(u\in\Kset\) and every \(0\le j\le i\le T\),
\[
e_{\mathrm{sig}}^\top
\frac{\partial B(u)_i}{\partial u_j}
e_{\mathrm{sig}}
=
\mathcal T_B^u(i,j).
\]

\item \textbf{Composition.}
Assume \(B_1\) has signal-blind exact scalar transport along \(e_{\mathrm{sig}}\)
over \(E_{\mathrm{ctrl}}\) on \(\Sat^{\mathrm{sig}}_\delta(\Kset)\), with kernel
\(\mathcal T_{B_1}^u\), and preserves the control channels exactly there.
Assume \(B_2\) has signal-blind exact scalar transport along \(e_{\mathrm{sig}}\)
over \(E_{\mathrm{ctrl}}\) on \(B_1(\Sat^{\mathrm{sig}}_\delta(\Kset))\), with kernel
\(\mathcal T_{B_2}^v\), and preserves the control channels exactly there.
Then \(B_2\circ B_1\) also has signal-blind exact scalar transport along
\(e_{\mathrm{sig}}\) over \(E_{\mathrm{ctrl}}\) on \(\Sat^{\mathrm{sig}}_\delta(\Kset)\),
and its kernel is the lower-triangular kernel product
\[
\mathcal T_{B_2\circ B_1}^u(i,j)
=
\sum_{r=j}^{i}\mathcal T_{B_2}^{B_1(u)}(i,r)\,\mathcal T_{B_1}^{u}(r,j).
\]
\end{enumerate}
\end{lemma}

\begin{proof}
For (i), fix \(u\in\Kset\), \(j\le i\), and define
\[
u^{(h)}:=u+h\,e_{\mathrm{sig}}\mathbf 1[\cdot=j].
\]
For \(|h|<\delta\), one has \(u^{(h)}\in \Sat^{\mathrm{sig}}_\delta(\Kset)\).
Because \(e_{\mathrm{sig}}\perp E_{\mathrm{ctrl}}\),
\[
\Pi_{\mathrm{ctrl}}u^{(h)}_t=\Pi_{\mathrm{ctrl}}u_t
\qquad\forall\,t,
\]
so the control stream is unchanged. Since the transport kernel depends only on the
control stream, the same kernel \(\mathcal T_B^u\) applies to both \(u\) and \(u^{(h)}\).
Therefore
\begin{align*}
\pi_{\mathrm{sig}}(B(u^{(h)})_i)-\pi_{\mathrm{sig}}(B(u)_i)
&=
\sum_{r=0}^{i}\mathcal T_B^u(i,r)
\bigl(x_r^{u^{(h)}}-x_r^{u}\bigr)\\
&=
h\,\mathcal T_B^u(i,j).
\end{align*}
Divide by \(h\) and let \(h\to 0\). Since \(B\) is \(C^1\),
\[
e_{\mathrm{sig}}^\top
\frac{\partial B(u)_i}{\partial u_j}
e_{\mathrm{sig}}
=
\mathcal T_B^u(i,j).
\]

For (ii), let \(u\in \Sat^{\mathrm{sig}}_\delta(\Kset)\).
Because \(B_1\) preserves the control channels exactly,
\[
\Pi_{\mathrm{ctrl}} B_1(u)_t = \Pi_{\mathrm{ctrl}} u_t,
\]
so the control stream of \(B_1(u)\) equals that of \(u\).
Hence
\[
\pi_{\mathrm{sig}}(B_1(u)_r)
=
\sum_{j=0}^{r}\mathcal T_{B_1}^u(r,j)\,x_j^u.
\]
Applying \(B_2\) and using exact control preservation again,
\begin{align*}
\pi_{\mathrm{sig}}(B_2(B_1(u))_i)
&=
\sum_{r=0}^{i}\mathcal T_{B_2}^{B_1(u)}(i,r)\,\pi_{\mathrm{sig}}(B_1(u)_r)\\
&=
\sum_{r=0}^{i}\mathcal T_{B_2}^{B_1(u)}(i,r)
\sum_{j=0}^{r}\mathcal T_{B_1}^{u}(r,j)\,x_j^u\\
&=
\sum_{j=0}^{i}
\left(
\sum_{r=j}^{i}\mathcal T_{B_2}^{B_1(u)}(i,r)\,\mathcal T_{B_1}^{u}(r,j)
\right)x_j^u.
\end{align*}
This is exactly the stated kernel-product formula.
\end{proof}

\begin{definition}[Transparent preprocessing]
\label{def:signal_transparent_preprocessing}
Fix \(T\ge 0\), a unit signal direction \(e_{\mathrm{sig}}\in\mathbb R^m\), and a control subspace
\(E_{\mathrm{ctrl}}\subset\mathbb R^m\) with \(e_{\mathrm{sig}}\perp E_{\mathrm{ctrl}}\).
Let
\[
\Pi_{\mathrm{ctrl}}:\mathbb R^m\to E_{\mathrm{ctrl}}
\]
be the orthogonal projection and
\[
\pi_{\mathrm{sig}}(v):=\langle v,e_{\mathrm{sig}}\rangle.
\]

A causal map
\[
R:(\mathbb R^m)^{T+1}\to(\mathbb R^m)^{T+1}
\]
is said to be \emph{signal-transparent along \(e_{\mathrm{sig}}\) over \(E_{\mathrm{ctrl}}\)} on a set
\(\Kset\subset(\mathbb R^m)^{T+1}\) if for every \(u\in\Kset\), every \(\tau\in\{0,\dots,T\}\), and every sufficiently small scalar \(a\) such that
\[
u^{(a,\tau)}:=u+a\,e_{\mathrm{sig}}\mathbf 1[\cdot=\tau]
\]
remains in the domain under consideration, one has
\[
\Pi_{\mathrm{ctrl}}R(u^{(a,\tau)})_t
=
\Pi_{\mathrm{ctrl}}R(u)_t
\qquad
\forall\,t,
\]
and
\[
\pi_{\mathrm{sig}}(R(u^{(a,\tau)})_t)
=
\pi_{\mathrm{sig}}(R(u)_t)+a\,\mathbf 1[t=\tau]
\qquad
\forall\,t.
\]
\end{definition}

\begin{lemma}[Transparent preprocessing and Jacobians]
\label{lem:transparent_preprocessing_exact_transport}
Fix \(T\ge 0\), \(e_{\mathrm{sig}}\), and \(E_{\mathrm{ctrl}}\).
Let
\[
R:(\mathbb R^m)^{T+1}\to(\mathbb R^m)^{T+1},
\qquad
B:(\mathbb R^m)^{T+1}\to(\mathbb R^m)^{T+1}
\]
be continuously differentiable on neighborhoods of \(\Kset\) and
\(\Sat^{\mathrm{sig}}_\delta(R(\Kset))\), respectively, for some \(\delta>0\).

Assume:
\begin{enumerate}[label=(\roman*), leftmargin=*, nosep]
\item \(R\) is signal-transparent along \(e_{\mathrm{sig}}\) over \(E_{\mathrm{ctrl}}\) on \(\Kset\);
\item \(B\) has signal-blind exact scalar transport along \(e_{\mathrm{sig}}\) over \(E_{\mathrm{ctrl}}\) on
\(\Sat^{\mathrm{sig}}_\delta(R(\Kset))\), with kernel
\[
\mathcal T_B^v(i,j),
\qquad v\in \Sat^{\mathrm{sig}}_\delta(R(\Kset)),\ \ 0\le j\le i\le T.
\]
\end{enumerate}
Then for every \(u\in\Kset\) and every \(0\le j\le i\le T\),
\[
e_{\mathrm{sig}}^\top
\frac{\partial (B\circ R)(u)_i}{\partial u_j}
e_{\mathrm{sig}}
=
\mathcal T_B^{R(u)}(i,j).
\]
\end{lemma}

\begin{proof}
Fix \(u\in\Kset\) and \(0\le j\le i\le T\). For sufficiently small \(a\), define
\[
u^{(a,j)}:=u+a\,e_{\mathrm{sig}}\mathbf 1[\cdot=j].
\]
Set
\[
v:=R(u),
\qquad
v^{(a)}:=R(u^{(a,j)}).
\]
By signal-transparency of \(R\),
\[
\Pi_{\mathrm{ctrl}}v^{(a)}_t=\Pi_{\mathrm{ctrl}}v_t
\qquad\forall\,t,
\]
and
\[
\pi_{\mathrm{sig}}(v^{(a)}_t)=\pi_{\mathrm{sig}}(v_t)+a\,\mathbf 1[t=j]
\qquad\forall\,t.
\]
Hence \(v^{(a)}\in \Sat^{\mathrm{sig}}_\delta(R(\Kset))\) for all sufficiently small \(|a|\),
and \(v^{(a)}\) and \(v\) have the same control stream. Therefore the same kernel
\(\mathcal T_B^{v}\) applies to both \(v\) and \(v^{(a)}\), so
\begin{align*}
\pi_{\mathrm{sig}}(B(v^{(a)})_i)-\pi_{\mathrm{sig}}(B(v)_i)
&=
\sum_{r=0}^{i}\mathcal T_B^{v}(i,r)
\bigl(\pi_{\mathrm{sig}}(v^{(a)}_r)-\pi_{\mathrm{sig}}(v_r)\bigr)\\
&=
a\,\mathcal T_B^{v}(i,j).
\end{align*}
Divide by \(a\) and let \(a\to 0\). Since \(B\circ R\) is continuously differentiable,
\[
e_{\mathrm{sig}}^\top
\frac{\partial (B\circ R)(u)_i}{\partial u_j}
e_{\mathrm{sig}}
=
\mathcal T_B^{R(u)}(i,j).
\]
\end{proof}

\begin{corollary}[Signal-fiber stability of the control-driven blocks]
\label{cor:concrete_blocks_signal_fiber_stability}
Fix \(\delta\ge 0\).
In each of Lemmas
\ref{lem:local_multiplier},
\ref{lem:selector_window},
\ref{lem:active_diffusive_factorized_transport},
\ref{lem:source0_tail_channel},
and \ref{lem:damped_predecessor_integrator},
replace the base compact set \(\Kset\) (or \(\Kset_H\)) by its bounded signal-fiber saturation
\(\Sat^{\mathrm{sig}}_\delta(\Kset)\) (or \(\Sat^{\mathrm{sig}}_\delta(\Kset_H)\)).
Then the same concrete block or network satisfies the same conclusion, with the same constants.

In particular, whenever one of these lemmas yields signal-blind exact scalar transport
along \(e_{\mathrm{sig}}\), that exact transport statement also holds on every bounded
signal-fiber saturation of the same control-side compact set.
\end{corollary}

\begin{proof}
In each listed lemma, the hypotheses and parameter choices depend only on channels orthogonal
to \(e_{\mathrm{sig}}\): ordered positional ranges, two-sided tail/profile bounds, exact
vanishing of designated scratch/profile channels, and carried control channels.
These quantities are unchanged when \(\Kset\) is replaced by
\(\Sat^{\mathrm{sig}}_\delta(\Kset)\).

Moreover, the concrete constructions preserve the relevant control channels exactly and treat
the \(e_{\mathrm{sig}}\)-channel linearly. Therefore the original proofs apply verbatim on the
saturated set, with the same constants.
\end{proof}

\begin{lemma}[Local multiplier]
\label{lem:local_multiplier}
Fix \(T\ge 0\) and \(\delta>0\).
Let \(\Kset\subset (\mathbb R^m)^{T+1}\) be compact.
Assume that for some unit vector \(e_{\mathrm{pos}}\in\mathbb R^m\),
\[
I_t:=\{\langle u_t,e_{\mathrm{pos}}\rangle:\ u\in\Kset\},
\qquad 0\le t\le T,
\]
are compact and strictly ordered in \((0,\infty)\).
Fix orthonormal directions
\[
e_{\mathrm{pos}},\ e_{\mathrm{sig}},\ e_{\mathrm{aux}}
\]
and let \(E_{\mathrm{carry}}\subset\mathbb R^m\) be any fixed subspace orthogonal to all three.
Assume moreover that \(m\ge 4\).
Assume moreover that the auxiliary channel is uniformly bounded:
\[
\sup_{u\in\Kset}\sup_{0\le t\le T}
\bigl|\langle u_t,e_{\mathrm{aux}}\rangle\bigr|
\le M
\]
for some finite \(M\).

Then there exists a single LN-free Sessa block
\[
M^{\mathrm{loc}}_{T,\delta}:(\mathbb R^m)^{T+1}\to(\mathbb R^m)^{T+1}
\]

such that the feedback branch is switched off, the
\(e_{\mathrm{pos}}\)-, \(e_{\mathrm{aux}}\)-, and \(E_{\mathrm{carry}}\)-channels are preserved exactly, and
\(M^{\mathrm{loc}}_{T,\delta}\) has signal-blind exact scalar transport along \(e_{\mathrm{sig}}\) over
\[
E_{\mathrm{ctrl}}:=\operatorname{span}\{e_{\mathrm{pos}},e_{\mathrm{aux}}\}\oplus E_{\mathrm{carry}},
\]
with diagonal kernel
\[
\mathcal T_{M^{\mathrm{loc}}}^u(i,j)
=
D_{\mathrm{loc}}^u(i)\,\mathbf 1[i=j];
\]
\[
\bigl|D_{\mathrm{loc}}^u(t)-\langle u_t,e_{\mathrm{aux}}\rangle\bigr|\le \delta
\qquad
\forall\,u\in\Kset,\ \forall\,0\le t\le T.
\]
In particular,
\[
e_{\mathrm{sig}}^\top
\frac{\partial M^{\mathrm{loc}}_{T,\delta}(u)_i}{\partial u_j}
e_{\mathrm{sig}}
=
D_{\mathrm{loc}}^u(i)\,\mathbf 1[i=j].
\]
\end{lemma}

\begin{proof}
Choose a parameter
\[
\mu\in(0,1)
\]
to be fixed later, and apply Lemma~\ref{lem:self_focus_forward} with this \(\mu\).

Choose a positive constant \(c_1\) such that
\[
\GELU(c_1)=1.
\]
Realize one forward value coordinate by the constant \(1\):
\[
v_t^{(0)}\equiv 1.
\]

Next read the auxiliary channel exactly using Corollary~\ref{cor:exact_channel_read}.
Choose two \(a\)-slots
\[
a_t^{(+)}=L\langle u_t,e_{\mathrm{aux}}\rangle,
\qquad
a_t^{(-)}=-L\langle u_t,e_{\mathrm{aux}}\rangle,
\]
for any fixed \(L>0\), and choose the value projection so that
\[
v_t^{(1)}
=
\frac1L\bigl(\bar a_t^{(+)}-\bar a_t^{(-)}\bigr)
=
\langle u_t,e_{\mathrm{aux}}\rangle.
\]

Choose two gate coordinates, both equal to the signal:
\[
g_t^{(0)}=\langle u_t,e_{\mathrm{sig}}\rangle,
\qquad
g_t^{(1)}=\langle u_t,e_{\mathrm{sig}}\rangle.
\]
Choose the output projection on the \(e_{\mathrm{sig}}\)-channel with coefficients \((-1,+1)\)
on these two gated coordinates and zero on all other output channels.

Since the forward row sums to \(1\),
\[
s_t^{(0)}=\sum_{j\le t}\alpha^f_{t,j}\cdot 1=1.
\]
Hence the signal output equals
\[
\bigl\langle M^{\mathrm{loc}}_{T,\delta}(u)_t,e_{\mathrm{sig}}\bigr\rangle
=
\langle u_t,e_{\mathrm{sig}}\rangle
-
s_t^{(0)}\langle u_t,e_{\mathrm{sig}}\rangle
+
s_t^{(1)}\langle u_t,e_{\mathrm{sig}}\rangle
=
s_t^{(1)}\langle u_t,e_{\mathrm{sig}}\rangle,
\]
where
\[
s_t^{(1)}=\sum_{j\le t}\alpha^f_{t,j}\,v_j^{(1)}
=
\sum_{j\le t}\alpha^f_{t,j}\,\langle u_j,e_{\mathrm{aux}}\rangle.
\]

Define
\[
D_{\mathrm{loc}}^u(t):=s_t^{(1)}.
\]
Then
\[
\bigl\langle M^{\mathrm{loc}}_{T,\delta}(u)_t,e_{\mathrm{sig}}\bigr\rangle
=
D_{\mathrm{loc}}^u(t)\,\langle u_t,e_{\mathrm{sig}}\rangle,
\]
which is exactly signal-blind exact scalar transport with diagonal kernel
\[
\mathcal T_{M^{\mathrm{loc}}}^u(i,j)=D_{\mathrm{loc}}^u(i)\mathbf 1[i=j].
\]

The coefficient \(D_{\mathrm{loc}}^u(t)\) depends only on the forward weights and on the auxiliary values
\(\langle u_j,e_{\mathrm{aux}}\rangle\). By construction, both depend only on the
\(e_{\mathrm{pos}}\)-, \(e_{\mathrm{aux}}\)-, and \(E_{\mathrm{carry}}\)-channels, not on the signal channel.
Thus the transport is signal-blind over
\[
E_{\mathrm{ctrl}}:=\operatorname{span}\{e_{\mathrm{pos}},e_{\mathrm{aux}}\}\oplus E_{\mathrm{carry}}.
\]

All output columns except the signal column are zero, so the \(e_{\mathrm{pos}}\)-, \(e_{\mathrm{aux}}\)-,
and \(E_{\mathrm{carry}}\)-channels are preserved exactly.

It remains to estimate \(D_{\mathrm{loc}}^u(t)\).
Since the auxiliary read is exact,
\[
D_{\mathrm{loc}}^u(t)-\langle u_t,e_{\mathrm{aux}}\rangle
=
\sum_{j\le t}\alpha^f_{t,j}
\bigl(\langle u_j,e_{\mathrm{aux}}\rangle-\langle u_t,e_{\mathrm{aux}}\rangle\bigr)
=
\sum_{j<t}\alpha^f_{t,j}
\bigl(\langle u_j,e_{\mathrm{aux}}\rangle-\langle u_t,e_{\mathrm{aux}}\rangle\bigr).
\]
Therefore,
\[
|D_{\mathrm{loc}}^u(t)-\langle u_t,e_{\mathrm{aux}}\rangle|
\le
2M\sum_{j<t}\alpha^f_{t,j}.
\]
By self-focusing,
\[
\sum_{j<t}\alpha^f_{t,j}\le \mu.
\]
Hence
\[
|D_{\mathrm{loc}}^u(t)-\langle u_t,e_{\mathrm{aux}}\rangle|
\le 2M\mu.
\]

Choose
\[
\mu\le \min\Bigl\{\frac12,\ \frac{\delta}{2\max\{M,1\}}\Bigr\}.
\]
Then
\[
|D_{\mathrm{loc}}^u(t)-\langle u_t,e_{\mathrm{aux}}\rangle|\le \delta
\qquad
\forall\,u\in\Kset,\ \forall\,0\le t\le T.
\]
For any \(\eta>0\), replacing \(\Kset\) by
\(\Sat^{\mathrm{sig}}_\eta(\Kset)\) leaves the ordered positional ranges
\((I_t)_{t=0}^{T}\) and the auxiliary bound \(M\) unchanged, since only the
\(e_{\mathrm{sig}}\)-channel is perturbed.
The same concrete construction therefore yields the same exact diagonal
signal-transport formula on \(\Sat^{\mathrm{sig}}_\eta(\Kset)\), with the same
coefficients \(D_{\mathrm{loc}}^u(i)\), because the forward weights depend only
on the positional-control stream and the exact auxiliary read depends only on
the \(e_{\mathrm{aux}}\)-channel.
Applying Lemma~\ref{lem:signal_blind_transport_calculus}(i) gives
\[
e_{\mathrm{sig}}^\top
\frac{\partial M^{\mathrm{loc}}_{T,\delta}(u)_i}{\partial u_j}
e_{\mathrm{sig}}
=
D_{\mathrm{loc}}^u(i)\,\mathbf 1[i=j].
\]
\end{proof}

\begin{lemma}[Two-block selector]
\label{lem:selector_window}
Fix \(T\ge 0\), \(\varepsilon\in(0,1)\), and a compact set \(\Kset\subset(\mathbb R^m)^{T+1}\).
Assume that for some unit vector \(e_{\mathrm{pos}}\in\mathbb R^m\) the scalar position ranges
\[
I_t:=\{\langle u_t,e_{\mathrm{pos}}\rangle:\ u\in\Kset\},
\qquad
0\le t\le T,
\]
are compact and strictly ordered:
\[
I_0<I_1<\cdots<I_T\subset(0,\infty).
\]
Fix a source index \(\tau_\ast\in\{0,\dots,T\}\) and orthonormal directions
\[
e_{\mathrm{pos}},\ e_{\mathrm{sig}},\ e_{\mathrm{aux}}.
\]
Let \(E_{\mathrm{carry}}\subset\mathbb R^m\) be any fixed subspace orthogonal to these three directions.
Assume moreover that \(m\ge 6\).

Then there exists a depth-\(2\) LN-free Sessa network
\[
S_{T,\tau_\ast,\varepsilon}:(\mathbb R^m)^{T+1}\to(\mathbb R^m)^{T+1}
\]

such that both constituent blocks have the feedback branch switched off,
the \(e_{\mathrm{pos}}\)-channel and every channel in \(E_{\mathrm{carry}}\) are preserved exactly, and
\(S_{T,\tau_\ast,\varepsilon}\) has signal-blind exact scalar transport along \(e_{\mathrm{sig}}\) over
\[
E_{\mathrm{ctrl}}:=\operatorname{span}\{e_{\mathrm{pos}}\}\oplus E_{\mathrm{carry}},
\]
with diagonal kernel
\[
\mathcal T_{S}^u(i,j)=D_{\mathrm{sel}}^u(i)\,\mathbf 1[i=j];
\]
Uniformly for all \(u\in\Kset\),
\[
\frac12\le D_{\mathrm{sel}}^u(\tau_\ast)\le 2,
\qquad
|D_{\mathrm{sel}}^u(t)|\le \varepsilon\quad (t\neq \tau_\ast).
\]
In particular,
\[
e_{\mathrm{sig}}^\top
\frac{\partial S_{T,\tau_\ast,\varepsilon}(u)_i}{\partial u_j}
e_{\mathrm{sig}}
=
D_{\mathrm{sel}}^u(i)\,\mathbf 1[i=j].
\]
\end{lemma}

\begin{proof}
Set
\[
\varepsilon_{\mathrm{wr}}:=\frac{\varepsilon}{4},
\qquad
\delta_{\mathrm{mul}}:=\frac{\varepsilon}{4}.
\]

Apply Lemma~\ref{lem:window_writer} with accuracy \(\varepsilon_{\mathrm{wr}}\).
This yields a forward-only block
\[
W^{\mathrm{write}}_{T,\tau_\ast,\varepsilon_{\mathrm{wr}}}
\]
which preserves the \(e_{\mathrm{pos}}\)-, \(e_{\mathrm{sig}}\)-, and \(E_{\mathrm{carry}}\)-channels exactly and writes an auxiliary channel
\[
a_t(u):=\big\langle W^{\mathrm{write}}_{T,\tau_\ast,\varepsilon_{\mathrm{wr}}}(u)_t,\ e_{\mathrm{aux}}\big\rangle
\]
satisfying
\[
|a_{\tau_\ast}(u)-1|\le \frac{\varepsilon}{4},
\qquad
|a_t(u)|\le \frac{\varepsilon}{4}\quad (t\neq \tau_\ast),
\]
and
\[
|a_t(u)|\le 2\qquad \forall t.
\]

Now apply Lemma~\ref{lem:local_multiplier} to the image
\[
\Kset':=W^{\mathrm{write}}_{T,\tau_\ast,\varepsilon_{\mathrm{wr}}}(\Kset),
\]
with the same \(e_{\mathrm{pos}},e_{\mathrm{sig}},e_{\mathrm{aux}},E_{\mathrm{carry}}\),
the bound \(M=2\), and accuracy \(\delta_{\mathrm{mul}}=\varepsilon/4\).
This yields a forward-only block
\[
M^{\mathrm{loc}}_{T,\delta_{\mathrm{mul}}}
\]
whose signal transport is exact and diagonal:
\[
\big\langle M^{\mathrm{loc}}_{T,\delta_{\mathrm{mul}}}(w)_t,e_{\mathrm{sig}}\big\rangle
=
D_{\mathrm{loc}}^{w}(t)\,\langle w_t,e_{\mathrm{sig}}\rangle
\qquad (w\in\Kset'),
\]
with
\[
|D_{\mathrm{loc}}^w(t)-\langle w_t,e_{\mathrm{aux}}\rangle|\le \frac{\varepsilon}{4}.
\]

Define
\[
S_{T,\tau_\ast,\varepsilon}
:=
M^{\mathrm{loc}}_{T,\delta_{\mathrm{mul}}}
\circ
W^{\mathrm{write}}_{T,\tau_\ast,\varepsilon_{\mathrm{wr}}}.
\]

Since the writer preserves the signal channel exactly,
\[
\langle W^{\mathrm{write}}_{T,\tau_\ast,\varepsilon_{\mathrm{wr}}}(u)_t,e_{\mathrm{sig}}\rangle
=
\langle u_t,e_{\mathrm{sig}}\rangle.
\]
Therefore
\[
\langle S_{T,\tau_\ast,\varepsilon}(u)_t,e_{\mathrm{sig}}\rangle
=
D_{\mathrm{loc}}^{W^{\mathrm{write}}(u)}(t)\,\langle u_t,e_{\mathrm{sig}}\rangle.
\]
Set
\[
D_{\mathrm{sel}}^u(t):=D_{\mathrm{loc}}^{W^{\mathrm{write}}(u)}(t).
\]
Then
\[
\langle S_{T,\tau_\ast,\varepsilon}(u)_t,e_{\mathrm{sig}}\rangle
=
D_{\mathrm{sel}}^u(t)\,\langle u_t,e_{\mathrm{sig}}\rangle,
\]
so the signal transport is exact and diagonal.

The coefficient \(D_{\mathrm{sel}}^u(t)\) depends only on the
\(e_{\mathrm{pos}}\)-, \(e_{\mathrm{aux}}\)-, and \(E_{\mathrm{carry}}\)-channels of the intermediate
state \(W^{\mathrm{write}}(u)\). The writer preserves \(e_{\mathrm{pos}}\) and \(E_{\mathrm{carry}}\) exactly,
and its written auxiliary channel \(a_t(u)\) is itself a deterministic function of the positional-control
coordinate only. Hence \(D_{\mathrm{sel}}^u(t)\) depends only on the original
\(e_{\mathrm{pos}}\)- and \(E_{\mathrm{carry}}\)-channels, not on the signal channel. Thus the transport
is signal-blind over \(E_{\mathrm{ctrl}}\).

The \(e_{\mathrm{pos}}\)-channel and all of \(E_{\mathrm{carry}}\) are preserved exactly by both blocks,
hence by the composition.

Finally, at the selected source,
\[
|D_{\mathrm{sel}}^u(\tau_\ast)-1|
\le
|D_{\mathrm{sel}}^u(\tau_\ast)-a_{\tau_\ast}(u)|
+
|a_{\tau_\ast}(u)-1|
\le
\frac{\varepsilon}{4}+\frac{\varepsilon}{4}
=
\frac{\varepsilon}{2},
\]
so since \(\varepsilon<1\),
\[
\frac12\le D_{\mathrm{sel}}^u(\tau_\ast)\le \frac32<2.
\]
For \(t\neq \tau_\ast\),
\[
|D_{\mathrm{sel}}^u(t)|
\le
|D_{\mathrm{sel}}^u(t)-a_t(u)|+|a_t(u)|
\le
\frac{\varepsilon}{4}+\frac{\varepsilon}{4}
=
\frac{\varepsilon}{2}
\le \varepsilon.
\]
For any \(\eta>0\), replacing \(\Kset\) by
\(\Sat^{\mathrm{sig}}_\eta(\Kset)\) leaves the ordered positional ranges
\((I_t)_{t=0}^{T}\) unchanged.
Moreover, in the concrete two-block construction, the writer depends only on
the positional coordinate and preserves the signal channel exactly, while the
local multiplier depends only on the positional and auxiliary channels and acts
diagonally on the signal channel.
Hence the same concrete construction yields the same exact diagonal
signal-transport formula on \(\Sat^{\mathrm{sig}}_\eta(\Kset)\), with the same
coefficients \(D_{\mathrm{sel}}^u(i)\).
Applying Lemma~\ref{lem:signal_blind_transport_calculus}(i) gives
\[
e_{\mathrm{sig}}^\top
\frac{\partial S_{T,\tau_\ast,\varepsilon}(u)_i}{\partial u_j}
e_{\mathrm{sig}}
=
D_{\mathrm{sel}}^u(i)\,\mathbf 1[i=j].
\]
\end{proof}

\begin{remark}[The selector depends only on position]
\label{rem:selector_positional_only}
In the concrete construction used in the proof of Lemma~\ref{lem:selector_window},
the diagonal transport coefficient \(D_{\mathrm{sel}}^u(t)\) depends only on the positional stream
\[
\bigl(\langle u_s,e_{\mathrm{pos}}\rangle\bigr)_{s=0}^{T},
\]
and is independent of the signal channel \(e_{\mathrm{sig}}\) and of the carried channels
\(E_{\mathrm{carry}}\).
\end{remark}

\begin{lemma}[Selector preserves signal fibers]
\label{lem:selector_signal_fiber_closure}
Under the hypotheses of Lemma~\ref{lem:selector_window}, let
\[
S_{T,\tau_\ast,\varepsilon}:(\mathbb R^m)^{T+1}\to(\mathbb R^m)^{T+1}
\]
be the selector block constructed there.
Then for every \(\delta\ge 0\) there exists \(\delta'=\delta'(\delta,\Kset)<\infty\) such that
\[
S_{T,\tau_\ast,\varepsilon}\bigl(\Sat^{\mathrm{sig}}_\delta(\Kset)\bigr)
\subset
\Sat^{\mathrm{sig}}_{\delta'}\bigl(S_{T,\tau_\ast,\varepsilon}(\Kset)\bigr).
\]
More precisely, if
\[
u' = u+\sum_{t=0}^{T} a_t e_{\mathrm{sig}}\mathbf 1[\cdot=t],
\qquad
u\in\Kset,
\qquad
\max_t |a_t|\le \delta,
\]
then
\[
S_{T,\tau_\ast,\varepsilon}(u')_i
=
S_{T,\tau_\ast,\varepsilon}(u)_i
+
D_{\mathrm{sel}}^u(i)\,a_i\,e_{\mathrm{sig}},
\qquad
0\le i\le T,
\]
where \(D_{\mathrm{sel}}^u(i)\) is the selector transport coefficient from
Lemma~\ref{lem:selector_window}. In particular, one may take
\[
\delta' := \delta\,
\sup_{u\in\Kset}\sup_{0\le i\le T}|D_{\mathrm{sel}}^u(i)|
\le 2\delta.
\]
\end{lemma}

\begin{proof}
Fix \(u\in\Kset\) and
\[
u' = u+\sum_{t=0}^{T} a_t e_{\mathrm{sig}}\mathbf 1[\cdot=t]
\]
with \(\max_t |a_t|\le \delta\).

By Remark~\ref{rem:selector_positional_only}, the coefficient \(D_{\mathrm{sel}}^u(i)\)
depends only on the positional stream
\[
(\langle u_s,e_{\mathrm{pos}}\rangle)_{s=0}^{T},
\]
which is unchanged under perturbations along \(e_{\mathrm{sig}}\).
Moreover, in the concrete construction of
\(S_{T,\tau_\ast,\varepsilon}\), all non-signal output channels are
independent of the input signal channel: the writer
\(W^{\mathrm{write}}_{T,\tau_\ast,\varepsilon_{\mathrm{wr}}}\) preserves
\(e_{\mathrm{sig}}\) exactly and writes only the auxiliary channel as a function
of the positional coordinate, while
\(M^{\mathrm{loc}}_{T,\delta_{\mathrm{mul}}}\) preserves the positional and
auxiliary channels exactly and modifies the output only on the signal channel.

Therefore
\[
S_{T,\tau_\ast,\varepsilon}(u')_i
=
S_{T,\tau_\ast,\varepsilon}(u)_i
+
D_{\mathrm{sel}}^u(i)\,a_i\,e_{\mathrm{sig}},
\]
and the claim follows.
\end{proof}

\begin{lemma}[Active diffusive transport]
\label{lem:active_diffusive_factorized_transport}
Fix \(\beta\in(0,1)\) and set \(\gamma:=1-\beta\).
Let \(T\ge 0\) and let \(\Kset\subset(\mathbb R^m)^{T+1}\) be compact.
Assume that for some orthonormal directions
\[
e_{\mathrm{pos}},\ e_{\mathrm{sig}},\ e_{\mathrm{src}},\ e_{\mathrm{tgt}}\in\mathbb R^m
\]
the scalar position ranges
\[
I_t:=\{\langle u_t,e_{\mathrm{pos}}\rangle:\ u\in\Kset\},
\qquad 0\le t\le T,
\]
are compact and strictly ordered:
\[
I_0<I_1<\cdots<I_T\subset(0,\infty).
\]
Let \(E_{\mathrm{carry}}\subset\mathbb R^m\) be any fixed subspace orthogonal to
\(e_{\mathrm{pos}},e_{\mathrm{sig}},e_{\mathrm{src}},e_{\mathrm{tgt}}\).

Then there exists a depth-\(2\) LN-free Sessa network
\[
A^{\mathrm{act}}_{T,\beta}:(\mathbb R^m)^{T+1}\to(\mathbb R^m)^{T+1}
\]

such that the first constituent block has the feedback branch switched off, while the second constituent block uses a strict-past uniform feedback solve with constant gain \(\gamma\), the \(e_{\mathrm{pos}}\)-channel and every channel in \(E_{\mathrm{carry}}\) are preserved exactly, and
\(A^{\mathrm{act}}_{T,\beta}\) has signal-blind exact scalar transport along \(e_{\mathrm{sig}}\) over
\[
E_{\mathrm{ctrl}}:=\operatorname{span}\{e_{\mathrm{pos}}\}\oplus E_{\mathrm{carry}},
\]
with kernel
\[
\mathcal T_{A^{\mathrm{act}}}^u(i,j)
=
D_{\mathrm{act}}^u(i)\,\mathbf 1[i=j]
+
K_{\mathrm{act}}^u(i,j)\,\mathbf 1[j<i];
\]
There exist constants
\[
0<\underline d_{\mathrm{act}}\le \overline d_{\mathrm{act}}<\infty,
\qquad
0<a^-_{\mathrm{act}}\le a^+_{\mathrm{act}}<\infty,
\]
depending only on \(\beta\), but independent of \(T\), such that
\[
\underline d_{\mathrm{act}}
\le
D_{\mathrm{act}}^u(i)
\le
\overline d_{\mathrm{act}},
\qquad 0\le i\le T,
\]
and
\[
a^-_{\mathrm{act}}(j+1)^{-\gamma}(i+1)^{-\beta}
\le
K_{\mathrm{act}}^u(i,j)
\le
a^+_{\mathrm{act}}(j+1)^{-\gamma}(i+1)^{-\beta},
\qquad 0\le j<i\le T.
\]
In particular,
\[
e_{\mathrm{sig}}^\top
\frac{\partial A^{\mathrm{act}}_{T,\beta}(u)_i}{\partial u_j}
e_{\mathrm{sig}}
=
D_{\mathrm{act}}^u(i)\,\mathbf 1[i=j]
+
K_{\mathrm{act}}^u(i,j)\,\mathbf 1[j<i].
\]
\end{lemma}

\begin{proof}
We construct
\[
A^{\mathrm{act}}_{T,\beta}=R_{T,\beta}\circ C_T,
\]
where \(C_T\) is a forward-only copy block and \(R_{T,\beta}\) is a single feedback-transport block.

\paragraph{Step 1: copy of the signal into a scratch source channel.}
Build a forward-only LN-free Sessa block
\[
C_T:(\mathbb R^m)^{T+1}\to(\mathbb R^m)^{T+1}
\]
such that
\[
\langle C_T(u)_t,e_{\mathrm{src}}\rangle
=
\langle u_t,e_{\mathrm{sig}}\rangle
\qquad (0\le t\le T),
\]
while the \(e_{\mathrm{pos}}\)-, \(e_{\mathrm{sig}}\)-, \(e_{\mathrm{tgt}}\)-, and \(E_{\mathrm{carry}}\)-channels are preserved exactly.

Switch off the feedback branch and choose two forward value coordinates equal to \(1\):
\[
v_t^{(0)}\equiv 1,\qquad v_t^{(1)}\equiv 1.
\]
Hence
\[
s_t^{(0)}=1,\qquad s_t^{(1)}=1.
\]
Choose the associated gate coordinates
\[
g_t^{(0)}=\langle u_t,e_{\mathrm{src}}\rangle,
\qquad
g_t^{(1)}=\langle u_t,e_{\mathrm{sig}}\rangle,
\]
and choose the output projection on the \(e_{\mathrm{src}}\)-channel with coefficients \((-1,+1)\).
Then
\[
\langle C_T(u)_t,e_{\mathrm{src}}\rangle
=
\langle u_t,e_{\mathrm{src}}\rangle
-
\langle u_t,e_{\mathrm{src}}\rangle
+
\langle u_t,e_{\mathrm{sig}}\rangle
=
\langle u_t,e_{\mathrm{sig}}\rangle.
\]

Let
\[
w:=C_T(u),
\qquad
x_j:=\langle u_j,e_{\mathrm{sig}}\rangle.
\]
Then
\begin{equation}
\label{eq:act_patch_copy}
\langle w_j,e_{\mathrm{src}}\rangle=x_j,
\qquad
\langle w_j,e_{\mathrm{sig}}\rangle=x_j.
\end{equation}

\paragraph{Step 2: the feedback-transport block.}
Now build a single LN-free Sessa block
\[
R_{T,\beta}:(\mathbb R^m)^{T+1}\to(\mathbb R^m)^{T+1}.
\]

On one dedicated feedback channel, choose all feedback queries and keys identically zero.
Then the strict-past feedback softmax is exactly uniform:
\[
\alpha^b_{i,j}=\frac1i,
\qquad 0\le j<i,\ \ 1\le i\le T.
\]
Choose the feedback gain to be the constant
\[
\gamma_i\equiv \gamma=1-\beta.
\]
Hence the scalar feedback matrix on that channel is
\[
B_{i,j}=\frac{\gamma}{i}\mathbf 1[j<i].
\]

For the forward branch, fix \(\mu_T\in(0,\frac12]\), to be chosen below, and apply
Lemma~\ref{lem:self_focus_forward} to the image \(C_T(\Kset)\) on the ordered positional-control coordinate.
Because \(C_T\) preserves the \(e_{\mathrm{pos}}\)-channel exactly, the hypotheses still hold.
This yields weights \(\alpha^f_{i,j}(w)\) satisfying
\begin{equation}
\label{eq:act_patch_self_focus}
\alpha^f_{i,i}(w)\ge 1-\mu_T,
\qquad
\sum_{j=0}^{i-1}\alpha^f_{i,j}(w)\le \mu_T,
\qquad 0\le i\le T.
\end{equation}
In particular, for every \(j<i\),
\begin{equation}
\label{eq:act_patch_single_offdiag}
\alpha^f_{i,j}(w)\le \mu_T.
\end{equation}

To read the source scratch channel exactly, use Corollary~\ref{cor:exact_channel_read} on the input \(w\) and the direction \(e_{\mathrm{src}}\):
choose two \(a\)-slots
\[
a_j^{(+)}=L\langle w_j,e_{\mathrm{src}}\rangle,
\qquad
a_j^{(-)}=-L\langle w_j,e_{\mathrm{src}}\rangle.
\]
Choose \(W_V\) so that one forward value coordinate is
\[
v_j^{\mathrm{src}}
=
\frac1L\bigl(\bar a_j^{(+)}-\bar a_j^{(-)}\bigr)
=
\langle w_j,e_{\mathrm{src}}\rangle
=
x_j.
\]
Let
\[
f_i:=\sum_{j\le i}\alpha^f_{i,j}(w)\,v_j^{\mathrm{src}}
=
\sum_{j\le i}\alpha^f_{i,j}(w)\,x_j
\]
be the forward signal entering the scalar feedback solve, and let \(s_i\) denote the corresponding solve output:
\[
s_0=f_0,\qquad
s_i=f_i+\frac{\gamma}{i}\sum_{j<i}s_j,\qquad 1\le i\le T.
\]
Choose the gate on that transport coordinate to be the constant \(1\), and choose the output projection so that
the signal channel receives exactly \(+s_i\), while the \(e_{\mathrm{pos}}\)- and \(E_{\mathrm{carry}}\)-channels are untouched.
Therefore
\[
\langle R_{T,\beta}(w)_i,e_{\mathrm{sig}}\rangle
=
\langle w_i,e_{\mathrm{sig}}\rangle+s_i
=
x_i+s_i.
\]

\paragraph{Step 3: resolvent kernel.}
Let
\[
\Theta_{i,j}:=[(I-B)^{-1}]_{i,j},\qquad 0\le j\le i\le T.
\]
Then \(\Theta_{i,i}=1\), and for \(j<i\),
\[
\Theta_{i,j}
=
\frac{\gamma}{i}\sum_{r=j}^{i-1}\Theta_{r,j}.
\]
As in the original proof, define
\[
S_i^{(j)}:=\sum_{r=j}^{i}\Theta_{r,j}.
\]
Then \(S_j^{(j)}=1\) and
\[
S_i^{(j)}=\Bigl(1+\frac{\gamma}{i}\Bigr)S_{i-1}^{(j)},
\]
hence
\[
S_i^{(j)}
=
\frac{\Gamma(i+1+\gamma)\Gamma(j+1)}{\Gamma(j+1+\gamma)\Gamma(i+1)}.
\]
Therefore, for \(j<i\),
\[
\Theta_{i,j}
=
\frac{\gamma}{i}S_{i-1}^{(j)}
=
\gamma\,
\frac{\Gamma(j+1)}{\Gamma(j+1+\gamma)}
\frac{\Gamma(i+\gamma)}{\Gamma(i+1)}.
\]
Since \(\gamma\in(0,1)\), standard Gamma-ratio bounds yield constants
\[
0<c^-_\Theta\le c^+_\Theta<\infty
\]
depending only on \(\beta\), such that
\begin{equation}
\label{eq:act_patch_theta_factorized}
c^-_\Theta (j+1)^{-\gamma}(i+1)^{-\beta}
\le
\Theta_{i,j}
\le
c^+_\Theta (j+1)^{-\gamma}(i+1)^{-\beta},
\qquad 0\le j<i\le T.
\end{equation}

Also, since \(\gamma=1-\beta\in(0,1)\),
\[
\sum_{r=1}^{n}r^{-\gamma}\lesssim_\beta n^\beta.
\]
Combining this with \eqref{eq:act_patch_theta_factorized}, there exists a constant \(C_\Sigma<\infty\),
depending only on \(\beta\), such that
\begin{equation}
\label{eq:act_patch_theta_sum_bound}
\sum_{k=j+1}^{i}\Theta_{i,k}\le C_\Sigma
\qquad (0\le j<i\le T).
\end{equation}
Finally, since \(j+1\le i+1\le T+1\),
\begin{equation}
\label{eq:act_patch_theta_lower_T}
\Theta_{i,j}
\ge
c^-_\Theta(i+1)^{-1}
\ge
\frac{c^-_\Theta}{T+1}.
\end{equation}

\paragraph{Step 4: transport formula.}
Since \(s=\Theta f\),
\[
s_i
=
\sum_{k=0}^{i}\Theta_{i,k}f_k
=
\sum_{k=0}^{i}\Theta_{i,k}\sum_{j=0}^{k}\alpha^f_{k,j}(w)x_j
=
\sum_{j=0}^{i}\Bigl(\sum_{k=j}^{i}\Theta_{i,k}\alpha^f_{k,j}(w)\Bigr)x_j.
\]
Therefore
\[
\langle A^{\mathrm{act}}_{T,\beta}(u)_i,e_{\mathrm{sig}}\rangle
=
x_i+s_i
=
\Bigl(1+\alpha^f_{i,i}(w)\Bigr)x_i
+
\sum_{j<i}\Bigl(\sum_{k=j}^{i}\Theta_{i,k}\alpha^f_{k,j}(w)\Bigr)x_j.
\]
Define
\[
D_{\mathrm{act}}^u(i):=1+\alpha^f_{i,i}(w),
\qquad
K_{\mathrm{act}}^u(i,j):=\sum_{k=j}^{i}\Theta_{i,k}\alpha^f_{k,j}(w)\quad (j<i).
\]
Then
\[
\langle A^{\mathrm{act}}_{T,\beta}(u)_i,e_{\mathrm{sig}}\rangle
=
D_{\mathrm{act}}^u(i)\,x_i
+
\sum_{j<i}K_{\mathrm{act}}^u(i,j)\,x_j.
\]

This is exact scalar transport. The coefficients depend only on the positional stream of \(w\), because the forward weights \(\alpha^f\) were built from the positional-control coordinate only; and \(C_T\) preserves the positional coordinate exactly, so this is the same as the positional stream of \(u\).
The \(e_{\mathrm{pos}}\)- and \(E_{\mathrm{carry}}\)-channels are preserved exactly by construction.
Thus the transport is signal-blind over \(E_{\mathrm{ctrl}}\).

\paragraph{Step 5: kernel bounds.}
From \eqref{eq:act_patch_self_focus},
\[
1-\mu_T\le \alpha^f_{i,i}(w)\le 1,
\]
so
\[
2-\mu_T\le D_{\mathrm{act}}^u(i)\le 2.
\]
Since \(\mu_T\le \frac12\),
\[
\frac32\le D_{\mathrm{act}}^u(i)\le 2.
\]
Thus we may take
\[
\underline d_{\mathrm{act}}:=\frac32,
\qquad
\overline d_{\mathrm{act}}:=2.
\]

For the off-diagonal coefficient, all summands are nonnegative. Hence for \(j<i\),
\[
K_{\mathrm{act}}^u(i,j)
\ge
\Theta_{i,j}\alpha^f_{j,j}(w)
\ge
(1-\mu_T)\Theta_{i,j}
\ge
\frac12\,\Theta_{i,j}.
\]
Combining with \eqref{eq:act_patch_theta_factorized} gives
\[
K_{\mathrm{act}}^u(i,j)\ge \frac12\,c^-_\Theta (j+1)^{-\gamma}(i+1)^{-\beta}.
\]

For the upper bound,
\[
K_{\mathrm{act}}^u(i,j)
=
\Theta_{i,j}\alpha^f_{j,j}(w)
+
\sum_{k=j+1}^{i}\Theta_{i,k}\alpha^f_{k,j}(w)
\le
\Theta_{i,j}
+
\mu_T\sum_{k=j+1}^{i}\Theta_{i,k},
\]
by \eqref{eq:act_patch_single_offdiag}.
Now choose
\[
\mu_T:=\min\Bigl\{\frac12,\ \frac{c^-_\Theta}{4C_\Sigma(T+1)}\Bigr\}.
\]
Then by \eqref{eq:act_patch_theta_sum_bound},
\[
\mu_T\sum_{k=j+1}^{i}\Theta_{i,k}
\le
\frac{c^-_\Theta}{4(T+1)}.
\]
By \eqref{eq:act_patch_theta_lower_T},
\[
\frac{c^-_\Theta}{4(T+1)}\le \frac14\,\Theta_{i,j}.
\]
Hence
\[
K_{\mathrm{act}}^u(i,j)\le \frac54\,\Theta_{i,j}.
\]
Using \eqref{eq:act_patch_theta_factorized},
\[
K_{\mathrm{act}}^u(i,j)\le \frac54\,c^+_\Theta (j+1)^{-\gamma}(i+1)^{-\beta}.
\]
Thus the stated two-sided bounds hold with
\[
a^-_{\mathrm{act}}:=\frac12\,c^-_\Theta,
\qquad
a^+_{\mathrm{act}}:=\frac54\,c^+_\Theta.
\]

For any \(\eta>0\), replacing \(\Kset\) by
\(\Sat^{\mathrm{sig}}_\eta(\Kset)\) leaves the ordered positional ranges
\((I_t)_{t=0}^{T}\) unchanged.
In the concrete construction, the copy block writes the source scratch channel
from the signal channel exactly and is independent of the incoming
\(e_{\mathrm{src}}\)-channel, while the transport block uses forward and
feedback weights depending only on the positional stream and reads the copied
source scratch channel exactly.
Hence the same concrete construction yields the same exact scalar transport
formula on \(\Sat^{\mathrm{sig}}_\eta(\Kset)\), with the same coefficients
\(D_{\mathrm{act}}^u(i)\) and \(K_{\mathrm{act}}^u(i,j)\).
Applying Lemma~\ref{lem:signal_blind_transport_calculus}(i) gives
\[
e_{\mathrm{sig}}^\top
\frac{\partial A^{\mathrm{act}}_{T,\beta}(u)_i}{\partial u_j}
e_{\mathrm{sig}}
=
D_{\mathrm{act}}^u(i)\,\mathbf 1[i=j]
+
K_{\mathrm{act}}^u(i,j)\,\mathbf 1[j<i].
\]
\end{proof}

\begin{remark}[Active diffusive transport depends only on position]
\label{rem:act_positional_only}
In the concrete construction used in the proof of
Lemma~\ref{lem:active_diffusive_factorized_transport},
the coefficients
\[
D_{\mathrm{act}}^u(i),
\qquad
K_{\mathrm{act}}^u(i,j),
\qquad 0\le j<i\le T,
\]
depend only on the positional stream
\[
\bigl(\langle u_s,e_{\mathrm{pos}}\rangle\bigr)_{s=0}^{T},
\]
and are independent of the signal channel \(e_{\mathrm{sig}}\) and of the carried channels
\(E_{\mathrm{carry}}\).
\end{remark}

\begin{lemma}[Transparent source-\(0\) tail channel]
\label{lem:source0_tail_channel}
Fix \(\beta\in(0,1)\), set \(\gamma:=1-\beta\), fix \(\tau_{\max}\ge 0\), and let
\[
L_H:=\tau_{\max}+H.
\]
Let \(\Kset_H\subset(\mathbb R^m)^{L_H+1}\) be compact.
Assume orthonormal directions
\[
e_{\mathrm{sig}},\ e_{\mathrm{pos}},\ e_{\mathrm{tail}},\ e_{\mathrm{aux}},\ e_{\mathrm{src}},\ e_{\mathrm{tgt}}\in\mathbb R^m
\]
and a subspace \(E_{\mathrm{carry}}\subset\mathbb R^m\) orthogonal to all six, such that
\[
I_t:=\{\langle u_t,e_{\mathrm{pos}}\rangle:\ u\in\Kset_H\},
\qquad 0\le t\le L_H,
\]
are compact and strictly ordered:
\[
I_0<I_1<\cdots<I_{L_H}\subset(0,\infty).
\]

Then there exists a constant-depth LN-free Sessa network
\[
T_H^{\mathrm{tail}}:(\mathbb R^m)^{L_H+1}\to(\mathbb R^m)^{L_H+1}
\]

such that the \(e_{\mathrm{sig}}\)-channel, the positional-control coordinate \(e_{\mathrm{pos}}\), and every channel in \(E_{\mathrm{carry}}\) are preserved exactly and, writing
\[
g_t(u):=\bigl\langle T_H^{\mathrm{tail}}(u)_t,\ e_{\mathrm{tail}}\bigr\rangle,
\qquad 0\le t\le L_H,
\]
there exist constants \(c_g^-,c_g^+>0\), independent of \(H\), such that
\[
c_g^-(t+1)^{-\beta}\le g_t(u)\le c_g^+(t+1)^{-\beta},
\qquad
0\le t\le L_H,\ \ u\in\Kset_H;
\]
\(T_H^{\mathrm{tail}}\) is signal-transparent along \(e_{\mathrm{sig}}\) with respect to the control pair
\[
\bigl(e_{\mathrm{pos}},e_{\mathrm{tail}}\bigr):
\]
for every \(u\in\Kset_H\), every \(\tau\in\{0,\dots,L_H\}\), and every scalar \(a\in\mathbb R\),
\[
\bigl\langle T_H^{\mathrm{tail}}(u+a\,e_{\mathrm{sig}}\mathbf 1[\cdot=\tau])_t,\ e_{\mathrm{pos}}\bigr\rangle
=
\bigl\langle T_H^{\mathrm{tail}}(u)_t,\ e_{\mathrm{pos}}\bigr\rangle,
\]
\[
\bigl\langle T_H^{\mathrm{tail}}(u+a\,e_{\mathrm{sig}}\mathbf 1[\cdot=\tau])_t,\ e_{\mathrm{tail}}\bigr\rangle
=
\bigl\langle T_H^{\mathrm{tail}}(u)_t,\ e_{\mathrm{tail}}\bigr\rangle,
\]
\[
\bigl\langle T_H^{\mathrm{tail}}(u+a\,e_{\mathrm{sig}}\mathbf 1[\cdot=\tau])_t,\ e_{\mathrm{sig}}\bigr\rangle
=
\bigl\langle T_H^{\mathrm{tail}}(u)_t,\ e_{\mathrm{sig}}\bigr\rangle
+
a\,\mathbf 1[t=\tau],
\qquad
0\le t\le L_H.
\]
\end{lemma}

\begin{proof}
All auxiliary directions used below are part of the hypotheses; no fresh direction is chosen inside the construction.
We construct
\[
T_H^{\mathrm{tail}}=A_H^{\mathrm{tail}}\circ S_H^{\mathrm{tail}}\circ C_H,
\]
where \(C_H\) writes a constant seed on the prescribed tail direction \(e_{\mathrm{tail}}\),
\(S_H^{\mathrm{tail}}\) selects source \(0\) on that tail channel, and
\(A_H^{\mathrm{tail}}\) transports the selected seed by the active diffusive block.

\paragraph{Step 1: constant seed writer on the prescribed tail direction.}
Build a forward-only LN-free Sessa block
\[
C_H:(\mathbb R^m)^{L_H+1}\to(\mathbb R^m)^{L_H+1}
\]
as follows.

Choose two forward value coordinates equal to \(1\):
\[
v_t^{(0)}\equiv 1,\qquad v_t^{(1)}\equiv 1.
\]
Hence the corresponding forward aggregates satisfy
\[
s_t^{(0)}=1,\qquad s_t^{(1)}=1.
\]
Choose two gate coordinates
\[
g_t^{(0)}=\langle u_t,e_{\mathrm{tail}}\rangle,
\qquad
g_t^{(1)}\equiv 1,
\]
and choose the output projection on the \(e_{\mathrm{tail}}\)-channel with coefficients
\((-1,+1)\) on these two gated coordinates and zero on all other output channels.
Then
\[
\bigl\langle C_H(u)_t,e_{\mathrm{tail}}\bigr\rangle
=
\langle u_t,e_{\mathrm{tail}}\rangle
-
s_t^{(0)}\langle u_t,e_{\mathrm{tail}}\rangle
+
s_t^{(1)}
=
1.
\]
Thus \(C_H\) overwrites the \(e_{\mathrm{tail}}\)-channel by the constant seed \(1\).

Because the output projection vanishes on the \(e_{\mathrm{sig}}\)-, \(e_{\mathrm{pos}}\)-, and
\(E_{\mathrm{carry}}\)-channels, these channels are preserved exactly:
\[
\bigl\langle C_H(u)_t,e_{\mathrm{sig}}\bigr\rangle=\langle u_t,e_{\mathrm{sig}}\rangle,
\qquad
\bigl\langle C_H(u)_t,e_{\mathrm{pos}}\bigr\rangle=\langle u_t,e_{\mathrm{pos}}\rangle,
\]
and likewise on \(E_{\mathrm{carry}}\).

Moreover, since the written tail seed is constant and independent of the input,
for every \(a\in\mathbb R\),
\[
\bigl\langle C_H(u+a\,e_{\mathrm{sig}}\mathbf 1[\cdot=\tau])_t,e_{\mathrm{tail}}\bigr\rangle
=
\bigl\langle C_H(u)_t,e_{\mathrm{tail}}\bigr\rangle
=
1,
\]
while the \(e_{\mathrm{sig}}\)-channel passes through exactly. So \(C_H\) is already
signal-transparent along \(e_{\mathrm{sig}}\) with respect to \((e_{\mathrm{pos}},e_{\mathrm{tail}})\).

\paragraph{Step 2: positional selector on the tail channel.}
Let
\[
\Kset_H^{(1)}:=C_H(\Kset_H).
\]
Apply Lemma~\ref{lem:selector_window} to \(\Kset_H^{(1)}\) with signal direction
\(e_{\mathrm{sig}}^{\mathrm{sel}}:=e_{\mathrm{tail}}\), positional-control direction
\(e_{\mathrm{pos}}\), auxiliary direction \(e_{\mathrm{aux}}\), source index
\(\tau_\ast=0\), and carried-through subspace
\[
E_{\mathrm{carry}}^{\mathrm{sel}}
:=
\operatorname{span}\{e_{\mathrm{sig}}\}\oplus E_{\mathrm{carry}}.
\]

Choose an exponent \(M>\beta\) and set
\[
\varepsilon_H:=c_0(H+1)^{-M},
\]
where \(c_0>0\) will be chosen later.
The lemma yields a depth-\(2\) network
\[
S_H^{\mathrm{tail}}:=S_{L_H,0,\varepsilon_H}
\]
which preserves \(e_{\mathrm{pos}}\), the original \(e_{\mathrm{sig}}\), and every channel in \(E_{\mathrm{carry}}\) exactly, and whose exact diagonal transport on the tail channel is
\[
\bigl\langle S_H^{\mathrm{tail}}(v)_t,e_{\mathrm{tail}}\bigr\rangle
=
D_{\mathrm{sel}}^v(t)\,\langle v_t,e_{\mathrm{tail}}\rangle.
\]
Since \(\langle C_H(u)_t,e_{\mathrm{tail}}\rangle\equiv 1\), the selected seed stream is
\[
z_t(u):=\bigl\langle S_H^{\mathrm{tail}}(C_H(u))_t,e_{\mathrm{tail}}\bigr\rangle
=
D_{\mathrm{sel}}^{C_H(u)}(t).
\]
By Lemma~\ref{lem:selector_window},
\[
\frac12\le z_0(u)\le 2,
\qquad
|z_t(u)|\le \varepsilon_H\quad (t\ge 1).
\]

By Remark~\ref{rem:selector_positional_only}, in the concrete construction of
\(S_H^{\mathrm{tail}}=S_{L_H,0,\varepsilon_H}\) the coefficient
\[
D_{\mathrm{sel}}^{C_H(u)}(t)
\]
depends only on the positional stream
\[
\bigl(\langle C_H(u)_s,e_{\mathrm{pos}}\rangle\bigr)_{s=0}^{L_H}.
\]
Since \(C_H\) preserves the positional coordinate exactly,
\[
\langle C_H(u)_s,e_{\mathrm{pos}}\rangle=\langle u_s,e_{\mathrm{pos}}\rangle,
\qquad 0\le s\le L_H,
\]
it follows that
\[
z_t(u)=D_{\mathrm{sel}}^{C_H(u)}(t)
\]
depends only on the original positional stream and not on the original signal channel.

\paragraph{Step 3: active diffusive transport on the same prescribed tail direction.}
Let
\[
\Kset_H^{(2)}:=S_H^{\mathrm{tail}}(\Kset_H^{(1)}).
\]
Apply Lemma~\ref{lem:active_diffusive_factorized_transport} to \(\Kset_H^{(2)}\) with
positional direction \(e_{\mathrm{pos}}\), signal direction
\(e_{\mathrm{sig}}^{\mathrm{act}}:=e_{\mathrm{tail}}\), scratch directions
\(e_{\mathrm{src}},e_{\mathrm{tgt}}\), and carried-through subspace
\[
E_{\mathrm{carry}}^{\mathrm{act}}
:=
\operatorname{span}\{e_{\mathrm{sig}}\}\oplus E_{\mathrm{carry}}.
\]
Denote the resulting network by
\[
A_H^{\mathrm{tail}}.
\]
By the lemma, \(A_H^{\mathrm{tail}}\) preserves \(e_{\mathrm{pos}}\), the original \(e_{\mathrm{sig}}\), and \(E_{\mathrm{carry}}\) exactly, and has exact scalar transport on the tail channel:
\[
\bigl\langle A_H^{\mathrm{tail}}(w)_t,e_{\mathrm{tail}}\bigr\rangle
=
D_{\mathrm{act}}^w(t)\,\langle w_t,e_{\mathrm{tail}}\rangle
+
\sum_{j<t}K_{\mathrm{act}}^w(t,j)\,\langle w_j,e_{\mathrm{tail}}\rangle.
\]
Therefore, for
\[
g_t(u):=\bigl\langle T_H^{\mathrm{tail}}(u)_t,e_{\mathrm{tail}}\bigr\rangle,
\]
we have
\[
g_t(u)
=
D_{\mathrm{act}}^{w}(t)\,z_t(u)+\sum_{j<t}K_{\mathrm{act}}^{w}(t,j)\,z_j(u),
\qquad
w:=S_H^{\mathrm{tail}}(C_H(u)).
\]

By Remark~\ref{rem:act_positional_only}, in the concrete construction of
\(A_H^{\mathrm{tail}}\) the coefficients
\[
D_{\mathrm{act}}^{w}(t),
\qquad
K_{\mathrm{act}}^{w}(t,j)
\]
depend only on the positional stream
\[
\bigl(\langle w_s,e_{\mathrm{pos}}\rangle\bigr)_{s=0}^{L_H}.
\]
Since both \(C_H\) and \(S_H^{\mathrm{tail}}\) preserve the positional coordinate exactly,
this is the same as the original positional stream of \(u\).
Hence these coefficients are independent of the original signal channel.

\paragraph{Step 4: two-sided tail bounds.}
At \(t=0\), the sum is empty, so
\[
g_0(u)=D_{\mathrm{act}}^{w}(0)\,z_0(u).
\]
By Lemma~\ref{lem:active_diffusive_factorized_transport},
\[
\underline d_{\mathrm{act}}\le D_{\mathrm{act}}^{w}(0)\le \overline d_{\mathrm{act}},
\]
hence
\[
\frac12\,\underline d_{\mathrm{act}}
\le
g_0(u)
\le
2\,\overline d_{\mathrm{act}}.
\]

Now fix \(t\ge 1\). Using the exact transport formula, the bounds on \(z_j(u)\), and the coefficient bounds from Lemma~\ref{lem:active_diffusive_factorized_transport}, we obtain
\begin{align*}
g_t(u)
&\ge
K_{\mathrm{act}}^{w}(t,0)\,z_0(u)
-
|D_{\mathrm{act}}^{w}(t)z_t(u)|
-
\sum_{j=1}^{t-1}K_{\mathrm{act}}^{w}(t,j)\,|z_j(u)|\\
&\ge
\frac12\,a^-_{\mathrm{act}}(t+1)^{-\beta}
-
\overline d_{\mathrm{act}}\,\varepsilon_H
-
a^+_{\mathrm{act}}\varepsilon_H\sum_{j=1}^{t-1}(j+1)^{-\gamma}(t+1)^{-\beta}.
\end{align*}
Since \(\gamma=1-\beta\in(0,1)\),
\[
\sum_{j=1}^{t-1}(j+1)^{-\gamma}\lesssim_\beta (t+1)^\beta,
\]
hence
\[
g_t(u)\ge c_1(t+1)^{-\beta}-c_2\varepsilon_H
\]
for constants \(c_1,c_2>0\) independent of \(H\).

Now \(M>\beta\), so
\[
\varepsilon_H=c_0(H+1)^{-M}\le c_0(H+1)^{-\beta}.
\]
Also \(0\le t\le L_H=\tau_{\max}+H\), hence
\[
(H+1)^{-\beta}\le (\tau_{\max}+1)^\beta (t+1)^{-\beta}.
\]
Therefore
\[
\varepsilon_H\lesssim_{\tau_{\max}} c_0 (t+1)^{-\beta}.
\]
Choosing \(c_0>0\) sufficiently small makes the error absorbable, so
\[
g_t(u)\ge c_g^-(t+1)^{-\beta}
\]
for some \(c_g^->0\) independent of \(H\).

Similarly,
\begin{align*}
g_t(u)
&\le
|D_{\mathrm{act}}^{w}(t)z_t(u)|
+
K_{\mathrm{act}}^{w}(t,0)|z_0(u)|
+
\sum_{j=1}^{t-1}K_{\mathrm{act}}^{w}(t,j)|z_j(u)|\\
&\le
\overline d_{\mathrm{act}}\varepsilon_H
+
2a^+_{\mathrm{act}}(t+1)^{-\beta}
+
a^+_{\mathrm{act}}\varepsilon_H\sum_{j=1}^{t-1}(j+1)^{-\gamma}(t+1)^{-\beta},
\end{align*}
hence
\[
g_t(u)\le c_g^+(t+1)^{-\beta}
\]
for some \(c_g^+<\infty\) independent of \(H\).

Thus
\[
c_g^-(t+1)^{-\beta}\le g_t(u)\le c_g^+(t+1)^{-\beta},
\qquad
0\le t\le L_H.
\]

\paragraph{Step 5: signal-transparency along \(e_{\mathrm{sig}}\).}
Let
\[
u^{(a,\tau)}:=u+a\,e_{\mathrm{sig}}\mathbf 1[\cdot=\tau].
\]
Since \(e_{\mathrm{sig}}\perp e_{\mathrm{pos}}\), we have
\[
\langle u^{(a,\tau)}_t,e_{\mathrm{pos}}\rangle=\langle u_t,e_{\mathrm{pos}}\rangle
\qquad \forall t.
\]
By Step 1,
\[
\langle C_H(u^{(a,\tau)})_t,e_{\mathrm{tail}}\rangle
=
\langle C_H(u)_t,e_{\mathrm{tail}}\rangle
=
1,
\]
and
\[
\langle C_H(u^{(a,\tau)})_t,e_{\mathrm{sig}}\rangle
=
\langle C_H(u)_t,e_{\mathrm{sig}}\rangle+a\,\mathbf 1[t=\tau].
\]

By the dependence analysis in Step 2, \(z_t(u)\) depends only on the positional stream, so
\[
z_t(u^{(a,\tau)})=z_t(u).
\]
By the dependence analysis in Step 3, the coefficients \(D_{\mathrm{act}}^{w},K_{\mathrm{act}}^{w}\) also depend only on the positional stream, hence they are unchanged under the perturbation.
Therefore the tail output is unchanged:
\[
g_t(u^{(a,\tau)})=g_t(u).
\]
Since each constituent block preserves the original \(e_{\mathrm{sig}}\)-channel exactly, the full composition satisfies
\[
\bigl\langle T_H^{\mathrm{tail}}(u^{(a,\tau)})_t,e_{\mathrm{sig}}\bigr\rangle
=
\bigl\langle T_H^{\mathrm{tail}}(u)_t,e_{\mathrm{sig}}\bigr\rangle
+
a\,\mathbf 1[t=\tau].
\]
The \(e_{\mathrm{pos}}\)-coordinate is preserved exactly at each stage as well. This proves signal-transparency.
\end{proof}

\begin{lemma}[Residual zero-writer]
\label{lem:profile_zero_writer}
Fix \(T\ge 0\), a compact set \(\Kset\subset (\mathbb R^m)^{T+1}\), orthonormal directions
\[
e_{\mathrm{sig}},\ e_{\mathrm{pos}},\ e_{\mathrm{zero}}\in\mathbb R^m,
\]
and a subspace \(E_{\mathrm{carry}}\subset\mathbb R^m\) orthogonal to all three.
Then there exists a single LN-free Sessa block
\[
Z_{T,e_{\mathrm{zero}}}:(\mathbb R^m)^{T+1}\to(\mathbb R^m)^{T+1}
\]

such that the feedback branch is switched off, the \(e_{\mathrm{sig}}\)-channel, the \(e_{\mathrm{pos}}\)-channel, and every channel in \(E_{\mathrm{carry}}\) are preserved exactly, the prescribed channel is written to zero exactly:
\[
\bigl\langle Z_{T,e_{\mathrm{zero}}}(u)_t,\ e_{\mathrm{zero}}\bigr\rangle=0
\qquad
\forall\,u\in\Kset,\ \forall\,0\le t\le T;
\]
\(Z_{T,e_{\mathrm{zero}}}\) is signal-transparent along \(e_{\mathrm{sig}}\) with respect to the control pair
\((e_{\mathrm{pos}},e_{\mathrm{zero}})\):
for every \(u\in\Kset\), every \(\tau\in\{0,\dots,T\}\), every scalar \(a\in\mathbb R\), and every \(0\le t\le T\),
\[
\bigl\langle Z_{T,e_{\mathrm{zero}}}(u+a\,e_{\mathrm{sig}}\mathbf 1[\cdot=\tau])_t,\ e_{\mathrm{pos}}\bigr\rangle
=
\bigl\langle Z_{T,e_{\mathrm{zero}}}(u)_t,\ e_{\mathrm{pos}}\bigr\rangle,
\]
\[
\bigl\langle Z_{T,e_{\mathrm{zero}}}(u+a\,e_{\mathrm{sig}}\mathbf 1[\cdot=\tau])_t,\ e_{\mathrm{zero}}\bigr\rangle
=
\bigl\langle Z_{T,e_{\mathrm{zero}}}(u)_t,\ e_{\mathrm{zero}}\bigr\rangle
=
0,
\]
and
\[
\bigl\langle Z_{T,e_{\mathrm{zero}}}(u+a\,e_{\mathrm{sig}}\mathbf 1[\cdot=\tau])_t,\ e_{\mathrm{sig}}\bigr\rangle
=
\bigl\langle Z_{T,e_{\mathrm{zero}}}(u)_t,\ e_{\mathrm{sig}}\bigr\rangle
+
a\,\mathbf 1[t=\tau].
\]
\end{lemma}

\begin{proof}
Switch off the feedback branch.

Choose a positive constant \(c_1\) such that
\[
\GELU(c_1)=1.
\]
Realize one forward value coordinate by the constant \(1\):
\[
v_t^{(0)}\equiv 1.
\]
Since every forward attention row sums to \(1\), the corresponding forward aggregate is
\[
s_t^{(0)}=\sum_{j\le t}\alpha^f_{t,j}\cdot 1=1
\qquad
(0\le t\le T).
\]

Choose one gate coordinate equal to the prescribed channel:
\[
g_t^{(0)}=\langle u_t,e_{\mathrm{zero}}\rangle.
\]
Choose the output projection so that this gated coordinate contributes
\[
-\,e_{\mathrm{zero}}
\]
and all other output columns are zero. Then the residual update adds
\[
-\,s_t^{(0)}\,g_t^{(0)}\,e_{\mathrm{zero}}
=
-\,\langle u_t,e_{\mathrm{zero}}\rangle e_{\mathrm{zero}}.
\]
Therefore
\[
Z_{T,e_{\mathrm{zero}}}(u)_t
=
u_t-\langle u_t,e_{\mathrm{zero}}\rangle e_{\mathrm{zero}}.
\]
Taking the \(e_{\mathrm{zero}}\)-coordinate gives
\[
\bigl\langle Z_{T,e_{\mathrm{zero}}}(u)_t,e_{\mathrm{zero}}\bigr\rangle
=
\langle u_t,e_{\mathrm{zero}}\rangle-\langle u_t,e_{\mathrm{zero}}\rangle
=0,
\]
which proves the exact zero-writing claim.

Because the update is supported only on the \(e_{\mathrm{zero}}\)-direction, and
\[
e_{\mathrm{sig}},\ e_{\mathrm{pos}},\ E_{\mathrm{carry}}\perp e_{\mathrm{zero}},
\]
the \(e_{\mathrm{sig}}\)-channel, the \(e_{\mathrm{pos}}\)-channel, and all channels in \(E_{\mathrm{carry}}\)
are preserved exactly. This proves the exact preservation claim.

For signal-transparency, let
\[
u^{(a,\tau)}:=u+a\,e_{\mathrm{sig}}\mathbf 1[\cdot=\tau].
\]
Since \(e_{\mathrm{sig}}\perp e_{\mathrm{zero}},e_{\mathrm{pos}}\), one has
\[
\langle u^{(a,\tau)}_t,e_{\mathrm{zero}}\rangle=\langle u_t,e_{\mathrm{zero}}\rangle,
\qquad
\langle u^{(a,\tau)}_t,e_{\mathrm{pos}}\rangle=\langle u_t,e_{\mathrm{pos}}\rangle.
\]
Applying the explicit formula for \(Z_{T,e_{\mathrm{zero}}}\) yields
\[
Z_{T,e_{\mathrm{zero}}}(u^{(a,\tau)})_t
=
u_t+a\,e_{\mathrm{sig}}\mathbf 1[t=\tau]
-\langle u_t,e_{\mathrm{zero}}\rangle e_{\mathrm{zero}}
=
Z_{T,e_{\mathrm{zero}}}(u)_t+a\,e_{\mathrm{sig}}\mathbf 1[t=\tau].
\]
Taking the \(e_{\mathrm{pos}}\)-, \(e_{\mathrm{zero}}\)-, and \(e_{\mathrm{sig}}\)-coordinates gives the stated signal-transparency property.
\end{proof}

\begin{lemma}[Exact reset of finitely many scratch channels]
\label{lem:multi_channel_reset}
Fix \(T\ge 0\), orthonormal directions
\[
e_{\mathrm{sig}},\ e_{z,1},\dots,e_{z,p}\in\mathbb R^m,
\]
and a subspace \(E_{\mathrm{keep}}\subset\mathbb R^m\) orthogonal to all of them.
Then there exists a single forward-only concrete LN-free Sessa block
\[
Z^{\mathrm{scr}}_{T,\{e_{z,r}\}}:(\mathbb R^m)^{T+1}\to (\mathbb R^m)^{T+1}
\]
such that \(Z^{\mathrm{scr}}_{T,\{e_{z,r}\}}\) preserves \(e_{\mathrm{sig}}\) and every channel in \(E_{\mathrm{keep}}\) exactly, and for every \(u\) and every \(t\),
\[
\bigl\langle Z^{\mathrm{scr}}_{T,\{e_{z,r}\}}(u)_t,e_{z,r}\bigr\rangle=0
\qquad
(r=1,\dots,p);
\]
\(Z^{\mathrm{scr}}_{T,\{e_{z,r}\}}\) is signal-transparent along \(e_{\mathrm{sig}}\) over \(E_{\mathrm{keep}}\).
\end{lemma}

\begin{proof}
Switch off the feedback branch and choose the forward queries and keys identically zero, so that every forward row has sum \(1\).

Choose a positive constant \(c_\ast\) with
\[
\GELU(c_\ast)=1.
\]
Activate one constant \(a\)-slot:
\[
a_t^{(1)}\equiv c_\ast.
\]
Then one post-GELU coordinate is identically \(1\). Choose \(W_V\) so that the first \(p\) forward value coordinates are all equal to \(1\). Since each forward row sums to \(1\), the corresponding forward aggregates satisfy
\[
s_t^{(r)}=1
\qquad
(r=1,\dots,p).
\]

Choose the first \(p\) gate coordinates as
\[
g_t^{(r)}=\langle u_t,e_{z,r}\rangle,
\qquad
r=1,\dots,p,
\]
and set all remaining gate coordinates to \(0\).
Finally choose \(W^{\mathrm{out}}\) so that the \(r\)-th active gated coordinate contributes \(-e_{z,r}\), with all other output columns equal to \(0\).
Then the residual update equals
\[
-\sum_{r=1}^{p}\langle u_t,e_{z,r}\rangle e_{z,r},
\]
so
\[
Z^{\mathrm{scr}}_{T,\{e_{z,r}\}}(u)_t
=
u_t-\sum_{r=1}^{p}\langle u_t,e_{z,r}\rangle e_{z,r}.
\]
Hence each scratch channel is reset exactly to zero, while \(e_{\mathrm{sig}}\) and every channel in \(E_{\mathrm{keep}}\) are preserved exactly.

Now let
\[
u^{(a,\tau)}:=u+a\,e_{\mathrm{sig}}\mathbf 1[\cdot=\tau].
\]
Because \(e_{\mathrm{sig}}\perp e_{z,r}\) for every \(r\), the reset term is identical for \(u^{(a,\tau)}\) and for \(u\). Therefore
\[
Z^{\mathrm{scr}}_{T,\{e_{z,r}\}}(u^{(a,\tau)})_t
=
Z^{\mathrm{scr}}_{T,\{e_{z,r}\}}(u)_t+a\,e_{\mathrm{sig}}\mathbf 1[t=\tau].
\]
This is exactly signal-transparency along \(e_{\mathrm{sig}}\) over \(E_{\mathrm{keep}}\).
\end{proof}

\begin{lemma}[Transparent damped predecessor integrator]
\label{lem:damped_predecessor_integrator}
Fix \(\beta\in(0,1)\), set \(\gamma:=1-\beta\), and let
\[
L_H:=\tau_{\max}+H.
\]
Let \(\Kset_H\subset(\mathbb R^m)^{L_H+1}\) be compact.
Assume orthonormal directions
\[
e_{\mathrm{sig}},\ e_{\mathrm{pos}},\ e_{\mathrm{tail}},\ e_{\mathrm{prof}}\in\mathbb R^m
\]
and a subspace \(E_{\mathrm{carry}}\subset\mathbb R^m\) orthogonal to all four, such that:

\begin{enumerate}[label=(\roman*), leftmargin=*, nosep]
\item the positional-control ranges
\[
I_t:=\{\langle u_t,e_{\mathrm{pos}}\rangle:\ u\in\Kset_H\},
\qquad 0\le t\le L_H,
\]
are compact and strictly ordered:
\[
I_0<I_1<\cdots<I_{L_H}\subset(0,\infty);
\]

\item the auxiliary tail input channel
\[
g_t(u):=\langle u_t,e_{\mathrm{tail}}\rangle
\]
satisfies
\[
c_g^-(t+1)^{-\beta}\le g_t(u)\le c_g^+(t+1)^{-\beta},
\qquad
0\le t\le L_H,\ \ u\in\Kset_H;
\]

\item the profile input channel is identically zero on \(\Kset_H\):
\[
\langle u_t,e_{\mathrm{prof}}\rangle=0
\qquad
\forall\,u\in\Kset_H,\ \forall\,0\le t\le L_H.
\]
\end{enumerate}

Then there exists a single LN-free Sessa block
\[
I_H:(\mathbb R^m)^{L_H+1}\to(\mathbb R^m)^{L_H+1}
\]

such that the \(e_{\mathrm{sig}}\)-channel, the \(e_{\mathrm{pos}}\)-coordinate, the \(e_{\mathrm{tail}}\)-channel, and every channel in \(E_{\mathrm{carry}}\) are preserved exactly and, writing
\[
r_t(u):=\bigl\langle I_H(u)_t,\ e_{\mathrm{prof}}\bigr\rangle,
\]
there exist constants \(c_r^-,c_r^+>0\), independent of \(H\), such that
\[
c_r^-(t+1)^\gamma\le r_t(u)\le c_r^+(t+1)^\gamma,
\qquad
0\le t\le L_H,\ \ u\in\Kset_H;
\]

\(I_H\) is signal-transparent along \(e_{\mathrm{sig}}\) with respect to the control pair
\[
\bigl(e_{\mathrm{pos}},e_{\mathrm{prof}}\bigr):
\]
for every \(u\in\Kset_H\), every \(\tau\in\{0,\dots,L_H\}\), every scalar \(a\in\mathbb R\), and every \(0\le t\le L_H\),
\[
\bigl\langle I_H(u+a\,e_{\mathrm{sig}}\mathbf 1[\cdot=\tau])_t,\ e_{\mathrm{pos}}\bigr\rangle
=
\bigl\langle I_H(u)_t,\ e_{\mathrm{pos}}\bigr\rangle,
\]
\[
\bigl\langle I_H(u+a\,e_{\mathrm{sig}}\mathbf 1[\cdot=\tau])_t,\ e_{\mathrm{prof}}\bigr\rangle
=
\bigl\langle I_H(u)_t,\ e_{\mathrm{prof}}\bigr\rangle,
\]
\[
\bigl\langle I_H(u+a\,e_{\mathrm{sig}}\mathbf 1[\cdot=\tau])_t,\ e_{\mathrm{sig}}\bigr\rangle
=
\bigl\langle I_H(u)_t,\ e_{\mathrm{sig}}\bigr\rangle
+
a\,\mathbf 1[t=\tau].
\]
\end{lemma}

\begin{proof}
Fix a small constant
\[
0<\kappa_\mu\le 1
\]
to be chosen later, and set
\[
\lambda_H:=1-\frac{1}{4(L_H+1)}\in(0,1),
\qquad
\mu_H:=\kappa_\mu (L_H+1)^{-3}.
\]

\paragraph{Step 1: choose the attention patterns.}
Use Lemma~\ref{lem:predecessor_focus} on the positional-control coordinate \(e_{\mathrm{pos}}\) with parameter \(\mu_H\).
This yields strict-past feedback attention satisfying
\[
\alpha^b_{t,t-1}\ge 1-\mu_H,
\qquad
\sum_{j=0}^{t-2}\alpha^b_{t,j}\le \mu_H.
\]
Use Lemma~\ref{lem:self_focus_forward} on the same positional-control coordinate, again with parameter \(\mu_H\), so that the forward row satisfies
\[
\alpha^f_{t,t}\ge 1-\mu_H,
\qquad
\sum_{j<t}\alpha^f_{t,j}\le \mu_H.
\]
Both \(\alpha^b\) and \(\alpha^f\) depend only on the positional stream.

\paragraph{Step 2: feed the tail channel into the solve.}
Read the tail input channel exactly using Corollary~\ref{cor:exact_channel_read}.
Choose two \(a\)-slots
\[
a_t^{(+)}=L\langle u_t,e_{\mathrm{tail}}\rangle,
\qquad
a_t^{(-)}=-L\langle u_t,e_{\mathrm{tail}}\rangle,
\]
for any fixed \(L>0\), and choose one dedicated transport value coordinate
\[
v_t^{\mathrm{tail}}
=
\frac1L\bigl(\bar a_t^{(+)}-\bar a_t^{(-)}\bigr)
=
\langle u_t,e_{\mathrm{tail}}\rangle
=
g_t(u).
\]

Choose the feedback gain constant
\[
\gamma_t\equiv \lambda_H.
\]
Let \(f_t(u)\) denote the forward signal entering the scalar solve on that dedicated coordinate:
\[
f_t(u)=\sum_{j\le t}\alpha^f_{t,j}(u)\,g_j(u).
\]
Let \(s_t(u)\) be the corresponding solve output:
\[
s_0(u)=f_0(u),
\qquad
s_t(u)=f_t(u)+\lambda_H\sum_{j<t}\alpha^b_{t,j}(u)\,s_j(u),
\qquad t\ge 1.
\]

Choose the gate on that dedicated coordinate to be the constant \(1\), and choose the output projection so that this solve output is written onto the prescribed profile direction \(e_{\mathrm{prof}}\), with all output columns on
\[
e_{\mathrm{sig}},\ e_{\mathrm{pos}},\ e_{\mathrm{tail}},\ E_{\mathrm{carry}}
\]
set to zero.

Because the input profile channel is identically zero on \(\Kset_H\), the residual formula gives
\[
\bigl\langle I_H(u)_t,e_{\mathrm{prof}}\bigr\rangle
=
\langle u_t,e_{\mathrm{prof}}\rangle+s_t(u)
=
s_t(u).
\]
Hence
\[
r_t(u):=\bigl\langle I_H(u)_t,e_{\mathrm{prof}}\bigr\rangle=s_t(u).
\]

The \(e_{\mathrm{sig}}\)-, \(e_{\mathrm{pos}}\)-, \(e_{\mathrm{tail}}\)-, and \(E_{\mathrm{carry}}\)-channels are preserved exactly, because the output projection vanishes on those directions.

\paragraph{Step 3: compare with the ideal predecessor recursion.}
Define the ideal predecessor recursion
\[
\widetilde r_0(u):=g_0(u),
\qquad
\widetilde r_t(u):=g_t(u)+\lambda_H\,\widetilde r_{t-1}(u),
\qquad t\ge 1,
\]
so that
\[
\widetilde r_t(u)=\sum_{m=0}^{t}\lambda_H^{\,t-m}g_m(u).
\]
Since \(0\le m\le t\le L_H\) and \(\lambda_H=1-\frac{1}{4(L_H+1)}\),
\[
e^{-1/4}\le \lambda_H^{\,t-m}\le 1.
\]
Therefore
\[
e^{-1/4}\sum_{m=0}^{t}g_m(u)
\le
\widetilde r_t(u)
\le
\sum_{m=0}^{t}g_m(u).
\]
Using
\[
c_g^-(m+1)^{-\beta}\le g_m(u)\le c_g^+(m+1)^{-\beta}
\]
and
\[
\sum_{m=0}^{t}(m+1)^{-\beta}\asymp (t+1)^{1-\beta}=(t+1)^\gamma,
\]
we obtain constants \(\widetilde c_r^-,\widetilde c_r^+>0\), independent of \(H\), such that
\[
\widetilde c_r^-(t+1)^\gamma
\le
\widetilde r_t(u)
\le
\widetilde c_r^+(t+1)^\gamma.
\]

\paragraph{Step 4: control the perturbation error.}
Let \(B_H(u)\) be the actual feedback matrix on the dedicated profile coordinate and \(B_H^\ast\) the ideal predecessor matrix
\[
(B_H^\ast)_{t,t-1}=\lambda_H,
\qquad
(B_H^\ast)_{t,j}=0\ \ (j<t-1).
\]
By the predecessor-focusing estimate,
\[
\sup_t\sum_{j<t}|(B_H(u)-B_H^\ast)_{t,j}|
\le
C\,\mu_H
\]
for an absolute constant \(C\).

Also,
\[
f_t(u)-g_t(u)
=
\sum_{j\le t}\alpha^f_{t,j}(u)\bigl(g_j(u)-g_t(u)\bigr)
=
\sum_{j<t}\alpha^f_{t,j}(u)\bigl(g_j(u)-g_t(u)\bigr),
\]
hence
\[
|f_t(u)-g_t(u)|
\le
2c_g^+\sum_{j<t}\alpha^f_{t,j}(u)
\le
2c_g^+\mu_H.
\]
Therefore
\[
\|f(u)-g(u)\|_\infty\le 2c_g^+\mu_H.
\]

Now
\[
r(u)=(I-B_H(u))^{-1}f(u),
\qquad
\widetilde r(u)=(I-B_H^\ast)^{-1}g(u),
\]
so
\[
r(u)-\widetilde r(u)
=
(I-B_H(u))^{-1}\Bigl((f(u)-g(u))+(B_H(u)-B_H^\ast)\widetilde r(u)\Bigr).
\]
Since the row sum of \(B_H(u)\) is at most \(\lambda_H<1\),
\[
\|(I-B_H(u))^{-1}\|_{\infty\to\infty}
\le
\frac{1}{1-\lambda_H}
=
4(L_H+1).
\]
Also
\[
\|\widetilde r(u)\|_\infty\lesssim (L_H+1)^\gamma.
\]
Therefore there exists a constant \(C_\ast>0\), independent of \(H\), such that
\[
\|r(u)-\widetilde r(u)\|_\infty
\le
C_\ast (L_H+1)^{\gamma+1}\mu_H
=
C_\ast \kappa_\mu (L_H+1)^{\gamma-2}.
\]

Since \(L_H=\tau_{\max}+H\ge \tau_{\max}+1\), we have
\[
(L_H+1)^{\gamma-2}\le (\tau_{\max}+2)^{\gamma-2}.
\]
Choose \(\kappa_\mu>0\) so small that
\[
C_\ast \kappa_\mu (\tau_{\max}+2)^{\gamma-2}
\le
\frac12 \widetilde c_r^-.
\]
Then uniformly in \(H\),
\[
\|r(u)-\widetilde r(u)\|_\infty
\le
\frac12 \widetilde c_r^-.
\]

Hence for every \(0\le t\le L_H\),
\[
r_t(u)
\ge
\widetilde r_t(u)-\frac12\widetilde c_r^-
\ge
\widetilde c_r^-(t+1)^\gamma-\frac12\widetilde c_r^-.
\]
Since \((t+1)^\gamma\ge 1\),
\[
\widetilde c_r^-(t+1)^\gamma-\frac12\widetilde c_r^-
\ge
\frac12 \widetilde c_r^-(t+1)^\gamma.
\]
So
\[
r_t(u)\ge \frac12 \widetilde c_r^-(t+1)^\gamma.
\]

Similarly,
\[
r_t(u)
\le
\widetilde r_t(u)+\frac12\widetilde c_r^-
\le
\widetilde c_r^+(t+1)^\gamma+\frac12\widetilde c_r^-.
\]
Again using \((t+1)^\gamma\ge 1\),
\[
r_t(u)\le \Bigl(\widetilde c_r^+ + \frac12\widetilde c_r^-\Bigr)(t+1)^\gamma.
\]

Thus the stated two-sided profile bound holds with
\[
c_r^-:=\frac12\widetilde c_r^-,
\qquad
c_r^+:=\widetilde c_r^+ + \frac12\widetilde c_r^-.
\]

\paragraph{Step 5: verify signal-transparency.}
Let
\[
u^{(a,\tau)}:=u+a\,e_{\mathrm{sig}}\mathbf 1[\cdot=\tau].
\]
Since \(e_{\mathrm{sig}}\perp e_{\mathrm{pos}},e_{\mathrm{tail}},e_{\mathrm{prof}}\), one has
\[
\langle u^{(a,\tau)}_t,e_{\mathrm{pos}}\rangle=\langle u_t,e_{\mathrm{pos}}\rangle,
\qquad
\langle u^{(a,\tau)}_t,e_{\mathrm{tail}}\rangle=\langle u_t,e_{\mathrm{tail}}\rangle,
\qquad
\langle u^{(a,\tau)}_t,e_{\mathrm{prof}}\rangle=\langle u_t,e_{\mathrm{prof}}\rangle=0.
\]
Therefore the feedback weights \(\alpha^b\) are unchanged, since they depend only on the positional stream. The forward weights \(\alpha^f\) are also unchanged for the same reason. Finally, the forward values \(g_t\) are unchanged, since they are exact reads of the tail channel.
Hence the actual forward signal \(f_t\), the actual feedback matrix \(B_H\), and therefore the solve output \(r_t\) are all unchanged under perturbations along \(e_{\mathrm{sig}}\):
\[
r_t(u^{(a,\tau)})=r_t(u).
\]
By construction, the output projection vanishes on the \(e_{\mathrm{sig}}\)-channel, so that channel passes through exactly:
\[
\bigl\langle I_H(u^{(a,\tau)})_t,e_{\mathrm{sig}}\bigr\rangle
=
\bigl\langle I_H(u)_t,e_{\mathrm{sig}}\bigr\rangle
+
a\,\mathbf 1[t=\tau].
\]
The \(e_{\mathrm{pos}}\)-coordinate is preserved exactly as well. This proves the stated signal-transparency property.
\end{proof}

\begin{corollary}[Transparent power-profile block]
\label{cor:preparatory_power_profile}
Fix \(\beta\in(0,1)\), set \(\gamma:=1-\beta\), fix \(H\ge 1\), and let \(L_H:=\tau_{\max}+H\).
Let \(\Kset_H\subset(\mathbb R^m)^{L_H+1}\) be the compact input set under consideration.

Assume \(\Kset_H\) carries orthonormal directions
\[
e_{\mathrm{sig}},\ e_{\mathrm{pos}}\in\mathbb R^m
\]
such that:
\begin{enumerate}[label=(\roman*), leftmargin=*, nosep]
\item the original signal channel is
\[
u\mapsto \langle u_t,e_{\mathrm{sig}}\rangle;
\]
\item the positional-control coordinate is
\[
u\mapsto \langle u_t,e_{\mathrm{pos}}\rangle,
\]
with ordered positive ranges
\[
I_0<I_1<\cdots<I_{L_H}\subset(0,\infty).
\]
\end{enumerate}

Fix additional orthonormal directions
\[
e_{\mathrm{prof}},\ e_{\mathrm{tail}},\ e_{\mathrm{aux}},\ e_{\mathrm{src}},\ e_{\mathrm{tgt}}\in\mathbb R^m
\]
orthogonal to both \(e_{\mathrm{sig}}\) and \(e_{\mathrm{pos}}\).

Then there exists a constant-depth LN-free Sessa network
\[
Q_H:(\mathbb R^m)^{L_H+1}\to(\mathbb R^m)^{L_H+1}
\]

such that the original signal channel is preserved exactly:
\[
\langle Q_H(u)_t,e_{\mathrm{sig}}\rangle
=
\langle u_t,e_{\mathrm{sig}}\rangle
\qquad (0\le t\le L_H,\ u\in\Kset_H);
\]

the positional-control coordinate is preserved exactly:
\[
\langle Q_H(u)_t,e_{\mathrm{pos}}\rangle
=
\langle u_t,e_{\mathrm{pos}}\rangle
\qquad (0\le t\le L_H,\ u\in\Kset_H);
\]

the profile channel on the prescribed direction \(e_{\mathrm{prof}}\) satisfies the uniform two-sided bound
\[
c_r^-(t+1)^\gamma
\le
\langle Q_H(u)_t,e_{\mathrm{prof}}\rangle
\le
c_r^+(t+1)^\gamma,
\qquad 0\le t\le L_H,\ u\in\Kset_H,
\]
with constants independent of \(H\); and \(Q_H\) is signal-transparent along \(e_{\mathrm{sig}}\) with respect to the control pair
\((e_{\mathrm{pos}},e_{\mathrm{prof}})\): for every \(u\in\Kset_H\), every \(\tau\in\{0,\dots,L_H\}\), and every scalar \(a\in\mathbb R\),
\[
\bigl\langle Q_H(u+a\,e_{\mathrm{sig}}\mathbf 1[\cdot=\tau])_t,\ e_{\mathrm{pos}}\bigr\rangle
=
\bigl\langle Q_H(u)_t,\ e_{\mathrm{pos}}\bigr\rangle,
\qquad
0\le t\le L_H,
\]
\[
\bigl\langle Q_H(u+a\,e_{\mathrm{sig}}\mathbf 1[\cdot=\tau])_t,\ e_{\mathrm{prof}}\bigr\rangle
=
\bigl\langle Q_H(u)_t,\ e_{\mathrm{prof}}\bigr\rangle,
\qquad
0\le t\le L_H,
\]
and
\[
\bigl\langle Q_H(u+a\,e_{\mathrm{sig}}\mathbf 1[\cdot=\tau])_t,\ e_{\mathrm{sig}}\bigr\rangle
=
\bigl\langle Q_H(u)_t,\ e_{\mathrm{sig}}\bigr\rangle
+
a\,\mathbf 1[t=\tau],
\qquad
0\le t\le L_H.
\]
\end{corollary}

\begin{proof}
The auxiliary orthonormal directions
\[
e_{\mathrm{prof}},\ e_{\mathrm{tail}},\ e_{\mathrm{aux}},\ e_{\mathrm{src}},\ e_{\mathrm{tgt}}
\]
are fixed by hypothesis and are orthogonal to both \(e_{\mathrm{sig}}\) and \(e_{\mathrm{pos}}\).

\paragraph{Step 1: clear the profile channel.}
Apply Lemma~\ref{lem:profile_zero_writer} with
\[
e_{\mathrm{zero}}:=e_{\mathrm{prof}},
\qquad
E_{\mathrm{carry}}:=\{0\}.
\]
This yields a forward-only block
\[
Z_H^{\mathrm{prof}}:(\mathbb R^m)^{L_H+1}\to(\mathbb R^m)^{L_H+1}
\]
such that
\[
\langle Z_H^{\mathrm{prof}}(u)_t,e_{\mathrm{sig}}\rangle=\langle u_t,e_{\mathrm{sig}}\rangle,
\qquad
\langle Z_H^{\mathrm{prof}}(u)_t,e_{\mathrm{pos}}\rangle=\langle u_t,e_{\mathrm{pos}}\rangle,
\qquad
\langle Z_H^{\mathrm{prof}}(u)_t,e_{\mathrm{prof}}\rangle=0.
\]
Moreover, \(Z_H^{\mathrm{prof}}\) is signal-transparent along \(e_{\mathrm{sig}}\) with respect to \((e_{\mathrm{pos}},e_{\mathrm{prof}})\).

Let
\[
\Kset_H^{(0)}:=Z_H^{\mathrm{prof}}(\Kset_H).
\]

\paragraph{Step 2: build the tail channel.}
Apply Lemma~\ref{lem:source0_tail_channel} to \(\Kset_H^{(0)}\), with
\[
E_{\mathrm{carry}}:=\operatorname{span}\{e_{\mathrm{prof}}\}.
\]
This yields a constant-depth network
\[
T_H^{\mathrm{tail}}:(\mathbb R^m)^{L_H+1}\to(\mathbb R^m)^{L_H+1}
\]
such that
\[
\langle T_H^{\mathrm{tail}}(v)_t,e_{\mathrm{sig}}\rangle=\langle v_t,e_{\mathrm{sig}}\rangle,
\qquad
\langle T_H^{\mathrm{tail}}(v)_t,e_{\mathrm{pos}}\rangle=\langle v_t,e_{\mathrm{pos}}\rangle,
\qquad
\langle T_H^{\mathrm{tail}}(v)_t,e_{\mathrm{prof}}\rangle=\langle v_t,e_{\mathrm{prof}}\rangle,
\]
and the tail channel
\[
g_t(v):=\langle T_H^{\mathrm{tail}}(v)_t,e_{\mathrm{tail}}\rangle
\]
satisfies
\[
c_g^-(t+1)^{-\beta}\le g_t(v)\le c_g^+(t+1)^{-\beta}.
\]
Because the carried profile channel is identically zero on \(\Kset_H^{(0)}\) and is preserved exactly by \(T_H^{\mathrm{tail}}\), one still has
\[
\langle T_H^{\mathrm{tail}}(v)_t,e_{\mathrm{prof}}\rangle=0
\qquad
\forall\,v\in \Kset_H^{(0)}.
\]

Let
\[
\Kset_H^{(1)}:=T_H^{\mathrm{tail}}(\Kset_H^{(0)}).
\]

\paragraph{Step 3: clear the scratch channels.}
Apply Lemma~\ref{lem:multi_channel_reset} to the scratch directions
\[
e_{\mathrm{aux}},\ e_{\mathrm{src}},\ e_{\mathrm{tgt}},
\]
with
\[
E_{\mathrm{keep}}:=\operatorname{span}\{e_{\mathrm{pos}},e_{\mathrm{tail}},e_{\mathrm{prof}}\}.
\]
This yields a forward-only concrete block
\[
Z_H^{\mathrm{scr}}:(\mathbb R^m)^{L_H+1}\to(\mathbb R^m)^{L_H+1}
\]
such that it preserves
\[
e_{\mathrm{sig}},\ e_{\mathrm{pos}},\ e_{\mathrm{tail}},\ e_{\mathrm{prof}}
\]
exactly and writes
\[
\langle Z_H^{\mathrm{scr}}(w)_t,e_{\mathrm{aux}}\rangle
=
\langle Z_H^{\mathrm{scr}}(w)_t,e_{\mathrm{src}}\rangle
=
\langle Z_H^{\mathrm{scr}}(w)_t,e_{\mathrm{tgt}}\rangle
=
0.
\]
Since \(Z_H^{\mathrm{scr}}\) preserves the tail channel exactly, the same bounds
\[
c_g^-(t+1)^{-\beta}
\le
\langle Z_H^{\mathrm{scr}}(w)_t,e_{\mathrm{tail}}\rangle
\le
c_g^+(t+1)^{-\beta}
\]
hold on the image.

Let
\[
\widetilde{\Kset}_H:=Z_H^{\mathrm{scr}}(\Kset_H^{(1)}).
\]
On \(\widetilde{\Kset}_H\) we therefore retain the same ordered positional ranges as on \(\Kset_H\), the same tail bounds \(c_g^\pm(t+1)^{-\beta}\), an identically zero profile channel, and identically zero scratch channels \(e_{\mathrm{aux}},e_{\mathrm{src}},e_{\mathrm{tgt}}\).

\paragraph{Step 4: integrate the tail channel.}
Apply Lemma~\ref{lem:damped_predecessor_integrator} to \(\widetilde{\Kset}_H\), with
\[
E_{\mathrm{carry}}:=\operatorname{span}\{e_{\mathrm{aux}},e_{\mathrm{src}},e_{\mathrm{tgt}}\}.
\]
Because these carried channels are already identically zero on \(\widetilde{\Kset}_H\), this application is fully legitimate and keeps them zero.
We obtain a single LN-free Sessa block
\[
I_H:(\mathbb R^m)^{L_H+1}\to(\mathbb R^m)^{L_H+1}
\]
such that
\[
\langle I_H(w)_t,e_{\mathrm{sig}}\rangle=\langle w_t,e_{\mathrm{sig}}\rangle,
\qquad
\langle I_H(w)_t,e_{\mathrm{pos}}\rangle=\langle w_t,e_{\mathrm{pos}}\rangle,
\qquad
\langle I_H(w)_t,e_{\mathrm{tail}}\rangle=\langle w_t,e_{\mathrm{tail}}\rangle,
\]
and
\[
c_r^-(t+1)^\gamma
\le
\langle I_H(w)_t,e_{\mathrm{prof}}\rangle
\le
c_r^+(t+1)^\gamma.
\]

\paragraph{Step 5: define the preparatory network.}
Set
\[
Q_H:=I_H\circ Z_H^{\mathrm{scr}}\circ T_H^{\mathrm{tail}}\circ Z_H^{\mathrm{prof}}.
\]
The exact preservation and two-sided profile bounds follow immediately from the four stages above.

\paragraph{Step 6: verify signal-transparency.}
Fix \(u\in\Kset_H\), \(\tau\in\{0,\dots,L_H\}\), and \(a\in\mathbb R\). Define
\[
u^{(a,\tau)}:=u+a\,e_{\mathrm{sig}}\mathbf 1[\cdot=\tau].
\]

By signal-transparency of \(Z_H^{\mathrm{prof}}\),
\[
Z_H^{\mathrm{prof}}(u^{(a,\tau)})
=
Z_H^{\mathrm{prof}}(u)+a\,e_{\mathrm{sig}}\mathbf 1[\cdot=\tau]
\]
on the signal channel, while the \(e_{\mathrm{pos}}\)- and \(e_{\mathrm{prof}}\)-channels are unchanged.

Applying signal-transparency of \(T_H^{\mathrm{tail}}\) then gives
\[
T_H^{\mathrm{tail}}(Z_H^{\mathrm{prof}}(u^{(a,\tau)}))
=
T_H^{\mathrm{tail}}(Z_H^{\mathrm{prof}}(u))
+
a\,e_{\mathrm{sig}}\mathbf 1[\cdot=\tau]
\]
on the signal channel, while the \(e_{\mathrm{pos}}\)- and \(e_{\mathrm{tail}}\)-channels are unchanged and the \(e_{\mathrm{prof}}\)-channel remains zero.

Now \(Z_H^{\mathrm{scr}}\) preserves \(e_{\mathrm{sig}},e_{\mathrm{pos}},e_{\mathrm{tail}},e_{\mathrm{prof}}\) exactly, so
\[
Z_H^{\mathrm{scr}}(T_H^{\mathrm{tail}}(Z_H^{\mathrm{prof}}(u^{(a,\tau)})))
=
Z_H^{\mathrm{scr}}(T_H^{\mathrm{tail}}(Z_H^{\mathrm{prof}}(u)))
+
a\,e_{\mathrm{sig}}\mathbf 1[\cdot=\tau]
\]
on the signal channel, and the \(e_{\mathrm{pos}}\)-, \(e_{\mathrm{tail}}\)-, and \(e_{\mathrm{prof}}\)-channels are unchanged.

Thus the two inputs fed into \(I_H\) differ only on the \(e_{\mathrm{sig}}\)-channel
and have the same \(e_{\mathrm{pos}}\)-, \(e_{\mathrm{tail}}\)-, and
\(e_{\mathrm{prof}}\)-streams.
In the concrete construction of Lemma~\ref{lem:damped_predecessor_integrator},
the feedback weights \(\alpha^b\) and forward weights \(\alpha^f\) depend only on
the positional stream, while the forward values \(g_t\) are exact reads of the
\(e_{\mathrm{tail}}\)-channel. Hence the forward signals \(f_t\), the feedback
matrices \(B_H\), and the solve outputs \(r_t\) are identical for the two inputs.
Moreover, the output projection of \(I_H\) vanishes on the \(e_{\mathrm{sig}}\)-,
\(e_{\mathrm{pos}}\)-, and \(e_{\mathrm{tail}}\)-channels, so the
\(e_{\mathrm{sig}}\)-channel passes through exactly and the
\(e_{\mathrm{pos}}\)-coordinate is unchanged.
Therefore
\[
\langle Q_H(u^{(a,\tau)})_t,e_{\mathrm{pos}}\rangle
=
\langle Q_H(u)_t,e_{\mathrm{pos}}\rangle,
\]
\[
\langle Q_H(u^{(a,\tau)})_t,e_{\mathrm{prof}}\rangle
=
\langle Q_H(u)_t,e_{\mathrm{prof}}\rangle,
\]
\[
\langle Q_H(u^{(a,\tau)})_t,e_{\mathrm{sig}}\rangle
=
\langle Q_H(u)_t,e_{\mathrm{sig}}\rangle
+
a\,\mathbf 1[t=\tau].
\]
This proves the stated signal-transparency property.
\end{proof}

\begin{lemma}[Profile-compensated macro-layer]
\label{lem:profile_weighted_macro}
Fix \(\beta\in(0,1)\), set \(\gamma:=1-\beta\), and fix \(T\ge 0\).
Let \(\Kset\subset(\mathbb R^m)^{T+1}\) be compact.
Assume orthonormal directions
\[
e_{\mathrm{sig}},\ e_{\mathrm{pos}},\ e_{\mathrm{prof}},\ e_{\mathrm{src}}\in\mathbb R^m
\]
and a subspace \(E_{\mathrm{carry}}\subset\mathbb R^m\) orthogonal to all four, such that:

\begin{enumerate}[label=(\roman*), leftmargin=*, nosep]
\item the positional-control ranges
\[
I_t:=\{\langle u_t,e_{\mathrm{pos}}\rangle:\ u\in\Kset\},
\qquad 0\le t\le T,
\]
are compact and strictly ordered:
\[
I_0<I_1<\cdots<I_T\subset(0,\infty);
\]

\item the profile channel
\[
r_t(u):=\langle u_t,e_{\mathrm{prof}}\rangle
\]
satisfies
\[
c_r^-(t+1)^\gamma\le r_t(u)\le c_r^+(t+1)^\gamma,
\qquad
0\le t\le T,\ \ u\in\Kset.
\]
\end{enumerate}

Then there exists a constant-depth LN-free Sessa macro-layer
\[
M_T:(\mathbb R^m)^{T+1}\to(\mathbb R^m)^{T+1}
\]

such that the \(e_{\mathrm{pos}}\)-channel, the \(e_{\mathrm{prof}}\)-channel, and every channel in \(E_{\mathrm{carry}}\) are preserved exactly, and
\(M_T\) has signal-blind exact scalar transport along \(e_{\mathrm{sig}}\) over
\[
E_{\mathrm{ctrl}}:=\operatorname{span}\{e_{\mathrm{pos}},e_{\mathrm{prof}}\}\oplus E_{\mathrm{carry}},
\]
with kernel
\[
\mathcal T_{M_T}^u(i,j)
=
D_{\mathrm{mac}}^u(i)\,\mathbf 1[i=j]
+
K_{\mathrm{mac}}^u(i,j)\,\mathbf 1[j<i];
\]

There exist constants
\[
1\le d^-_{\mathrm{mac}}\le d^+_{\mathrm{mac}}<\infty,
\qquad
0<a^-_{\mathrm{mac}}\le a^+_{\mathrm{mac}}<\infty,
\]
depending only on \((\beta,c_r^-,c_r^+)\), but independent of \(T\), such that
\[
d^-_{\mathrm{mac}}
\le
D_{\mathrm{mac}}^u(i)
\le
d^+_{\mathrm{mac}},
\qquad 0\le i\le T,
\]
and
\[
a^-_{\mathrm{mac}}(i+1)^{-\beta}
\le
K_{\mathrm{mac}}^u(i,j)
\le
a^+_{\mathrm{mac}}(i+1)^{-\beta},
\qquad 0\le j<i\le T.
\]
In particular,
\[
K_{\mathrm{mac}}^u(i,j)\le a^+_{\mathrm{mac}}(i-j+1)^{-\beta}.
\]

Consequently,
\[
e_{\mathrm{sig}}^\top
\frac{\partial M_T(u)_i}{\partial u_j}
e_{\mathrm{sig}}
=
D_{\mathrm{mac}}^u(i)\,\mathbf 1[i=j]
+
K_{\mathrm{mac}}^u(i,j)\,\mathbf 1[j<i].
\]
\end{lemma}

\begin{proof}
Write
\[
x_t:=\langle u_t,e_{\mathrm{sig}}\rangle,
\qquad
r_t(u):=\langle u_t,e_{\mathrm{prof}}\rangle,
\qquad 0\le t\le T.
\]

We construct
\[
M_T=A_T^{\mathrm{diff}}\circ W_T^{\mathrm{src}},
\]
where \(W_T^{\mathrm{src}}\) is a local source writer and \(A_T^{\mathrm{diff}}\) is the diffuse
transport-bearing block.

\paragraph{Step 1: local source writer.}
Choose a parameter \(\mu\in(0,\frac12]\) and apply Lemma~\ref{lem:self_focus_forward} to the
ordered positional-control coordinate \(e_{\mathrm{pos}}\).
This yields a forward attention row satisfying
\[
\alpha^f_{t,t}\ge 1-\mu,
\qquad
\sum_{j<t}\alpha^f_{t,j}\le \mu,
\qquad 0\le t\le T.
\]

We now build a forward-only LN-free Sessa block
\[
W_T^{\mathrm{src}}:(\mathbb R^m)^{T+1}\to(\mathbb R^m)^{T+1}.
\]

Choose one forward value coordinate equal to \(1\):
\[
v_t^{(0)}\equiv 1.
\]
Hence
\[
s_t^{(0)}=\sum_{j\le t}\alpha^f_{t,j}\cdot 1=1.
\]

Next read the profile channel exactly using Corollary~\ref{cor:exact_channel_read}.
Choose two \(a\)-slots
\[
a_t^{(+)}=L\langle u_t,e_{\mathrm{prof}}\rangle,
\qquad
a_t^{(-)}=-L\langle u_t,e_{\mathrm{prof}}\rangle
\]
for any fixed \(L>0\), and choose the value projection so that
\[
v_t^{(1)}
=
\frac1L\bigl(\bar a_t^{(+)}-\bar a_t^{(-)}\bigr)
=
\langle u_t,e_{\mathrm{prof}}\rangle
=
r_t(u).
\]
Let
\[
m_t^u:=s_t^{(1)}:=\sum_{j\le t}\alpha^f_{t,j}\,r_j(u).
\]

Choose two gate coordinates
\[
g_t^{(0)}=\langle u_t,e_{\mathrm{src}}\rangle,
\qquad
g_t^{(1)}=\langle u_t,e_{\mathrm{sig}}\rangle=x_t,
\]
and choose the output projection on the \(e_{\mathrm{src}}\)-channel with coefficients \((-1,+1)\).
Then
\[
\langle W_T^{\mathrm{src}}(u)_t,e_{\mathrm{src}}\rangle
=
\langle u_t,e_{\mathrm{src}}\rangle-s_t^{(0)}\langle u_t,e_{\mathrm{src}}\rangle+s_t^{(1)}x_t
=
m_t^u\,x_t.
\]

All other output columns are zero, so the \(e_{\mathrm{sig}}\)-, \(e_{\mathrm{pos}}\)-,
\(e_{\mathrm{prof}}\)-, and \(E_{\mathrm{carry}}\)-channels are preserved exactly.

It remains to bound \(m_t^u\).
Since every \(r_j(u)\ge 0\),
\[
m_t^u
\ge
\alpha^f_{t,t}\,r_t(u)
\ge
(1-\mu)c_r^-(t+1)^\gamma.
\]
Also, for every \(j\le t\),
\[
r_j(u)\le c_r^+(j+1)^\gamma\le c_r^+(t+1)^\gamma,
\]
so
\[
m_t^u
=
\sum_{j\le t}\alpha^f_{t,j}r_j(u)
\le
c_r^+(t+1)^\gamma.
\]
Therefore
\[
m^-(t+1)^\gamma\le m_t^u\le m^+(t+1)^\gamma,
\qquad
m^-:=(1-\mu)c_r^-,
\qquad
m^+:=c_r^+.
\]

\paragraph{Step 2: diffuse transport block.}
Let
\[
w:=W_T^{\mathrm{src}}(u).
\]
We now build a single LN-free Sessa block
\[
A_T^{\mathrm{diff}}:(\mathbb R^m)^{T+1}\to(\mathbb R^m)^{T+1}
\]
as follows.

\emph{Forward branch.}
Choose all forward queries and keys equal to zero:
\[
q_k^f\equiv 0,
\qquad
k_j^f\equiv 0.
\]
Hence the forward row is exactly uniform on the visible prefix:
\[
\alpha^f_{k,j}=\frac1{k+1}\mathbf 1[j\le k].
\]

Read the source scratch channel exactly using Corollary~\ref{cor:exact_channel_read}.
Choose two \(a\)-slots
\[
a_j^{(+)}=L\langle w_j,e_{\mathrm{src}}\rangle,
\qquad
a_j^{(-)}=-L\langle w_j,e_{\mathrm{src}}\rangle,
\]
and choose the value projection so that
\[
v_j^{\mathrm{src}}
=
\frac1L\bigl(\bar a_j^{(+)}-\bar a_j^{(-)}\bigr)
=
\langle w_j,e_{\mathrm{src}}\rangle
=
m_j^u\,x_j.
\]
Thus the forward signal is
\[
f_k
=
\sum_{j\le k}\alpha^f_{k,j}v_j^{\mathrm{src}}
=
\frac1{k+1}\sum_{j=0}^{k}m_j^u x_j.
\]

\emph{Feedback branch.}
Choose all feedback queries and keys equal to zero and the feedback gain constant:
\[
q_i^b\equiv 0,
\qquad
k_j^b\equiv 0,
\qquad
\gamma_i\equiv \gamma=1-\beta.
\]
Therefore the strict-past feedback row is exactly uniform:
\[
\alpha^b_{i,k}=\frac1{i}\mathbf 1[k<i],
\qquad 1\le i\le T,
\]
and the scalar feedback matrix is
\[
B_{i,k}=\frac{\gamma}{i}\mathbf 1[k<i].
\]

Let
\[
\Theta_{i,k}:=[(I-B)^{-1}]_{i,k},
\qquad 0\le k\le i\le T.
\]
Exactly as in the proof of Lemma~\ref{lem:active_diffusive_factorized_transport}, one has
\[
\Theta_{i,i}=1,
\]
and for \(k<i\),
\[
\Theta_{i,k}
=
\gamma\,
\frac{\Gamma(k+1)}{\Gamma(k+1+\gamma)}
\frac{\Gamma(i+\gamma)}{\Gamma(i+1)}.
\]
Hence there exist constants
\[
0<c^-_\Theta\le c^+_\Theta<\infty
\]
depending only on \(\beta\), such that
\[
c^-_\Theta (k+1)^{-\gamma}(i+1)^{-\beta}
\le
\Theta_{i,k}
\le
c^+_\Theta (k+1)^{-\gamma}(i+1)^{-\beta},
\qquad 0\le k<i\le T.
\]

\emph{Write transport into the signal channel.}
Choose one gate coordinate identically \(1\), and choose the output projection so that the solve output
adds \(+s_i\) to the \(e_{\mathrm{sig}}\)-channel and all output columns on
\[
e_{\mathrm{pos}},\ e_{\mathrm{prof}},\ E_{\mathrm{carry}}
\]
vanish.

Therefore
\[
\langle A_T^{\mathrm{diff}}(w)_i,e_{\mathrm{sig}}\rangle
=
\langle w_i,e_{\mathrm{sig}}\rangle+s_i
=
x_i+s_i,
\]
where
\[
s_i
=
\sum_{k=0}^{i}\Theta_{i,k}f_k.
\]

Since \(W_T^{\mathrm{src}}\) preserves \(e_{\mathrm{sig}},e_{\mathrm{pos}},e_{\mathrm{prof}},E_{\mathrm{carry}}\) exactly,
the full macro-layer \(M_T=A_T^{\mathrm{diff}}\circ W_T^{\mathrm{src}}\) also preserves
\(e_{\mathrm{pos}},e_{\mathrm{prof}},E_{\mathrm{carry}}\) exactly.

\paragraph{Step 3: exact transport formula.}
Substituting the expression for \(f_k\), we get
\[
s_i
=
\sum_{k=0}^{i}\Theta_{i,k}\frac1{k+1}\sum_{j=0}^{k}m_j^u x_j
=
\sum_{j=0}^{i}
\left(
m_j^u\sum_{k=j}^{i}\frac{\Theta_{i,k}}{k+1}
\right)x_j.
\]
Define
\[
L(i,j):=\sum_{k=j}^{i}\frac{\Theta_{i,k}}{k+1},
\qquad 0\le j\le i\le T.
\]
Then
\[
\langle M_T(u)_i,e_{\mathrm{sig}}\rangle
=
x_i+\sum_{j=0}^{i}m_j^u\,L(i,j)\,x_j.
\]

Since \(\Theta_{i,i}=1\), we have
\[
L(i,i)=\frac1{i+1}.
\]
Therefore
\[
\langle M_T(u)_i,e_{\mathrm{sig}}\rangle
=
\left(1+\frac{m_i^u}{i+1}\right)x_i
+
\sum_{j<i}m_j^u\,L(i,j)\,x_j.
\]
Define
\[
D_{\mathrm{mac}}^u(i):=1+\frac{m_i^u}{i+1},
\qquad
K_{\mathrm{mac}}^u(i,j):=m_j^u\,L(i,j)\quad (j<i).
\]
This yields exact scalar transport on the signal channel:
\[
\langle M_T(u)_i,e_{\mathrm{sig}}\rangle
=
D_{\mathrm{mac}}^u(i)\,x_i
+
\sum_{j<i}K_{\mathrm{mac}}^u(i,j)\,x_j.
\]

The coefficient \(m_j^u\) depends only on the \(e_{\mathrm{pos}}\)- and \(e_{\mathrm{prof}}\)-control streams,
because the source writer uses positional self-focusing and an exact read of the profile channel only.
The kernel \(L(i,j)\) depends only on the fixed diffuse transport block.
Hence \(D_{\mathrm{mac}}^u(i)\) and \(K_{\mathrm{mac}}^u(i,j)\) depend only on the control stream
\[
(\Pi_{\mathrm{ctrl}}u_t)_{t=0}^{T},
\qquad
E_{\mathrm{ctrl}}:=\operatorname{span}\{e_{\mathrm{pos}},e_{\mathrm{prof}}\}\oplus E_{\mathrm{carry}}.
\]
Thus \(M_T\) has signal-blind exact scalar transport over \(E_{\mathrm{ctrl}}\).

\paragraph{Step 4: diagonal bounds.}
Since
\[
m^-(i+1)^\gamma\le m_i^u\le m^+(i+1)^\gamma,
\]
we obtain
\[
1
\le
D_{\mathrm{mac}}^u(i)
=
1+\frac{m_i^u}{i+1}
\le
1+m^+(i+1)^{\gamma-1}
=
1+m^+(i+1)^{-\beta}
\le
1+m^+.
\]
Hence we may take
\[
d^-_{\mathrm{mac}}:=1,
\qquad
d^+_{\mathrm{mac}}:=1+m^+.
\]

\paragraph{Step 5: off-diagonal upper bound.}
Fix \(0\le j<i\le T\).
Using \(\Theta_{i,i}=1\) and the upper bound on \(\Theta_{i,k}\) for \(k<i\),
\[
L(i,j)
\le
\frac1{i+1}
+
c^+_\Theta (i+1)^{-\beta}\sum_{k=j}^{i-1}(k+1)^{-1-\gamma}.
\]
Since
\[
\frac1{i+1}
\le
(j+1)^{-\gamma}(i+1)^{-\beta},
\]
and
\[
\sum_{k=j}^{i-1}(k+1)^{-1-\gamma}
\le
\sum_{k=j}^{\infty}(k+1)^{-1-\gamma}
\lesssim_\gamma (j+1)^{-\gamma},
\]
there exists \(C_L^+<\infty\), depending only on \(\beta\), such that
\[
L(i,j)\le C_L^+ (j+1)^{-\gamma}(i+1)^{-\beta}.
\]
Therefore
\[
K_{\mathrm{mac}}^u(i,j)
=
m_j^u\,L(i,j)
\le
m^+(j+1)^\gamma\cdot C_L^+ (j+1)^{-\gamma}(i+1)^{-\beta}.
\]
Hence
\[
K_{\mathrm{mac}}^u(i,j)\le a^+_{\mathrm{mac}}(i+1)^{-\beta},
\qquad
a^+_{\mathrm{mac}}:=m^+ C_L^+.
\]

\paragraph{Step 6: off-diagonal lower bound.}
Fix \(0\le j<i\le T\).

\emph{Case 0: \(j=0\).}
Since \(\Theta_{i,0}\) appears in the sum defining \(L(i,0)\), we have
\[
L(i,0)\ge \Theta_{i,0}.
\]
By the resolvent bound,
\[
\Theta_{i,0}\ge c^-_\Theta (0+1)^{-\gamma}(i+1)^{-\beta}
= c^-_\Theta (i+1)^{-\beta}.
\]
Also \(m_0^u\ge m^-\). Therefore
\[
K_{\mathrm{mac}}^u(i,0)
=
m_0^u\,L(i,0)
\ge
m^- c^-_\Theta (i+1)^{-\beta}.
\]

\emph{Case 1: \(1\le j\le i/2\).}
Then \(2j\le i\), so
\[
L(i,j)
\ge
\sum_{k=j}^{2j-1}\frac{\Theta_{i,k}}{k+1}
\ge
c^-_\Theta (i+1)^{-\beta}\sum_{k=j}^{2j-1}(k+1)^{-1-\gamma}.
\]
Since the sum over one dyadic block is comparable to \((j+1)^{-\gamma}\),
there exists \(c_L^{(1)}>0\), depending only on \(\beta\), such that
\[
L(i,j)\ge c_L^{(1)}(j+1)^{-\gamma}(i+1)^{-\beta}.
\]
Hence
\[
K_{\mathrm{mac}}^u(i,j)
=
m_j^u\,L(i,j)
\ge
m^-(j+1)^\gamma\cdot c_L^{(1)}(j+1)^{-\gamma}(i+1)^{-\beta}
=
m^- c_L^{(1)}(i+1)^{-\beta}.
\]

\emph{Case 2: \(j> i/2\).}
Then
\[
L(i,j)\ge \frac1{i+1},
\]
so
\[
K_{\mathrm{mac}}^u(i,j)
=
m_j^u\,L(i,j)
\ge
\frac{m_j^u}{i+1}
\ge
\frac{m^-(j+1)^\gamma}{i+1}.
\]
Since \(j+1>\frac{i+1}{2}\),
\[
(j+1)^\gamma\ge 2^{-\gamma}(i+1)^\gamma.
\]
Therefore
\[
K_{\mathrm{mac}}^u(i,j)\ge m^- 2^{-\gamma}(i+1)^{\gamma-1}
=
m^-2^{-\gamma}(i+1)^{-\beta}.
\]

Combining the three cases gives
\[
K_{\mathrm{mac}}^u(i,j)\ge a^-_{\mathrm{mac}}(i+1)^{-\beta},
\qquad
a^-_{\mathrm{mac}}
:=
\min\{m^-c^-_\Theta,\ m^-c_L^{(1)},\ m^-2^{-\gamma}\}.
\]
For any \(\eta>0\), replacing \(\Kset\) by
\(\Sat^{\mathrm{sig}}_\eta(\Kset)\) leaves the ordered positional ranges and the
two-sided profile bounds unchanged, since only the \(e_{\mathrm{sig}}\)-channel
is perturbed.
The same source-writer plus diffuse-transport construction therefore yields the
same exact scalar transport formula on \(\Sat^{\mathrm{sig}}_\eta(\Kset)\), with
the same coefficients \(D_{\mathrm{mac}}^u(i)\) and \(K_{\mathrm{mac}}^u(i,j)\),
because these coefficients depend only on the control stream
\((e_{\mathrm{pos}},e_{\mathrm{prof}},E_{\mathrm{carry}})\).
Applying Lemma~\ref{lem:signal_blind_transport_calculus}(i) gives
\[
e_{\mathrm{sig}}^\top
\frac{\partial M_T(u)_i}{\partial u_j}
e_{\mathrm{sig}}
=
D_{\mathrm{mac}}^u(i)\,\mathbf 1[i=j]
+
K_{\mathrm{mac}}^u(i,j)\,\mathbf 1[j<i].
\]
\end{proof}

\begin{corollary}[Macro-layer transport]
\label{cor:macro_signal_fiber_transport}
Under the hypotheses of Lemma~\ref{lem:profile_weighted_macro}, let
\[
E_{\mathrm{ctrl}}
:=
\operatorname{span}\{e_{\mathrm{pos}},e_{\mathrm{prof}}\}\oplus E_{\mathrm{carry}},
\qquad
\Pi_{\mathrm{ctrl}}:\mathbb R^m\to E_{\mathrm{ctrl}},
\qquad
\pi_{\mathrm{sig}}(v):=\langle v,e_{\mathrm{sig}}\rangle,
\]
and let \(M_T\) be the concrete macro-layer constructed there.
Then for every \(\delta\ge 0\), \(M_T\) has signal-blind exact scalar transport
along \(e_{\mathrm{sig}}\) over \(E_{\mathrm{ctrl}}\) on
\(\Sat^{\mathrm{sig}}_\delta(\Kset)\), with the same scalar transport kernel
\(\mathcal T_{M_T}^u(i,j)\) as on \(\Kset\).

More precisely, if
\[
v=u+\sum_{t=0}^{T} a_t e_{\mathrm{sig}}\mathbf 1[\cdot=t],
\qquad
u\in\Kset,
\]
then
\[
\Pi_{\mathrm{ctrl}}M_T(v)_i=\Pi_{\mathrm{ctrl}}v_i,
\qquad 0\le i\le T,
\]
and
\[
\pi_{\mathrm{sig}}(M_T(v)_i)
=
\sum_{j=0}^{i}\mathcal T_{M_T}^{u}(i,j)\,\pi_{\mathrm{sig}}(v_j),
\qquad 0\le i\le T.
\]
The right-hand side depends only on the control stream of \(v\), hence is independent
of the choice of \(u\in\Kset\) with the same control stream.
\end{corollary}

\begin{proof}
Write
\[
M_T=A_T^{\mathrm{diff}}\circ W_T^{\mathrm{src}}
\]
exactly as in the proof of Lemma~\ref{lem:profile_weighted_macro}.

Fix
\[
v=u+\sum_{t=0}^{T} a_t e_{\mathrm{sig}}\mathbf 1[\cdot=t],
\qquad
u\in\Kset.
\]
Since \(v\) differs from \(u\) only on the \(e_{\mathrm{sig}}\)-channel, the
\(e_{\mathrm{pos}}\)-, \(e_{\mathrm{prof}}\)-, and \(E_{\mathrm{carry}}\)-streams are unchanged.
Hence the self-focused profile averages from the source-writer stage are unchanged:
\[
m_t^v=m_t^u,
\qquad 0\le t\le T.
\]

Therefore the explicit source-writer formula gives
\[
\bigl\langle W_T^{\mathrm{src}}(v)_t,e_{\mathrm{src}}\bigr\rangle
=
m_t^u\,\pi_{\mathrm{sig}}(v_t),
\qquad 0\le t\le T.
\]
Moreover, \(W_T^{\mathrm{src}}\) preserves the channels in \(E_{\mathrm{ctrl}}\) exactly,
because it modifies only the \(e_{\mathrm{src}}\)-channel.

In the diffuse stage, the forward row is the exact uniform prefix average, so the
forward signal entering the fixed feedback solve is
\[
f_k(v)
=
\frac1{k+1}\sum_{j=0}^{k} m_j^u\,\pi_{\mathrm{sig}}(v_j),
\qquad 0\le k\le T.
\]
The feedback matrix \(B\), its resolvent \(\Theta\), and the kernel
\[
L(i,j):=\sum_{k=j}^{i}\frac{\Theta_{i,k}}{k+1}
\]
depend only on \(\beta\), hence are independent of \(v\).
Thus the solve output satisfies
\[
s_i(v)
=
\sum_{k=0}^{i}\Theta_{i,k}f_k(v)
=
\sum_{j=0}^{i} m_j^u\,L(i,j)\,\pi_{\mathrm{sig}}(v_j).
\]
Using the definitions from Lemma~\ref{lem:profile_weighted_macro},
\[
D_{\mathrm{mac}}^u(i):=1+\frac{m_i^u}{i+1},
\qquad
K_{\mathrm{mac}}^u(i,j):=m_j^u\,L(i,j)\quad (j<i),
\]
we obtain
\[
\pi_{\mathrm{sig}}(M_T(v)_i)
=
\pi_{\mathrm{sig}}(v_i)+s_i(v)
=
\sum_{j=0}^{i}\mathcal T_{M_T}^{u}(i,j)\,\pi_{\mathrm{sig}}(v_j).
\]

Finally, \(A_T^{\mathrm{diff}}\) modifies only the \(e_{\mathrm{sig}}\)-channel and preserves
\(e_{\mathrm{pos}},e_{\mathrm{prof}},E_{\mathrm{carry}}\) exactly.
Hence \(M_T\) preserves \(E_{\mathrm{ctrl}}\) exactly on
\(\Sat^{\mathrm{sig}}_\delta(\Kset)\).
Since the coefficients \(m_j^u\), and therefore
\(\mathcal T_{M_T}^{u}(i,j)\), depend only on the control stream, the displayed
kernel is independent of the choice of \(u\in\Kset\) with the same control stream.
This proves the claim.
\end{proof}

\begin{lemma}[Projected macro-layer]
\label{lem:projected_macro_signal_fiber}
Under the hypotheses of Lemma~\ref{lem:profile_weighted_macro}, let
\[
\Pi_{\mathrm{src}}(v)_t:=v_t-\langle v_t,e_{\mathrm{src}}\rangle e_{\mathrm{src}},
\qquad 0\le t\le T,
\]
be the tokenwise orthogonal projection that kills the \(e_{\mathrm{src}}\)-channel, and define
\[
\bar M_T:=\Pi_{\mathrm{src}}\circ M_T.
\]
Then:

\begin{enumerate}[label=(\roman*), leftmargin=*, nosep]
\item \(M_T\) is blind to the incoming \(e_{\mathrm{src}}\)-channel:
\[
M_T = M_T\circ \Pi_{\mathrm{src}}.
\]

\item \(\bar M_T\) preserves the \(e_{\mathrm{pos}}\)-channel, the \(e_{\mathrm{prof}}\)-channel,
and every channel in \(E_{\mathrm{carry}}\) exactly.

\item \(\bar M_T\) has signal-blind exact scalar transport along \(e_{\mathrm{sig}}\) over
\[
E_{\mathrm{ctrl}}:=\operatorname{span}\{e_{\mathrm{pos}},e_{\mathrm{prof}}\}\oplus E_{\mathrm{carry}},
\]
with exactly the same scalar transport kernel as \(M_T\):
\[
\mathcal T_{\bar M_T}^u(i,j)=\mathcal T_{M_T}^u(i,j),
\qquad
0\le j\le i\le T.
\]

\item For every \(\delta\ge 0\) there exists \(\delta'=\delta'(\delta,\Kset)<\infty\) such that
\[
\bar M_T\bigl(\Sat^{\mathrm{sig}}_\delta(\Kset)\bigr)
\subset
\Sat^{\mathrm{sig}}_{\delta'}\bigl(\bar M_T(\Kset)\bigr).
\]
More precisely, if
\[
u' = u+\sum_{t=0}^{T} a_t e_{\mathrm{sig}}\mathbf 1[\cdot=t],
\qquad
u\in\Kset,
\qquad
\max_t |a_t|\le \delta,
\]
then
\[
\bar M_T(u')_i
=
\bar M_T(u)_i
+
\left(
\sum_{j=0}^{i}\mathcal T_{M_T}^u(i,j)a_j
\right)e_{\mathrm{sig}},
\qquad
0\le i\le T.
\]

\item For every \(\delta\ge 0\), \(\bar M_T\) has signal-blind exact scalar transport
along \(e_{\mathrm{sig}}\) over
\[
E_{\mathrm{ctrl}}:=\operatorname{span}\{e_{\mathrm{pos}},e_{\mathrm{prof}}\}\oplus E_{\mathrm{carry}}
\]
on \(\Sat^{\mathrm{sig}}_\delta(\Kset)\), with the same scalar transport kernel as \(M_T\).
More precisely, if
\[
v=u+\sum_{t=0}^{T} a_t e_{\mathrm{sig}}\mathbf 1[\cdot=t],
\qquad
u\in\Kset,
\]
then
\[
\Pi_{\mathrm{ctrl}}\bar M_T(v)_i=\Pi_{\mathrm{ctrl}}v_i,
\qquad 0\le i\le T,
\]
and
\[
\pi_{\mathrm{sig}}(\bar M_T(v)_i)
=
\sum_{j=0}^{i}\mathcal T_{M_T}^u(i,j)\,\pi_{\mathrm{sig}}(v_j),
\qquad 0\le i\le T.
\]
The right-hand side depends only on the control stream of \(v\), hence is independent of the
choice of \(u\in\Kset\) with the same control stream.
\end{enumerate}
\end{lemma}

\begin{proof}
Write
\[
M_T=A_T^{\mathrm{diff}}\circ W_T^{\mathrm{src}}
\]
as in the proof of Lemma~\ref{lem:profile_weighted_macro}.

For item~(i), the explicit source-writer formula there gives
\[
\langle W_T^{\mathrm{src}}(u)_t,e_{\mathrm{src}}\rangle
=
m_t^u\,\langle u_t,e_{\mathrm{sig}}\rangle,
\]
where \(m_t^u\) depends only on the control stream
\((e_{\mathrm{pos}},e_{\mathrm{prof}},E_{\mathrm{carry}})\), and not on the incoming
\(e_{\mathrm{src}}\)-coordinate.
All other channels used by \(W_T^{\mathrm{src}}\) are likewise independent of the incoming
\(e_{\mathrm{src}}\)-channel. Hence
\[
W_T^{\mathrm{src}}(u)=W_T^{\mathrm{src}}(\Pi_{\mathrm{src}}u).
\]
Applying \(A_T^{\mathrm{diff}}\) yields
\[
M_T(u)=M_T(\Pi_{\mathrm{src}}u),
\]
which is item~(i).

Item~(ii) follows because \(M_T\) already preserves
\(e_{\mathrm{pos}},e_{\mathrm{prof}},E_{\mathrm{carry}}\) exactly by
Lemma~\ref{lem:profile_weighted_macro}, and \(\Pi_{\mathrm{src}}\) acts as the identity on those channels.

For item~(iii), \(\Pi_{\mathrm{src}}\) acts as the identity on the \(e_{\mathrm{sig}}\)-coordinate, so
\[
\langle \bar M_T(u)_i,e_{\mathrm{sig}}\rangle
=
\langle M_T(u)_i,e_{\mathrm{sig}}\rangle.
\]
Since \(M_T\) has signal-blind exact scalar transport with kernel \(\mathcal T_{M_T}^u\),
the same is true for \(\bar M_T\), with the same kernel.

For item~(iv), fix \(u\in\Kset\) and
\[
u' = u+\sum_{t=0}^{T} a_t e_{\mathrm{sig}}\mathbf 1[\cdot=t],
\qquad
\max_t |a_t|\le \delta.
\]
The control stream is unchanged, so the same transport kernel
\(\mathcal T_{M_T}^u\) applies to both \(u\) and \(u'\).
By item~(iii),
\[
\langle \bar M_T(u')_i-\bar M_T(u)_i,e_{\mathrm{sig}}\rangle
=
\sum_{j=0}^{i}\mathcal T_{M_T}^u(i,j)a_j.
\]
In the concrete construction of Lemma~\ref{lem:profile_weighted_macro},
the source writer modifies only the \(e_{\mathrm{src}}\)-channel and the diffuse block
modifies only the \(e_{\mathrm{sig}}\)-channel; every channel orthogonal to
\[
\operatorname{span}\{e_{\mathrm{sig}},e_{\mathrm{pos}},e_{\mathrm{prof}},e_{\mathrm{src}}\}\oplus E_{\mathrm{carry}}
\]
is preserved exactly. Thus the only possible signal-dependent non-signal output channel is \(e_{\mathrm{src}}\), and \(\Pi_{\mathrm{src}}\) removes it.
Hence
\[
\bar M_T(u')_i-\bar M_T(u)_i
=
\left(
\sum_{j=0}^{i}\mathcal T_{M_T}^u(i,j)a_j
\right)e_{\mathrm{sig}},
\]
which is exactly a bounded signal-fiber perturbation over \(\bar M_T(u)\).
Since \(T\) is finite and \(\Kset\) is compact, the quantity
\[
\sup_{u\in\Kset}\sup_{0\le i\le T}\sum_{j=0}^{i}|\mathcal T_{M_T}^u(i,j)|
\]
is finite, so one may take
\[
\delta'
:=
\delta\,
\sup_{u\in\Kset}\sup_{0\le i\le T}\sum_{j=0}^{i}|\mathcal T_{M_T}^u(i,j)|.
\]

For item~(v), fix \(\delta\ge 0\) and \(v\in \Sat^{\mathrm{sig}}_\delta(\Kset)\).
Write
\[
v=u+\sum_{t=0}^{T} a_t e_{\mathrm{sig}}\mathbf 1[\cdot=t]
\qquad\text{with }u\in\Kset.
\]
By item~(iv),
\[
\bar M_T(v)_i
=
\bar M_T(u)_i
+
\left(
\sum_{j=0}^{i}\mathcal T_{M_T}^u(i,j)a_j
\right)e_{\mathrm{sig}}.
\]
Taking the \(e_{\mathrm{sig}}\)-coordinate and using item~(iii) on \(u\in\Kset\), we obtain
\begin{align*}
\pi_{\mathrm{sig}}(\bar M_T(v)_i)
&=
\pi_{\mathrm{sig}}(\bar M_T(u)_i)
+
\sum_{j=0}^{i}\mathcal T_{M_T}^u(i,j)a_j\\
&=
\sum_{j=0}^{i}\mathcal T_{M_T}^u(i,j)\,\pi_{\mathrm{sig}}(u_j)
+
\sum_{j=0}^{i}\mathcal T_{M_T}^u(i,j)a_j\\
&=
\sum_{j=0}^{i}\mathcal T_{M_T}^u(i,j)\,\pi_{\mathrm{sig}}(v_j).
\end{align*}
Moreover, from the explicit construction,
\(W_T^{\mathrm{src}}\) modifies only the \(e_{\mathrm{src}}\)-channel,
\(A_T^{\mathrm{diff}}\) modifies only the \(e_{\mathrm{sig}}\)-channel,
and \(\Pi_{\mathrm{src}}\) kills only the \(e_{\mathrm{src}}\)-channel.
Hence \(\bar M_T\) acts as the identity on
\[
E_{\mathrm{ctrl}}=\operatorname{span}\{e_{\mathrm{pos}},e_{\mathrm{prof}}\}\oplus E_{\mathrm{carry}}
\]
for every input, and therefore
\[
\Pi_{\mathrm{ctrl}}\bar M_T(v)_i=\Pi_{\mathrm{ctrl}}v_i.
\]
Finally, since \(\mathcal T_{M_T}^u\) depends only on the control stream, the displayed kernel is
independent of the choice of \(u\in\Kset\) with the same control stream as \(v\).
Thus \(\bar M_T\) has signal-blind exact scalar transport on
\(\Sat^{\mathrm{sig}}_\delta(\Kset)\) with the same kernel as \(M_T\).
This proves the claim.
\end{proof}

\begin{lemma}[Balanced path lower bound]
\label{lem:balanced_profile_paths}
Fix $\beta\in(0,1)$, set $\gamma:=1-\beta$, fix $k\ge 1$, and fix $\tau_{\max}\ge 0$.
Then there exists a constant $c^{\mathrm{bal}}_{k,\beta,\tau_{\max}}>0$ such that for every
$0\le \tau_\ast\le \tau_{\max}$ and every $\ell\ge k$, with $t=\tau_\ast+\ell$,
\[
\sum_{\substack{\tau_\ast=i_0<i_1<\cdots<i_k=t\\
\frac{\ell}{2k}\le i_r-i_{r-1}\le \frac{2\ell}{k}\ \forall r}}
\ \prod_{r=1}^{k}(i_r+1)^{-\beta}
\ge
c^{\mathrm{bal}}_{k,\beta,\tau_{\max}}(1+\ell)^{k(1-\beta)-1}.
\]
\end{lemma}

\begin{proof}
The number of balanced paths is
$\gtrsim_k \ell^{k-1}$ for all $\ell\ge k$.

For every balanced path and every $r=1,\dots,k$,
\[
i_r+1\asymp_{k,\tau_{\max}} 1+\ell.
\]
Hence every balanced path contributes at least
\[
C_{k,\beta,\tau_{\max}}^{-1}(1+\ell)^{-k\beta}.
\]
Multiplying by the number of balanced paths gives
\[
\gtrsim \ell^{k-1}(1+\ell)^{-k\beta}
\asymp
(1+\ell)^{k-1-k\beta}
=
(1+\ell)^{k(1-\beta)-1}.
\]
\end{proof}

\begin{lemma}[Competitor suppression]
\label{lem:competitor_suppression_profile}
Fix \(\beta\in(0,1)\), set \(\gamma:=1-\beta\), fix \(k\ge 1\), and fix \(\tau_{\max}\ge 0\).
Consider a depth-\((k+1)\) exact scalar transport stack on a distinguished signal channel,
consisting of one selector block \(S_{H,\tau_\ast,\varepsilon_H}\) followed by \(k\) diffuse profile-compensated macro-layers.
Let
\[
\mathcal T_{\mathrm{stack}}^u(t,\tau)
\]
denote the resulting exact scalar transport kernel on that signal channel.
Assume the selector satisfies
\[
\frac12\le D_{\mathrm{sel}}^u(\tau_\ast)\le 2,
\qquad
|D_{\mathrm{sel}}^u(\tau)|\le \varepsilon_H\quad (\tau\neq \tau_\ast),
\]
uniformly in \(u\), and each macro-layer satisfies
\[
1\le D_{\mathrm{mac}}^u(i)\le d^+_{\mathrm{mac}},
\qquad
K_{\mathrm{mac}}^u(i,j)\le a^+_{\mathrm{mac}}(i+1)^{-\beta}.
\]
Then there exists \(C_{\mathrm{comp}}<\infty\), independent of \(H\), such that for every
\[
t=\tau_\ast+\ell,\qquad 1\le \ell\le H,
\]
one has
\[
\sum_{\substack{0\le \tau<t\\ \tau\neq \tau_\ast}}
\bigl|\mathcal T_{\mathrm{stack}}^u(t,\tau)\bigr|
\le
C_{\mathrm{comp}}\,\varepsilon_H\,(1+\ell)^{k(1-\beta)}.
\]
In particular, if
\[
\varepsilon_H\le c_0(H+1)^{-1}
\]
with \(c_0>0\) small enough, then
\[
\sum_{\substack{0\le \tau<t\\ \tau\neq \tau_\ast}}
\bigl|\mathcal T_{\mathrm{stack}}^u(t,\tau)\bigr|
\le
\frac12 c_{\mathrm{sig}}(1+\ell)^{k(1-\beta)-1}
\]
for any prescribed \(c_{\mathrm{sig}}>0\) after reducing \(c_0\).
\end{lemma}

\begin{proof}
Fix a competitor source \(\tau\neq\tau_\ast\) with \(\tau<t\).
Any path from \(\tau\) to \(t\) through the selector-plus-\(k\)-macro-layer stack must contain at least one genuine jump,
because diagonal propagation alone cannot change the time index.

Fix a path with exactly \(j\) jump layers, where \(1\le j\le k\), and let
\[
\tau=i_0<i_1<\cdots<i_j=t
\]
be the corresponding jump times.
The selector contributes at most \(\varepsilon_H\) at the source \(\tau\neq\tau_\ast\).
Each jump contributes at most
\[
a^+_{\mathrm{mac}}(i_r+1)^{-\beta},
\qquad r=1,\dots,j.
\]
Each non-jump macro-layer contributes at most the diagonal bound \(d^+_{\mathrm{mac}}\).

Hence every such path has weight bounded by
\[
C_0\,\varepsilon_H
\prod_{r=1}^{j}(i_r+1)^{-\beta},
\]
where \(C_0\) depends only on \(k\) and \(d^+_{\mathrm{mac}}\).

Now sum over all jump times for fixed \(j\):
\[
\sum_{\tau=i_0<i_1<\cdots<i_j=t}
\prod_{r=1}^{j}(i_r+1)^{-\beta}
=
(t+1)^{-\beta}
\sum_{\tau<i_1<\cdots<i_{j-1}<t}
\prod_{r=1}^{j-1}(i_r+1)^{-\beta}.
\]
Using the elementary symmetric-sum bound,
\[
\sum_{\tau<i_1<\cdots<i_{j-1}<t}
\prod_{r=1}^{j-1}(i_r+1)^{-\beta}
\le
\frac{1}{(j-1)!}
\left(\sum_{m=1}^{t-1}(m+1)^{-\beta}\right)^{j-1},
\]
and
\[
\sum_{m=1}^{t-1}(m+1)^{-\beta}\lesssim (1+t)^{1-\beta},
\]
we obtain
\[
\sum_{\tau=i_0<i_1<\cdots<i_j=t}
\prod_{r=1}^{j}(i_r+1)^{-\beta}
\le
C_j (1+t)^{j(1-\beta)-1}.
\]
Therefore
\[
|\mathcal T_{\mathrm{stack}}^u(t,\tau)|
\le
C_1\,\varepsilon_H
\sum_{j=1}^{k}(1+t)^{j(1-\beta)-1}
\le
C_2\,\varepsilon_H(1+t)^{k(1-\beta)-1},
\]
since \(k\) is fixed.

Now \(t=\tau_\ast+\ell\) with \(0\le \tau_\ast\le\tau_{\max}\), so
\[
1+t\asymp_{\tau_{\max}}1+\ell.
\]
Hence
\[
|\mathcal T_{\mathrm{stack}}^u(t,\tau)|
\lesssim
\varepsilon_H(1+\ell)^{k(1-\beta)-1}.
\]

Finally sum over all competitors \(\tau<t\). There are at most \(t\lesssim_{\tau_{\max}}1+\ell\) of them, so
\[
\sum_{\substack{0\le \tau<t\\ \tau\neq \tau_\ast}}
|\mathcal T_{\mathrm{stack}}^u(t,\tau)|
\lesssim
\varepsilon_H(1+\ell)^{k(1-\beta)}.
\]
This proves the first claim.

For the in-particular clause, use \(1+\ell\le H+1\):
\[
\varepsilon_H(1+\ell)^{k(1-\beta)}
\le
c_0(H+1)^{-1}(1+\ell)^{k(1-\beta)}
\le
c_0(1+\ell)^{k(1-\beta)-1}.
\]
Reducing \(c_0\) if necessary yields the desired factor \(\frac12 c_{\mathrm{sig}}\).
\end{proof}

\begin{remark}[Width bookkeeping]
\label{rem:flexible_route_B_width7}
After the positional writer has fixed the direction \(e_{\mathrm{pos}}\), choose once and for all six orthonormal directions
\[
e_{\mathrm{sig}},\ e_{\mathrm{prof}},\ e_{\mathrm{tail}},\ e_{\mathrm{aux}},\ e_{\mathrm{src}},\ e_{\mathrm{tgt}},
\]
all orthogonal to \(e_{\mathrm{pos}}\).

The preparatory network \(Q_H\) uses \(e_{\mathrm{prof}},e_{\mathrm{tail}},e_{\mathrm{aux}},e_{\mathrm{src}},e_{\mathrm{tgt}}\);
the selector block reuses \(e_{\mathrm{aux}}\) and preserves \(e_{\mathrm{prof}}\);
each diffuse profile-compensated macro-layer reuses \(e_{\mathrm{src}}\) and preserves \(e_{\mathrm{prof}}\);
the direction \(e_{\mathrm{tgt}}\) remains available as an auxiliary spare scratch direction.
No block requires any additional fresh ambient direction beyond these seven coordinates.

In the concrete architecture, each width-\(D\) block also provides \(D\) \(a\)-slots and \(D\) \(g\)-slots in the split
\[
(a,g)=\mathrm{split}(xW^{\mathrm{in}}+b^{\mathrm{in}}).
\]
The constructions below use at most six active \(a\)-slots and at most three active \(g\)-slots in any single block:
the plateau window uses four \(a\)-slots, the window writer uses six \(a\)-slots and two \(g\)-slots, the local multiplier uses four \(a\)-slots and two \(g\)-slots, the repaired source writer uses four \(a\)-slots and two \(g\)-slots, the repaired diffuse transport block uses two \(a\)-slots and one \(g\)-slot, the damped predecessor integrator uses three \(a\)-slots and one \(g\)-slot, and the simultaneous scratch reset uses one \(a\)-slot and three \(g\)-slots.

Hence the same condition
\[
D\ge 7
\]
simultaneously provides the seven persistent ambient directions and enough concrete \(a\)-/\(g\)-slots for every primitive block.
\end{remark}

\begin{proof}[Proof of Theorem~\ref{thm:flexible_route_B_sessa}]
Fix \(H\ge 1\) and \(0\le \tau_\ast\le \tau_{\max}\).
Set
\[
L_H:=\tau_{\max}+H,
\qquad
T_H:=L_H+1.
\]

\paragraph{Composite architecture.}
For each horizon parameter \(H\ge 1\) and source index \(0\le \tau_\ast\le \tau_{\max}\), we construct
\[
G_{H,\tau_\ast}
=
M_{H,k}\circ\cdots\circ M_{H,1}\circ
S_{H,\tau_\ast,\varepsilon_H}\circ Q_H\circ P_H.
\]
Here \(P_H\) writes a one-directional positional code, \(Q_H\) builds a signal-transparent preparatory
power-profile channel, \(S_{H,\tau_\ast,\varepsilon_H}\) is a selector that isolates the chosen source
\(\tau_\ast\), and \(M_{H,1},\dots,M_{H,k}\) are the diffuse profile-compensated macro-layers that
generate the target polynomial transport envelope.

Inside the proof we also introduce projected variants of the macro-layers in order to expose the exact
signal-channel transport kernel while removing an auxiliary scratch channel. This internal projection does
not change the realized map on the relevant signal fibers, so it is used only as a bookkeeping device in
the kernel calculation.

\paragraph{Step 1: write the positional code.}
Apply Corollary~\ref{cor:pos_code_one_direction} on the finite prefix \(\{0,\dots,L_H\}\).
This yields a block
\[
P_H:(\mathbb R^D)^{T_H}\to(\mathbb R^D)^{T_H}
\]
and a unit direction \(e_{\mathrm{pos}}\) such that
\[
P_H(h)_t=h_t+\lambda_t e_{\mathrm{pos}},
\qquad 0\le t\le L_H,
\]
for some scalars \(\lambda_t\), and such that on
\[
\Kset_H:=P_H(\mathcal X_0^{(H)})
\]
the scalar ranges
\[
I_t:=\{\langle u_t,e_{\mathrm{pos}}\rangle:\ u\in\Kset_H\}
\]
are compact and strictly ordered:
\[
I_0<\cdots<I_{L_H}\subset(0,\infty).
\]

Since \(D\ge 7\), after fixing \(e_{\mathrm{pos}}\) we may choose orthonormal directions
\[
e_{\mathrm{sig}},\ e_{\mathrm{prof}},\ e_{\mathrm{tail}},\ e_{\mathrm{aux}},\ e_{\mathrm{src}},\ e_{\mathrm{tgt}}
\]
all orthogonal to \(e_{\mathrm{pos}}\); see Remark~\ref{rem:flexible_route_B_width7}.

By Corollary~\ref{cor:pos_code_signal_transparency}, for every \(x\in\mathcal X_0^{(H)}\),
every \(\tau\), and every scalar \(a\),
\[
P_H(x+a\,e_{\mathrm{sig}}\mathbf 1[\cdot=\tau])_t
=
P_H(x)_t+a\,e_{\mathrm{sig}}\mathbf 1[t=\tau].
\]
In particular,
\[
\langle P_H(x+a\,e_{\mathrm{sig}}\mathbf 1[\cdot=\tau])_t,e_{\mathrm{pos}}\rangle
=
\langle P_H(x)_t,e_{\mathrm{pos}}\rangle.
\]

\paragraph{Step 2: build the preparatory power-profile network.}
Apply Corollary~\ref{cor:preparatory_power_profile} to the compact set \(\Kset_H\), with the fixed orthonormal directions
\[
e_{\mathrm{sig}},\ e_{\mathrm{pos}},\ e_{\mathrm{prof}},\ e_{\mathrm{tail}},\ e_{\mathrm{aux}},\ e_{\mathrm{src}},\ e_{\mathrm{tgt}},
\]
which satisfy the hypotheses of that corollary.
This yields a constant-depth network
\[
Q_H:(\mathbb R^D)^{T_H}\to(\mathbb R^D)^{T_H}
\]
with the following properties.

\paragraph{Signal preservation.}
The signal channel is preserved exactly:
\[
\langle Q_H(u)_t,e_{\mathrm{sig}}\rangle=\langle u_t,e_{\mathrm{sig}}\rangle.
\]

\paragraph{Positional preservation.}
The positional-control coordinate is preserved exactly:
\[
\langle Q_H(u)_t,e_{\mathrm{pos}}\rangle=\langle u_t,e_{\mathrm{pos}}\rangle.
\]

\paragraph{Profile growth.}
The profile channel on the prescribed direction \(e_{\mathrm{prof}}\) satisfies
\[
c_r^-(t+1)^\gamma
\le
\langle Q_H(u)_t,e_{\mathrm{prof}}\rangle
\le
c_r^+(t+1)^\gamma,
\qquad
\gamma=1-\beta.
\]

\paragraph{Signal transparency.}
The map \(Q_H\) is signal-transparent relative to \((e_{\mathrm{pos}},e_{\mathrm{prof}})\): for every \(u\), every \(\tau\), and every scalar \(a\),
\[
\langle Q_H(u+a\,e_{\mathrm{sig}}\mathbf 1[\cdot=\tau])_t,e_{\mathrm{pos}}\rangle
=
\langle Q_H(u)_t,e_{\mathrm{pos}}\rangle,
\]
\[
\langle Q_H(u+a\,e_{\mathrm{sig}}\mathbf 1[\cdot=\tau])_t,e_{\mathrm{prof}}\rangle
=
\langle Q_H(u)_t,e_{\mathrm{prof}}\rangle,
\]
\[
\langle Q_H(u+a\,e_{\mathrm{sig}}\mathbf 1[\cdot=\tau])_t,e_{\mathrm{sig}}\rangle
=
\langle Q_H(u)_t,e_{\mathrm{sig}}\rangle+a\,\mathbf 1[t=\tau].
\]

Write
\[
R_H:=Q_H\circ P_H.
\]

\paragraph{Step 3: select the source index.}
Apply Lemma~\ref{lem:selector_window} on the image of \(R_H\), using the already fixed directions
\(e_{\mathrm{pos}},e_{\mathrm{sig}},e_{\mathrm{aux}}\), with
\[
E_{\mathrm{carry}}:=\operatorname{span}\{e_{\mathrm{prof}}\},
\qquad
\varepsilon_H:=c_0(H+1)^{-1},
\]
where \(c_0>0\) will be fixed later.
This yields a selector module
\[
S_{H,\tau_\ast,\varepsilon_H}
\]
which preserves the positional and profile channels and has exact diagonal signal transport
\[
\mathcal T_S^u(i,j)=D_{\mathrm{sel}}^u(i)\mathbf 1[i=j]
\]
with
\[
\frac12\le D_{\mathrm{sel}}^u(\tau_\ast)\le 2,
\qquad
|D_{\mathrm{sel}}^u(\tau)|\le \varepsilon_H\quad (\tau\neq \tau_\ast).
\]

\paragraph{Step 4: add the \(k\) macro-layers.}
Define
\[
\Kset_{H,0}^{\mathrm{mac}}
:=
S_{H,\tau_\ast,\varepsilon_H}\bigl(R_H(\mathcal X_0^{(H)})\bigr).
\]
This is compact. By Step~2 and Step~3, on \(\Kset_{H,0}^{\mathrm{mac}}\) the positional-control ranges are still
\[
I_0<\cdots<I_{L_H}\subset(0,\infty),
\]
and the profile channel still satisfies
\[
c_r^-(t+1)^\gamma
\le
\langle u_t,e_{\mathrm{prof}}\rangle
\le
c_r^+(t+1)^\gamma,
\qquad
0\le t\le L_H.
\]

Apply Lemma~\ref{lem:profile_weighted_macro} with \(T=L_H\) to
\(\Kset_{H,0}^{\mathrm{mac}}\), using the fixed directions
\[
e_{\mathrm{sig}},\ e_{\mathrm{pos}},\ e_{\mathrm{prof}},\ e_{\mathrm{src}},
\qquad
E_{\mathrm{carry}}:=\{0\},
\]
to obtain \(M_{H,1}\). Define
\[
\bar M_{H,1}:=\Pi_{\mathrm{src}}\circ M_{H,1}.
\]
If \(k\ge 2\), set
\[
\Kset_{H,1}^{\mathrm{mac}}
:=
\bar M_{H,1}\bigl(\Kset_{H,0}^{\mathrm{mac}}\bigr).
\]

Inductively, suppose that for some \(1\le r\le k-1\) we have already constructed
\[
M_{H,1},\dots,M_{H,r},
\qquad
\bar M_{H,1},\dots,\bar M_{H,r},
\]
and compact sets
\[
\Kset_{H,0}^{\mathrm{mac}},\dots,\Kset_{H,r}^{\mathrm{mac}}
\]
such that for each \(1\le s\le r\),
\[
\Kset_{H,s}^{\mathrm{mac}}
=
\bar M_{H,s}\bigl(\Kset_{H,s-1}^{\mathrm{mac}}\bigr),
\]
and on every \(\Kset_{H,s}^{\mathrm{mac}}\) the same ordered positional ranges
\[
I_0<\cdots<I_{L_H}\subset(0,\infty)
\]
and the same two-sided profile bounds
\[
c_r^-(t+1)^\gamma
\le
\langle u_t,e_{\mathrm{prof}}\rangle
\le
c_r^+(t+1)^\gamma
\]
hold.

Apply Lemma~\ref{lem:profile_weighted_macro} to
\(\Kset_{H,r}^{\mathrm{mac}}\), with the same fixed directions, to obtain
\(M_{H,r+1}\). Define
\[
\bar M_{H,r+1}:=\Pi_{\mathrm{src}}\circ M_{H,r+1}.
\]
If \(r+1\le k-1\), set
\[
\Kset_{H,r+1}^{\mathrm{mac}}
:=
\bar M_{H,r+1}\bigl(\Kset_{H,r}^{\mathrm{mac}}\bigr).
\]
By Lemma~\ref{lem:projected_macro_signal_fiber}(ii)--(iii), each
\(\bar M_{H,r}\) preserves the \(e_{\mathrm{pos}}\)- and
\(e_{\mathrm{prof}}\)-channels exactly and has the same exact signal-channel
transport kernel as \(M_{H,r}\). Therefore the induction is well-posed, and
after \(k\) steps we obtain macro-layers
\[
M_{H,1},\dots,M_{H,k},
\qquad
\bar M_{H,1},\dots,\bar M_{H,k-1},
\]
all preserving the positional and profile channels and having exact signal
transport kernels
\[
\mathcal T_{M_{H,r}}^u(i,j)
=
D_{\mathrm{mac},r}^u(i)\mathbf 1[i=j]
+
K_{\mathrm{mac},r}^u(i,j)\mathbf 1[j<i],
\]
with uniform bounds
\[
1\le D_{\mathrm{mac},r}^u(i)\le d^+_{\mathrm{mac}},
\]
\[
a^-_{\mathrm{mac}}(i+1)^{-\beta}
\le
K_{\mathrm{mac},r}^u(i,j)
\le
a^+_{\mathrm{mac}}(i+1)^{-\beta}
\qquad (j<i).
\]

Moreover, by Lemma~\ref{lem:projected_macro_signal_fiber}(i),
\[
M_{H,r+1}=M_{H,r+1}\circ \Pi_{\mathrm{src}}
\qquad (r=1,\dots,k-1),
\]
hence the actual network from the theorem statement satisfies
\[
G_{H,\tau_\ast}
=
M_{H,k}\circ\cdots\circ M_{H,1}\circ
S_{H,\tau_\ast,\varepsilon_H}\circ Q_H\circ P_H
=
\widehat G_{H,\tau_\ast}\circ R_H,
\]
where
\[
\widehat G_{H,\tau_\ast}
:=
M_{H,k}\circ \bar M_{H,k-1}\circ\cdots\circ \bar M_{H,1}\circ
S_{H,\tau_\ast,\varepsilon_H},
\qquad
R_H:=Q_H\circ P_H.
\]
By Lemma~\ref{lem:projected_macro_signal_fiber}(iii), each \(\bar M_{H,r}\)
has the same signal-channel transport kernel as the corresponding \(M_{H,r}\),
so all of the above kernel bounds remain unchanged.

\paragraph{Step 5: identify the score with the transport kernel.}
Take the normalized probes in Definition~\ref{def:uniform_finite_horizon_profile_realization} to be
\[
c^{(H,\tau_\ast)}:=e_{\mathrm{sig}},
\qquad
\rho_t^{(H,\tau_\ast)}:=e_{\mathrm{sig}}
\qquad (0\le t\le L_H).
\]
These are independent of \(x\), common to all source indices \(\tau\), and satisfy
\[
\|c^{(H,\tau_\ast)}\|_2=1,
\qquad
\|\rho_t^{(H,\tau_\ast)}\|_2=1.
\]

Set
\[
R_H:=Q_H\circ P_H.
\]
By Step~1 and Step~2, \(R_H\) is signal-transparent along \(e_{\mathrm{sig}}\) over
\[
E_{\mathrm{ctrl}}:=\operatorname{span}\{e_{\mathrm{pos}},e_{\mathrm{prof}}\}
\]
on \(\mathcal X_0^{(H)}\).

Fix some \(\delta_\ast>0\), for example \(\delta_\ast=1\), and define
\[
\mathcal Y_H
:=
\Sat^{\mathrm{sig}}_{\delta_\ast}\bigl(R_H(\mathcal X_0^{(H)})\bigr).
\]
This set is compact.

Define
\[
\mathcal Y_{H,0}
:=
S_{H,\tau_\ast,\varepsilon_H}(\mathcal Y_H).
\]
By Lemma~\ref{lem:selector_signal_fiber_closure}, there exists a finite
\(\delta_{H,0}\) such that
\[
\mathcal Y_{H,0}
\subset
\Sat^{\mathrm{sig}}_{\delta_{H,0}}\bigl(\Kset_{H,0}^{\mathrm{mac}}\bigr).
\]

For \(r=1,\dots,k-1\), define inductively
\[
\mathcal Y_{H,r}
:=
\bar M_{H,r}(\mathcal Y_{H,r-1}).
\]
By Lemma~\ref{lem:projected_macro_signal_fiber}(iv), there exists a finite
\(\delta_{H,r}\) such that
\[
\mathcal Y_{H,r}
\subset
\Sat^{\mathrm{sig}}_{\delta_{H,r}}\bigl(\Kset_{H,r}^{\mathrm{mac}}\bigr),
\qquad
r=1,\dots,k-1.
\]

By Corollary~\ref{cor:concrete_blocks_signal_fiber_stability}, the selector
\(S_{H,\tau_\ast,\varepsilon_H}\) has signal-blind exact scalar transport
along \(e_{\mathrm{sig}}\) over
\[
E_{\mathrm{ctrl}}=\operatorname{span}\{e_{\mathrm{pos}},e_{\mathrm{prof}}\}
\]
on \(\mathcal Y_H\).
For each \(r=1,\dots,k-1\), Lemma~\ref{lem:projected_macro_signal_fiber}(v)
shows that \(\bar M_{H,r}\) has signal-blind exact scalar transport along
\(e_{\mathrm{sig}}\) over the same control subspace on \(\mathcal Y_{H,r-1}\).
Finally, since
\[
\mathcal Y_{H,k-1}
\subset
\Sat^{\mathrm{sig}}_{\delta_{H,k-1}}\bigl(\Kset_{H,k-1}^{\mathrm{mac}}\bigr),
\]
Corollary~\ref{cor:macro_signal_fiber_transport} implies that the
final macro-layer \(M_{H,k}\) has signal-blind exact scalar transport along
\(e_{\mathrm{sig}}\) over the same control subspace on \(\mathcal Y_{H,k-1}\),
with the same kernel \(\mathcal T_{M_{H,k}}^u\) as on
\(\Kset_{H,k-1}^{\mathrm{mac}}\).

Repeated application of Lemma~\ref{lem:signal_blind_transport_calculus}(ii)
therefore yields that the full post-preparatory stack
\[
\widehat G_{H,\tau_\ast}
=
M_{H,k}\circ \bar M_{H,k-1}\circ\cdots\circ \bar M_{H,1}\circ S_{H,\tau_\ast,\varepsilon_H}
\]
has signal-blind exact scalar transport along \(e_{\mathrm{sig}}\) over
\[
E_{\mathrm{ctrl}}=\operatorname{span}\{e_{\mathrm{pos}},e_{\mathrm{prof}}\}
\]
on \(\mathcal Y_H\), with transport kernel
\[
\mathcal T_{\widehat G_{H,\tau_\ast}}^{u}(t,\tau).
\]

Hence Lemma~\ref{lem:transparent_preprocessing_exact_transport} applies with
\[
R=R_H,
\qquad
B=\widehat G_{H,\tau_\ast},
\qquad
\Kset=\mathcal X_0^{(H)}.
\]
Therefore, for every \(x\in\mathcal X_0^{(H)}\) and every \(0\le \tau\le t\le L_H\),
\[
e_{\mathrm{sig}}^\top
\frac{\partial G_{H,\tau_\ast,t}(x)}{\partial x_\tau}
e_{\mathrm{sig}}
=
\mathcal T_{\widehat G_{H,\tau_\ast}}^{\,R_H(x)}(t,\tau).
\]
By our choice of score channels,
\[
\mathsf S^{(H,\tau_\ast)}_{t,\tau}(x)
=
\bigl(\rho_t^{(H,\tau_\ast)}\bigr)^\top
J^{G_{H,\tau_\ast}}_{t,\tau}(x)\,
c^{(H,\tau_\ast)}
=
e_{\mathrm{sig}}^\top
J^{G_{H,\tau_\ast}}_{t,\tau}(x)\,
e_{\mathrm{sig}}
=
\mathcal T_{\widehat G_{H,\tau_\ast}}^{\,R_H(x)}(t,\tau).
\]
Set
\[
u:=R_H(x).
\]

\paragraph{Step 6: lower-bound the balanced paths.}
Fix
\[
t=\tau_\ast+\ell,\qquad \ell\ge k.
\]
Expand the kernel product along the intermediate states. Writing
\[
u^{(0)}:=u,
\qquad
u^{(r)}:=\bar M_{H,r}\circ\cdots\circ \bar M_{H,1}\circ S_{H,\tau_\ast,\varepsilon_H}(u)
\quad (1\le r\le k-1),
\]
one has
\[
\mathcal T_{\widehat G_{H,\tau_\ast}}^{u}
=
\mathcal T_{M_{H,k}}^{u^{(k-1)}}\,
\mathcal T_{\bar M_{H,k-1}}^{u^{(k-2)}}\cdots
\mathcal T_{\bar M_{H,1}}^{u^{(0)}}\,
\mathcal T_{S_{H,\tau_\ast,\varepsilon_H}}^{u}.
\]
Since every factor preserves the control channels exactly and its kernel depends only on the
control stream, all intermediate control streams equal that of \(u\). Hence the same pathwise
kernel bounds apply throughout. Moreover, by Lemma~\ref{lem:projected_macro_signal_fiber},
\[
\mathcal T_{\bar M_{H,r}}^{u^{(r-1)}}(i,j)=\mathcal T_{M_{H,r}}^{u^{(r-1)}}(i,j)
\qquad (r=1,\dots,k-1).
\]
Consider the family of paths that use all \(k\) macro-layers as jumps and whose jump times are balanced:
\[
\tau_\ast=i_0<i_1<\cdots<i_k=t,
\qquad
\frac{\ell}{2k}\le i_r-i_{r-1}\le \frac{2\ell}{k}.
\]
For each such path, the selector contributes at least \(\frac12\), and each jump contributes at least
\[
a^-_{\mathrm{mac}}(i_r+1)^{-\beta}.
\]
Hence
\[
\mathcal T_{\widehat G_{H,\tau_\ast}}^{u}(t,\tau_\ast)
\ge
\frac12 (a^-_{\mathrm{mac}})^k
\sum_{\substack{\tau_\ast=i_0<\cdots<i_k=t\\ \text{balanced}}}
\prod_{r=1}^{k}(i_r+1)^{-\beta}.
\]
By Lemma~\ref{lem:balanced_profile_paths},
\[
\mathcal T_{\widehat G_{H,\tau_\ast}}^{u}(t,\tau_\ast)
\ge
c_{\mathrm{good}}(1+\ell)^{k(1-\beta)-1}.
\]

\paragraph{Step 7: handle small lags.}
There are only finitely many pairs \((\tau_\ast,\ell)\) with
\[
0\le \tau_\ast\le \tau_{\max},
\qquad
1\le \ell<k.
\]
For each such pair, choose the path that jumps in the first \(\ell\) macro-layers and then propagates diagonally.
Since all indices lie in the finite set \(\{0,\dots,\tau_{\max}+k-1\}\), the corresponding exact path weight is bounded below by
a positive constant depending only on \((k,\beta,\tau_{\max})\).
Therefore there exists
\[
c_{\mathrm{small}}>0
\]
such that
\[
\mathcal T_{\widehat G_{H,\tau_\ast}}^{u}(\tau_\ast+\ell,\tau_\ast)\ge c_{\mathrm{small}}
\qquad (1\le \ell<k).
\]

Combining the large- and small-lag cases, there exists \(c_{\mathrm{sig}}>0\) such that for all
\(1\le \ell\le H\),
\[
\mathcal T_{\widehat G_{H,\tau_\ast}}^{u}(\tau_\ast+\ell,\tau_\ast)
\ge
c_{\mathrm{sig}}(1+\ell)^{\nu_k(\beta)},
\qquad
\nu_k(\beta)=k(1-\beta)-1.
\]

\paragraph{Step 8: suppress the competitors.}
Apply Lemma~\ref{lem:competitor_suppression_profile} to the
selector-plus-macro transport kernel. By
Lemma~\ref{lem:projected_macro_signal_fiber}(iii), each projected macro-layer
\(\bar M_{H,r}\) has exactly the same signal-channel transport kernel as the
corresponding macro-layer \(M_{H,r}\), so the lemma applies verbatim to the
post-preparatory stack
\[
\widehat G_{H,\tau_\ast}
=
M_{H,k}\circ \bar M_{H,k-1}\circ\cdots\circ \bar M_{H,1}\circ S_{H,\tau_\ast,\varepsilon_H}.
\]
Since the exact transport coefficient equals the Jacobian score coefficient on the signal channel,
\[
\sum_{\substack{0\le \tau<t\\ \tau\neq \tau_\ast}}
\bigl|\mathsf S^{(H,\tau_\ast)}_{t,\tau}(x)\bigr|
=
\sum_{\substack{0\le \tau<t\\ \tau\neq \tau_\ast}}
\bigl|\mathcal T_{\widehat G_{H,\tau_\ast}}^{u}(t,\tau)\bigr|
\le
C_{\mathrm{comp}}\varepsilon_H(1+\ell)^{k(1-\beta)}.
\]
Choose \(c_0>0\) small enough that
\[
C_{\mathrm{comp}}\varepsilon_H(1+\ell)^{k(1-\beta)}
\le
\frac12 c_{\mathrm{sig}}(1+\ell)^{\nu_k(\beta)}
\qquad (1\le \ell\le H).
\]
Then
\[
\mathsf M^{(H,\tau_\ast)}_{\tau_\ast+\ell,\tau_\ast}(x)
\ge
\frac12 c_{\mathrm{sig}}(1+\ell)^{\nu_k(\beta)}.
\]
So we may take
\[
c_-:=\frac12 c_{\mathrm{sig}}.
\]

\paragraph{Step 9: anchor bounds.}
At \(\ell=1\),
\[
\mathsf M^{(H,\tau_\ast)}_{\tau_\ast+1,\tau_\ast}(x)
\ge
c_-(1+1)^{\nu_k(\beta)}
=
2^{\nu_k(\beta)}c_-.
\]
Hence we may take
\[
m_-:=2^{\nu_k(\beta)}c_->0.
\]

For the anchor upper bound, note first that
\[
\mathsf M^{(H,\tau_\ast)}_{\tau_\ast+1,\tau_\ast}(x)
\le
\bigl|\mathsf S^{(H,\tau_\ast)}_{\tau_\ast+1,\tau_\ast}(x)\bigr|.
\]
By Step~5,
\[
\mathsf S^{(H,\tau_\ast)}_{\tau_\ast+1,\tau_\ast}(x)
=
\mathcal T_{\widehat G_{H,\tau_\ast}}^{\,R_H(x)}(\tau_\ast+1,\tau_\ast).
\]

Since the selector is diagonal, any path from \(\tau_\ast\) to \(\tau_\ast+1\)
through
\[
\widehat G_{H,\tau_\ast}
=
M_{H,k}\circ \bar M_{H,k-1}\circ\cdots\circ \bar M_{H,1}\circ S_{H,\tau_\ast,\varepsilon_H}
\]
must contain exactly one off-diagonal jump, and that jump must occur in one of
the \(k\) macro-layers. Therefore
\begin{align*}
&\mathcal T_{\widehat G_{H,\tau_\ast}}^{\,R_H(x)}(\tau_\ast+1,\tau_\ast) \\
&\qquad=
D_{\mathrm{sel}}^{u}(\tau_\ast)
\sum_{r=1}^{k}
\left(\prod_{q<r} D_{\mathrm{mac},q}^{u}(\tau_\ast)\right)
K_{\mathrm{mac},r}^{u}(\tau_\ast+1,\tau_\ast)
\left(\prod_{q>r} D_{\mathrm{mac},q}^{u}(\tau_\ast+1)\right),
\end{align*}
where \(u=R_H(x)\).

Using
\[
D_{\mathrm{sel}}^{u}(\tau_\ast)\le 2,
\qquad
D_{\mathrm{mac},q}^{u}(i)\le d_{\mathrm{mac}}^{+},
\qquad
K_{\mathrm{mac},r}^{u}(\tau_\ast+1,\tau_\ast)
\le
a_{\mathrm{mac}}^{+}(\tau_\ast+2)^{-\beta}
\le
a_{\mathrm{mac}}^{+},
\]
we obtain
\[
\bigl|\mathcal T_{\widehat G_{H,\tau_\ast}}^{\,R_H(x)}(\tau_\ast+1,\tau_\ast)\bigr|
\le
2k\,(d_{\mathrm{mac}}^{+})^{k-1}a_{\mathrm{mac}}^{+}.
\]
Hence one may take
\[
m_+
:=
2k\,(d_{\mathrm{mac}}^{+})^{k-1}a_{\mathrm{mac}}^{+},
\]
which is independent of \(H\), \(\tau_\ast\), and \(x\). Consequently,
\[
\mathsf M^{(H,\tau_\ast)}_{\tau_\ast+1,\tau_\ast}(x)\le m_+.
\]

This verifies Definition~\ref{def:uniform_finite_horizon_profile_realization}.
The sign classification follows immediately from the sign of
\[
\nu_k(\beta)=k(1-\beta)-1.
\]
\end{proof}

\end{document}